\documentclass[11pt] {article}
\usepackage[round]{natbib}
\usepackage{chapterbib}
\usepackage{amssymb}
\usepackage{amsmath}
\usepackage{fullpage}
\usepackage{mathtools}
\usepackage{amsthm}
\usepackage{float} 
\usepackage{booktabs}
\usepackage{multirow}
\usepackage{geometry}
\usepackage{makecell}
\usepackage{colortbl}
\usepackage{graphicx}
\usepackage{wrapfig}
\usepackage[dvipsnames]{xcolor}
\usepackage{subfigure}
\graphicspath{{figures/}}
\usepackage[small]{caption}

\usepackage{placeins}
\let\algnewcommand\newcommand
\algnewcommand\algorithmicinput{\textbf{INPUT:}}
\algnewcommand\INPUT{\item[\algorithmicinput]}
\algnewcommand\algorithmicoutput{\textbf{OUTPUT:}}
\algnewcommand\OUTPUT{\item[\algorithmicoutput]}

\usepackage{enumitem}
\usepackage{verbatim}
\colorlet{arxivlinkblue}{blue}
\usepackage{xurl}
\usepackage{hyperref}
\usepackage{cleveref}
\hypersetup{
    colorlinks=true,
    linkcolor=arxivlinkblue,
    filecolor=magenta,      
    urlcolor=cyan,
      citecolor=arxivlinkblue,
}

\newenvironment{keywords}
{\bgroup\leftskip 20pt\rightskip 20pt \small\noindent{\bf Keywords:} }%
{\par\egroup\vskip 0.25ex}

  \makeatletter
\AtBeginDocument{%
  \expandafter\renewcommand\expandafter\subsection\expandafter
    {\expandafter\@fb@secFB\subsection}%
  \newcommand\@fb@secFB{\FloatBarrier
    \gdef\@fb@afterHHook{\@fb@topbarrier \gdef\@fb@afterHHook{}}}%
  \g@addto@macro\@afterheading{\@fb@afterHHook}%
  \gdef\@fb@afterHHook{}%
}
\makeatother

\graphicspath{{figures/}}

\DeclareMathAlphabet{\mathpzc}{OT1}{pzc}{m}{it}

\newtheorem{theorem}{Theorem}
\newtheorem{lemma}[theorem]{Lemma}
\newtheorem{proposition}[theorem]{Proposition}
\newtheorem{corollary}[theorem]{Corollary}

\newtheorem{definition}{Definition}
\newtheorem{remark}{Remark}
\newtheorem{assumption}{Assumption}

\allowdisplaybreaks

\crefname{theorem}{Theorem}{Theorems}
\Crefname{theorem}{Theorem}{Theorems}
\crefname{lemma}{Lemma}{Lemmas}
\Crefname{lemma}{Lemma}{Lemmas}
\crefname{proposition}{Proposition}{Propositions}
\Crefname{proposition}{Proposition}{Propositions}
\crefname{corollary}{Corollary}{Corollaries}
\Crefname{corollary}{Corollary}{Corollaries}
\crefname{definition}{Definition}{Definitions}
\Crefname{definition}{Definition}{Definitions}
\crefname{assumption}{Assumption}{Assumptions}
\Crefname{assumption}{Assumption}{Assumptions}
\crefname{remark}{Remark}{Remarks}
\Crefname{remark}{Remark}{Remarks}
\crefname{equation}{equation}{equations}
\Crefname{equation}{Equation}{Equations}

\allowdisplaybreaks

\newcommand\itbf{\bfseries\itshape}

 \newcommand{\p}{\mathbb P}
  \newcommand{\Q}{ \widehat {\mathcal Q}}

\newcommand{\I}{{\mathcal I }} 
\newcommand{\lbg}{ \langle}
\newcommand{\rbg}{ \rangle}
  \newcommand{\J}{{  \mathcal J } }
   \newcommand{\An}{{  \mathcal A_n } }
   \newcommand{\Anm}{{  \mathcal A_{nm} } }
   \newcommand{\lt}{{  \mathcal L_2 } }
   \newcommand{\T}{{  \mathcal T } }

   \newcommand{\tv}{\text{TV}}

     \newcommand{\E}{ { \mathbb{E} } }
 
  \newcommand {\st} {{ \psi_2}}
    \newcommand {\so} {{ \psi_1}} 
 \newcommand{\s}{ { \mathfrak{s}  } }
 \newcommand{\hatp}{  {\widehat  { \mathcal P }  } }

  \newcommand{\sx}{ {\Sigma }}

\newcommand{\ou}{\"{o}}

\newcommand{\f}{\mathcal F}
\newcommand{\diag}{\text{diag}}

\newcommand{\lc}{ \lceil} 
\newcommand{\rc}{ \rceil } 
\newcommand{\kse}{\kappa_{\max}^{s,e}}
\newcommand{\dint}{\,\mathrm{d}}

\newcommand{\lb}{ \langle}
\newcommand{\rb}{ \rangle}
\newcommand{\N}{\mathcal N }
\newcommand{\h}{\mathcal H}

\newcommand{\lf}{\lfloor}
\newcommand{\rf}{\rfloor}
\newcommand{\hr}{{ \mathcal H^{r} } }

\newtheorem*{theorem*}{Theorem}

\newenvironment{IEEEkeywords}{\begin{keywords}}{\end{keywords}}
\newcommand{\IEEEpeerreviewmaketitle}{}

\newcommand{\panelcaption}[1]{\par\smallskip{\centering\footnotesize #1\par}}
\colorlet{Magenta}{black}
\colorlet{magenta}{black}
\colorlet{red}{black}
\colorlet{blue}{black}
\colorlet{green}{black}
\colorlet{Maroon}{black}
\hypersetup{
  pdftitle={Dense ReLU Neural Networks for Temporal-Spatial Modeling},
  pdfauthor={Carlos Misael Madrid Padilla; Zhi Zhang; Xiaokai Luo; Daren Wang; Oscar Hernan Madrid Padilla}
}

\title{Dense ReLU Neural Networks for Temporal-Spatial Modeling}
\author{Carlos Misael Madrid Padilla$^1$ \and Zhi Zhang$^2$ \and Xiaokai Luo$^3$ \and Daren Wang$^4$ \and Oscar Hernan Madrid Padilla$^2$}
\date{%
    $^1$Department of Statistics and Data Science, Washington University in St. Louis.\\%
    $^2$Department of Statistics, University of California, Los Angeles.\\
    $^3$Department of ACMS, University of Notre Dame.\\%
    $^4$Department of Mathematics, University of California, San Diego.
}

\begin{document}

\maketitle

\begin{abstract}

We focus on fully connected deep neural networks utilizing the Rectified Linear Unit (ReLU) activation function for nonparametric estimation. We derive non-asymptotic bounds that lead to convergence rates, addressing temporal and spatial dependence in the observed measurements. By accounting for dependencies across time and space, our models better capture real-world complexity, enhancing predictive performance and theoretical robustness. We also tackle the curse of dimensionality by modeling the data on a manifold, exploring the intrinsic dimensionality of high-dimensional data.  Our analysis extends existing temporal-spatial theory by applying them to neural networks in more general contexts and demonstrate that our proof techniques are effective for models with short-range dependence. Simulations on synthetic functions and results on real data show our method consistently outperforms established approaches, highlighting the strength of dense neural networks (Dense NN) for temporal-spatial modeling.

\end{abstract}


%
\IEEEpeerreviewmaketitle

\section{Introduction}
%
%
%
%

\label{sec:intro}

Neural networks have become a cornerstone in nonparametric statistical methodologies, widely applied in
vision and image classification \citep{krizhevsky2017imagenet}, facial recognition \citep{taigman2014deepface}, time series forecasting \citep{graves2012long}, generative models \citep{goodfellow2014generative}, and advanced architectures  like transformers \citep{vaswani2017attention} and GPT models \citep{radford2018improving}.

The rise of high-dimensional, complex data has made the curse of dimensionality a key challenge in nonparametric regression \citep{donoh2000high}, though deep neural networks have shown promise in addressing it \citep{bauer2019deep}. Meanwhile, datasets with temporal and spatial dependencies are increasingly studied in fields like environmental science, epidemiology, and economics \citep{wikle2019spatio, lawson2013statistical, wu2021varying}.

Despite the demand from applications, most of the statistical theory literature on deep learning assumes that the data are independent (e.g. \cite{mccaffrey1994convergence,kohler2005adaptive,hamers2006nonasymptotic,schmidt2020nonparametric,bauer2019deep,padilla2022quantile}), with the exception of  \cite{ma2022theoretical} who considered time series data. 

In this paper, we concentrate on fully connected deep neural networks utilizing the ReLU activation function due to its applicability \citep{nair2010rectified,glorot2010understanding,krizhevsky2012imagenet}. We derive non-asymptotic bounds for nonparametric estimation with deep neural networks. Distinct from previous studies, our approach accounts for the dependence of measurements. By integrating these dependencies, our models more faithfully mirror the underlying data structure, thereby improving both predictive performance and theoretical robustness.

\subsection{Temporal-Spatial Model}

This paper examines a nonparametric regression model that captures temporal and spatial dependencies. Our model assumes that we are given data $\left\{\left(y_{ij}, x_{ij}\right)\right\}_{i=1, j=1}^{n,m_i}$ generated as
\begin{align} 
\label{true_model}
 y_{ij} = f^*(x_{ij}) + \gamma_i(x_{ij}) + \epsilon_{ij}.
 \end{align}  
Here, the index $i$ represents time, and $1\leq j \leq  \textcolor{black}{m_i}$ indicates the different measurements sampled at time $i$. The observations 
$\left\{x_{ij}\right\}_{i=1,j=1}^{n,m_i} \subseteq [0,1]^d$ represent the random locations where the (noisy) temporal-spatial data $\left\{y_{ij}\right\}_{i=1,j=1}^{n,m_i} \subseteq \mathbb{R}$ are observed. 
The function $f^*:[0,1]^d \rightarrow \mathbb{R}$ denotes the deterministic function of interest, while the stochastic processes $\left\{\gamma_i:[0,1]^d \rightarrow \mathbb{R}\right\}_{i=1}^n$ represent the functional spatial noise. Finally, $\left\{\epsilon_{ij}\right\}_{i=1,j=1}^{n,m_i} \subseteq \mathbb{R}$ account for the measurement error.

Without smoothness assumptions on the regression function $f^*$, deriving convergence rates for nonparametric regression is infeasible \citep{cover1968rates}. For a $(p,C)$-smooth $f^*$,
\cite{stone1982optimal} established the minimax rate of $n^{-2p/(2p+d)}$, highlighting the curse of dimensionality: higher dimensions $d$ dramatically slow convergence. Structural assumptions like additivity, as shown by \cite{stone1985additive, stone1994use}, can overcome this limitation.

This paper extends the hierarchical interaction model of \cite{kohler2021rate}, encompassing additive and single index models, to include temporal and spatial dependencies, as well as data on manifolds, demonstrating the adaptability of  dense neural network estimators to these complex structures.

\subsection{Summary of Results}

In this paper, we focus on a neural network estimator that minimizes the empirical $\mathcal{L}_2$ risk over the observed data within the network class $\mathcal{F}(L, r)$, as defined in \eqref{eq:space of neural network}, where the network consists of $L$ hidden layers and $r$ neurons per layer. The estimator is given by:
\begin{align}
\label{eqn_estimator}
\widehat{f} 
= \underset{f \in \mathcal{F}(L, r)}{\arg \min} \frac{1}{n} \sum_{i=1}^n \frac{1}{m_i} \sum_{j=1}^{m_i} \left( y_{ij} - f(x_{ij}) \right)^2.
\end{align}
The presence of \textcolor{black}{$m_i $} measurements for each $i$ is reflected in the loss used in (\ref{eqn_estimator}). Notably, the work by \cite{ma2022theoretical} addresses only temporal data, without accounting for spatial noise, which is a novel consideration in our approach. 

\color{black}

\paragraph{Beyond Independence: Temporal--Spatial Results}
Our main result studies this estimator under short-range temporal
dependence in the covariates, measurement errors, and spatial random
effects. Following \cite{kohler2021rate}, we assume that the regression
function has a hierarchical composition structure. Such functions are
built by nesting smooth component functions whose smoothness and input
dimensions may vary.

For this class,
Theorem~\ref{thm:main thm in the temporal-spatial model}
gives a risk bound with two main \textcolor{black}{terms}. The first
\textcolor{black}{term} is driven by the total number of measurements.
Its rate depends on the smoothness and input dimension of the component
function that is hardest to estimate. The relevant input dimension may
be much smaller than the ambient dimension. In that case, the first
\textcolor{black}{term} can mitigate the curse of dimensionality.
When temporal dependence and spatial random effects are absent, its
algebraic order agrees, up to logarithmic factors, with the rate for
fully connected networks in \cite{kohler2021rate}.

The second \textcolor{black}{term} is driven mainly by the number of
time points. It captures the effects of temporal dependence and the
time-indexed spatial random effects.

For the ordered hierarchical subclass specified in
Lemma~\ref{Mlb-M}, \textcolor{black}{the corresponding terms in the
lower and upper bounds have the same algebraic orders}.
They therefore match up to logarithmic factors. This provides a
temporal--spatial counterpart to the no-spatial-noise lower bound in
\cite{schmidt2020nonparametric}.

Theorem~\ref{thm:main thm in the temporal-spatial model}
establishes the upper bound for two explicit fully connected
architectures, one wide and one deep, following
\cite{kohler2021rate}.
Theorem~\ref{thm:general_architecture_temporal_spatial} gives a risk
bound for \textcolor{black}{general estimation families} under
approximation and entropy conditions.
\textcolor{black}{These families need not consist of neural networks.}
Corollary~\ref{cor:explicit_architecture_temporal_spatial}
\textcolor{black}{specializes this result to fully connected ReLU
networks and} makes the dependence on depth and width explicit.
This facilitates comparison with the architecture-dependent bounds in
\cite{jiao2023deep}.

The framework also yields two useful special cases. Without temporal
dependence, we obtain a functional-regression result for independent
functional observations. This connects with the classical
functional-regression literature
\cite{cardot1999functional,cardot2003testing,cai2006prediction,
hall2007methodology,cai2012minimax,petersen2016functional}.
Without spatial random effects and with one measurement at each time
point, we obtain a temporal-regression result. The temporal result
extends the independent setting of \cite{kohler2021rate} and is related
to the temporally dependent setting of \cite{ma2022theoretical}.
Precise assumptions and rates are given in
Section~\ref{sec_main}.

\paragraph{Overcoming the Curse of Dimensionality with Deep Neural
Networks and Manifold Learning}
This paper also explores the estimation problem for nonparametric
temporal--spatial data when the covariates have low-dimensional
geometric structure. We consider exact manifold support, support near
a manifold, and support on a more general set of low geometric
complexity. Across these settings, the risk bounds retain the
\textcolor{black}{two-term} structure described above. The first
\textcolor{black}{term} depends on an intrinsic or effective dimension
of the covariate support rather than directly on the ambient dimension,
while the second \textcolor{black}{term} is governed mainly by the
number of time points. Precise statements are given in
Section~\ref{sec_manifold}.

When temporal dependence and spatial random effects are absent, the
exact-manifold result recovers the upper-bound rate in
\cite{kohler2023estimation}, up to logarithmic factors. The
approximate-manifold and low-Minkowski-dimension results provide
temporal--spatial counterparts to the corresponding i.i.d.\ results in
\cite{jiao2023deep}.

\color{Black}

\subsection{Notation}
\label{sec:notation}

The set of positive integers is denoted by $\mathbf{Z}^{+}$, and the set of natural numbers is denoted by $\mathbf{N} = \mathbf{Z}^{+} \cup \{0\}$. For $n \in \mathbf{Z}^+$, we write $[n]:=\{1,\ldots,n\}$. 
 For two positive sequences $\{a_n\}_{n\in \mathbf{Z}^{+} }$ and $\{b_n\}_{n\in \mathbf{Z}^{+} }$, we write $a_n = O(b_n)$ or $a_n\lesssim b_n$ if $a_n\le Cb_n$ with some constant $C > 0$ that does not depend on $n$, and $a_n = \Theta(b_n)$ or $a_n\asymp b_n$ if $a_n = O(b_n)$ and $b_n = O(a_n)$. Given a $d$-dimensional multi-index $j=\left(j^{(1)}, \ldots, j^{(d)}\right)^T \in \mathbf{N}^d$,  we write $\|j\|_1=j^{(1)}+\cdots+j^{(d)}, j !=j^{(1)} ! \cdots j^{(d)}$ !, $x^{j}=x_{1}^{j^{(1)}} \cdots x_{d}^{j^{(d)}}$ and $\partial^{j}= (\partial^{j^{(1)}}\cdots \partial^{j^{(d)}})/(\partial x_{1}^{j^{(1)}}\cdots \partial x_{d}^{j^{(d)}}).$

For a random variable $X$, we denote by $\sigma(X)$ the $\sigma$-algebra generated by $X$, that is, the collection of events determined by $X$. More generally, for a collection $\{X_t : t \in I\}$, we write $\sigma(X_t,\, t\in I)$ for the $\sigma$-algebra generated by these random variables. For a $d$-dimensional vector $x \in \mathbb{R}^d$, the Euclidean and the supremum norms of $x$ are denoted by $\|x\|$ and $\|x\|_{\infty}$, respectively.  The $\ell_\infty$ norm of a function $f: \mathbb{R}^d \rightarrow \mathbb{R}$ is defined by
$ 
\|f\|_{\infty}=\sup _{x \in \mathbb{R}^d}|f(x)|.
$ 
We denote the supremum norm of $f$ restricted to a subset $\mathcal A\subset \mathbb{R}^d$ by $\|f\|_{\infty, \mathcal A}$. Given a function $f: \mathcal{X} \rightarrow \mathbb{R}$ and a probability distribution $\mathbb{P}$ over $\mathcal{X}$, the usual $\mathcal L_2(\mathbb{P})$-norm is given by
$
\|f\|_{\mathcal L_2(\mathbb{P})}:=\bigl( \int_\mathcal{X} f^2(x) \mathbb{P}(d x)\bigr)^{1/2}=\mathbb{E}\left[f^2(X)\right]^{1/2},
$
where $X$ is a random variable with distribution $\mathbb{P}$. The space of functions such that $\|f\|_{\mathcal L_2(\mathbb{P})} < \infty$ is denoted by $\mathcal L_2(\mathbb{P})$. Similarly, the $\mathcal L_2(\mathbb{P})$-inner product is given by
$
\langle f, g\rangle_{\mathcal L_2(\mathbb{P})}:=\int_\mathcal{X} f(x)g(x) \mathbb{P}(d x)=\mathbb{E}\left[f(X)g(X)\right].
$ Given a collection of samples $\left\{x_{ij}\right\}_{i=1,j=1}^{n,m_i}$ that are identically distributed according to a probability distribution $\mathbb{P}$, define the empirical probability distribution
as $
\mathbb{P}_{nm}(x):=\frac{1}{n} \sum_{i=1}^n \frac{1}{m_i} \sum_{j=1}^{m_i}\delta_{x_{ij}}(x).
$
Then the empirical $\mathcal L_2$-norm is given by
\[
\begin{aligned}
\|f\|_{\mathcal L_2\left(\mathbb{P}_{nm}\right)}
:=\bigl(\frac{1}{n} \sum_{i=1}^n \frac{1}{m_i}\sum_{j=1}^{m_i} f^2\left(x_{ij}\right)\bigr)^{1/2} =\bigl(\int_{\mathcal{X}} f^2(x) \mathbb{P}_{nm}(d x)\bigr)^{1/2}.
\end{aligned}
\] 
Furthermore,
\[
\begin{aligned}
\langle f, g\rangle_{\mathcal L_2\left(\mathbb{P}_{nm}\right)} :=\frac{1}{n} \sum_{i=1}^n \frac{1}{m_i} \sum_{j=1}^{m_i} f\left(x_{ij}\right)g\left(x_{ij}\right) =\int_{\mathcal{X}} f(x)g(x) \mathbb{P}_{nm}(d x),
\end{aligned}
\] defines the empirical \(\mathcal{L}_2\)-inner product.
When there is no ambiguity, $\|f\|_{nm}$, $\|f\|_2$, $\langle f, g\rangle_{nm}$ and $\langle f, g\rangle _2$ are written as convenient shorthands for $\|f\|_{\mathcal L_2(\mathbb{P}_{nm})}$, $\|f\|_{\mathcal L_2(\mathbb{P})}$, $\langle f, g\rangle_{\mathcal L_2(\mathbb{P}_{nm})}$ and $\langle f, g\rangle_{\mathcal L_2(\mathbb{P})}$, respectively. Additionally, if the probability distribution is supported on a subset $\mathcal{A}$, we write $\|f\|_{nm}$ as $\|f\|_{nm, \mathcal{A}}$ to highlight the dependence on $\mathcal{A}$.
Furthermore, for random variable $x$, if $\mathbb E(\exp(tx)) \leq \exp(\frac{t^2}{2\varpi^2})$ for any $t>0$, we say $x$ is a sub-Gaussian with parameter $\varpi >0$. For a vector $a \in \mathbb{R}^d$, we denote by $a^{(j)}$, for $j\in \{1,\ldots,d\}$, its $j$ coordinate.

\subsection{Outline}

The paper is organized as follows: Section~\ref{sec_nn} provides background on nonparametric regression and neural network approximation. Section~\ref{sec_main} explores the convergence rates of neural network estimators under different dependence structures. Section~\ref{sec_manifold} extends these rates to input data on lower-dimensional manifolds. Section~\ref{sec_exp} presents numerical experiments on simulated and real-world data. Section~\ref{sec_conclusion} concludes with future research directions, with full proofs in the Supplementary Material.

\section{Overview of Fully Connected Neural Networks }
\label{sec_nn}

In this section, we provide a comprehensive overview of the architecture of fully connected neural networks. The architecture of a neural network, denoted by $(L, k)$, is characterized by two primary components: the number of hidden layers $L$, which is a positive integer, and the width vector $k = (k_1, \ldots, k_L) \in \mathbf{N}^L$, which specifies the number of neurons in each of the $L$ hidden layers. 

A multilayer feedforward neural network with architecture $(L, k)$, employing the ReLU activation function $\rho(x)=\max(0,x)$ for any $x \in \mathbb{R}$, can be mathematically represented as a real-valued function $f: \mathbb{R}^d \rightarrow \mathbb{R}$ defined as
\begin{align}
f(x)=\sum_{i=1}^{k_L} c_{1, i}^{(L)} f_i^{(L)}(x)+c_{1,0}^{(L)}  \label{eq:form of approximation function 0}
\end{align}
for some weights $c_{1,0}^{(L)}, \ldots, c_{1, k_L}^{(L)} \in \mathbb{R}$ and for $f_i^{(L)}$ 's recursively defined by
\begin{align}
f_i^{(s)}(x)=\rho\left(\sum_{j=1}^{k_{s-1}} c_{i, j}^{(s-1)} f_j^{(s-1)}(x)+c_{i, 0}^{(s-1)}\right) \label{eq:form of approximation function L}
\end{align} for some $c_{i,0}^{(s-1)}, \dots, c_{i, k_{s-1}}^{(s-1)} \in \mathbb{R}$,
$s \in \{2, \dots, L\}$,
and $f_i^{(1)}(x) = \rho \left(\sum_{j=1}^d c_{i,j}^{(0)} x^{(j)} + c_{i,0}^{(0)} \right)$
for some $c_{i,0}^{(0)}, \dots, c_{i,d}^{(0)} \in \mathbb{R}$.

For simplification, we assume that all hidden layers possess an identical number of neurons. Consequently, we define the space $\mathcal{F}(L, r)$, as per \cite{kohler2021rate}, to represent the set of neural networks with $L$ hidden layers and $r$ neurons per layer:
\begin{definition}
\label{def:fully Dense NN space}
The space of neural networks with $L$ hidden layers and $r$ neurons per layer is defined by
\begin{equation}
\begin{aligned}
\mathcal{F}(L, r)=\bigl\{f:\; f\text{ is of the form \eqref{eq:form of approximation function 0}}\; \text{and \eqref{eq:form of approximation function L} with }k_1=k_2=\ldots=k_L=r\bigr\}.
\end{aligned}
\label{eq:space of neural network}
\end{equation}
\end{definition}

Here, the neurons in the neural network can be viewed as a computational unit. The input of any neuron $i$ in the $s-$th layer is associated with the output of all $k_{s-1}$ units $j$ in the $(s-1)-$th layer with weights $c_{i,j}^{(s-1)}$ through the ReLU function $\rho(\cdot)$. Such a kind of recursive construction of a fully connected neural network can be represented as an acyclic graph.

One key feature of our neural network architecture is that no network sparsity assumption is needed. The network class $\mathcal{F}(L, r)$ represents a fully connected feedforward neural network, where each neuron is connected to every neuron in the previous layer.

To avoid the so--called curse of dimensionality, we consider $f^*$ be in the space of hierarchical composition models $\mathcal{H}(l,\mathcal{P})$ for some $l \in \mathbf{Z}^{+}$ and $\mathcal{P} \subseteq[1, \infty) \times \mathbf{Z}^{+} $, which is described next.

\begin{definition}\label{definition: p,c smoothness}
Let $p=q+s$ for some $q \in \mathbf{N}=\mathbf{Z}^{+}\cup\left\{0\right\}$ and $0<s \leq 1$. A function $g: \mathbb{R}^d \rightarrow \mathbb{R}$ is called $(p, C)$-smooth, if for every $\alpha=\left(\alpha_1, \ldots, \alpha_d\right) \in \mathbf{N}^d$, where $d \in \mathbf{Z}^{+}$, with $\|\alpha\|_1=q$, the partial derivative $\partial^q g /\left(\partial z_1^{\alpha_1} \ldots \partial z_d^{\alpha_d}\right)$ exists and satisfies
$$
\left|\frac{\partial^q g}{\partial z_1^{\alpha_1} \ldots \partial z_d^{\alpha_d}}\left(z\right)-\frac{\partial^q g}{\partial z_1^{\alpha_1} \ldots \partial z_d^{\alpha_d}}\left(w\right)\right| \leq C\|z-w\|^s
$$
for all $z, w \in \mathbb{R}^d$, where $\|\cdot\|$ denotes the Euclidean norm.	 
\end{definition}

\begin{definition} [Space of Hierarchical Composition Models, \cite{kohler2021rate}]\label{def:Generalized Hierarchical Interaction Model}
Let $l \in \mathbf{Z}^{+}$ and the order and smoothness constraint $\mathcal{P} \subseteq[1, \infty) \times \mathbf{Z}^{+}$, the space of hierarchical composition models is defined recursively. For $l=1$, 
$$
\begin{aligned}
&\mathcal{H}(1,\mathcal{P}):={}\bigl\{h:\mathbb{R}^d\rightarrow\mathbb{R}:\quad h(z)=u\left(z^{(\pi(1))},\ldots,z^{(\pi(K))}\right),\\
&\quad \text{where }u:\mathbb{R}^K\rightarrow\mathbb{R}\text{ is}\; (p,C)\text{-smooth for some }(p,K)\in\mathcal{P}\\
&\quad \text{and }\pi:\{1,\ldots,K\}\rightarrow\{1,\ldots,d\}\bigr\}.
\end{aligned}
$$
For $l>1$, we recursively define 
$$
\begin{aligned}
&\mathcal{H}(l,\mathcal{P}):={}\bigl\{h:\mathbb{R}^d\rightarrow\mathbb{R}:\quad h(z)=u\left(f_1(z),\ldots,f_K(z)\right),\\
&\quad \text{where }u:\mathbb{R}^K\rightarrow\mathbb{R}\text{ is}\; (p,C)\text{-smooth for some }(p,K)\in\mathcal{P}\quad \text{and }f_i\in\mathcal{H}(l-1,\mathcal{P})\bigr\}.
\end{aligned}
$$	
\end{definition}

\color{black}
The pair $(p,K)\in\mathcal P$ may vary across component functions in
the recursive representation. Thus, $\mathcal P$ records the smoothness
levels and input dimensions that may occur throughout the hierarchy.
Notably, additive models, single index models, and projection pursuit
can be viewed as special cases of the hierarchical composition model.
\color{Black}

Here, the function class $\mathcal{H}(l,\mathcal{P})$ describes the relationships between the input and output of the network, where $(p, K) \in \mathcal{P}$ describes the smoothness and order constraint of the hierarchical composition model. Notably, additive models, single index models, and projection pursuit  can be viewed as special cases of the hierarchical composition model. 

\section{ReLU Dense NN for Temporal-Spatial Models}
\label{sec_main}

\subsection{Main Result}

\color{black}

In this section, we establish the main risk bound for the dense ReLU neural network estimator under temporal-spatial dependence. We begin by recalling the notion of $\beta$-mixing, which we use to quantify temporal dependence; see \cite{doukhan2012mixing} for a comprehensive review.

Using the notation introduced in Section~\ref{sec:notation} we define $\beta$-mixing as follows. For two $\sigma$-fields $\mathcal A$ and $\mathcal B$, define
\[
\beta(\mathcal A,\mathcal B)
=
\mathbb E\Big[\sup_{C\in \mathcal B}\big|\mathbb P(C)-\mathbb P(C\mid \mathcal A)\big|\Big].
\]
A sequence of $\sigma$-fields $\{\mathcal F_t\}_{t\in\mathbf Z^+}$ is said to be $\beta$-mixing if
\[
\begin{aligned}
\beta_s=\sup_{t\in\mathbf Z^+}
\beta\!\left(\sigma(\mathcal F_\ell,\ \ell\le t),\,
\sigma(\mathcal F_\ell,\ \ell\ge t+s)\right) \longrightarrow 0\qquad\text{as }s\to\infty.
\end{aligned}
\]
Accordingly, a sequence of random variables $\{X_t\}_{t\in\mathbf Z^+}$ is called $\beta$-mixing if the associated sequence of $\sigma$-fields $\{\sigma(X_t)\}_{t\in\mathbf Z^+}$ is $\beta$-mixing. If the data are independent, then the corresponding $\beta$-mixing coefficients are zero.

We now state the assumptions under which the main theorem is established. They are grouped into four categories: the observation model, the design and dependence structure, the noise conditions, and the regularity assumptions on the regression function and sampling scheme.

\begin{assumption}[Observation model]
\label{assumption:ts_model}
We consider the temporal-spatial regression model introduced in \eqref{true_model}, namely
\begin{align}
\label{eq:ts_model_main}
y_{ij} = f^*(x_{ij})+\gamma_i(x_{ij})+\epsilon_{ij},
\end{align}
for $1\le i \le n$ and $1\le j \le m_i$, where $x_{ij}\in[0,1]^d$, $f^*:[0,1]^d\to\mathbb{R}$, $\gamma_i:[0,1]^d\to\mathbb{R}$, and $\epsilon_{ij}\in\mathbb{R}$.
\end{assumption}

Assumption~\ref{assumption:ts_model} states the temporal-spatial regression model for clarity and establishes notation used throughout the analysis. The response is decomposed into a deterministic signal $f^*$, a time-indexed random spatial effect $\gamma_i$, and a measurement error term $\epsilon_{ij}$. The component $\gamma_i$ captures variation in the underlying spatial process across the domain where the data are collected. This variation may arise from natural randomness or application-specific variability and may still be present even in the absence of measurement error. The term $\epsilon_{ij}$ represents additional noise due to imperfect observations, such as measurement devices or external interference. This separation distinguishes variability in the underlying signal from variability due to measurement error and arises naturally in applications with repeated spatial observations over time.

We next formalize the assumptions on these noise components.

\begin{assumption}[Noise conditions]
\label{assumption:ts_noise}
{\bf a.} [sub-Gaussian property of $\epsilon$] For a fixed $i\in \{1,\ldots,n\}$, let $j \in \{1, \ldots, m_i \}$. We assume
\[
\mathbb{E}(\epsilon_{ij} \mid x_{ij})= 0,
\]
and for all $t\in \mathbb R$, there exists a constant $\sigma_\epsilon>0$ such that for all $i$ and $j$, almost surely,
\[
\mathbb{E}\{ \exp ( t \epsilon_{ij})  \mid x_{ij}\} \le  \exp\biggl(\frac{\sigma_\epsilon^2t ^2}{2}\biggr).
\]
\\
{\bf b.} [Marginal spatial distribution] 
We assume that the spatial noise processes $\{\gamma_i\}_{i\ge 1}$ are identically distributed with mean zero, i.e., 
\(
\E[\gamma_i(x)] = 0 \quad \text{for all } x \in [0,1]^d.
\)
Moreover, defining
\(
\sigma_\gamma^2 := \sup_{x \in [0,1]^d} \E[\gamma_i(x)^2],
\)
we assume that $\sigma_\gamma^2 < \infty$ and that there exists a constant $c > 0$ such that
\[
\sup_{x \in [0,1]^d} \E \bigl(\exp \{ c \gamma_i^2(x)\}\bigr) < \infty.
\]
\\
{\bf c.} [Moment condition] There exist $r, q>1$ such that $\frac{1}{r}+\frac{1}{q}=1$, and for all $i \in \mathbf{Z}^{+}$,
\[
\sup_{i\in \mathbf{Z}^{+} } \,\mathbb{E}\left(\|\gamma_i\|_{\mathcal{L}_2}^{2q}\right)<\infty.
\]
See Lemma {\color{magenta}{21}} in the Supplementary Material.
\end{assumption}

Assumption~\ref{assumption:ts_noise} controls the stochastic behavior of the noise components. Part {\bf a} imposes a standard sub-Gaussian condition on the measurement errors. The exponential moment condition in part {\bf b} ensures that the spatial noise has uniformly controlled tails across the domain. Part {\bf c} imposes a moment condition on the spatial noise $\gamma_i$. This condition is standard in the analysis of stochastic processes and is used to control higher-order stochastic terms in the analysis, for instance when applying Hölder-type inequalities and bounding $\mathcal{L}_2$-norm fluctuations. These conditions are satisfied in many common settings, including Gaussian processes, bounded stochastic processes, processes with sub-Gaussian tails, as well as random basis expansion and Karhunen--Lo\`eve models. See Supplementary Material {\color{magenta}{S10}} for concrete examples.

We now turn to the remaining components of the model, namely the underlying regression function and the sampling scheme.

\begin{assumption}[Regularity of $f^*$ and sampling balance]
\label{assumption:ts_regularity}
{\bf a.} [Same order condition of $m_i$] There exist positive constants $c$ and $C$ such that for all $i \in \{1,\ldots,n\}$ we have
\[
cm \leq m_i \leq C m,
\]
where $m$ is the harmonic mean of $m_i$ defined by
\begin{align}
\label{eq:harmonic_mean_m}
m=\left(\frac{1}{n} \sum_{i=1}^n \frac{1}{m_i}\right)^{-1}.
\end{align}
\\
{\bf b.} [$\ell_\infty$-bound of regression function] Suppose that
\[
\left\|f^*\right\|_\infty \leq \frac{\mathcal{A}_{nm}}{4},
\]
where $\mathcal{A}_{nm}$ is some parameter depending on $n$ and $m$, with $m$ as before. We allow $\mathcal{A}_{nm}$ to grow to infinity as $n$ or $m$ increase.
\\
\color{black}
\\
{\bf c.} [Hierarchical composition structure]
Let $f^*$ be in the class $\mathcal{H}(l,\mathcal{P})$ as in
Definition  \ref{def:Generalized Hierarchical Interaction Model}. By
Definition \ref{definition: p,c smoothness},
each component function $g$ in the recursive representation of $f^*$
is associated with a pair $(p_g,K_g)\in\mathcal P$, where
\[
p_g=q_g+s_g,
\qquad q_g\in\mathbf N,
\qquad s_g\in(0,1].
\]
Let $K_{\max}$ be the largest input dimension and $p_{\max}$ the
largest smoothness among the component functions. In addition to these defining conditions, impose the following uniform
regularity requirements. For every component function $g$, all partial
derivatives of order at most $q_g$ are uniformly bounded by
$C_{\mathrm{Smooth}}$. Moreover, each $g$ is also Lipschitz continuous with
Lipschitz constant $C_{\mathrm{Lip}}\geq1$. Also, assume that
$\max\{p_{\max},K_{\max}\}=O(1)$. Let
\[
\phi_{nm}
=
\max_{(p,K)\in\mathcal P}
(nm)^{-\frac{2p}{2p+K}}.
\]
\end{assumption}

\color{Black}

Assumption~\ref{assumption:ts_regularity} controls both the complexity of the regression function and the balance of the sampling scheme across time. The condition in {\bf a} requires that the number of observations per time index is of the same order. This is standard in the literature with spatial or functional noise. For instance, \cite{wang2022low} studies nonparametric covariance estimation for functional data and imposes a similar assumption; see Page 3 of that paper. The bound in {\bf b} imposes a uniform $\ell_\infty$ constraint on the regression function. This supports the truncation step and ensures stability of the estimator in the presence of noise. Allowing the bound $\mathcal{A}_{nm}$ to grow with the sample size provides flexibility while still controlling large deviations. Finally, the structural assumption in {\bf c} places $f^*$ in a hierarchical composition class. This captures low-dimensional structure within high-dimensional inputs and enables neural networks to adapt to compositional features of the regression function; see \cite{kohler2021rate, kohler2023estimation}.
\color{black}
 In addition, part
{\bf c} imposes uniform derivative and Lipschitz bounds on the component
functions and requires their maximal smoothness and input dimension to
remain bounded. These additional conditions are used to obtain the
neural-network approximation bound in Supplementary Lemma~{\color{Magenta}{1}};
see also \cite{kohler2021rate,kohler2023estimation}.
\color{Black}

Having specified the conditions on each component of the model, we now turn to the sampling design and temporal dependence structure.

\begin{assumption}[Design and dependence structure]
\label{assumption:ts_design_dependence}
{\bf a.} [Support of $x$] The design points $\{ x_{ij} \}_{i=1,j=1}^{n,m_i}$ are supported on $[0,1]^d$.
\\
{\bf b.} [$\beta$-mixing time series] Let
\[
\begin{gathered}
\sigma_x(i)=\sigma(\{x_{ij}\}_{j=1}^{m_i}),\qquad
\sigma_\gamma(i)=\sigma(\gamma_i),\qquad \sigma_\epsilon(i)=\sigma(\{\epsilon_{ij}\}_{j=1}^{m_i}).
\end{gathered}
\]
Suppose that $\{\sigma_x(i)\}_{i=0}^n$, $\{\sigma_\gamma(i)\}_{i=0}^n$, and $\{\sigma_\epsilon(i)\}_{i=0}^n$ are $\beta$-mixing with $\beta$ coefficients $\beta(s)$. See Supplementary Material {\color{magenta}{S3-B1}} for the $\beta$-mixing definition. The $\beta$ coefficients are assumed to have exponential decay,
\[
\beta(s) \lesssim e^{-C_{\beta}s}
\]
for some constant $C_\beta > 0$.
\\
{\bf c.} [Independence between $x$, $\gamma$, and $\epsilon$] For $(i, j) \in \{1,\ldots,n\} \times \{1,\ldots,m_i\}$, the $\sigma$-fields
\[
\begin{gathered}
\sigma(x_{ij}),\qquad\sigma_\gamma=\sigma(\{\gamma_i\}_{i=1}^{n}),\qquad \sigma_\epsilon^{(ij)}=\sigma\!\left(\{\epsilon_{i'j'}\}_{i'=1,j'=1}^{n,m_i}\backslash\{\epsilon_{ij}\}\right)
\end{gathered}
\]
are mutually independent. In addition, for each fixed $i\in \{1,\ldots,n\}$, we assume that $\{x_{ij}\}_{j=1}^{m_i}$ are independent and identically distributed.
\end{assumption}

Assumption~\ref{assumption:ts_design_dependence} specifies how dependence is incorporated into the model. It fixes the support of the design in Part {\bf a}. Then, in Part {\bf b} we allow for weak temporal dependence through $\beta$-mixing with exponentially decaying coefficients. This captures short-range dependence across time, which is commonly observed in spatio-temporal data.  
Our framework allows for dependence in the design, the spatial noise, and the measurement errors. In contrast, prior work such as \cite{kohler2021rate, ma2022theoretical,kohler2023estimation} typically assumes independence and does not account for spatial noise. As a result, the concentration analysis remains valid under structured temporal-spatial dependence.

Part {\bf c} imposes a standard exogeneity condition, meaning that the design variables are independent of the noise components. Specifically, the covariates are assumed to be independent of the spatial noise and of measurement errors across different indices, as is common in the literature; see, among others, \cite{schmidt2020nonparametric,kohler2021rate, wang2022low}. At the same time, the formulation allows for localized dependence at the observation level. This is because independence is only required with respect to $\sigma_\epsilon^{(ij)}$. In particular, this permits dependence between $x_{ij}$ and its corresponding measurement error $\epsilon_{ij}$, while preserving independence across different indices.  Such localized dependence is related to the notion of endogeneity, which arises naturally in many econometric applications, where covariates may be correlated with unobserved disturbances; see, e.g., \cite{wooldridge2010econometric}. In our setting, this effect is restricted to the observation level. This maintains tractability while still allowing for additional modeling flexibility.

With these assumptions in place, we can now state the main result.

\color{Black}
\begin{theorem}
\label{thm:main thm in the temporal-spatial model}
Suppose that Assumptions~\ref{assumption:ts_model}--\ref{assumption:ts_design_dependence} hold and let $\widehat{f}$ be the estimator in $\mathcal F(L,r)$ defined in \eqref{eqn_estimator}, with its number of layers $L$ and number of neurons per layer $r$ set as either one of the following two cases:
\begin{equation}
\begin{aligned}
&\text{\textbf{Case 1.} a wide network with}\quad L\asymp\log(nm)\quad\text{and}\quad r\asymp\max_{(p,K)\in\mathcal{P}}(nm)^{\frac{K}{2(2p+K)}},\\[0.5ex]
&\text{\textbf{Case 2.} a deep network with}\quad L\asymp\log(nm)\cdot\max_{(p,K)\in\mathcal{P}}(nm)^{\frac{K}{2(2p+K)}}\quad \text{and}\quad r=r_0,
\end{aligned}
\label{network_arch_proof}
\end{equation}
where $r_0\in\mathbf Z^+$ is a sufficiently large constant independent of $n$ and $m$.
Let
\[
\mathcal{A}_{nm} \asymp \max \left\{\sigma_\epsilon,\sigma_\gamma\right\}\cdot \log^{1/2}(nm).
\]
and $(a_n)_{n\ge 1}$ be any sequence such that $a_n \to \infty$.
Then with probability approaching one, it holds that
\begin{equation}
\begin{aligned}
\|\widehat f_{\mathcal{A}_{nm}}-f^*\|_\lt^2&\lesssim a_n\left(\sigma_\epsilon^2+\sigma_\gamma^2+1\right)
\max_{(p,K)\in\mathcal P}\left(\frac{1}{nm}\right)^{\frac{2p}{2p+K}}
\log^{7.3}(nm)\\
&\quad +a_n\frac{\left(\sigma_\epsilon^2+\sigma_\gamma^2\right)\log(nm)}{n}.
\end{aligned}
\label{rate_in_theorem_temporal_spatial_theorem}
\end{equation}
\end{theorem}

\color{black}

The first term in \eqref{rate_in_theorem_temporal_spatial_theorem} is the main approximation--estimation term. It is driven by the total number of observations, of order \(nm\), together with the compositional complexity of \(f^*\), encoded through the pairs \((p,K)\in\mathcal P\). In the independent case, as established in Corollary~\ref{thm:main thm in the functional model}, this term is of order
\(
\max_{(p,K)\in\mathcal P} (nm)^{-\frac{2p}{2p+K}}
\)
up to logarithmic factors. 
In the dependent setting, temporal dependence enters through the blocking and coupling argument for the \(\beta\)-mixing sequence, applied to the full time-indexed process (including the covariates, measurement errors, and spatial random effects). At the level of the proof, this leads to an intermediate complexity term of order
\(
\max_{(p,K)\in\mathcal P}\left(\frac{S}{nm}\right)^{\frac{2p}{2p+K}},
\)
where \(S\) denotes the block length. Under exponential \(\beta\)-mixing, choosing \(S\asymp \log n\) affects this contribution only through an additional logarithmic factor, so that the algebraic rate is preserved up to logarithmic corrections. Consequently, the first term matches the standard neural-network approximation rate associated with hierarchical function classes, up to logarithmic factors, see Theorem~3 of Kohler and Langer~\citeyearpar{kohler2021rate}.

\color{black}
We now clarify how the logarithmic factors arising at different stages of the proof combine to produce the factor $\log^{7.3}(nm)$. The first contribution comes from the metric entropy bound for the truncated version of the neural network class $\mathcal F(L,r)$. This bound controls the complexity of the network class and contributes a factor of order $\log^{4.3}(nm)$. It follows from the VC-dimension bound of \cite{bartlett2019nearly}, together with the empirical covering-number argument of \cite{gyorfi2002distribution}; see also Lemma~19 of \cite{kohler2021rate} and Supplementary Lemma~24 in Section~S8. The next contribution arises from the localization and symmetrization arguments used to compare the empirical and population norms. These arguments contribute an additional factor of order $\log(nm)$, raising the logarithmic exponent from $4.3$ to $5.3$ and leading to the intermediate term $a_n(\mathcal A_{nm}^2+1)\phi_{nm/S}\log^{5.3}(nm)$. Another logarithmic factor comes from the truncation level, since $\mathcal A_{nm}^2+1\lesssim(\sigma_\epsilon^2+\sigma_\gamma^2+1)\log(nm)$. Finally, temporal dependence is handled through a blocking and coupling argument with block size $S\asymp\log n$. Because $\phi_{nm/S}\leq S\phi_{nm}$, this step contributes at most one further logarithmic factor. Combining these contributions gives the displayed $\log^{7.3}(nm)$ term. Thus, the exponent $7.3$ reflects the cumulative effect of the entropy, localization, truncation, and blocking bounds used in our proof and is not claimed to be optimal.

Related work shows that the exponent of the logarithmic factor can be reduced in other settings. In particular, Ou and B\"olcskei~\citeyearpar{ou2024covering} develop refined covering-number bounds for deep ReLU networks, including networks with unrestricted weights and truncated outputs. In a related i.i.d.\ Lipschitz-regression setting, they combine these bounds with bounded-weight approximation to remove a $\log^6 n$ factor. Obtaining a corresponding improvement for our hierarchical $\beta$-mixing temporal--spatial model would require a separate analysis of the approximation, empirical-process, norm-comparison, and spatial-random-effect arguments, which is beyond the scope of the present paper.

\color{Black}

\begingroup
\color{Black}

Theorem~\ref{thm:main thm in the temporal-spatial model} treats the
two canonical wide and deep architectures used in \cite{kohler2021rate}.
Its proof depends on the architecture only through a uniform
approximation bound for $f^*$ and empirical covering-number bounds
for the corresponding truncated network class.
Replacing these architecture-specific inputs by abstract approximation
and entropy conditions yields a risk bound for
\textcolor{black}{general estimation families, not only neural networks}.
\textcolor{black}{Theorem~\ref{thm:general_architecture_temporal_spatial}
states this bound. We present the explicit rate first to emphasize the
main statistical conclusion.}

\par
\endgroup

To state the result, we use the block notation introduced in Supplementary Material~{\color{Magenta}{S11}}. For a block length $S_n$ and $s\in\{0,\ldots,S_n-1\}$, let $\mathcal J_{e,s}$ and $\mathcal J_{o,s}$ denote the even and odd block index sets. The associated empirical norms are denoted by $\|\cdot\|_{|\mathcal J_{e,s}|m}$ and $\|\cdot\|_{|\mathcal J_{o,s}|m}$.

\begingroup
\color{Black}

\begin{theorem}[General \textcolor{black}{function-class} bound]
\label{thm:general_architecture_temporal_spatial}
Suppose that Assumptions~\ref{assumption:ts_model}--\ref{assumption:ts_design_dependence} hold\textcolor{black}{, except for Assumption~\ref{assumption:ts_regularity}{\bf c}}.  For each $n$ and $m$, let \textcolor{black}{$\mathfrak F_{nm}$ be a nonempty deterministic pointwise measurable class of measurable functions from $[0,1]^d$ to $\mathbb R$}, and let $\widehat f$ be \textcolor{black}{a measurable minimizer} over $\textcolor{black}{\mathfrak F_{nm}}$ \textcolor{black}{of the empirical criterion in} \eqref{eqn_estimator}. Assume that there exists \textcolor{black}{a deterministic} \(\overline f \in \textcolor{black}{\mathfrak F_{nm}}\) such that
\[
\|f^*-\overline f\|_\infty^2 \le \varepsilon_{nm}^2
\textcolor{black}{,\qquad\|\overline f\|_\infty\le\mathcal A_{nm}}
\]
for some sequence \(\varepsilon_{nm}>0\).
Let \((S_n)_{n\ge1}\) be any sequence satisfying
\(
C_\beta \log(n) \le S_n \le n,
\)
for a sufficiently large constant \(C_\beta>0\) depending only on the mixing rate in Assumption~\ref{assumption:ts_design_dependence}.
Suppose further that there exist positive sequences \((\rho_n)_{n\ge1}\) and \((\widetilde\rho_n)_{n\ge1}\) such that, for all sufficiently large \(n\),
\begin{align}
\begin{aligned}
\sup_{0<\delta\le1}\;\max_{0\le s\le S_n-1}
\sup_{z_1,\dots,z_{|\mathcal J_{e,s}|m}\in[0,1]^d} \frac{\log\mathcal N\!\left(\delta,
\textcolor{black}{\mathfrak F_{nm,\mathcal A_{nm}}},\|\cdot\|_{|\mathcal J_{e,s}|m}\right)}
{|\mathcal J_{e,s}|m\,\log(\delta^{-1})}\le C\,\rho_n,
\end{aligned}\label{eq:general_entropy_complexity_temporal_spatial_block_even}\\
\begin{aligned}
\sup_{0<\delta\le1}\;\max_{0\le s\le S_n-1}
\sup_{z_1,\dots,z_{|\mathcal J_{o,s}|m}\in[0,1]^d} \frac{\log\mathcal N\!\left(\delta,
\textcolor{black}{\mathfrak F_{nm,\mathcal A_{nm}}},\|\cdot\|_{|\mathcal J_{o,s}|m}\right)}
{|\mathcal J_{o,s}|m\,\log(\delta^{-1})}\le C\,\rho_n,
\end{aligned}\label{eq:general_entropy_complexity_temporal_spatial_block_odd}
\end{align}
and
\begin{align}
\label{eq:general_entropy_complexity_temporal_spatial_global}
\sup_{0<\delta\le 1}\;
\sup_{z_1,\dots,z_{nm}\in[0,1]^d}
\frac{
\log \mathcal N\!\left(
\delta,\textcolor{black}{\mathfrak F_{nm,\mathcal A_{nm}}},\|\cdot\|_{nm}
\right)
}{
nm\,\log(\delta^{-1})
}
&\le C\,\widetilde\rho_n,
\end{align}
where \(C>0\) is a universal constant and
\(
\textcolor{black}{\mathfrak F_{nm,\mathcal A_{nm}}}
:=
\{f_{\mathcal A_{nm}}:f\in\textcolor{black}{\mathfrak F_{nm}}\}.
\)
\\
Let
\(
\mathcal A_{nm}\asymp \max\{\sigma_\epsilon,\sigma_\gamma\}\sqrt{\log(nm)},
\)
and let \((a_n)_{n\ge 1}\) be any sequence such that \(a_n \to \infty\). Then, with probability approaching one,
\begin{equation}
\begin{aligned}
\|\widehat f_{\mathcal A_{nm}}-f^*\|_{\mathcal L_2}^2&\lesssim\varepsilon_{nm}^2+a_n\mathcal A_{nm}^2\,\rho_n\,\log(nm)\\
&\quad +a_n\mathcal A_{nm}^2\,\widetilde\rho_n\,\log(nm) +a_n\mathcal A_{nm}^2\beta(S_n)+\frac{\sigma_\gamma^2}{n}.
\end{aligned}
\label{eq:general_risk_bound_temporal_spatial}
\end{equation}
\end{theorem}

\endgroup

\color{black}

Theorem~\ref{thm:general_architecture_temporal_spatial} decomposes the
risk bound into an approximation term, two entropy-driven stochastic
terms, and contributions arising from temporal dependence and spatial
noise. The quantity $\varepsilon_{nm}^2$ measures how well $f^*$ can be
approximated by the chosen \textcolor{black}{family $\mathfrak F_{nm}$}.
\textcolor{black}{The hierarchical composition assumption is not
required for this result. However, when
$\mathfrak F_{nm}=\mathcal F(L_{nm},r_{nm})$ and $f^*$ belongs to the
hierarchical composition class $\mathcal H(l,\mathcal P)$,}
the approximation term reflects the compositional structure of $f^*$.
In particular, the hierarchy level $l$ is naturally associated with the
depth \textcolor{black}{used to approximate} the recursive composition,
while the local pairs $(p,K)\in\mathcal P$ determine the approximation
difficulty of the component functions.

The stochastic terms are determined by the covering number complexity of the
truncated class
$\textcolor{black}{\mathfrak F_{nm,\mathcal A_{nm}}}$.
The blockwise entropy rate $\rho_n$, defined through
\eqref{eq:general_entropy_complexity_temporal_spatial_block_even}--\eqref{eq:general_entropy_complexity_temporal_spatial_block_odd},
controls the empirical processes arising from the blocking argument.
The global entropy rate $\widetilde\rho_n$, defined through
\eqref{eq:general_entropy_complexity_temporal_spatial_global}, controls
the corresponding full-sample empirical process.

The remaining terms have distinct sources. The term
$a_n\mathcal A_{nm}^2\beta(S_n)$ is the coupling error arising from
temporal dependence. \textcolor{black}{The term}
$\sigma_\gamma^2/n$ captures the contribution of the time-indexed
spatial random effects.

\textcolor{black}{We now specialize
Theorem~\ref{thm:general_architecture_temporal_spatial} to the fully
connected ReLU class $\mathcal F(L_{nm},r_{nm})$ in
Definition~\ref{def:fully Dense NN space}.}
 Lemma~{\color{Magenta}{24}} in the Supplementary material allows the entropy rates
$\rho_n$ and $\widetilde\rho_n$ to be expressed directly in terms of the
network depth and width.
\textcolor{black}{The depth and width may follow deterministic
sequences other than those in \eqref{network_arch_proof}.}
This gives the following architecture-dependent bound.

\begingroup
\color{Black}

\begin{corollary}[Architecture-dependent risk bound]
\label{cor:explicit_architecture_temporal_spatial}

Retain the hypotheses and notation of Theorem~\ref{thm:general_architecture_temporal_spatial}, except for the entropy conditions \eqref{eq:general_entropy_complexity_temporal_spatial_block_even}--\eqref{eq:general_entropy_complexity_temporal_spatial_global}.
\textcolor{black}{For each $n$ and $m$, let
$L_{nm},r_{nm}\in\mathbf Z^+$ be deterministic and set
$\mathfrak F_{nm}=\mathcal F(L_{nm},r_{nm})$,
to be the network functions class defined in Equation (\ref{def:fully Dense NN space}).}
Set $\mathfrak V_{nm}:=L_{nm}^2r_{nm}^2\log\!\left(eL_{nm}r_{nm}^2\right)$. Then \textcolor{black}{the entropy conditions
\eqref{eq:general_entropy_complexity_temporal_spatial_block_even}--\eqref{eq:general_entropy_complexity_temporal_spatial_global}} hold with $\rho_n\lesssim S_n\mathfrak V_{nm}(nm)^{-1}\log^{1.3}(nm)$ and $\widetilde\rho_n\lesssim\mathfrak V_{nm}(nm)^{-1}\log^{1.3}(nm)$. Consequently, with probability approaching one,
\begin{equation}
\begin{aligned}
\|\widehat f_{\mathcal A_{nm}}-f^*\|_{\mathcal L_2}^2 \lesssim\;\varepsilon_{nm}^2+a_n\mathcal A_{nm}^2
\frac{S_n\mathfrak V_{nm}}{nm}\log^{2.3}(nm) +a_n\mathcal A_{nm}^2\beta(S_n)+\frac{\sigma_\gamma^2}{n}.
\end{aligned}
\label{eq:explicit_architecture_temporal_spatial}
\end{equation}
\end{corollary}

\endgroup

The proof is given in Supplementary Material~{\color{Magenta}{S11}}.
\textcolor{black}{The corollary keeps the approximation error explicit.
Different depth--width choices can therefore be assessed through
$\varepsilon_{nm}^2$ and $\mathfrak V_{nm}$.}

\textcolor{black}{To recover
Theorem~\ref{thm:main thm in the temporal-spatial model}, impose
Assumption~\ref{assumption:ts_regularity}{\bf c} and choose either
architecture in \eqref{network_arch_proof}.}
Coinsider the term $\phi_N:=\max_{(p,K)\in\mathcal P}N^{-2p/(2p+K)}$.
 Lemma~{\color{Magenta}{1}} in the Supplementary material gives
$\varepsilon_{nm}^2\lesssim\phi_{nm}$, and both architecture regimes
satisfy $\mathfrak V_{nm}\lesssim nm\,\phi_{nm}\log^3(nm)$.
Assumption~\ref{assumption:ts_regularity}{\bf b}, together with
$\varepsilon_{nm}\to0$, ensures that
$\|\textcolor{black}{\overline f}\|_\infty\leq\mathcal A_{nm}$ for all
sufficiently large $n$.

Under exponential $\beta$-mixing,
\textcolor{black}{choose $S_n\asymp\log n$ with a sufficiently large
proportionality constant. Then}
$\beta(S_n)\lesssim n^{-1}$.
Together with
$\mathcal A_{nm}^2\lesssim(\sigma_\epsilon^2+\sigma_\gamma^2)\log(nm)$,
these relations reduce
Corollary~\ref{cor:explicit_architecture_temporal_spatial} to
\eqref{rate_in_theorem_temporal_spatial_theorem}.
Hence both canonical architecture regimes are covered by the corollary.

\color{Black}

\color{black}

A complementary line of work is \cite{jiao2023deep}. It studies nonasymptotic error bounds for deep neural network estimators in the classical i.i.d.\ regression model
\[
Y=f_0(X)+\eta,
\]
with independent observations and sub-exponential noise. In terms of our notation, their regression function \(f_0\) corresponds to \(f^*\), their covariate \(X\) corresponds to \(x_{ij}\), and their response \(Y\) corresponds to \(y_{ij}\) in the special case of \eqref{true_model} with \(\gamma_i=0\) and \(m_i=1\).

There are several important differences. First, their analysis assumes independent data, whereas our framework allows for temporal dependence through a \(\beta\)-mixing structure acting on the full time-indexed process, including covariates, measurement errors, and spatial random effects. Second, \cite{jiao2023deep} assumes that \(f_0\) is Hölder smooth and exploits structural assumptions on the distribution of the covariates, such as approximate manifold support, exact manifold support, or low Minkowski dimension. By contrast, our framework captures complexity through the compositional structure of the regression function via \(\mathcal H(l,\mathcal P)\). \textcolor{black}{Third, their results are formulated in terms of explicit architectural quantities for the feedforward network class $F_{D,W,U,S,B}$, where $D$ denotes depth, $W$ width, $U$ the number of hidden-layer neurons, $S$ the number of network parameters, and $B$ a uniform output bound.  Theorem~\ref{thm:general_architecture_temporal_spatial} instead
separates the approximation and entropy inputs required by the
temporal--spatial analysis
\textcolor{black}{for a general estimation family $\mathfrak F_{nm}$,
which need not consist of neural networks}.
Corollary~\ref{cor:explicit_architecture_temporal_spatial}
\textcolor{black}{specializes this result to fully connected ReLU
networks and} expresses the entropy contribution directly through
$L_{nm}$ and $r_{nm}$.}

\textcolor{black}{
Despite these differences, both approaches decompose the risk into approximation and stochastic complexity terms. Under the correspondence above, the result of \cite{jiao2023deep} can be stated directly in our notation. Specializing their network class to an equal-width fully connected network replaces their depth and width by $L_{nm}$ and $r_{nm}$, respectively. Since $m_i=1$ in this comparison, $nm=n$. For a fixed input dimension and a fixed uniform output bound, Lemma~3.2 of \cite{jiao2023deep} bounds the expected square $\mathcal L_2$ error by the network approximation error plus a constant multiple of
$L_{nm}^2r_{nm}^2\log(L_{nm}r_{nm}^2)(nm)^{-1}\log^5(nm)$.
The corresponding stochastic term in Corollary~\ref{cor:explicit_architecture_temporal_spatial} is
$a_n\mathcal A_{nm}^2S_nL_{nm}^2r_{nm}^2
\log\!\left(eL_{nm}r_{nm}^2\right)(nm)^{-1}\log^{2.3}(nm)$,
where $S_n$ is the block length used to control temporal dependence.} Thus, up to logarithmic and dependence-related factors, the two
stochastic bounds have the same algebraic dependence on the network
depth and width. 

A further distinction is that temporal dependence in our framework is handled through a blocking and coupling argument. This introduces the additional dependence term
\(
a_n\mathcal A_{nm}^2\beta(S_n),
\)
and requires controlling the empirical process under \(\beta\)-mixing rather than independence. 
In addition, the presence of time-indexed spatial random effects contributes the term \(\sigma_\gamma^2/n\), reflecting variability at the temporal scale. Thus, while the algebraic approximation--estimation trade-off remains analogous, our result extends architectural risk bounds for deep neural network estimators to temporally dependent data with spatial structure.

\color{black}
\begin{remark}
\label{remark:example convergence rate in the temporal-spatial model}
Suppose that each level $i\in\{1,\ldots,l\}$ of the hierarchical
composition model has smoothness $p_i$ and dimension $K_i$, with
$K_l\leq K_{l-1}\leq\cdots\leq K_1\leq d$. In the absence of spatial
noise, the leading term in the rate of
Theorem~\ref{thm:main thm in the temporal-spatial model} is
\[
(\sigma_\epsilon^2+1)
\max_{(p,K)\in\mathcal P}
\left(\frac{1}{nm}\right)^{\frac{2p}{2p+K}}
\log^{7.3}(nm).
\]
Up to a logarithmic factor, this term matches the $\lt$ lower bound
in Theorem~3 of \cite{schmidt2020nonparametric}; see also
Supplementary Lemma~{\color{Magenta}{23}}.
\end{remark}
\color{Black}

In the presence of spatial noise, the rate for the same special case becomes
\[
\max_{(p, K) \in \mathcal P}(nm)^{\frac{-2p}{2p+K}} + n^{-1},
\]
ignoring logarithmic factors. A natural question arises: is this rate optimal for this special case? A partial answer is provided in the following lemma, which extends Theorem 3 in \cite{schmidt2020nonparametric} to the temporal-spatial setting.

\begin{lemma}
\label{Mlb-M}
Suppose Assumptions~\ref{assumption:ts_model}--\ref{assumption:ts_design_dependence} hold. Recall that $\mathcal{H}(l,\mathcal{P})$ denotes the hierarchical composition model for the function class.
\color{black}
Consider the case
\[
\mathcal P=\widetilde{\mathcal P}
:=
\left\{
\left(p_1,K_1\right),
\left(p_2,K_2\right),
\ldots,
\left(p_l,K_l\right)
\right\},
\]
where $p_i\in[1,\infty)$ and $K_i\in\mathbf Z^+$ for every
$i\in\{1,\ldots,l\}$, with
$K_l\leq K_{l-1}\leq\cdots\leq K_1\leq d$.
\color{Black}
Let $\widetilde{f}$ be any estimator based on the observations $\{(x_{ij},y_{ij})\}_{i=1,j=1}^{n,m_i}$. Then, it holds that
\[
\begin{aligned}
\underset{n\rightarrow\infty}{\lim\sup}\,
\underset{f^*\in\mathcal{H}(l,\widetilde{\mathcal{P}})}{\sup} \mathbb{P}\Bigl(\|\widetilde{f}-f^*\|_\lt^2
\geq C_{\text{opt}}\Bigl(\frac{1}{n} +\max_{(p,K)\in\mathcal P}\left(\frac{1}{nm}\right)^{\frac{2p}{2p+K}}\Bigr)\Bigr)>0,
\end{aligned}
\]
where $C_{\text{opt}} > 0$ is a constant.
\end{lemma}

\color{black}
\textcolor{black}{Specializing
Theorem~\ref{thm:main thm in the temporal-spatial model}
to the setting of Lemma~\ref{Mlb-M} yields an upper bound that
matches the lower bound in that lemma up to logarithmic factors.}
Hence the rate displayed above is minimax optimal up to
logarithmic factors in this particular setting.
Lemma~\ref{Mlb-M} thereby extends the lower-bound result
of Theorem~3 in \cite{schmidt2020nonparametric} from the model
without spatial noise to the temporal--spatial setting. In this
extension, the additional $n^{-1}$ term accounts for the
time-indexed spatial random effects.
\color{Black}

\color{black}
Finally, it is worth considering the implications of relaxing
Assumption~4{\bf b}. Suppose instead that the mixing coefficients decay
polynomially, so that $\beta(s)\lesssim s^{-\alpha}$ for some
$\alpha>0$. In this case, we obtain the rate 
\begin{align*}
a_n\frac{\sigma_\epsilon^2+\sigma_\gamma^2+1}{n^\alpha}
\log^{6.3}(nm)
+
\frac{(\sigma_\epsilon^2+\sigma_\gamma^2)\log(nm)}{n}
+
\frac{\sigma_\gamma^2}{n},
\end{align*}
provided that
\[
m\geq
\max_{(p,K)\in\mathcal P}
n^{\alpha\left(\frac{2p+K}{2p}\right)}.
\]

The difference from the exponential-decay setting arises in the
blocking argument. Under exponential decay, choosing
$S\asymp\log n$ reduces the coupling error $\beta(S)$ to order at most
$n^{-1}$, while the effective blockwise sample size $nm/S$ differs
from $nm$ only by a logarithmic factor. Under polynomial decay, the
same choice gives only
$\beta(S)\lesssim(\log n)^{-\alpha}$, so longer blocks are required.
Taking $S=n$ makes the coupling error of order at most $n^{-\alpha}$
and reduces the effective blockwise sample size to $m$. The condition
above ensures that
$\max_{(p,K)\in\mathcal P}m^{-2p/(2p+K)}\leq n^{-\alpha}$.
Consequently, the blockwise estimation and coupling terms are both of
order at most $n^{-\alpha}$, up to logarithmic and noise factors. The
remaining $n^{-1}$-scale terms are unaffected by the change in the
mixing rate. See Supplementary Corollary~{\color{Magenta}{3}} for
details.
\color{Black}

\subsection{Functional Regression with Independent Observations}
\label{indsection}

In the case that there is no temporal dependence, our model in \eqref{true_model} reduces to the usual functional regression model. This includes as a particular case the linear functional regression setting studied in  \cite{cardot1999functional, cardot2003testing, cai2006prediction,hall2007methodology,cai2012minimax,  petersen2016functional}.   

We now give the estimation error of estimator under functional regression model with no temporal dependence. 

\begin{corollary}  \label{thm:main thm in the functional model} 

Suppose Assumptions \ref{assumption:ts_model}--\ref{assumption:ts_design_dependence} hold except Assumption \ref{assumption:ts_design_dependence}{\bf b}.
Consider the model  $  y_{ij} = f^*(x_{ij} )+\gamma _i(x_{ij}) +\epsilon_{ij},$   where $\{ x_{ij}\}_{i=1, j=1}^{n,m_i}$ are i.i.d. $d$-dimensional random variables supported on $[0,1]^d$, $ \{ \epsilon_{ij}\}_{i=1, j=1}^{n,m_i}$ are i.i.d sub-Gaussian random variables with sub-Gaussian parameter $\sigma_\epsilon$, $ \{  \gamma_i :[0,1]^d \to \mathbb{R} \}_{i=1}^n$ is a collection of i.i.d. separable stochastic processes.
Then with $L$, $r$ and $\mathcal{A}_{nm}$ as in \Cref{thm:main thm in the temporal-spatial model}, 
\begin{equation}
\begin{aligned}
\|\widehat f_{\mathcal{A}_{nm}}-f^*\|_\lt^2 \lesssim a_n\left(\sigma_\epsilon^2+\sigma_\gamma^2+1\right)
\max_{(p,K)\in\mathcal P}\left(\frac{1}{nm}\right)^{\frac{2p}{2p+K}}
\log^{6.3}(nm) +\frac{\sigma_\gamma^2}{n},
\end{aligned}
\label{bound_functional_regression}
\end{equation} holds with probability approaching to one. Here \((a_n)_{n\ge 1}\) is any sequence such that \(a_n\to\infty\). \end{corollary} 

 Thus, compared to the original  temporal-spatial model upper bound in Theorem \ref{thm:main thm in the temporal-spatial model}, the rate for functional regression with independent observations does not include the term  
which is due to the temporal dependence. 

\subsection{Model with Temporal Dependence without Spatial Dependence}

\color{black}

When $\gamma_i(\cdot)=0$ and $m_i=1$ for all $i$, we simplify notation by writing $x_i$ instead of $x_{i1}$, $y_i$ instead of $y_{i1}$, and $\epsilon_i$ instead of $\epsilon_{i1}$. In this case, the model in \eqref{true_model} reduces to a time series regression model without spatial dependence,
\[
y_i = f^*(x_i) + \epsilon_i.
\]
In this setting, the observations are indexed only by time, and dependence arises through the temporal structure of the sequence $\{x_i,\epsilon_i\}_{i=1}^n$.

This formulation is closely related to the framework studied in \cite{kohler2021rate}, where independence is assumed. Here, we relax this assumption and allow for temporal dependence through a $\beta$-mixing condition, while maintaining the same estimator defined in \eqref{eqn_estimator}. Compared to the full model in \eqref{true_model}, the absence of the spatial component $\gamma_i$ eliminates the additional variability associated with time-indexed random effects, and naturally leads to the case $m_i=1$.

We now establish the corresponding risk bound under temporal dependence.

\color{Black}

\begin{corollary}\label{thm:main thm in the temporal model}
Suppose that Assumptions \ref{assumption:ts_model}--\ref{assumption:ts_design_dependence} hold with $m_i =1$ and $\gamma_i =0$ for all $i \in \{1,\ldots,n\}$. 
Then the model can be written as  $y_{i } = f^*(x_{i } ) +\epsilon_{i }. $ Let $(a_n)_{n\ge 1}$ be any sequence such that $a_n \to \infty$.
Moreover, with the $L$ and $r$ chosen as \Cref{thm:main thm in the temporal-spatial model},  
and $ \mathcal{A}_n \asymp \sigma_\epsilon \cdot \log^{1/2}(n)$, the estimator \eqref{eqn_estimator}
satisfies 
\begin{align*}
\left\|\widehat{f}_{\mathcal{A}_n}-f^*\right\|_{\mathcal{L}_2}^2 
\lesssim &  a_n(\sigma^2_\epsilon+1) \log^{7.3}(n) \max_{(p,K) \in \mathcal{P}}\left( \frac{1}{n}\right)^{\frac{2p}{2p+K}},
\end{align*} with probability approaching to one.\end{corollary}

Compared to the rate in \Cref{thm:main thm in the temporal-spatial model}, the upper bound in Corollary~\ref{thm:main thm in the temporal model} does not involve $\sigma_\gamma$, reflecting the absence of spatial random effects. As a result, the additional $n^{-1}$-scale contribution present in the temporal-spatial model disappears, and the rate is driven entirely by the approximation--estimation term at the scale $n$.

\color{black}

This setting is also closely related to \cite{ma2022theoretical}, which studies deep neural network estimators for temporally dependent data. In their framework, the data are generated according to
\[
Y_t = f_0(X_t) + \epsilon_t,
\]
where $\{X_t\}$ is a strictly stationary and exponentially $\alpha$-mixing process, while the noise sequence $\{\epsilon_t\}$ satisfies a Bernstein-type moment condition and is independent of $\{X_s : s \le t\}$. Translating their notation into ours, their regression function $f_0$ corresponds to $f^*$, their covariate process $\{X_t\}$ corresponds to $\{x_i\}$, and their observations $\{Y_t\}$ correspond to $\{y_i\}$ in the temporal model obtained by setting $\gamma_i=0$ and $m_i=1$ in \eqref{true_model}.

There are, however, several important differences. First,
Ma and Safikhani~\citeyearpar{ma2022theoretical} impose dependence only on the covariate
process, whereas our framework allows temporal dependence to act on
the full time-indexed process, including the covariates and the
measurement errors; see
Assumption~\ref{assumption:ts_design_dependence}.

\color{black}
Second, Ma and Safikhani~\citeyearpar{ma2022theoretical} consider a hierarchical composition
class in which each coordinate function $g_{i,j}$ at level $i$
depends on at most $t_i$ of its $d_i$ inputs. Although they focus on
the regime $t_i\ll d_i$, the formal definition of their class also
permits $t_i=d_i$. In particular, taking $q=0$ and $t_0=d_0=d$
recovers the $d$-dimensional H\"older class. Their rate avoids the
curse of dimensionality only when the relevant $t_i$ are sufficiently
small relative to the corresponding $d_i$.
The same qualification applies to $\mathcal H(l,\mathcal P)$. Its
definition does not require the component input dimensions $K$ to be
small. The rate benefits from compositional structure only when the
relevant values of $K$ are sufficiently small relative to the ambient
dimension. Our formulation has the additional flexibility
that the smoothness and input dimension may vary among component
functions within the same level.

At the estimator level, the two approaches differ.
Ma and Safikhani~\citeyearpar{ma2022theoretical} consider the sparse neural network class
$F(L,p,s,F)$, where $L$ denotes the depth, $p$ the layer widths, $s$
the number of nonzero parameters, and $F$ a uniform bound on the
network output. We instead estimate over the fully connected class
$\mathcal F(L,r)$.
\color{Black}

Despite these differences, both approaches yield rates of the same general form. In \cite{ma2022theoretical}, the prediction error is of order
\[
\phi_n \log^7 n, \qquad \text{with } \phi_n = \max_i n^{-\frac{2\beta_i^*}{2\beta_i^* + t_i}}.
\]
\color{black}
Here $\beta_i^*$ is the effective smoothness associated with level
$i$. The quantity $t_i$ is an upper bound on the number of input
variables on which each component function at that level depends.
\color{Black}. \color{black}
Equivalently, setting $p_i:=\beta_i^*$ and $K_i:=t_i$, their rate
can be expressed in our notation as
\[
\max_i n^{-\frac{2p_i}{2p_i+K_i}}\log^7 n.
\]
\color{Black} This is directly comparable to our rate $n^{-\frac{2p}{2p+K}}$ up to logarithmic factors, with $p$ playing the role of smoothness and $K$ the role of intrinsic dimension. 

In both works, temporal dependence manifests through additional logarithmic factors relative to the independent case. In \cite{ma2022theoretical}, the prediction error rate is of order $\phi_n \log^7 n$, compared to $\phi_n \log^3 n$ under independence, so dependence effectively contributes an additional $\log^4 n$ factor. 

A similar phenomenon appears in our analysis, where the $\beta$-mixing dependence is handled through a blocking and coupling argument. In our case, the final rate involves a logarithmic factor of order $\log^{7.3}(n)$. As detailed above, this exponent arises from the accumulation of several steps in the proof, including metric entropy bounds, localization arguments, truncation, and the blocking construction used to control dependence. Thus, while the precise logarithmic exponents differ, both approaches show that temporal dependence does not alter the algebraic convergence rate, but instead introduces additional logarithmic complexity reflecting the cost of controlling dependence.

 Overall, our result complements \cite{ma2022theoretical} by showing that comparable neural-network rates remain valid under a different dependence structure and without restricting attention to sparse network architectures.

\subsection{Summary of Rates and Dominant Regimes}
\label{subsec:summary_rates_regimes}

The preceding results show how the temporal-spatial risk bound simplifies under different structural assumptions. In the full model, the rate contains two components,
\[
\max_{(p,K)\in\mathcal P}(nm)^{-\frac{2p}{2p+K}}
\qquad \text{and} \qquad
n^{-1},
\]
up to logarithmic factors. The first is the main approximation--estimation term, determined by the total number of observations of order \(nm\) together with the compositional complexity of \(f^*\). The second appears at the temporal scale \(n\), since the spatial random effects are indexed by time and shared across observations within each time point, and because controlling temporal dependence introduces an additional penalty at that scale.

Balancing these two contributions yields the transition regime given by
\[
m \asymp n^{\frac{K}{2p}},
\]
up to polynomial factors in \(\log(nm)\). More precisely, if
\[
m \le n^{K/(2p)}\,\mathrm{poly}(\log(nm)),
\]
then the approximation--estimation term is dominant, and the rate is mainly driven by the total sample size \(nm\). Here \(\mathrm{poly}(\log(nm))\) denotes a polynomial factor in \(\log(nm)\); see Table~\ref{tab:special_cases_temporal_spatial}. If instead
\[
m \ge n^{K/(2p)}\,\mathrm{poly}(\log(nm)),
\]
then the \(n^{-1}\)-scale contribution becomes dominant. At the boundary, both terms are of the same order.

This comparison also clarifies the role of the different model components. Removing temporal dependence leaves the algebraic form of the approximation--estimation term unchanged, but improves the logarithmic factor by eliminating the blocking and coupling cost. Removing the spatial random effects eliminates the \(n^{-1}\)-scale contribution entirely, leaving only the approximation--estimation term. Table~\ref{tab:special_cases_temporal_spatial} summarizes these special cases and places them alongside the most closely related results in the literature.

\begin{table*}[!t]
\centering
\begin{small}
\caption{Summary of rates, assumptions, and dominant regimes. The regime boundaries are stated up to polynomial factors in $\log(nm)$. That is, $\mathrm{poly}(\log(nm)) = (\log(nm))^{c}$ with $c = \frac{6.3(2p+K)}{2p}$, arising from balancing the two terms in the rate.}
\setlength{\tabcolsep}{3pt}
\renewcommand{\arraystretch}{1.35}
\begin{tabular}{|c|c|c|}
\hline
\textbf{Setting / Assumptions} & \textbf{Rate} & \textbf{Dominant regime} \\ \hline

\multicolumn{3}{|c|}{\textbf{This work}} \\ \hline

\begin{tabular}{@{}c@{}}
Temporal-spatial \\
$\gamma_i + \epsilon_{ij}$,\; $\beta$-mixing
\end{tabular}
&
$\displaystyle
\underbrace{\max_{(p,K)\in\mathcal P}(nm)^{-\frac{2p}{2p+K}}\log^{7.3}(nm)}_{\text{Term 1}}
+
\underbrace{\frac{\log(nm)}{n}}_{\text{Term 2}}
$
&
\begin{tabular}{@{}c@{}}
$m \le n^{K/(2p)}\,\mathrm{poly}(\log(nm))$ \\
Term 1 \\[0.3em]
$m \ge n^{K/(2p)}\,\mathrm{poly}(\log(nm))$ \\
Term 2
\end{tabular}
\\ \hline

\begin{tabular}{@{}c@{}}
i.i.d. functional \\
$\gamma_i + \epsilon_{ij}$,\; i.i.d.
\end{tabular}
&
$\displaystyle
\underbrace{\max_{(p,K)\in\mathcal P}(nm)^{-\frac{2p}{2p+K}}\log^{6.3}(nm)}_{\text{Term 1}}
+
\underbrace{\frac{1}{n}}_{\text{Term 2}}
$
&
\begin{tabular}{@{}c@{}}
$m \le n^{K/(2p)}\,\mathrm{poly}(\log(nm))$ \\
Term 1 \\[0.3em]
$m \ge n^{K/(2p)}\,\mathrm{poly}(\log(nm))$ \\
Term 2
\end{tabular}
\\ \hline

\begin{tabular}{@{}c@{}}
Temporal only \\
$\epsilon_{i,1}$,\; $\gamma_i=0$,\\$\beta$-mixing,\; $m_i=1$
\end{tabular}
&
$\displaystyle
\max_{(p,K)\in\mathcal P}
n^{-\frac{2p}{2p+K}}\log^{7.3}(n)
$
&
N/A
\\ \hline

\multicolumn{3}{|c|}{\textbf{Related work}} \\ \hline

\begin{tabular}{@{}c@{}}
\cite{kohler2021rate} \\
$\epsilon_{i,1}$, $\gamma_i=0$, $m_i=1$,\\i.i.d., dense NN
\end{tabular}
&
$\displaystyle
\max_{(p,K)\in\mathcal P}
n^{-\frac{2p}{2p+K}}\log^{c'}(n)
$
&
N/A
\\ \hline

\begin{tabular}{@{}c@{}}
\cite{schmidt2020nonparametric} \\
$\epsilon_{i,1}$, $\gamma_i=0$, $m_i=1$,\\i.i.d., sparse NN
\end{tabular}
&
$\displaystyle
\max_i n^{-\frac{2\beta_i^*}{2\beta_i^*+t_i}}\log^{3}(n)
$
&
N/A
\\ \hline

\begin{tabular}{@{}c@{}}
\cite{ma2022theoretical} \\
$\epsilon_{i,1}$, $\gamma_i=0$, $m_i=1$,\\$\alpha$-mixing, sparse NN
\end{tabular}
&
$\displaystyle
\max_i n^{-\frac{2\beta_i^*}{2\beta_i^* + t_i}}\log^{7}(n)
$
&
N/A
\\ \hline

\end{tabular}
\label{tab:special_cases_temporal_spatial}
\end{small}
\end{table*}

\color{Black}

\section{Theoretical /Results for Manifold Data}
\label{sec_manifold}

Many real-world datasets exhibit low-dimensional structures, such as images as 3D object projections or speech data governed by grammatical rules. Visual, acoustic, and textual data often reflect these structures due to regularities, symmetries, or redundancy, making them well-suited for modeling as samples near low-dimensional manifolds \cite{Tenenbaum2000,Roweis2000}.

\color{black}
Unlike classical nonparametric estimators, deep neural network estimators have been shown to adapt effectively to low intrinsic-dimensional structures, often mitigating the curse of dimensionality and achieving fast convergence rates. We note that classical methods, such as kernel-based estimators, can also adapt to low-dimensional manifold structures under suitable conditions. However, neural networks offer additional flexibility in capturing complex compositional features of the regression function, which can be advantageous in high-dimensional and structured settings. Studies, including \cite{cloninger2021deep, jiao2023deep, kohler2023estimation}, have demonstrated their adaptivity to (nearly) manifold inputs. 
\color{Black}
\subsection{Exact Manifold Support}
\label{subsec:exact_manifold_temporal_spatial}

\color{black}
To formalize the low-dimensional structure discussed in Section~\ref{sec_manifold}, we consider the setting in which the covariates lie exactly on a lower-dimensional manifold embedded in the ambient space. This provides a canonical framework to study how the estimation rate depends on the intrinsic dimension rather than the ambient one.

We adopt the notion of a Lipschitz manifold from \cite{kohler2023estimation}.

\color{Black}

\begin{definition}[Definition 2 of \cite{kohler2023estimation},
$d^*$-dimensional Lipschitz-manifold]
\label{def:lipschitz_manifold}
Let $\mathcal{M} \subseteq \mathbb{R}^d$ be compact set and let $d^* \in\{1, \ldots, d\}$. Then $\mathcal{M}$ is a $d^*$-dimensional Lipschitz-manifold if 
$$
\mathcal{M}=\bigcup_{l=1}^R \mathcal{M} \cap U_l=\bigcup_{l=1}^R \psi_l\left((0,1)^{d^*}\right),
$$
where $U_1, \ldots, U_R$ is an open covering of $\mathcal{M}$, 
where $\psi_i:[0,1]^{d^*} \rightarrow \mathbb{R}^d(i \in\{1, \ldots, R\})$ are bi-Lipschitz functions, 
such that $ \psi_l\left((0,1)^{d^*}\right)=\mathcal{M} \cap U_l. $  Here,  bi-Lipschitz  means that there exist  constants $0<C_{\psi, 1} \leq C_{\psi, 2}<\infty$ such that
\begin{align}\label{eq:bi-Sipschitz property}
\begin{aligned}
C_{\psi, 1} \cdot\left\|a_1-a_2\right\|
\leq\left\|\psi_l\left(a_1\right)-\psi_l\left(a_2\right)\right\| \leq C_{\psi, 2} \cdot\left\|a_1-a_2\right\|
\end{aligned}
\end{align}
for any $a_1, a_2 \in[0,1]^{d^*}$ and any $l \in\{1, \ldots, R\}$. 
\end{definition}

We now incorporate this geometric structure into our temporal--spatial model.

\begin{assumption}\label{assumption: assumption in the temporal-spatial model on manifold}
For $1\le i \le n $ and $1\le j \le m_i$, let $y_{ij} = f^*(x_{ij} )+\gamma _i(x_{ij}) +\epsilon_{ij}, $
 where  $x_{ij}$ is supported on some $d^*$-dimensional Lipschitz-manifold $\mathcal{M} \subseteq[-1, 1]^d$, and $d^* \in\{1, \ldots, d\}$. We also assume that the rest of the conditions from Assumptions \ref{assumption:ts_model}--\ref{assumption:ts_design_dependence} hold except for Assumption \ref {assumption:ts_regularity}{\bf c} which is replaced with  $f^*$ being
  $(p,C)$-smooth, where $p= q +s$ with $q$ and $s$ as in   Definition \ref{definition: p,c smoothness}. Suppose that for  $f^*$ all its partial derivatives of order less than or equal to $q$ are uniformly bounded.
\end{assumption}

With Assumption \ref{assumption: assumption in the temporal-spatial model on manifold} in hand, we are now in position to state the main result of this section.

\begin{theorem}  \label{thm:main thm in the temporal-spatial model on manifold} 
Suppose that  \Cref{assumption: assumption in the temporal-spatial model on manifold} holds and consider the estimator $\widehat f$ defined (\ref{eqn_estimator}) 
with its number of layers $L$ and number of neurons per layer $r$ set as either a wide network with $L  \asymp \log(nm) $ and $r  \asymp  (nm)^{\frac{d^*}{2(2p+d^*)}}$, or a deep network with
$ L \asymp \log(nm) \cdot  (nm)^{\frac{d^*}{2(2p+d^*)}} $ and  r = O(1).
Let $\mathcal{A}_{nm} \asymp \max \left\{\sigma_\epsilon,\sigma_\gamma\right\}\cdot \log^{1/2}(nm)$ and $(a_n)_{n\ge 1}$ be any sequence such that $a_n \to \infty$. 
Then with probability approaching to one, it holds that 
\begin{align} \label{eq:functional rate wanted in the temporal-spatial model on manifold}
\begin{aligned}
\left\| \widehat f_{\mathcal{A}_{nm}} - f^* \right\|_\lt ^2&\lesssim a_n\left(\sigma_\epsilon^2+\sigma_\gamma^2+1\right)
\left( \frac{\log(nm)}{nm}\right)^{\frac{2p}{2p+d^*}}\log^{6.3}(nm)\\
&\quad + a_n\left(\sigma_\epsilon^2+\sigma_\gamma^2\right)\frac{\log(nm)}{n}.
\end{aligned}
\end{align}
\end{theorem} 

When no spatial noise is present, $\sigma_\gamma = 0$,
we recover the upper bound obtained in Theorem 1 in \cite{kohler2023estimation}, up to a logarithmic factor.

\color{black}
Theorem~\ref{thm:main thm in the temporal-spatial model on manifold}
is closely related to the exact manifold regime studied in Theorem~6.2 of
\cite{jiao2023deep}. In both settings, the convergence rate is governed by the
intrinsic dimension of the support rather than the ambient dimension, illustrating
the adaptivity of deep neural networks to low-dimensional structure.

To make this comparison precise, we briefly translate their setup into our notation.
In \cite{jiao2023deep}, the regression model is given by \(Y=f_0(X)+\eta\), where
the covariates \(X\) are supported on a compact \(d_M\)-dimensional Riemannian manifold
\(M\subset[0,1]^d\). The target function \(f_0\) is assumed to be \(\tilde p\)-Hölder
smooth on the manifold. In our notation, this corresponds to a regression function
\(f^*\) with smoothness level \(p\), so that \(\tilde p\) in \cite{jiao2023deep}
plays the same role as \(p\) in our formulation. After optimizing the network
architecture in Theorem~6.2, \cite{jiao2023deep} obtain a rate of order
\(n^{-2\tilde p/(2\tilde p+d_M)}\) up to logarithmic factors. In comparison,
our result yields a rate consisting of two components: a term of order
\((nm)^{-2p/(2p+d^*)}\), reflecting the intrinsic-dimension dependence of the
estimation problem, and an additional term of order
\((\sigma_\epsilon^2+\sigma_\gamma^2)\log(nm)/n\), which captures the effect of
temporal dependence and spatial variability. Thus, both results exhibit the same
intrinsic-dimension dependence in the first component of the rate. Despite this similarity, the two results are derived under different assumptions
and apply to distinct statistical frameworks. In \cite{jiao2023deep}, the analysis
relies on a compact Riemannian manifold structure with additional geometric quantities
such as the condition number (equivalently, reach) and surface area, and is developed
in an i.i.d.\ regression setting. In contrast, our result is formulated for a
temporal--spatial model under \(\beta\)-mixing dependence and spatial random effects,
and is based on a Lipschitz-manifold assumption, which does not explicitly impose
curvature or reach conditions. This allows us to extend manifold-adapted rates to
dependent data settings, at the cost of an additional lower-order term of order
\((\sigma_\epsilon^2+\sigma_\gamma^2)\log(nm)/n\), reflecting temporal dependence
and spatial variability.

The theorem above assumes that the design points lie exactly on a
$d^*$-dimensional Lipschitz manifold. While this setting is natural, it
can be restrictive in practice, as data may concentrate near a
lower-dimensional manifold without lying exactly on it. Recent work
such as \cite{jiao2023deep} shows that neural network estimators remain
adaptive when the exact-support requirement is relaxed to approximate
manifold support. It also considers more general low-dimensional sets
characterized by their Minkowski dimension. Motivated by these results,
we extend our temporal--spatial analysis to both settings. The next two
results show that analogous risk bounds continue to hold, with the
intrinsic dimension replaced by an appropriate effective dimension.
Proofs are deferred to Supplementary Material,
Sections~{\color{Magenta}{S12}} and~{\color{Magenta}{S13}}.

\color{Black}

\subsection{Approximate Manifold Support}
\label{subsec:approx_manifold_temporal_spatial}

The exact manifold assumption in the previous subsection can be restrictive in practice, as real data often lie near a low-dimensional structure rather than exactly on it. To account for this, we consider a more flexible setting in which the design points are supported on a neighborhood of a low-dimensional manifold.

 Let $\mathcal M \subset [0,1]^d$ be a compact $d_M$-dimensional Riemannian manifold (in the standard sense; see, e.g., \cite{jiao2023deep}), and for $\tau>0$ define its $\tau$-neighborhood by
\[
\mathcal M_\tau
:=
\left\{
x\in[0,1]^d:\inf_{y\in \mathcal M}\|x-y\|\le \tau
\right\}.
\]
This neighborhood formulation captures the idea that the design points concentrate around a low-dimensional manifold while allowing for small perturbations away from it. The parameter $\tau$ controls the scale of this deviation and determines how closely the support adheres to the underlying manifold structure. 
\color{black}
The Riemannian manifold used in this setting is related to the
Lipschitz manifold in the exact-support setting. Every compact
Riemannian submanifold of $\mathbb R^d$ admits a finite bi-Lipschitz
atlas and therefore satisfies
Definition~\ref{def:lipschitz_manifold}.
\textcolor{black}{However, a compact Lipschitz manifold need not be
a smooth submanifold of $\mathbb R^d$. The} bi-Lipschitz charts in
that definition are not required to be differentiable.
Thus, the Riemannian framework of
\cite{jiao2023deep} imposes greater regularity on the reference
manifold than the Lipschitz framework of
\cite{kohler2023estimation}.

The support condition changes in the opposite direction. Exact support
requires the design points to lie on the reference manifold, whereas
approximate support permits them to lie in its $\tau$-neighborhood.
The displayed definition also gives $\mathcal M_0=\mathcal M$.
Therefore, on the common class of Riemannian manifolds, exact support
corresponds to the zero-radius case. The two settings consequently
trade regularity of the reference manifold against exactness of the
support and are not directly ordered.

These geometric assumptions agree with the approximation results used
in the corresponding proofs. The exact-support proof invokes
Theorem~2 of \cite{kohler2023estimation}, while the approximate-support
proof adapts the construction in the proof of Theorem~6.1 of
\cite{jiao2023deep}.
\color{Black}

We now incorporate this geometric condition into the temporal--spatial model.
\begin{assumption}
\label{assumption:approx_manifold_temporal_spatial}
For $1\le i \le n$ and $1\le j \le m_i$, let
\(
y_{ij}=f^*(x_{ij})+\gamma_i(x_{ij})+\epsilon_{ij},
\)
where $x_{ij}$ is supported on $\mathcal M_\tau$.
We also assume that the remaining conditions from
Assumptions~\ref{assumption:ts_model}--\ref{assumption:ts_design_dependence}
hold, except that Assumption~\ref{assumption:ts_regularity}{\bf c} is replaced by
$f^*$ being $(p,C)$-smooth, where $p=q+s$ with $q$ and $s$ as in
Definition~\ref{definition: p,c smoothness}. Suppose that all partial derivatives
of $f^*$ of order less than or equal to $q$ are uniformly bounded.
\end{assumption}

We now state the corresponding risk bound.

\begin{theorem}[Approximate manifold support]
\label{thm:approx_manifold_temporal_spatial}
Suppose that Assumption~\ref{assumption:approx_manifold_temporal_spatial} holds.
Fix $\delta\in(0,1)$. Let $d_\delta$ satisfy
\(
d_M \le d_\delta < d,
\qquad
d_\delta \lesssim \frac{d_M\log(d/\delta)}{\delta^2}.
\)
Assume the approximation-scale condition
\(
\tau^{\,p\wedge 1}
\lesssim
(nm)^{-\frac{p}{2p+d_\delta}}.
\)
Let $\widehat f$ be the estimator in \eqref{eqn_estimator}, with $(L,r)$ chosen according to either
\(
L\asymp \log(nm), \quad
r\asymp (nm)^{\frac{d_\delta}{2(2p+d_\delta)}},
\)
or
\(
L\asymp \log(nm)\,(nm)^{\frac{d_\delta}{2(2p+d_\delta)}},
\quad
r=O(1).
\)
Let
\(
\mathcal A_{nm}\asymp \max\{\sigma_\epsilon,\sigma_\gamma\}\sqrt{\log(nm)},
\) and $(a_n)_{n\ge 1}$ be any sequence such that $a_n \to \infty$.
Then, with probability approaching one,
\begin{align}
\begin{aligned}
\|\widehat f_{\mathcal A_{nm}}-f^*\|_{\mathcal L_2}^2 \lesssim\; a_n(\sigma_\epsilon^2+\sigma_\gamma^2+1)
(nm)^{-\frac{2p}{2p+d_\delta}}
\log^{7.3}(nm) + a_n\frac{(\sigma_\epsilon^2+\sigma_\gamma^2)\log(nm)}{n}.
\end{aligned}\nonumber
\end{align}
\end{theorem}

Theorem~\ref{thm:approx_manifold_temporal_spatial} shows that the temporal--spatial risk bound continues to hold under approximate manifold support, with the intrinsic dimension replaced by an effective dimension \(d_\delta\). This reflects the additional geometric complexity introduced by the tubular neighborhood. The lower-order \(n^{-1}\)-scale contribution remains unchanged, since it is driven by temporal dependence and the time-indexed spatial random effects rather than by the geometry of the support. 
In particular, up to logarithmic factors, the first term in the bound matches the convergence rate \(n^{-2p/(2p+d_\delta)}\) established in \cite{jiao2023deep} for the approximate manifold setting. Our result shows that this rate continues to hold in the temporal--spatial framework, with an additional term capturing temporal dependence and spatial variability.
This result is the temporal--spatial analogue of the approximate-manifold regime studied in Theorem~6.1 of \cite{jiao2023deep}. A proof of Theorem~\ref{thm:approx_manifold_temporal_spatial} is given in Supplementary Material, Section~{\color{magenta}{S12}}.

\subsection{Low Minkowski Dimension Support}
\label{subsec:minkowski_temporal_spatial}

We now move to a more general regime in which no manifold structure is assumed. Instead, we only require that the support of the design distribution has low geometric complexity, quantified through its Minkowski dimension \cite{jiao2023deep}. This setting captures a broad class of low-dimensional structures beyond smooth manifolds.

\begin{definition}[Minkowski dimension]
Let $\mathcal M \subset \mathbb R^d$ be bounded. For $\epsilon > 0$, let 
$\mathcal N(\epsilon, \mathcal M, \|\cdot\|)$ denote the covering number of $\mathcal M$ by Euclidean balls of radius $\epsilon$. 

The (upper) Minkowski dimension of $\mathcal M$ is defined as
\[
\dim_M(\mathcal M)
:=
\limsup_{\epsilon \to 0}
\frac{\log \mathcal N(\epsilon, \mathcal M, \|\cdot\|)}{\log(\epsilon^{-1})}.
\]
\end{definition}

The Minkowski dimension provides a measure of intrinsic geometric complexity based solely on covering properties, without requiring any smooth parametrization. In contrast to the manifold-based settings considered above, this allows for more general low-dimensional supports.

We now incorporate this geometric condition into the temporal--spatial model.

\color{black}
This condition is also related to, but distinct from, the two preceding
manifold assumptions. A compact $d^*$-dimensional Lipschitz manifold
has Minkowski dimension $d^*$, see \cite{falconer2003fractal}. Hence, when $d^*<d$, exact manifold
support is a special case of the low-Minkowski-dimension setting. By
contrast, for any fixed $\tau>0$, the neighborhood
$\mathcal M_\tau$ generally has ambient Minkowski dimension $d$, even
when $\mathcal M$ is low-dimensional. Approximate manifold support
therefore does not generally imply low Minkowski dimension, see for instance Chapter 3 in \cite{falconer2003fractal}.

Conversely, a set of low Minkowski dimension need not possess either
the bi-Lipschitz parametrizations of the exact-support setting or the
Riemannian structure of the approximate-support setting. The
approximate-manifold and low-Minkowski-dimension assumptions thus
describe different extensions of exact manifold support.
\color{Black}
\begin{assumption}
\label{assumption:minkowski_temporal_spatial}
For $1\le i \le n$ and $1\le j \le m_i$, let
\(
y_{ij}=f^*(x_{ij})+\gamma_i(x_{ij})+\epsilon_{ij},
\)
where $x_{ij}$ is supported on a set $\mathcal M \subset [0,1]^d$ with Minkowski dimension
\(
d^* := \dim_M(\mathcal M) < d.
\)
We also assume that the remaining conditions from
Assumptions~\ref{assumption:ts_model}--\ref{assumption:ts_design_dependence}
hold, except that Assumption~\ref{assumption:ts_regularity}{\bf c} is replaced by
$f^*$ being $(p,C)$-smooth, where $p=q+s$ with $q$ and $s$ as in
Definition~\ref{definition: p,c smoothness}. Suppose that all partial derivatives
of $f^*$ of order less than or equal to $q$ are uniformly bounded.
\end{assumption}

 Now we state the main result of this section.

\begin{theorem}[Low Minkowski dimension support]
\label{thm:minkowski_temporal_spatial}
Suppose that Assumption~\ref{assumption:minkowski_temporal_spatial} holds.
Fix $\delta\in(0,1)$. Let $d_0\in\mathbf Z^+$ satisfy
\(
d \ge d_0 \ge \kappa \frac{d^*}{\delta^2},
\)
for some constant $\kappa>0$.
Let $\widehat f$ be defined as in \eqref{eqn_estimator}, and choose $(L,r)$ according to either one of the following two regimes:
\begin{align*}
\text{\textbf{Case 1.}}\quad
&L\asymp \log(nm),
\qquad
r\asymp (nm)^{\frac{d_0}{2(2p+d_0)}},
\\
\text{\textbf{Case 2.}}\quad
&L\asymp \log(nm)\,(nm)^{\frac{d_0}{2(2p+d_0)}},
\qquad
r=O(1).
\end{align*}
Let
\(
\mathcal A_{nm}\asymp \max\{\sigma_\epsilon,\sigma_\gamma\}\sqrt{\log(nm)},
\) and $(a_n)_{n\ge 1}$ be any sequence such that $a_n \to \infty$.
Then, with probability approaching one,
\begin{align}
\label{eq:minkowski_temporal_spatial_rate}
\begin{aligned}
\|\widehat f_{\mathcal A_{nm}}-f^*\|_{\mathcal L_2}^2 \lesssim\; a_n(\sigma_\epsilon^2+\sigma_\gamma^2+1)
(nm)^{-\frac{2p}{2p+d_0}}
\log^{7.3}(nm) + a_n\frac{(\sigma_\epsilon^2+\sigma_\gamma^2)\log(nm)}{n}.
\end{aligned}\nonumber
\end{align}
\end{theorem}

Theorem~\ref{thm:minkowski_temporal_spatial} extends the temporal--spatial
risk bound to the case where the support of the design distribution is a
general bounded set of low Minkowski dimension. In contrast to the manifold
settings above, no smooth chart structure is required. The effective dimension
in the rate is \(d_0\), which arises from the random projection argument used
to reduce the ambient dimension. In the corresponding analysis of
\cite{jiao2023deep}, \(d_0\) can be chosen of order \(d^*/\delta^2\), up to
additional logarithmic terms. This is precisely the low-Minkowski-dimension
regime studied in Theorem~6.3 of \cite{jiao2023deep}.  This matches the intrinsic-dimension-dependent rate obtained in
\cite{jiao2023deep} up to logarithmic factors, with \(d_0\) playing the role
of the effective dimension. A proof of
Theorem~\ref{thm:minkowski_temporal_spatial} is given in Section~{\color{magenta}{S13}}
of the Supplementary Material.

\color{Black}
\section{Numerical Experiments}
\label{sec_exp}
\subsection{Simulated data}

To conduct our simulations, we consider various generative models, referred to as scenarios. We describe these scenarios in the following subsections. The performance of each method is evaluated by the MSE obtained when using the true regression function. We then report the average MSE over 100 Monte Carlo simulations.

For each setting below, we consider dimension $d\in\{2,5,7,10\}.$ For each fix dimension,
we vary \( n \) in \( \{500, 1000, 2000\} \). Moreover, for each fix value \( n \) we then vary a parameter \( m_{mult} \) in \( \{ 1, 2\} \), to obtain different sets of values for \( \{m_i\}_{i=1}^{n} \). Specifically, we consider: 1). \( m_1 = \ldots = m_{n/4} = 16m_{mult} \); 2). \( m_{n/4+1} = \ldots = m_{n/2} = 24m_{mult} \); 3). \( m_{n/2+1} = \ldots = m_{3n/4} = 20m_{mult} \); 4). \( m_{3n/4+1} = \ldots = m_n = 10m_{mult} \).

The spatial noise is generated as \(\gamma_i(x) = 0.5\gamma_{i-1}(x) + \sum_{t=1}^{25} t^{-1}b_{t,i}h_t(x)\), where
\[
\Bigl\{h_t(x) = \prod_{j=1}^d\frac{1}{\sqrt{2}\pi} \sin(t\pi x^{(j)})\Bigr\}_{t=1}^{25},
\]
are basis functions and $\{b_{t,i}\}_{t = 1, i=1}^{25,n}$ are considered to be i.i.d.~$\mathcal{N}(0,1)$.

The measurement error is generated as follows: For each \(i \in [n]\), the vector \(\tilde \epsilon_i \in \mathbb{R}^M\) is generated recursively as, \(\tilde \epsilon_i = 0.3 \tilde \epsilon_{i-1} + \xi_i\), where \(\{\xi_{i}\}_{i=1}^{n}\) are independent errors with $\xi_i \sim \mathcal{N}(0,0.5\mathbf{I}_{M})$, and $\mathbf{I}_{M}$ is the $M \times M$ identity matrix. 
To form the measurement errors, for each \(i \in [n]\), we take the first \(m_i\) entries of the vector \(\tilde \epsilon_i\) and denote the resulting vector as \(\epsilon_i\), since the $m_i$ vary across $i$. 

We observe the noisy temporal-spatial data \(\{y_{ij}\}_{i=1,j=1}^{n,m_i}\) at design points \(\{x_{ij}\}_{j=1}^{m_1}\)  sampled in \([0,1]^d\) as follows. 

First, we generate \(\{x_{1j}\}_{j=1}^{m_1} \sim \text{Unif}([0,1]^d)\). Then, for any \(1 < i \leq n\),
\[
x_{ij} = 
\begin{cases}
x_{i-1j}, & \text{with probability equals to }  0.1,\\
x_{ij} \sim \text{Unif}([0,1]^d), & \text{otherwise},
\end{cases}
\]
for all \(j = 1, \ldots, m_i\).

We investigate the subsequent scenarios.  

{$\bullet$ {\bf{Scenario 1} }} In this scenario, we use a step function. The function $f^*:[0,1]^d \rightarrow \mathbb{R}$ is defined as
$$
f^*(q)= \begin{cases}
2 & \begin{aligned}[t]
\text { if }\left\|q-q_1\right\|_2<\min \left\{\left\|q-q_2\right\|_2,\left\|q-q_3\right\|_2,\left\|q-q_4\right\|_2\right\},
\end{aligned}\\[0.6em]
1 & \begin{aligned}[t]
\text { if }\left\|q-q_2\right\|_2<\min, \left\{\left\|q-q_1\right\|_2,\left\|q-q_3\right\|_2,\left\|q-q_4\right\|_2\right\},
\end{aligned}\\[0.6em]
0 & \begin{aligned}[t]
\text { if }\left\|q-q_3\right\|_2<\min \left\{\left\|q-q_1\right\|_2,\left\|q-q_2\right\|_2,\left\|q-q_4\right\|_2\right\},
\end{aligned}\\[0.6em]
-1 & \text { otherwise, }
\end{cases}
$$
where
\[
\begin{aligned}
q_1&=\left(\frac{1}{4} \mathbf{1}_{\lfloor d / 2\rfloor}^T, \frac{1}{2} \mathbf{1}_{d-\lfloor d / 2\rfloor}^T\right)^T,\\
q_2&=\left(\frac{1}{2} \mathbf{1}_{\lfloor d / 2\rfloor}^T, \frac{1}{4} \mathbf{1}_{d-\lfloor d / 2\rfloor}^T\right)^T,\\
q_3&=\left(\frac{3}{4} \mathbf{1}_{\lfloor d / 2\rfloor}^T, \frac{1}{2} \mathbf{1}_{d-\lfloor d / 2\rfloor}^T\right)^T
\end{aligned}
\]
and
\[
q_4=\left(\frac{1}{2} \mathbf{1}_{\lfloor d / 2\rfloor}^T, \frac{3}{4} \mathbf{1}_{d-\lfloor d / 2\rfloor}^T\right)^T.
\]

{$\bullet$ {\bf{Scenario 2} }} In this scenario, we apply a simple sine function to the sum of inputs to examine smooth, periodic variation across the input space.
  The function $f^*:[0,1]^d \rightarrow \mathbb{R}$ is defined as $f^*(q) = \sin(\mathbf{1}_d^Tq).$  

{$\bullet$ {\bf{Scenario 3} }} This scenario uses a two-level hierarchical function to test the performance on layered non-linearities.
  We set 
\[
\begin{aligned}
f^*(q)&=g_2\circ g_1(q),\qquad \forall q\in \mathbb{ R}^d,\\
g_1(q)&=\left( \sqrt{q_1} + \sum_{i=1}^{d-1} q_i \cdot q_{i+1}, \cos\left(2\pi \sum_{i=1}^d q_i\right) \right)^\top,\qquad \forall q\in \mathbb{ R}^d,\\
g_2(q)&=\sqrt{q_1 + q_2^2 }+q_1^2\cdot q_2,\qquad \forall q\in \mathbb{ R}^2.
\end{aligned}
\]

{$\bullet$ {\bf{Scenario 4} }} 
This scenario employs a three-level hierarchical function with distinct transformations at each level to test adaptability to deeper compositions with layered structures with varying smoothness.
 Specifically, the function $f^* \,:\,[0,1]^{d} \rightarrow \mathbb{R}$ is defined as $f^*(q) = g_3  \circ g_2 \circ g_1(q)$, where
\[
\begin{aligned}
g_1(q)&=(\sqrt{q_1^2+\sum_{j=2}^{d}q_j}, (\sum_{j=1}^{d}q_j)^3)^{\top},\qquad q\in[0,1]^d,\\
g_2(q)&=(\vert q_1\vert,q_2\cdot q_1)^{\top},\qquad q\in[0,1]^2,\\
g_3(q)&=q_1+\sqrt{q_1+q_2},\qquad q\in[0,1]^2.
\end{aligned}
\]

{$\bullet$ {\bf{Scenario 5} }} In this scenario, we use a highly non-linear function with interaction terms:
\begin{align*}
f^*(q)&=\begin{cases}
\begin{aligned}[t]
\sin(2\pi q_1 q_2),\qquad \text{if }d=2,\quad q\in[0,1]^d.
\end{aligned}\\[0.6em]
\begin{aligned}[t]
\sin(2\pi q_1 q_2)+\cos(2\pi q_3),\qquad \text{if }d=3,
\end{aligned}\\[0.6em]
\begin{aligned}[t]
\sin(2\pi q_1 q_2)+\cos(2\pi q_3 q_4) +\ldots+\sin(2\pi q_{d-1}q_d),\qquad \text{if }d\text{ is even }d\geq4,
\end{aligned}\\[0.6em]
\begin{aligned}[t]
\sin(2\pi q_1 q_2)+\cos(2\pi q_3 q_4) +\ldots+\sin(2\pi q_d),\qquad \text{if }d\text{ is odd }d>4.
\end{aligned}
\end{cases}
\end{align*}

{$\bullet$ {\bf{Scenario 6} }} 
In this scenario, we use a function with a hierarchical structure that combines quadratic and sinusoidal transformations across separate partitions of the input space.
Here, we use a function with a hierarchical structure:
$$
f^*(q) = \Bigl(\sum_{i=1}^{\lfloor d/2 \rfloor} q_i^2\Bigr) \cdot \sin\Bigl(\sum_{i=\lfloor d/2\rfloor + 1}^d q_i\Bigr), \quad q \in [0,1]^d. 
$$

We compare our estimator with four competitors, with their abbreviated names provided in parentheses. 
    \textbf{ $K$-NN-FL} (KNN-FL, \cite{padilla2018adaptive}). We construct  the $K$-NN   graphs  using standard  \verb=Matlab= functions such 
as \verb=knnsearch= and \verb=bsxfun=. 
We solve the fused lasso  with the parametric max-flow  algorithm from \cite{chambolle2009total}, for which  software is available from the authors' website, \url{http://www.cmap.polytechnique.fr/~antonin/software/}. We let  the number of neighbors in $K$-NN-FL be  $K=5, 10, 15$, and the  parameter $\lambda$ to be adjusted with $5$-fold cross-validation. 
    \textbf{Additive Model with Trend Filtering} (Trend Filtering, \cite{sadhanala2019additive}), a multivariate estimator with additive modes. This is implemented by using the backfitting algorithm to fit each component of the univariate trend filtering estimator. For the univariate trend filtering solver,
we used the {\tt trendfilter} function  from the \texttt{R}  package {\tt glmgen}, which is an
implementation of the fast ADMM algorithm given in \cite{ramdas2016fast}. The order of trend filtering (the order of weak derivatives of total variation difference) $k$ is chosen from $0,1,2$,  and the regularization parameter $\lambda$ is chosen by cross-validation. 
    \textbf{Generalized Additive Models} (GAM, \cite{wood2017generalized}), implemented in the \texttt{Python} package \texttt{pyGAM}. The feature functions are built using penalized B splines. The number of splines, the spline order, the type of basis functions, the strength of smoothing penalty $\lambda$ are all chosen by cross-validation. \textbf{ Kernel ridge regression} (RKHS, \cite{hofmann2008kernel}), implemented in the \texttt{Python}  package \texttt{sklearn.kernel\_ridge}, with the regularization strength parameter $\alpha$ and type of kernels treated as tuning parameters, the best parameters are selected with $5$-fold cross-validation. 
    \color{black}
    For the Dense NN, we employ a fully connected feedforward architecture comprising an input layer, $L$ hidden layers, and an output layer. Each hidden layer consists of $r$ units and uses the ReLU activation function. The network parameters are updated using the Adam optimizer with a learning rate of $0.001$. 
The Adam optimizer is run for a maximum of \(800\) iterations with early stopping enabled, so that training terminates automatically once the improvement in the loss falls below a prescribed tolerance. This iteration budget is chosen to be sufficiently large to ensure convergence across all scenarios and hyperparameter configurations. 
The best parameters $L,r$ are selected with $5$-fold cross-validation.
    \color{Black}

Figure \ref{fig_box_d_7} displays the best-performing methods for {\bf{Scenarios 1}} to {\bf{6}} with dimensions $d = 7$. Figures \ref{fig_box_d_10} shows the results for $d=10$. Detailed performance values for each method, an the other values of $d$, along with visualizations for $d=2$,  are available  in Supplementary Material {\color{magenta}{S9}}.
The evaluation metric we used is the relative error, which is defined as $\text{Relative Error} = \|\hat f - f^* \|_{nm}^2/\|f^*\|_{nm}^2.$

\begin{figure*}[!t]
    \centering
        \begin{minipage}[c]{0.9\textwidth}
                \includegraphics[width=\linewidth, height=0.18\linewidth]{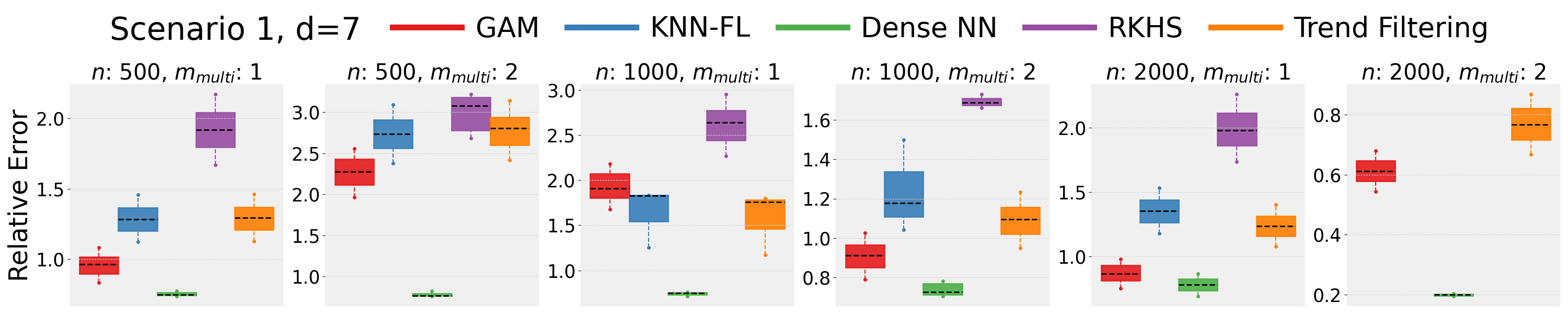}
    \end{minipage}\hspace{-4pt}
                
        \begin{minipage}[c]{0.9\textwidth}
                \includegraphics[width=\linewidth, height=0.18\linewidth]{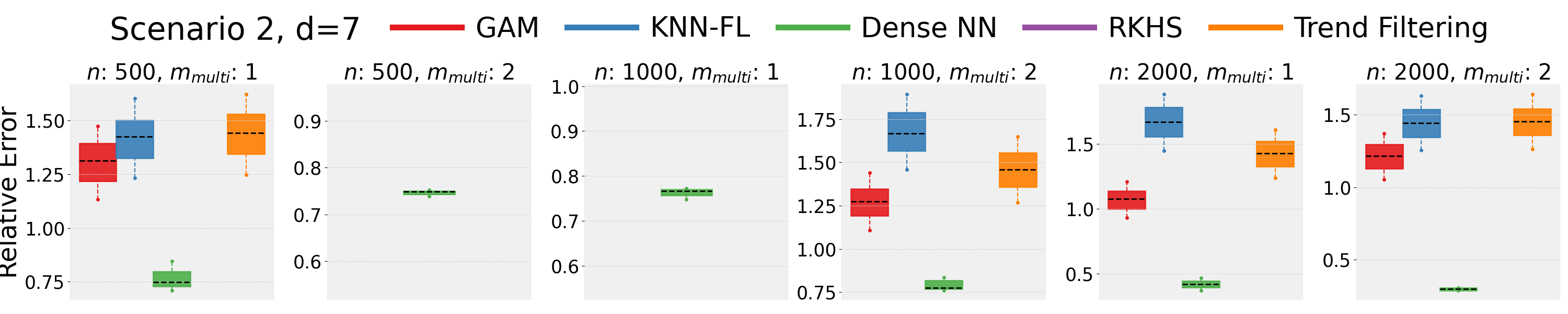}
    \end{minipage}\hspace{-4pt}

        \begin{minipage}[c]{0.9\textwidth}
                \includegraphics[width=\linewidth, height=0.18\linewidth]{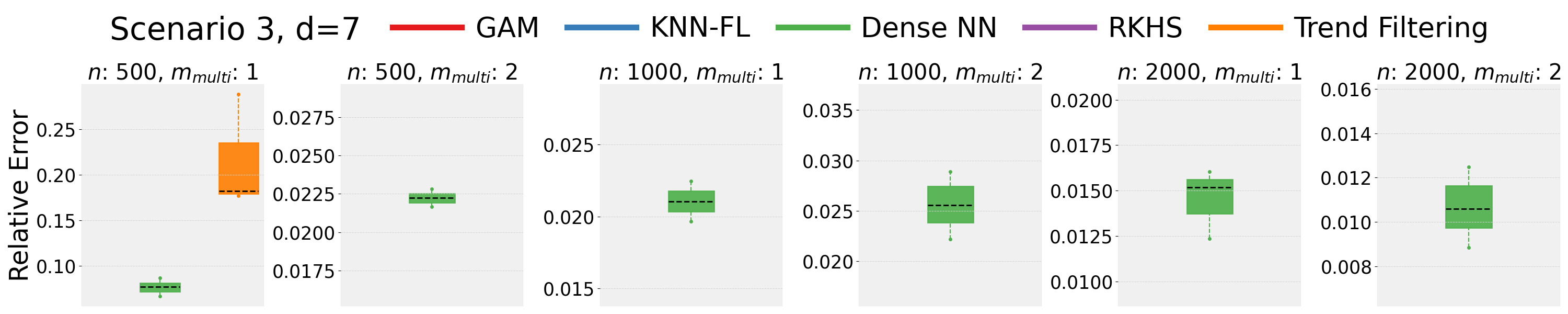}
    \end{minipage}\hspace{-4pt}
                
        \begin{minipage}[c]{0.9\textwidth}
                \includegraphics[width=\linewidth, height=0.18\linewidth]{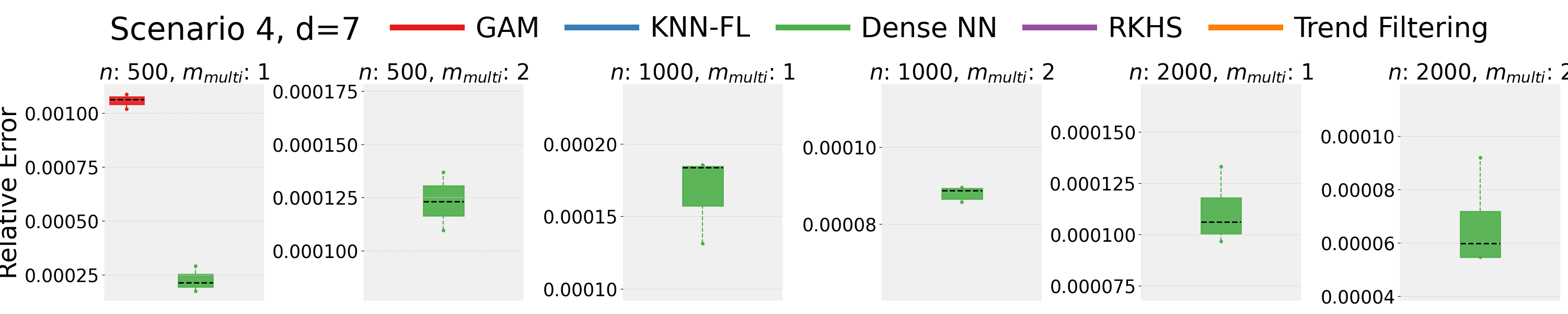}
    \end{minipage}\hspace{-4pt}
                
        \begin{minipage}[c]{0.9\textwidth}
                \includegraphics[width=\linewidth, height=0.18\linewidth]{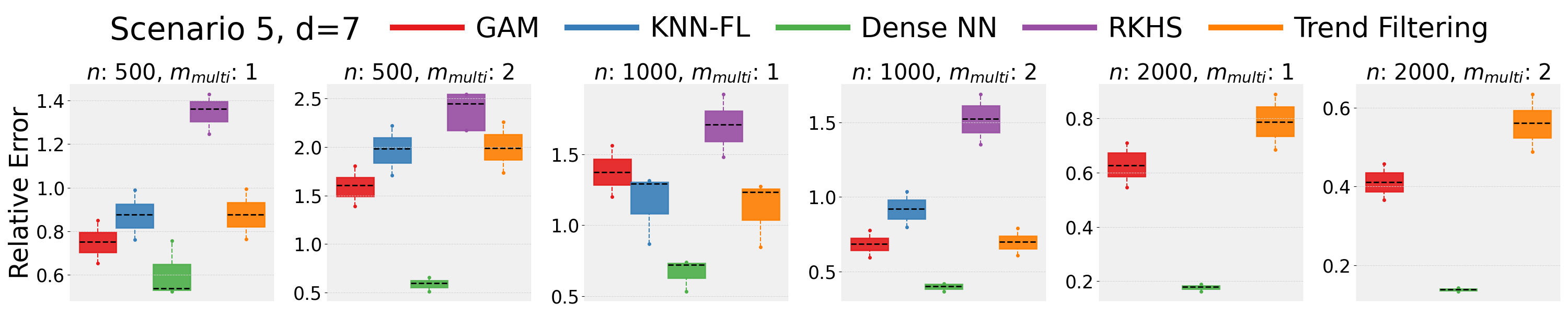}
    \end{minipage}\hspace{-4pt}

        \begin{minipage}[c]{0.9\textwidth}
                \includegraphics[width=\linewidth, height=0.18\linewidth]{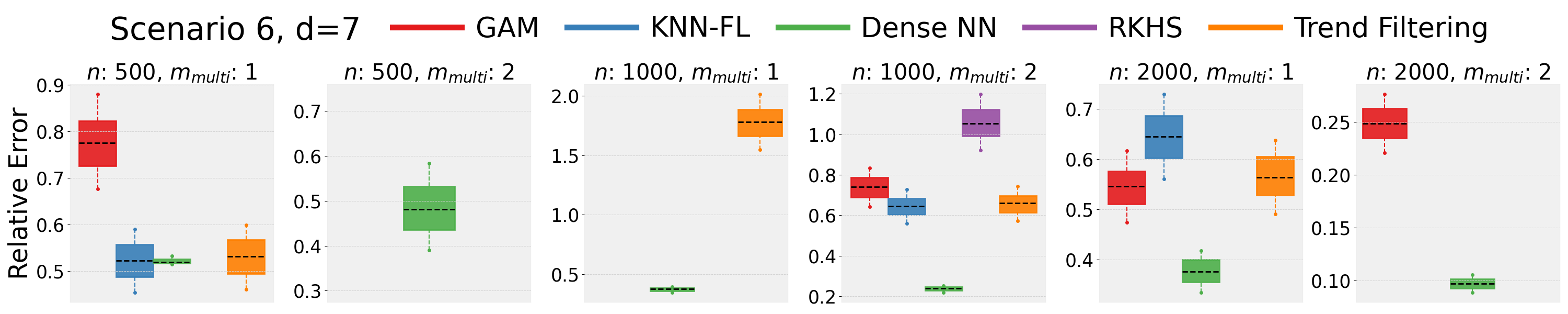}
    \end{minipage}\hspace{-4pt}
                
    \caption{Comparison of methods across different scenarios, $d = 7$, $n$, and $m_{mult}$ by box-plot. For each $(n,m_{mult})$ setting, we present the errors of the competitors in the range of within 5 times the smallest error. }
    \label{fig_box_d_7}
\end{figure*}

\begin{figure*}[!t]
    \centering
                        \begin{minipage}[c]{0.9\textwidth}
                \includegraphics[width=\linewidth, height=0.18\linewidth]{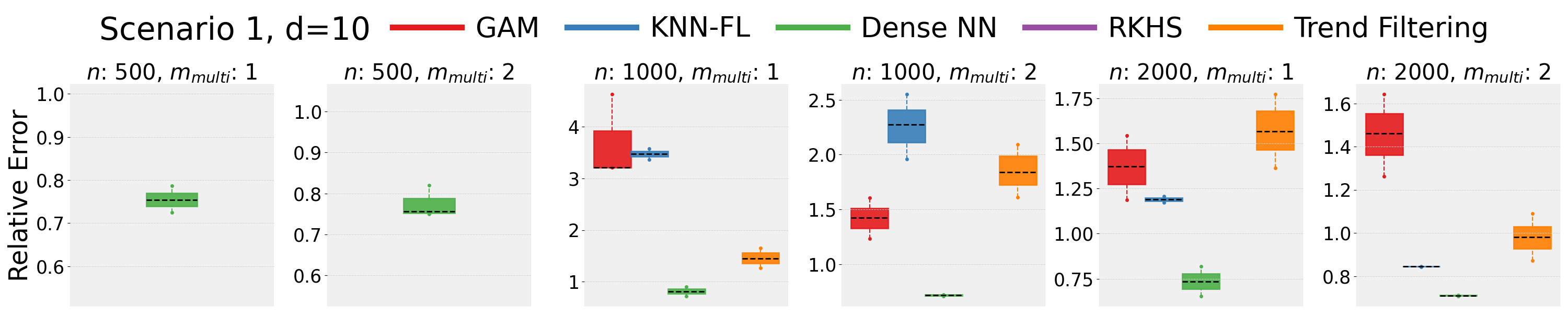}
    \end{minipage}\hspace{-4pt}

                        \begin{minipage}[c]{0.9\textwidth}
                \includegraphics[width=\linewidth, height=0.18\linewidth]{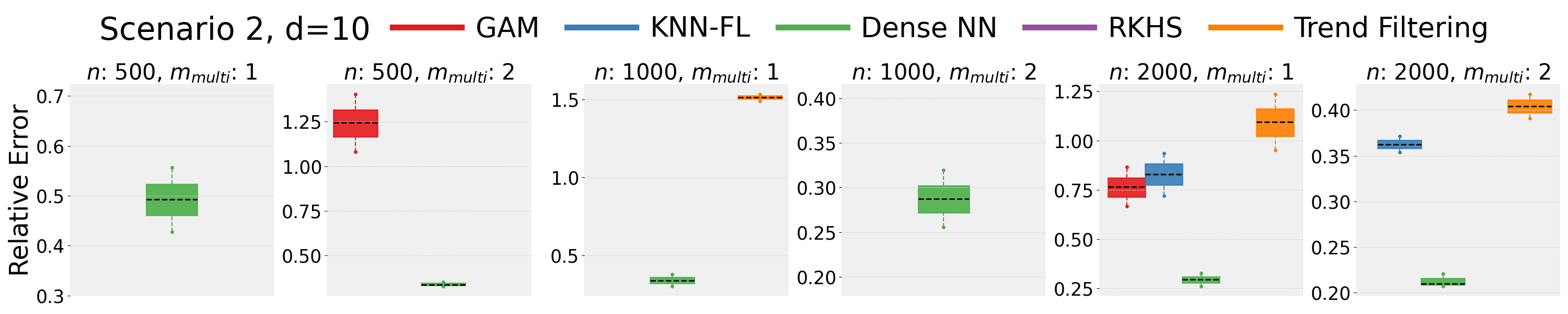}
    \end{minipage}\hspace{-4pt}

                        \begin{minipage}[c]{0.9\textwidth}
                \includegraphics[width=\linewidth, height=0.18\linewidth]{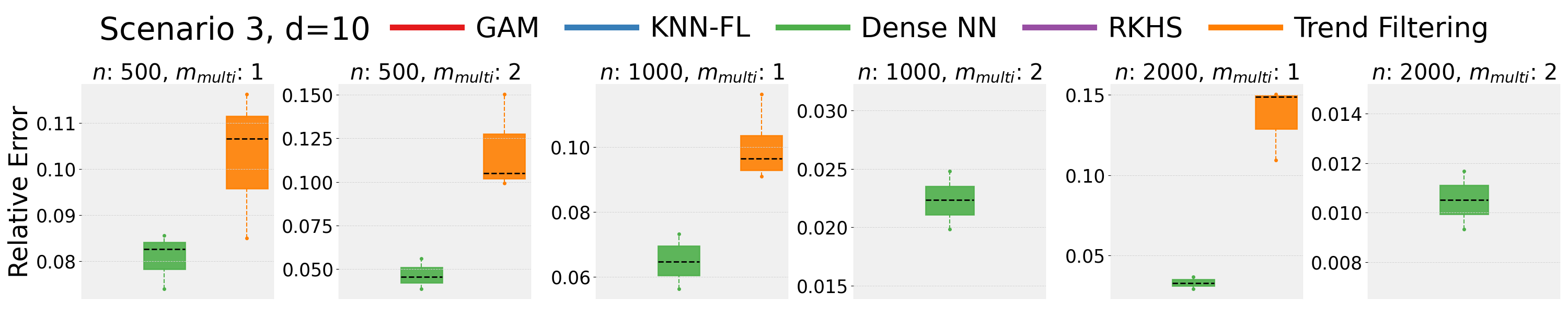}
    \end{minipage}\hspace{-4pt}

                        \begin{minipage}[c]{0.9\textwidth}
                \includegraphics[width=\linewidth, height=0.18\linewidth]{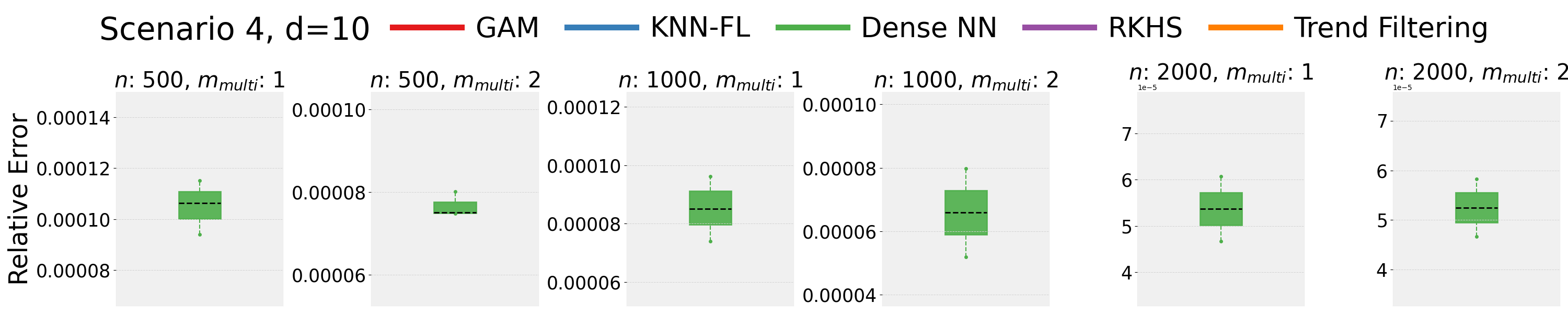}
    \end{minipage}\hspace{-4pt}

                        \begin{minipage}[c]{0.9\textwidth}
                \includegraphics[width=\linewidth, height=0.18\linewidth]{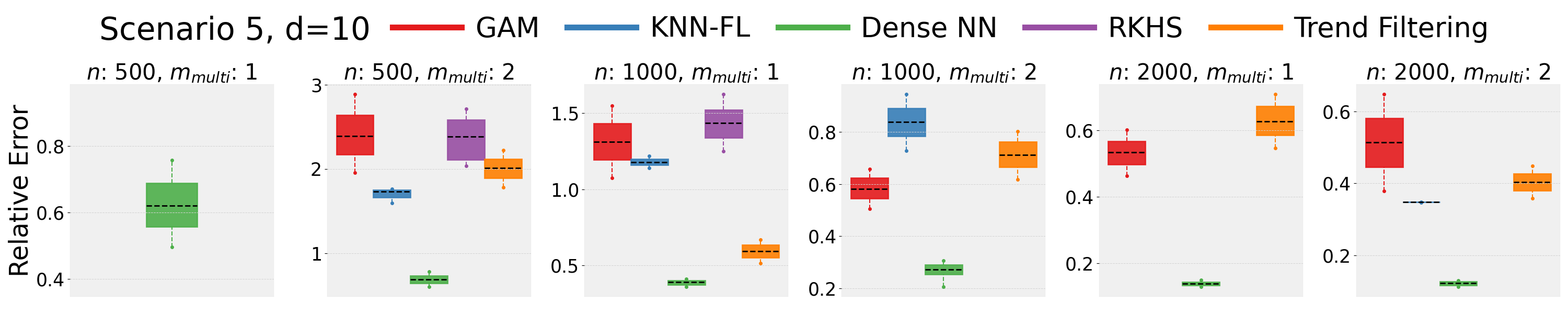}
    \end{minipage}\hspace{-4pt}

                        \begin{minipage}[c]{0.9\textwidth}
                \includegraphics[width=\linewidth, height=0.18\linewidth]{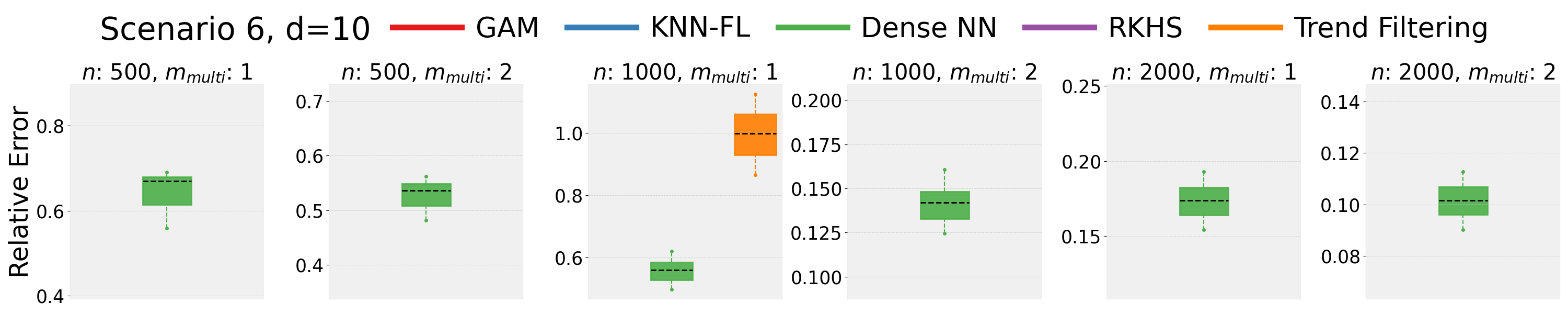}
    \end{minipage}\hspace{-4pt}

    \caption{Comparison of methods across different scenarios, $d = 10$, $n$, and $m_{mult}$ by box-plot. 
        }
    \label{fig_box_d_10}
\end{figure*}

Overall, the Dense NN estimator exhibits a clear decreasing trend in relative error as the sample size increases, and consistently achieves the smallest (or among the smallest) errors across scenarios, although for small sample sizes and moderate increases (e.g., doubling from \(n=500\) to \(n=1000\)) the decrease may not be strictly monotonic.

 Notably, the Dense NN estimator defined in (\ref{eqn_estimator})  consistently outperforms all other methods. 
  \color{black}
In each scenario, as spatial noise, measurement error, and mixture probability vary, the Dense NN estimator exhibits an overall decreasing trend in relative error as the sample size increases. Importantly, this trend is observed in most settings, including relatively small increases in sample size (e.g., doubling from \(n=500\) to \(n=1000\)), although it is not strictly monotonic in every case. The overall behavior aligns with our theoretical findings, where the Dense NN estimator attains fast convergence rates for hierarchical composition functions under broad structural assumptions.
\color{Black}

\subsection{Real data application}

We consider the EPA Regional dataset, which consists of daily ozone measurements collected across various monitoring stations in different regions of the United States (\url{https://www.epa.gov/enviro/data-downloads}). The dataset includes measurements of ozone concentration and associated variables such as Air Quality Index (AQI), wind speed, temperature, Latitude, and  Longitude. The specific variables in the dataset are: State Code, County Code, Site Number, Latitude, Longitude, Date (Year, Month, Day), Ozone (in parts per million), AQI, Wind (in miles per hour), and Temperature (in Fahrenheit). Data were collected daily over multiple years.

\begin{figure*}[!t]
    \centering

        \begin{minipage}[c]{0.3\textwidth}
    \includegraphics[width= 0.95\linewidth, height=0.5\linewidth]{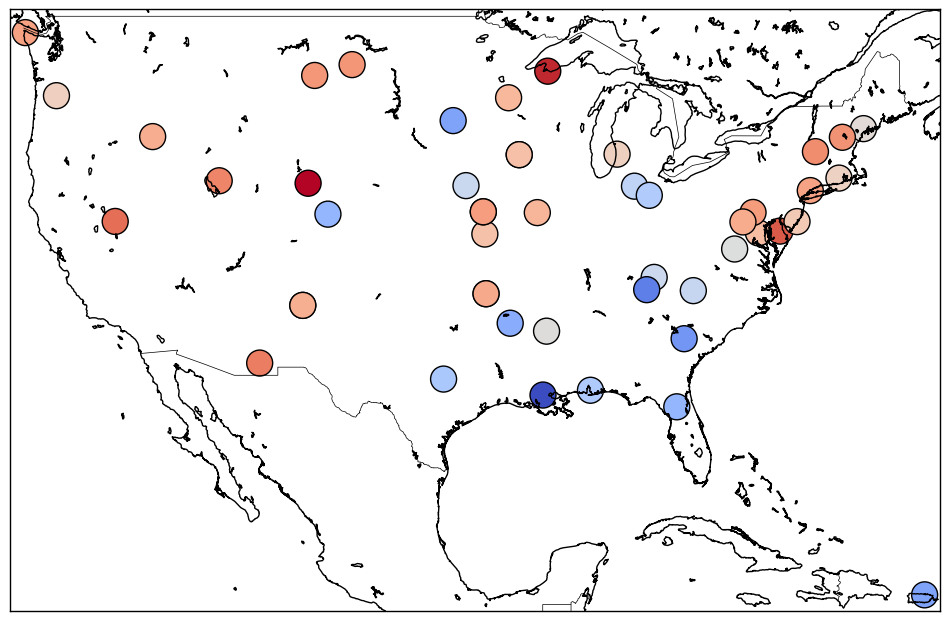}        \panelcaption{(a) Ozone, $y$ }
    \end{minipage}\hspace{-1pt}
    \begin{minipage}[c]{0.3\textwidth}
        \includegraphics[width=0.95\linewidth, height=0.5\linewidth]{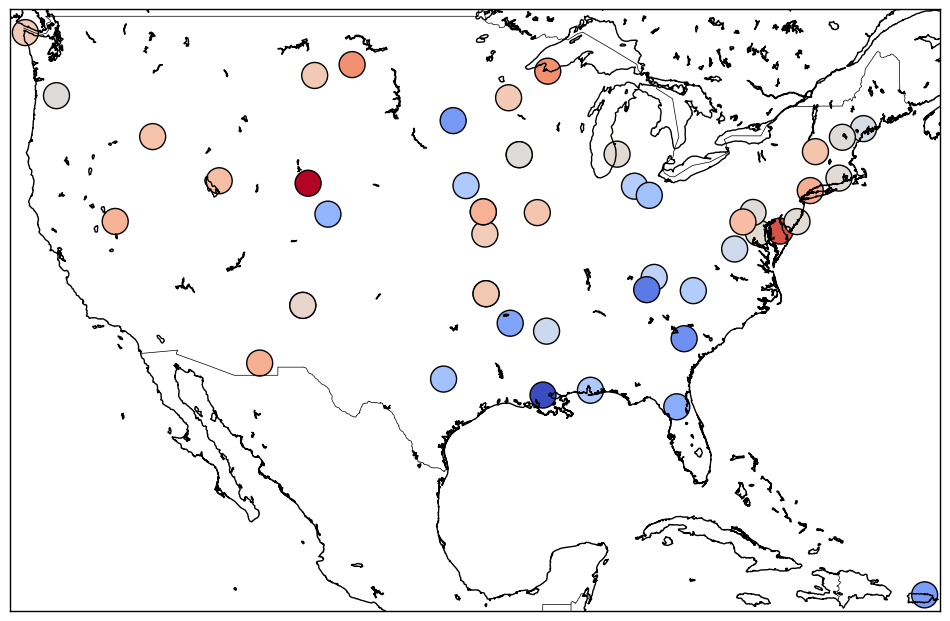}
        \panelcaption{(b) Dense NN, $\hat y$}
    \end{minipage}\hspace{-1pt}
    \begin{minipage}[c]{0.3\textwidth}
        \includegraphics[width=0.95\linewidth, height=0.5\linewidth]{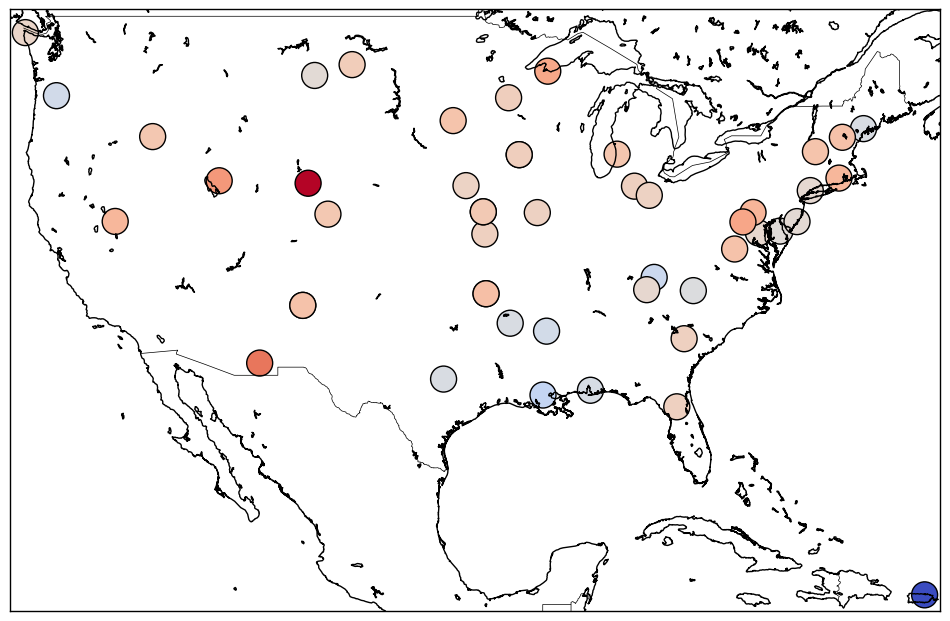}
        \panelcaption{(c) GAM, $\hat y$}
    \end{minipage}\hspace{-4pt}
    
        \begin{minipage}[c]{0.3\textwidth}
        \includegraphics[width=0.95\linewidth, height=0.5\linewidth]{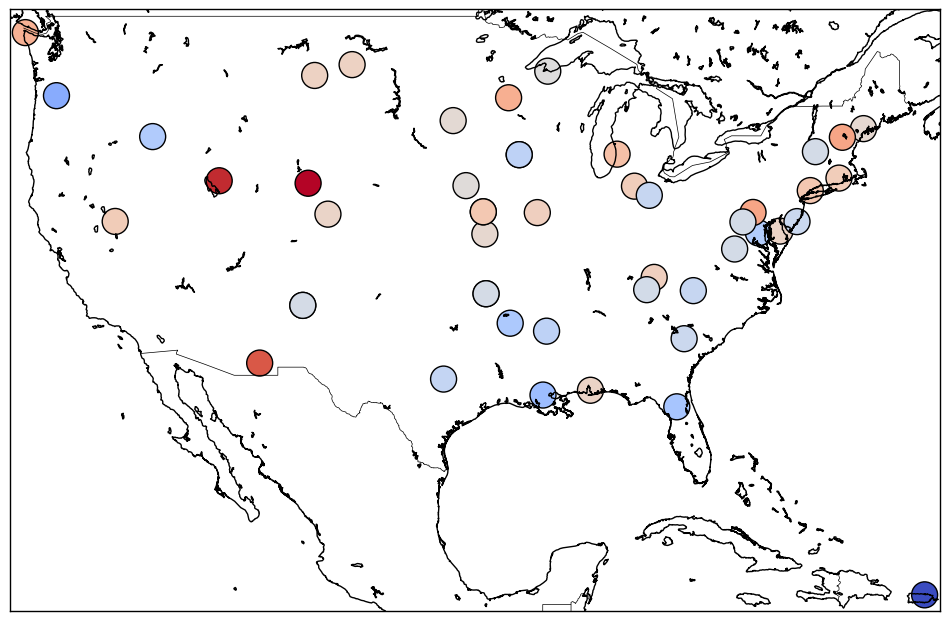}
        \panelcaption{(d) KNN-FL, $\hat y$}
    \end{minipage}\hspace{-1pt}
    \begin{minipage}[c]{0.3\textwidth}
        \includegraphics[width=0.95\linewidth, height=0.5\linewidth]{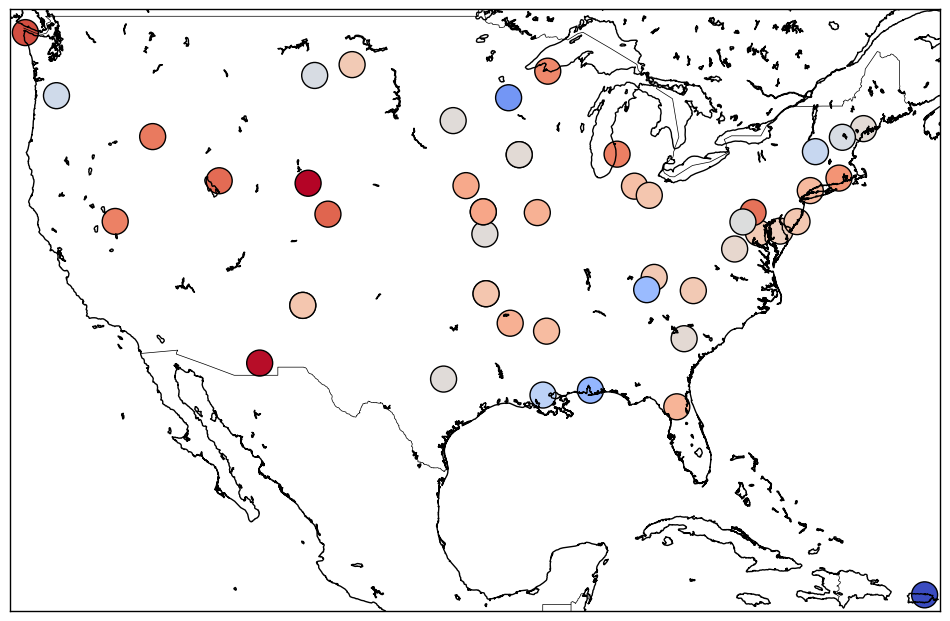}
        \panelcaption{(e) RKHS, $\hat y$}
    \end{minipage}\hspace{-1pt}
    \begin{minipage}[c]{0.3\textwidth}
        \includegraphics[width=0.95\linewidth, height=0.5\linewidth]{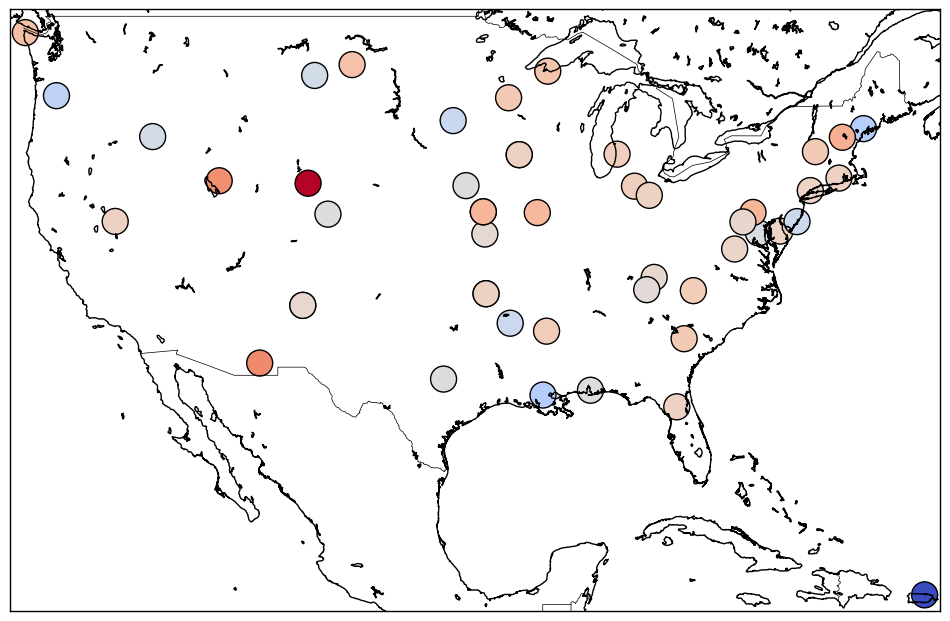}
        \panelcaption{(f) Trend Filtering, $\hat y$}
    \end{minipage}\hspace{-4pt}
                                       \vspace{0.5em}      \includegraphics[width=0.6\textwidth, height=0.08\linewidth]{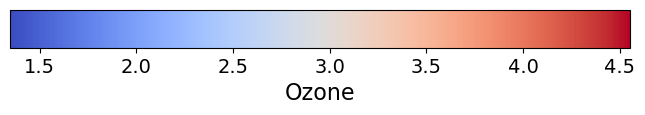}  \vspace{-4pt}
    \caption{Visualization of predicted ozone levels for the 45 monitoring locations  on March 3rd. Subplots (a) show the $y$, the observed ozone levels, while subplots (b) to (f) represent predictions from Dense NN, GAM, KNN-FL, RKHS, and Trend Filtering, respectively.}
    \label{OzoneRegions}
    \vspace{-4pt}
\end{figure*}

For our analysis, we focus on predicting daily ozone levels based on these geographic and meteorological variables.  The data consist of  monitoring stations across different states in 10  EPA regions,  ensuring a representative variety of locations. The observations cover an entire year of daily measurements. The regions considered include latitude and longitude ranges of $63^\circ N - 70^\circ N$ and $148^\circ W - 153^\circ W$, respectively, including several representative sites across the continental United States. 

To implement our analysis, we have applied our proposed Dense NN method alongside other methods such as GAM, KNN-FL, Trend Filtering, and RKHS. For each model, we used daily data from the year 2023. A total of 260 days were considered for each region, with an average of 40 sites or stations per day, resulting in 12,240 measurements. Categorical variables were encoded using one-hot encoding, and numeric variables were scaled to the range $[0,1]$. We perform 50 random train-test splits, dividing the data into training and test sets using a random $75\%/25\%$ split  for each EPA region. All models were trained on the training data, and 5-fold cross-validation was used to select tuning parameters.  Prediction performance was evaluated on the test set for each method. We assessed the accuracy of the ozone predictions based on relative error defined as: 
$\text{Relative Error} = \|\hat f - \tilde{y} \|_{nm}^2/\|\tilde{y}\|_{nm}^2$, where $\tilde{y}$ is the test set, and with an abuse of notation $\hat f$ is the vectorized version of the estimator at the locations corresponding to the test data. The final prediction performance is reported as the mean of the test set relative errors and the standard deviation of these errors across the 50 test splits.

Table \ref{tab_realdata_results} shows that the Dense NN   achieved the lowest relative error among all methods. Figure \ref{OzoneRegions} provides visualizations of the selected regions analyzed in this study. March 3rd was chosen as a representative day from the year-long dataset to streamline the analysis and enhance clarity. The figure's results further corroborate those presented in Table~\ref{tab_realdata_results}. Our proposed method, Dense NN, produces predictions that are closest to the true ozone levels. 

      \begin{table*}[!t]
\centering
 \begin{small}
         \caption{Prediction performance (Relative Error) of different methods for ozone level prediction. The values in parentheses represent standard deviations.}
\begin{tabular}{|c|c|c|c|c|}
\hline
 \textbf{Dense NN} & \textbf{GAM}   & \textbf{KNN-FL} & \textbf{RKHS}  & \textbf{Trend Filtering} \\ \hline
 0.0086 (0.009)    & 0.0574 (0.015) & 0.0255 (0.010)  & 0.1374 (0.022) & 0.0518 (0.020)          \\ \hline
\end{tabular}
\label{tab_realdata_results}
\end{small}
  \end{table*}

\raggedbottom
\section{Conclusion}\label{sec_conclusion}

This paper examines fully connected deep neural networks (Dense NN) with ReLU activation for nonparametric regression under temporal and spatial dependencies, demonstrating robust non-asymptotic convergence rates despite such complexities.

Future research could explore relaxing the assumption of exponentially decaying mixing coefficients to polynomial decay or weaker dependencies, broadening the applicability of the results. Additionally, investigating Dense NN with alternative activation functions could enhance versatility, and extending the theoretical framework to accommodate these functions would be a valuable contribution.

%

\section*{Acknowledgment}

OHMP was supported in part by the NSF CAREER Award DMS-2541747.

\clearpage
\setcounter{theorem}{0}
\setcounter{definition}{0}
\setcounter{remark}{0}
\setcounter{assumption}{0}
\setcounter{claim}{0}
\setcounter{equation}{0}
\setcounter{figure}{0}
\setcounter{table}{0}
\setcounter{footnote}{0}
\def\theHsection{supp.\arabic{section}}
\def\theHsubsection{\theHsection.\arabic{subsection}}
\def\theHsubsubsection{\theHsubsection.\arabic{subsubsection}}
\def\theHtheorem{supp.\arabic{theorem}}
\def\theHdefinition{supp.\arabic{definition}}
\def\theHremark{supp.\arabic{remark}}
\def\theHassumption{supp.\arabic{assumption}}
\def\theHclaim{supp.\arabic{claim}}
\def\theHequation{supp.\arabic{equation}}
\def\theHfigure{supp.\arabic{figure}}
\def\theHtable{supp.\arabic{table}}
\def\theHfootnote{supp.\arabic{footnote}}
\crefname{theorem}{theorem}{theorems}
\Crefname{theorem}{Theorem}{Theorems}
\crefname{lemma}{lemma}{lemmas}
\Crefname{lemma}{Lemma}{Lemmas}
\crefname{proposition}{proposition}{propositions}
\Crefname{proposition}{Proposition}{Propositions}
\crefname{corollary}{corollary}{corollaries}
\Crefname{corollary}{Corollary}{Corollaries}
\crefname{definition}{definition}{definitions}
\Crefname{definition}{Definition}{Definitions}
\crefname{assumption}{assumption}{assumptions}
\Crefname{assumption}{Assumption}{Assumptions}
\crefname{remark}{remark}{remarks}
\Crefname{remark}{Remark}{Remarks}
\crefname{equation}{equation}{equations}
\Crefname{equation}{Equation}{Equations}

\renewcommand\itbf{\bfseries\itshape}

\renewcommand{\p}{\mathbb P}
\renewcommand{\Q}{\widehat{\mathcal Q}}
\renewcommand{\I}{{\mathcal I}}
\renewcommand{\lbg}{\langle}
\renewcommand{\rbg}{\rangle}
\renewcommand{\J}{{\mathcal J}}
\renewcommand{\An}{{\mathcal A_n}}
\renewcommand{\Anm}{{\mathcal A_{nm}}}
\renewcommand{\lt}{{\mathcal L_2}}
\renewcommand{\T}{{\mathcal T}}
\renewcommand{\tv}{\text{TV}}
\renewcommand{\E}{\mathbb{E}}
\renewcommand{\st}{{\psi_2}}
\renewcommand{\so}{{\psi_1}}
\renewcommand{\s}{{\mathfrak{s}}}
\renewcommand{\hatp}{{\widehat{\mathcal P}}}
\renewcommand{\sx}{{\Sigma}}
\renewcommand{\ou}{\"{o}}
\renewcommand{\f}{\mathcal F}
\renewcommand{\diag}{\text{diag}}
\renewcommand{\lc}{\lceil}
\renewcommand{\rc}{\rceil}
\renewcommand{\kse}{\kappa_{\max}^{s,e}}
\renewcommand{\dint}{\,\mathrm{d}}
\renewcommand{\lb}{\langle}
\renewcommand{\rb}{\rangle}
\renewcommand{\N}{\mathcal N}
\renewcommand{\h}{\mathcal H}
\renewcommand{\lf}{\lfloor}
\renewcommand{\rf}{\rfloor}
\renewcommand{\hr}{{\mathcal H^{r}}}

\begin{center}
{\Large\bfseries Supplementary material for ``Dense ReLU Neural Networks for Temporal-Spatial Modeling''\par}
\end{center}

\begin{abstract}
We provide the proofs for all the results in the main paper. Furthermore, additional simulations are presented.

\end{abstract}

\begin{IEEEkeywords}
Nonparametric statistics, deep learning, manifold data.
\end{IEEEkeywords}

\renewcommand{\thesection}{S\arabic{section}}
\setcounter{section}{0}

\section*{Supplementary Material}
\addcontentsline{toc}{section}{Supplementary Material}

\section{Outline of the proof of Theorem 1, Theorem 5 in the Main Manuscript}
\label{outline_proof_main}
{\bf{Theorem 1}}:
The proof of Theorem 1 in the Main Manuscript is divided into two critical steps. 
The initial step entails a decomposition of the estimation error, measured via the $\mathcal L_2$ norm. In doing this, we use two key features. A minimization attribute and the coupling characteristic between empirical and $\mathcal L_2$ norms. The subsequent step is dedicated to the derivation of deviation bounds. These bounds incorporate Rademacher complexities and pertain to measurement errors as well as spatial noise.

Both stages come with a significant challenge, the interplay between time and space, known as the temporal-spatial dependence. To provide an overview of how we tackled this throughout the proof, in the following we expose our strategy. In Supplementary Material \ref{beta mixing proof temporal-spatial model}, a blocking method is crafted, which effectively transforms dependent $\beta$-mixing sequences into a sequence of independent blocks. 
The number of these blocks scales as $O(\log(n))$. By analyzing each block independently, we ensure that the examination of the $\beta$-mixing sequences is confined within these blocks. This approach ultimately allows us to analyze a process of length $O(\log(n))$, leveraging the independence within each block to manage the temporal-spatial dependence.
\\
\textbf{Proof Sketch:}
\begin{itemize}
  \item \textbf{Step $1$:}
	We find an approximator $  \overline f  \in \mathcal F(L, r) $ such that 
 \begin{align*}
\left\|f^* - \overline f \right\|_\infty^2 \le 2\phi_{nm},
\end{align*} where $\phi_{nm}$ is the upper bound of approximation error for class $\mathcal F(L,r)$, as shown in \eqref{approx_phi_app}. Choosing $ \mathcal{A}_{nm}=\max\{\sigma_\epsilon, \sigma_\gamma\}\cdot o \left(\log(nm)\right)$, leads to $ \left\| \overline f \right\|_\infty \le {\mathcal{A}_{nm}}$.
	
  \item \textbf{Step $2$:} Next, we decompose the $\mathcal L_2$ error of estimator as
\begin{align}
&\| \widehat f_{\mathcal{A}_{nm}} - f^*\|_\lt^2   
\le 2\| \overline f_{\mathcal{A}_{nm}} - f^*\|_\lt ^2 + 2\| \overline f_{\mathcal{A}_{nm}} - \widehat f_{\mathcal{A}_{nm}} \|_\lt ^2 \nonumber
\\
&\le  4 \phi_{nm} + 2\| \overline  f_{\mathcal{A}_{nm}} - \widehat f_{\mathcal{A}_{nm}} \|_\lt ^2.\label{l2_decomp}
\end{align}

  \item \textbf{Step $3$:} Our subsequent objective is to transition from the $\mathcal L_2$ norm to the empirical norm.  \Cref{lemma:rate in the temporal-spatial model} provides a coupling of empirical and $\mathcal L_2$ norms. Specificially, with high probability, 
  \begin{align}
&\bigl\| \widehat{f}_{\mathcal{A}_{nm}} -\overline{f}_{\mathcal{A}_{nm}} \bigr\|_\lt ^2 \le 2 \bigl\| \widehat f_{\mathcal{A}_{nm}} -\overline f_{\mathcal{A}_{nm}} \bigr\|_{nm} ^2   
+ C \bigl\{\phi_{\frac{nm}{S}} \log^{5.3}(\frac{nm}{S})+ \mathcal A^2_{nm}\beta(S)\bigr\},\label{eq:norm comparison 1 in the temporal-spatial model outline}
\end{align} 
 where $S \asymp \log(n)$.
 \item \textbf{Step $4$:} At this point, coupled with the minimization property of $\widehat f_{\mathcal{A}_{nm}}$, it follows that

\begin{align*}
&\displaystyle      \frac{1}{2} \left\| \widehat f _{\mathcal{A}_{nm}}- \overline f _{\mathcal{A}_{nm}} \right\|_{ nm} ^2
\displaystyle \le  \frac{1}{32}\| \widehat f_{\mathcal{A}_{nm}}  - \overline f_{\mathcal{A}_{nm}}\|_{nm} ^2 + 16\phi_{nm}
\\
&\displaystyle \,\,\,\,+ \sum_{i=1}^n \sum_{j=1}^{m_i} \frac{1}{nm_i} \bigl\{\widehat f _{\mathcal{A}_{nm}}(x_{ij})   
- \overline f _{\mathcal{A}_{nm}}(x_{ij}) \bigr\} \bigl\{ \gamma_i (x_{ij}) + \epsilon_{ij} \bigr\} .
\end{align*}

\item \textbf{Step $5$:}
To analyze the second term where involving $\epsilon_{ij}$, by \Cref{lemma:threshold space entropy 2} and \Cref{coro:complexity bound 2 beta mn}, with high probability, we have 

\begin{align*}
&\sum_{i=1}^n \sum_{j=1}^{m_i} \frac{1}{nm_i} \left\{\widehat f_{\mathcal{A}_{nm}} - \overline f_{\mathcal{A}_{nm}} \right\} (x_{ij}) \epsilon_{ij}
\\
&\le C\mathcal{A}_{nm} \cdot \phi_{\frac{nm}{S}} \bigl\{ \log^{4.3}\bigl(nm\bigr)\log(\| \widehat f_{\mathcal{A}_{nm}}   
-\overline f_{\mathcal{A}_{nm}}\|_{nm}^{-1} ) \bigr\} + \frac{1}{32} \bigl\| \widehat f_{\mathcal{A}_{nm}}-\overline f_{\mathcal{A}_{nm}}\bigr\|_{nm}^2.
\end{align*}

Then by \Cref{lemma:bound on the expected norm of gamma}, \Cref{coro:deviation bound for gamma with xi mn}, \Cref{lemma:rate in the temporal-spatial model}, also with high probability,   we have 
\begin{align*}
&\sum_{i=1}^n \sum_{j=1}^{m_i} \frac{1}{nm_i} \left\{\widehat f_{\mathcal{A}_{nm}} - \overline f_{\mathcal{A}_{nm}} \right\} (x_{ij}) \gamma_{i}(x_{ij})
\\
&\le \frac{8\sigma_\gamma^2}{n} + \frac{3}{16} \bigl\| \widehat f _{\mathcal{A}_{nm}}-\overline f _{\mathcal{A}_{nm}} \bigr\|_{nm} ^2   
+ C\bigl\{\mathcal{A}^2_{nm}\phi_{\frac{nm}{S}} \log^{5.3}(nm)+ \mathcal{A}^2_{nm}\beta(S)\bigr\}.
\end{align*}

 Putting the previous steps together, it follows that
\begin{align}
&\bigl\| \widehat{f}_{\mathcal{A}_n} - \overline{f}_{\mathcal{A}_{nm}} \bigr\|_{nm}^2   
\leq C_4 \bigl\{ \mathcal{A}_{nm}^2 \phi_{\frac{nm}{S}} \log^{5.3}(nm) + \mathcal{A}_{nm}^2 \beta(S) + \frac{1}{n} \bigr\},\label{eq2279}
\end{align}
which, together with \eqref{l2_decomp} and \eqref{eq:norm comparison 1 in the temporal-spatial model outline}, holds with high probability.

Then,  assuming Assumption 1 \textbf{a}, choosing $\mathcal{A}_{nm} = \max\{\sigma_\epsilon, \sigma_\gamma\} \cdot o\left(\log(nm)\right)$, and recalling

\begin{align*}
&\phi_{\frac{nm}{S}} = \max_{(p, K) \in \mathcal{P}} \left(\frac{nm}{S}\right)^{-\frac{2p}{2p + K}},
\end{align*}

we obtain the final bound in Equation ({\color{blue}{3}}) of the Main Manuscript.
\end{itemize}

{\bf{Theorem 5}}:
For inputs on a low-dimensional manifold, the approximation theory leads to a different result. 
. 

When the input variable $x_{ij}$ is supported on some $d^*$-dimensional Lipschitz-manifold $\mathcal{M} \subseteq[0, 1]^d$, and $d^* \in\{1, \ldots, d\}$, the approximation error is $\sqrt{\phi_{nm, \mathcal{M}}}$, presented in \eqref{approximation_error_manifold}. 

 We use $\phi_{nm, \mathcal{M}}$ throughout our proof to derive new deviation bounds.

 The proof steps for the manifold case follow a similar logical structure to those in the non-manifold case, thus we omit them.

\section{Proof of Theorem 1 in the Main Manuscript}
\label{proof of Theorem thm:main thm in the temporal-spatial model}
\subsection{Preliminary Result for Proof of Theorem 1 in the Main Manuscript}

\begin{lemma}[Neural Networks Approximation Result.]
\label{lemma_approximation}
   Suppose that $f^* \in \mathcal{H}(l,\mathcal{P})$, for some $l \in \mathbf{Z}^{+}$ and $\mathcal{P} \subseteq[1, \infty) \times \mathbf{Z}^{+}$, where $\mathcal{H}(l,\mathcal{P})$ is defined in Definition 3 in the Main Manuscript. In addition, assume that each function $g$ in the definition of $f^*$ can have different smoothness $p_g =  q_g +s_g$, for $q_g \in \mathbf{N}$, $s_g \in (0,1]$,  and of potentially different input dimension $K_g $, so that $(p_g, K_g ) \in \mathcal{P}$. Let $K_{\max}$ be the largest input dimension and $p_{\max}$ the largest smoothness of any of the functions $g$. Suppose that for each $g$ all the partial derivatives of order less than or equal to $q_g$ are uniformly bounded by constant $C_{\mathrm{Smooth} }$, and each function $g$ is Lipschitz continuous
with Lipschitz constant $C_{\mathrm{Lip} } \geq1 $. Also, assume that $\max\{p_{\max},K_{\max} \} =O(1)$.

Let 
\begin{align*}
&\phi_{nm} =  \max_{(p, K) \in \mathcal P } \left(nm\right)^{-\frac{2p}{2p +K } },
\end{align*} 
 where $(p,K) \in \mathcal P$, with smoothness constraint $p$ and order constraint $K$.
   
   Let $L_{nm}, r_{nm} \in \mathbf{Z}^+$,  then there exists a neural network $f \in \mathcal{F}(L, r)$ such that it can either take the form as

\begin{align}
&\textbf{Case 1.} \quad \text{a wide network with} \quad L   
= L_{nm} \asymp \log(nm), \quad \nonumber
\\
&\text{and} \quad r = r_{nm} \asymp \max_{(p, K) \in \mathcal{P}} (nm)^{\frac{K}{2(2p+K)}}, \nonumber
\\
&\textbf{Case 2.} \quad \text{a deep network with} \quad L = L_{nm}   
\asymp \log(nm) \cdot \max_{(p, K) \in \mathcal{P}} (nm)^{\frac{K}{2(2p+K)}}, \quad \nonumber
\\
&\text{and} \quad r = O(1).\label{network_arch_approx}
\end{align}

In either case, it holds that
\begin{align} 
\label{approx_phi_app}
\inf_{f \in \mathcal{F}(L , r )} \|f - f^*\|_\infty \lesssim \sqrt{\phi_{nm}}. 
\end{align}

Thus, by the choice of $L$ and $r$ in \eqref{network_arch_approx}, we have

\begin{align}
    Lr 
    &\asymp \log\left(nm\right)\max_{(p, K) \in \mathcal{P}} \left(nm\right)^{\frac{K}{2(2p+K)}} , \nonumber \\
    Lr^2 
    &\lesssim \log\left(nm\right)\left[\max_{(p, K) \in \mathcal{P}} \left(nm\right)^{\frac{K}{2(2p+K)}}\right]^2 \lesssim nm. \label{L_r_size}
\end{align}

\end{lemma}

\begin{proof}
The result follows from Theorem 3 of \cite{kohler2021rate}. 

Under the conditions in \eqref{network_arch_approx}, 
then there exists a neural network $\hat{f}_{\text {net}} \in \mathcal{F}(L, r)$ such that
\begin{align*}
&\left\|f^*-\hat{f}_{\text{net}}\right\|_{\infty} \lesssim   \sqrt{\phi_{nm}}.
\end{align*}
 
Overall, we have 

\begin{align*}
	& Lr 
  \asymp \log \left(nm\right)\max_{(p, K) \in \mathcal P  } \left(nm\right)^{\frac{K}{2(2p+K)}}, \\
	& Lr^2 \lesssim \log\left(nm\right)\left[\max_{(p, K) \in \mathcal P  }\left(nm\right)^{\frac{K}{2(2p+K)}}\right]^2 \lesssim nm.	
\end{align*}
\end{proof}

\begin{remark}
The rate above indicates that the networks of the class $\mathcal{F}(L, r)$, which is the so-called fully connected feedfoward neural networks class, can achieve the $\ell_\infty$-error bound that does not depend on the dimension of predictors $d$ and therefore circumvents the curse of dimensionality.

\end{remark}

\color{black}
\begin{lemma}
\label{lemma_epsilon_gamma_infinity_bound}
Suppose Assumption 1 holds. Define
\(
\left\|\gamma_i\right\|_{\infty,nm}:=\max_{1\le j\le m_i}\left|\gamma_i\left(x_{ij}\right)\right|.
\)
Then there exist parameters $\mathcal{B}_{nm}$ and $\mathcal{C}_{nm}$ depending only on $n,m$ such that
\begin{align}
&\p \bigl(|\epsilon_{ij}| \le \sigma_\epsilon \cdot \frac{\mathcal{B}_{nm}}{4}   
\text{ for all } 1\le i \le n,\, 1\le j \le m_i\bigr)=1-o(1),\label{eq:tail-bound event}
\\
&\p \bigl(\|\gamma_i\|_{\infty,nm} \le \sigma_\gamma \cdot \frac{\mathcal{C}_{nm}}{4}   
\text{ for all } 1 \le i\le n \bigr) =1-o(1),
\end{align}
where $\E(\epsilon_{ij}^2) \leq \sigma_\epsilon^2 <\infty$, and
\(
\sigma_\gamma^2=\sup_{x \in [0,1]^d} \E \gamma_i(x)^2<\infty.
\) Specifically, \(
\mathcal{C}_{nm}\asymp \sqrt{\log(nm)}
\) and \(
\mathcal{B}_{nm}\asymp \sqrt{\log(nm)}
\).
\end{lemma}

\begin{proof}
By Assumption~1, $\{\epsilon_{ij}\}_{i=1,j=1}^{n,m_i}$ are sub-Gaussian random variables with sub-Gaussian parameter $\sigma_{\epsilon}$. Hence
\begin{align*}
&\mathbb P \left(|\epsilon_{ij}| \le \sigma_\epsilon \cdot \frac{\mathcal{B}_{nm}}{4}
\text{ for all } 1\le i\le n,\ 1\le j \le m_i \right) 
=\, 1- \mathbb P \left(\max_{i=1,\ldots,n \atop j=1,\ldots,m_i}|\epsilon_{ij}|
\geq \sigma_\epsilon \cdot \frac{1}{4}\mathcal{B}_{nm}\right) \\
\geq&\, 1- \left(\sum_{i=1}^{n} m_i\right)
\exp\left(\sigma_\epsilon^2 \cdot \frac{-\mathcal{B}_{nm}^2}{32 \sigma_\epsilon^2}\right) \\
\geq&\, 1-\exp\left(
\sigma_\epsilon^2 \cdot \frac{-\mathcal{B}_{nm}^2}{32 \sigma_\epsilon^2}
+\log \left(\sum_{i=1}^n m_i\right)\right).
\end{align*}
Therefore, it suffices to choose
\begin{align*}
&\mathcal{B}_{nm}\asymp \sqrt{\log(nm)}.
\end{align*}

Next, by Assumption~1{\bf b}, there exist constants $c>0$ and $M>0$ such that
\begin{align*}
&\sup_{x\in[0,1]^d}\E\bigl[\exp\{c\,\gamma_i(x)^2\}\bigr]\le M.
\end{align*}
Fix $t>0$. Then, for every $x\in[0,1]^d$, Markov's inequality gives
\begin{align*}
\mathbb P\bigl(|\gamma_i(x)|>t\bigr)
&=\mathbb P\bigl(\exp\{c\,\gamma_i(x)^2\}>\exp\{ct^2\}\bigr)\\
&\le \exp(-ct^2)\,
\E\bigl[\exp\{c\,\gamma_i(x)^2\}\bigr]\\
&\le M e^{-ct^2}.
\end{align*}
Since this bound is uniform in $x$, conditioning on the design points and applying a union bound yields
\begin{align*}
&\mathbb P\left(\|\gamma_i\|_{\infty,nm} \le \sigma_\gamma \cdot \frac{\mathcal{C}_{nm}}{4}
\text{ for all }1\le i\le n\right) 
=\, 1-\mathbb P\left(\max_{1\le i\le n \atop 1\le j\le m_i}|\gamma_i(x_{ij})| >
\sigma_\gamma \cdot \frac{\mathcal{C}_{nm}}{4}\right)\\
\geq&\, 1-\sum_{i=1}^n\sum_{j=1}^{m_i}
\mathbb P\left(|\gamma_i(x_{ij})|>\sigma_\gamma \cdot \frac{\mathcal{C}_{nm}}{4}\right)\\
\geq&\, 1- M\left(\sum_{i=1}^{n} m_i\right)\exp\left(-c\,\sigma_\gamma^2\frac{\mathcal{C}_{nm}^2}{16}\right).
\end{align*}
Therefore, it suffices to choose
\begin{align*}
&\mathcal{C}_{nm}\asymp \sqrt{\log(nm)}.
\end{align*}

Combining the two bounds above gives
\begin{align*}
&\mathbb P\bigl(\bigl|\epsilon_{ij}\bigr| \le \sigma_\epsilon \cdot \frac{\mathcal{B}_{nm}}{4}   
\text{ for all } 1 \le i \le n,\,1 \le j \le m_i \bigr)=1-o(1),
\end{align*}
and
\begin{align*}
&\mathbb P\bigl(\|\gamma_i\|_{\infty,nm} \le \sigma_\gamma \cdot \frac{\mathcal{C}_{nm}}{4}   
\text{ for all } 1 \le i \le n\bigr)=1-o(1).
\end{align*}
This completes the proof.
\end{proof}
\color{Black}

\subsection{Proof of Theorem~1 in the Main Manuscript}
\label{sec_proof_Theorem_dependent}
To facilitate reading, we provide the Theorem~1 in the Main Manuscript again. 

\begin{theorem*}[Theorem~1] 
Suppose Assumption 1 holds and let $\widehat{f}$ be the estimator in $\mathcal F(L,r)$ defined in Equation (2) in the Main Manuscript, 
with its number of layers $L$ and number of neurons per layer $r$ set as either one of the two cases: 
\begin{align}
&\text{\textbf{Case 1.} a wide network with} \quad L   
\asymp \log(nm) \quad \text{and} \quad r \asymp \max_{(p, K) \in \mathcal{P}} (nm)^{\frac{K}{2(2p+K)}}, \nonumber
\\
&\text{\textbf{Case 2.} a deep network with} \quad L \notag \\
&\quad \asymp \log(nm) \cdot \max_{(p, K) \in \mathcal{P}} (nm)^{\frac{K}{2(2p+K)}} \quad \text{and} \quad r = O(1).\label{network_arch}
\end{align}Let $\mathcal{A}_{nm} \asymp \max \left\{\sigma_\epsilon,\sigma_\gamma\right\}\cdot \log^{1/2}(nm)$ and $(a_n)_{n\ge 1}$ be any sequence such that $a_n \to \infty$.
Then with probability approaching to one, it holds that 

\begin{align}
&\| \widehat f_{\mathcal{A}_{nm}} - f^*\|_\lt ^2 \le a_n\bigl(\sigma_\epsilon^2+\sigma_\gamma^2 +1\bigr)\max_{(p,K)\in \mathcal {P}}\bigl(   
\frac{1}{nm}\bigr)^{\frac{2p}{2p+K}}\log^{7.3}(nm) \notag \\
&\quad + \frac{\bigl(\sigma_\epsilon^2+\sigma_\gamma^2\bigr)\log(nm)}{n}.\label{rate_in_theorem_temporal_spatial}
\end{align}
\end{theorem*}

\begin{proof}

 Let $\mathcal{A}_{nm} \asymp \max \left\{\sigma_\epsilon \cdot \mathcal{B}_{nm}, \sigma_\gamma \cdot \mathcal{C}_{n}\right\}\asymp \max \left\{\sigma_\epsilon \cdot \sqrt{\log(nm)}, \sigma_\gamma \cdot \sqrt{\log(n)} \right\}$,\\ or $\max\{\sigma_\epsilon, \sigma_\gamma \}\cdot \sqrt{\log(nm)}$. Recall $\phi_{nm} =  \max_{(p, K) \in \mathcal P } \left(nm\right)^{-\frac{2p}{2p +K } }$.

Then, we will show that 
with probability approaching to one, it holds that 
\begin{align*}
&\| \widehat f_{\mathcal{A}_{nm}} - f^*\|_\lt ^2   
\lesssim {\mathcal{A}^2_{nm}} \phi_{\frac{nm}{S}}\log^{5.3}(nm) + \mathcal{A}^2_{nm}\beta(S)+ \frac{\sigma_\gamma^2}{n},
\end{align*}
for sufficiently large $n$, with $S \asymp \log(n)$. 
Finally, by Assumption 1 for $\beta$ coefficients, let $\beta(S) \lesssim \exp(-C_{\beta}S)$, we obtain the bound. 

Let 
$  \overline f  \in \mathcal F(L, r)   $  be such that 
\begin{align}
\label{eq: error of barf in temporal-spatial model}
&\|f^* - \overline f \|_\infty^2 \le 2\phi_{nm}, 
\end{align} where the existence of $\overline f$ is valid by \Cref{lemma_approximation}.
Since $\mathcal{A}_{nm}$ can be chosen appropriately large with respect to $n$ while $\phi_{nm}$ decays with respect to $n$, assume that $\sqrt{2\phi_{nm}} \le \frac{3}{4}\mathcal{A}_{nm}$. In addition, assume $ 1 \leq \mathcal{A}_{nm}=\max\{\sigma_\epsilon, \sigma_\gamma\}\cdot \sqrt{\log(nm)}$ where $m$ and $m_i$ satisfy $cm\leq m_i \leq Cm$ for any $i=1,2,\ldots, n$. Then,
\begin{align*}
&\bigl\| \overline f \bigr\|_\infty \le \bigl\|f^* \bigr\|_\infty + \bigl\| f^* -\overline f \bigr\|_\infty   
\le \frac{1}{4}\mathcal{A}_{nm} + \sqrt {2 \phi_{nm}} \le {\mathcal{A}_{nm}},
\end{align*} where the second inequality holds by Assumption 1 \textbf{g} and \eqref{eq: error of barf in temporal-spatial model}.
It follows that $ \overline f=\overline f_{\mathcal{A}_{nm}}$.
Since 
\begin{align*}
&\| \widehat f_{\mathcal{A}_{nm}} - f^*\|_\lt^2 \le 2\| \overline f_{\mathcal{A}_{nm}} - f^*\|_\lt ^2 + 2\| \overline f_{\mathcal{A}_{nm}}   
- \widehat f_{\mathcal{A}_{nm}} \|_\lt ^2 \le 4 \phi_{nm} + 2\| \overline f_{\mathcal{A}_{nm}} - \widehat f_{\mathcal{A}_{nm}} \|_\lt ^2 ,
\end{align*} where the second inequality holds by \eqref{eq: error of barf in temporal-spatial model}, 
it suffices to show with probability approaching to one,
\begin{align}
&\| \overline f_{\mathcal{A}_{nm}} - \widehat f_{\mathcal{A}_{nm}} \|_\lt ^2   
\lesssim {\mathcal{A}^2_{nm}} \phi_{\frac{nm}{S}}\log^{5.3}(nm) + \mathcal{A}^2_{nm}\beta(S)+ \frac{\sigma_\gamma^2}{n}.
\end{align}

The following steps are made conditioning on $\{x_{ij}\}^{n,m_i}_{i=1,j=1}$ except for \eqref{eq:norm comparison 1 in the temporal-spatial model}.
\
\\
{\bf Step 1.}
Denote the event
\begin{align*}
&\mathcal E=\bigl\{ |\epsilon_{ij}| \le \frac{1}{4}\mathcal{A}_{nm} \text{ for all } 1\le i\le n, 1\le j   
\le m_i \notag \\
&\quad \text{and}\quad \bigl\|\gamma_i\bigr\|_\infty \le \frac{1}{4}\mathcal{A}_{nm} \text{ for all } 1\le i\le n \bigr\}.
\end{align*}

Suppose
the event $\mathcal E $ holds. 
By \Cref{lemma_epsilon_gamma_infinity_bound}, we have \begin{align}
\mathbb{P}\Big(&\left|\epsilon_{i j}\right| \le \frac{\mathcal{A}_{nm}}{4} \text{ for all } 1 \le i \le n, 1 \le j \le m_i \nonumber \\
&\text{and } \left\|\gamma_i\right\|_{\infty} \le \frac{\mathcal{A}_{nm}}{4} \text{ for all } 1 \le i \le n\Big) \nonumber \\
&= 1 - o(1), \text{ or } \nonumber \\
\mathbb{P}\left( \mathcal{E}^c \right) &= o(1).
\end{align}

Note that under $\mathcal E$, 
$|y_{ij}| \le \|f^*\|_\infty + |\epsilon_{ij}| + \|\gamma_i\|_{\infty,nm} \le {\mathcal{A}_{nm}} $ for any $1\le i\le n, 1\le j \le m_i $. Here $\left\|\gamma_i\right\|_{\infty,nm}:=\max_{1\le j\le m_i}\left|\gamma_i\left(x_{ij}\right)\right|$.
It  follows that 
  \begin{align}
&\sum_{i=1}^n \sum_{j=1}^{m_i} \frac{1}{nm_i} \bigl( y_{ij} - \widehat{f}_{\mathcal{A}_{nm}} (x_{ij}) \bigr)^2   
\le \sum_{i=1}^n \sum_{j=1}^{m_i} \frac{1}{nm_i}\bigl( y_{ij} - \widehat{f} (x_{ij}) \bigr)^2 .\label{inequality_truncation}
\end{align}
Additionally, from the minimization property of $\widehat{f}(\cdot)$, 
\begin{align*}
&\sum_{i=1}^n \sum_{j=1}^{m_i} \frac{1}{nm_i} \bigl( y_{ij} - \widehat{f}(x_{ij}) \bigr)^2   
\le \sum_{i=1}^n \sum_{j=1}^{m_i} \frac{1}{nm_i} \bigl(y_{ij} - \overline{f}(x_{ij}) \bigr)^2
\\
&=   \sum_{i=1}^n \sum_{j=1}^{m_i} \frac{1}{nm_i} \left(  y_{ij} -   \overline{f}_{\mathcal{A}_{nm}}   (x_{ij})  \right)^2,
\end{align*}
where the second identity follows from the fact $ \overline f=\overline f_{\mathcal{A}_{nm}}$ mentioned above.
Therefore 
\begin{align}
&\sum_{i=1}^n \sum_{j=1}^{m_i} \frac{1}{nm_i} \bigl( y_{ij} - \widehat{f}_{\mathcal{A}_{nm}} (x_{ij}) \bigr)^2   
\le \sum_{i=1}^n \sum_{j=1}^{m_i} \frac{1}{nm_i} \bigl( y_{ij} - \overline{f}_{\mathcal{A}_{nm}} (x_{ij}) \bigr)^2 .\label{eq:minimization property}
\end{align} 
\
\\
Rewriting the left-hand side and expanding the squared term, we get 
\begin{align*}
&\sum_{i=1}^n \sum_{j=1}^{m_i} \frac{1}{nm_i} \left( y_{ij} - \widehat{f}_{\mathcal{A}_{nm}}(x_{ij}) \right)^2
\\
&= \sum_{i=1}^n \sum_{j=1}^{m_i} \frac{1}{nm_i} \bigl[ \bigl( y_{ij} - \overline{f}_{\mathcal{A}_{nm}}(x_{ij}) \bigr)   
+ \bigl( \overline{f}_{\mathcal{A}_{nm}}(x_{ij}) - \widehat{f}_{\mathcal{A}_{nm}}(x_{ij}) \bigr) \bigr]^2
\\
&=\sum_{i=1}^n \sum_{j=1}^{m_i} \frac{1}{nm_i} \left( y_{ij} - \overline{f}_{\mathcal{A}_{nm}}(x_{ij}) \right)^2
\\
&\quad + \sum_{i=1}^n \sum_{j=1}^{m_i}\frac{2}{nm_i} \bigl( y_{ij}   
- \overline{f}_{\mathcal{A}_{nm}}(x_{ij}) \bigr) \bigl( \overline{f}_{\mathcal{A}_{nm}}(x_{ij}) - \widehat{f}_{\mathcal{A}_{nm}}(x_{ij}) \bigr)
\\
&\quad +  \sum_{i=1}^n \sum_{j=1}^{m_i} \frac{1}{nm_i} \left( \overline{f}_{\mathcal{A}_{nm}}(x_{ij}) - \widehat{f}_{\mathcal{A}_{nm}}(x_{ij}) \right)^2
\\
&\leq \sum_{i=1}^n \sum_{j=1}^{m_i} \frac{1}{nm_i} \left(  y_{ij} -   \overline{f}_{\mathcal{A}_{nm}}   (x_{ij})  \right)^2.
\end{align*}
  
Subtracting \( \sum_{i=1}^n \sum_{j=1}^{m_i} \frac{1}{nm_i} \left( y_{ij} - \overline{f}_{\mathcal{A}_{nm}}(x_{ij}) \right)^2\) from both sides, we get
\begin{align*}
&\sum_{i=1}^n \sum_{j=1}^{m_i} \frac{1}{nm_i} \left( \overline{f}_{\mathcal{A}_{nm}}(x_{ij}) - \widehat{f}_{\mathcal{A}_{nm}}(x_{ij}) \right)^2
\\
&\le\,- \sum_{i=1}^n \sum_{j=1}^{m_i} {\frac{2}{nm_i}}\bigl( y_{ij}   
- \overline{f}_{\mathcal{A}_{nm}}(x_{ij}) \bigr) \bigl( \overline{f}_{\mathcal{A}_{nm}}(x_{ij}) - \widehat{f}_{\mathcal{A}_{nm}}(x_{ij}) \bigr)
\end{align*}

which implies 
\begin{align*}
&\frac{1}{2} \bigl\| \widehat{f}_{\mathcal{A}_{nm}} - \overline{f}_{\mathcal{A}_{nm}} \bigr\|_{nm}^2 \le \sum_{i=1}^n \sum_{j=1}^{m_i} \frac{1}{nm_i} \bigl\{ \widehat{f}_{\mathcal{A}_{nm}}(x_{ij})   
- \overline{f}_{\mathcal{A}_{nm}}(x_{ij}) \bigr\} \bigl\{ y_{ij} - \overline{f}_{\mathcal{A}_{nm}}(x_{ij}) \bigr\}.
\end{align*}

It follows that
\begin{align*}
&\frac{1}{2} \left\| \widehat f _{\mathcal{A}_{nm}}- \overline f _{\mathcal{A}_{nm}} \right\|_{ nm} ^2
\\
&\le \sum_{i=1}^n \sum_{j=1}^{m_i} \frac{1}{nm_i} \bigl\{\widehat f _{\mathcal{A}_{nm}}(x_{ij}) - \overline f _{\mathcal{A}_{nm}}(x_{ij}) \bigr\} \bigl\{ f^*(x_{ij})   
+ \gamma_i (x_{ij}) + \epsilon_{ij} - \overline f_{\mathcal{A}_{nm}}(x_{ij})\bigr\}
\\
&= \sum_{i=1}^n \sum_{j=1}^{m_i} \frac{1}{nm_i} \bigl\{\widehat f _{\mathcal{A}_{nm}}(x_{ij})   
- \overline f _{\mathcal{A}_{nm}}(x_{ij}) \bigr\} \bigl\{ f^*(x_{ij}) - \overline f_{\mathcal{A}_{nm}}(x_{ij})\bigr\}
\\
&+ \sum_{i=1}^n \sum_{j=1}^{m_i} \frac{1}{nm_i} \bigl\{\widehat f _{\mathcal{A}_{nm}}(x_{ij})   
- \overline f _{\mathcal{A}_{nm}}(x_{ij}) \bigr\} \bigl\{ \gamma_i (x_{ij}) + \epsilon_{ij} \bigr\} .
\end{align*}
\
\\
{\bf Step 2.} Note that 
\begin{align*}
&\sum_{i=1}^n \sum_{j=1}^{m_i} \frac{1}{nm_i} \bigl\{\widehat f _{\mathcal{A}_{nm}}(x_{ij})   
- \overline f _{\mathcal{A}_{nm}}(x_{ij}) \bigr\} \bigl\{ f^*(x_{ij}) - \overline f_{\mathcal{A}_{nm}}(x_{ij})\bigr\}
\\
&= \sum_{i=1}^n \sum_{j=1}^{m_i} \frac{1}{nm_i} \bigl\{\widehat f _{\mathcal{A}_{nm}}(x_{ij})   
- \overline f _{\mathcal{A}_{nm}}(x_{ij}) \bigr\} \bigl\{ f^*(x_{ij}) - \overline{f} (x_{ij})\bigr\}
\\
&\le  \left\| \widehat f_{\mathcal{A}_{nm}}  - \overline f_{\mathcal{A}_{nm}}\right\|_{nm} \sqrt {2\phi_{nm}}
\\
&\le \frac{1}{32}\| \widehat f_{\mathcal{A}_{nm}}  - \overline f_{\mathcal{A}_{nm}}\|_{nm} ^2 + 16\phi_{nm},
\end{align*} 
where the first inequality follows from the definition of empirical norm $\left\|\cdot\right\|_{nm}$ and \eqref{eq: error of barf in temporal-spatial model}.

\color{black}
{\bf Step 3.} By \Cref{lemma:threshold space entropy 2} and \Cref{coro:complexity bound 2 beta mn} applied with \(t=1\),
\begin{align*} 
& \sum_{i=1}^n \sum_{j=1}^{m_i} \frac{1}{nm_i}
\left\{\widehat f_{\mathcal{A}_{nm}} - \overline f_{\mathcal{A}_{nm}} \right\} (x_{ij}) \epsilon_{ij}
\le\;
C_{\epsilon,1}\sigma_\epsilon\,
\Lambda_{n,S}\!\left(\widehat f_{\mathcal{A}_{nm}}-\overline f_{\mathcal{A}_{nm}}\right),
\end{align*}
with high probability, where
\begin{align*}
&\Lambda_{n,S}\!\bigl(\widehat f_{\mathcal{A}_{nm}}-\overline f_{\mathcal{A}_{nm}}\bigr)    
:={} \sum_{s=0}^{S-1}\frac{|\mathcal J_{e,s}|}{n}\, \bigl\|\widehat f_{\mathcal{A}_{nm}}-\overline f_{\mathcal{A}_{nm}}\bigr\|_{|\mathcal J_{e,s}|m}
\\
&\qquad\qquad\times   
\sqrt{ \phi_{|\mathcal J_{e,s}|m}\log^{4.3}(nm)\, \log\!\bigl( \bigl\|\widehat f_{\mathcal{A}_{nm}}-\overline f_{\mathcal{A}_{nm}}\bigr\|_{|\mathcal J_{e,s}|m}^{-1} \bigr) }
\\
&\quad+
\sum_{s=0}^{S-1}\frac{|\mathcal J_{o,s}|}{n}\,
\left\|\widehat f_{\mathcal{A}_{nm}}-\overline f_{\mathcal{A}_{nm}}\right\|_{|\mathcal J_{o,s}|m}
\\
&\qquad\qquad\times   
\sqrt{ \phi_{|\mathcal J_{o,s}|m}\log^{4.3}(nm)\, \log\!\bigl( \bigl\|\widehat f_{\mathcal{A}_{nm}}-\overline f_{\mathcal{A}_{nm}}\bigr\|_{|\mathcal J_{o,s}|m}^{-1} \bigr) }.
\end{align*}
\color{Black}

\color{black}
{\bf Step 4.}
Let 
\begin{align*}
&g := \widehat f_{\mathcal{A}_{nm}} - \overline f_{\mathcal{A}_{nm}}.
\end{align*}
Since both \(\widehat f_{\mathcal{A}_{nm}}\) and \(\overline f_{\mathcal{A}_{nm}}\) are truncated at level \(\mathcal{A}_{nm}\), we have
\begin{align*}
&\|g\|_\infty \le 2\mathcal{A}_{nm}.
\end{align*}

Define
\begin{align*}
&Z_n := \sup_{h \in \mathcal F} \bigl|   
\frac{\|h\|_{nm}^2 - \|h\|_{\mathcal L_2}^2} {\eta_n \|h\|_{\mathcal L_2}^2 + \tau C_{\eta_n}\phi_{nm/S}\log^{5.3}(nm/S) + \tau^2\beta(S)} \bigr|,
\end{align*}
with \(\tau = 2\mathcal{A}_{nm}\).

By Lemma \ref{lemma:rate in the temporal-spatial model}, we have
\begin{align*}
&\E(Z_n) \le 12.
\end{align*}

Let \((a_n)\) be any sequence such that \(a_n \to \infty\), and define
\begin{align*}
&\eta_n := \frac{1}{8 a_n},
\end{align*}
so that \(\eta_n \in (0,1/4)\) for all sufficiently large \(n\), and
\begin{align*}
&a_n \eta_n = \frac{1}{8}.
\end{align*}
Applying Markov's inequality to \(Z_n\), we obtain
\begin{align*}
&\Pr(Z_n>a_n)\le \frac{\E(Z_n)}{a_n}\le \frac{12}{a_n}\to 0.
\end{align*}
Hence, with probability tending to one,
\begin{align*}
&Z_n\le a_n.
\end{align*}
Therefore, uniformly over \(h\in\mathcal F\),
\begin{align*}
&\|h\|_{\mathcal L_2}^2-\|h\|_{nm}^2 \le a_n\Bigl( \eta_n\|h\|_{\mathcal L_2}^2   
+ \tau C_{\eta_n}\phi_{nm/S}\log^{5.3}(nm/S) + \tau^2\beta(S) \Bigr).
\end{align*}
Since \(a_n\eta_n=1/8\), it follows that
\begin{align*}
&\bigl(1-\frac18\bigr)\|h\|_{\mathcal L_2}^2   
\le \|h\|_{nm}^2 + a_n\Bigl( \tau C_{\eta_n}\phi_{nm/S}\log^{5.3}(nm/S) + \tau^2\beta(S) \Bigr),
\end{align*}
that is,
\begin{align*}
&\|h\|_{\mathcal L_2}^2 \le \frac87\|h\|_{nm}^2   
+ \frac87 a_n\Bigl( \tau C_{\eta_n}\phi_{nm/S}\log^{5.3}(nm/S) + \tau^2\beta(S) \Bigr).
\end{align*}
Using \(\tau=2\mathcal A_{nm}\), we conclude that
\begin{align*}
&\|h\|_{\mathcal L_2}^2 \le \frac87\|h\|_{nm}^2 + \frac{16}{7}a_n\mathcal A_{nm} C_{\eta_n}\phi_{nm/S}   
\log^{5.3}(nm/S) + \frac{32}{7}a_n\mathcal A_{nm}^2\beta(S).
\end{align*}
Since \(8/7<2\), this yields
\begin{align*}
&\|h\|_{\mathcal L_2}^2 \le 2\|h\|_{nm}^2 + \frac{16}{7}a_n\mathcal A_{nm} C_{\eta_n}\phi_{nm/S}   
\log^{5.3}(nm/S) + \frac{32}{7}a_n\mathcal A_{nm}^2\beta(S).
\end{align*}
Applying this bound with \(h = \widehat f_{\mathcal{A}_{nm}} - \overline f_{\mathcal{A}_{nm}}\), we obtain that with probability tending to one,
\begin{align}
&\bigl\| \widehat{f}_{\mathcal{A}_{nm}} -\overline{f}_{\mathcal{A}_{nm}} \bigr\|_\lt ^2 \le 2 \bigl\| \widehat f_{\mathcal{A}_{nm}} -\overline f_{\mathcal{A}_{nm}} \bigr\|_{nm} ^2 \notag \\
&\quad + C_{\mathcal L_2}\, a_n \bigl\{\mathcal A_{nm} C_{\eta_n}\phi_{\frac{nm}{S}} \log^{5.3}\!\bigl(\frac{nm}{S}\bigr)+ \mathcal A^2_{nm}\beta(S)\bigr\},\label{eq:norm comparison 1 in the temporal-spatial model}
\end{align}
where \(C_{\mathcal L_2}\) is a universal constant satisfying \(C_{\mathcal L_2}>0\); for instance, one may take \(C_{\mathcal L_2}=32/7\).
\color{Black}

Observe that 
\begin{align*}
&\sum_{i=1}^n \sum_{j=1}^{m_i} \frac{1}{nm_i}   \left[\widehat f _{\mathcal{A}_{nm}}- \overline f _{\mathcal{A}_{nm}} \right] (x_{ij}) \gamma_i(x_{ij})
\\
&= \sum_{i=1}^n \sum_{j=1}^{m_i} \frac{1}{nm_i} \bigl\{\bigl[\widehat f _{\mathcal{A}_{nm}}- \overline{f} _{\mathcal{A}_{nm}}\bigr] (x_{ij}) \gamma_i(x_{ij})   
-\langle \gamma_i, \widehat f _{\mathcal{A}_{nm}} -\overline f _{\mathcal{A}_{nm}} \rangle_{\mathcal L_2}\bigr\}
\\
&\quad+\frac{1}{n} \sum_{i=1}^n  \langle \gamma_i, \widehat f  _{\mathcal{A}_{nm}}-\overline{f}_{\mathcal{A}_{nm}} \rangle_{\mathcal L_2}.
\end{align*}
Also,
\begin{align*}
&\frac{1}{n} \sum_{i=1}^n \bigl\langle \gamma_i, \widehat{f}_{\mathcal{A}_{nm}} -\overline{f}_{\mathcal{A}_{nm}} \bigr\rangle_\lt   
= \bigl\langle \frac{1}{n} \sum_{i=1}^n \gamma_i, \widehat{f}_{\mathcal{A}_{nm}} -\overline{f} _{\mathcal{A}_{nm}} \bigr\rangle_\lt
\\
&\le \left\| \frac{1}{n} \sum_{i=1}^n \gamma_i  \right\|_\lt \left\| \widehat{f}_{\mathcal{A}_{nm}}-\overline{f}_{\mathcal{A}_{nm}} \right\|_\lt
\\
&\le \frac{\sigma_\gamma^2}{\sqrt{n}} \bigl\| \widehat{f} _{\mathcal{A}_{nm}}-\overline{f}_{\mathcal{A}_{nm}}\bigr\|_\lt   
\le \frac{8\sigma_\gamma^2}{n} + \frac{1}{32} \bigl\| \widehat f _{\mathcal{A}_{nm}}-\overline f _{\mathcal{A}_{nm}}\bigr\|_\lt ^2
\\
&\le \frac{8\sigma_\gamma^2}{n} + \frac{1}{16 } \bigl\| \widehat f _{\mathcal{A}_{nm}}-\overline f _{\mathcal{A}_{nm}} \bigr\|_{nm} ^2   
+ C_{\gamma,1} a_n\bigl\{\mathcal A_{nm} C_{\eta_n}\phi_{\frac{nm}{S}} \log^{5.3}\!\bigl(\frac{nm}{S}\bigr)+ \mathcal{A}^2_{nm}\beta(S)\bigr\},
\end{align*}
where the second inequality follows from Lemma \ref{lemma:bound on the expected norm of gamma} and the last inequality follows from \eqref{eq:norm comparison 1 in the temporal-spatial model}, for some universal constant \(C_{\gamma,1}\) satisfying \(C_{\gamma,1}>0\).

In addition, 
\begin{align*}
&\sum_{i=1}^n \sum_{j=1}^{m_i} \frac{1}{nm_i} \bigl\{ \big[ \widehat f_{\mathcal{A}_{nm}}- \overline f _{\mathcal{A}_{nm}}\big] (x_{ij}) \gamma_i(x_{ij})   
-\langle \gamma_i, \widehat f_{\mathcal{A}_{nm}} -\overline f _{\mathcal{A}_{nm}}\rangle_{\mathcal L_2}\bigr\}
\\
&\le  \frac{1}{16} \| \widehat f_{\mathcal{A}_{nm}}- \overline f_{\mathcal{A}_{nm}}\| _\lt^2 + C_{\gamma,2} \sigma_\gamma \mathcal C_{n}  \phi_{nm}  \log^{5.3}(nm)
\\
&\le \frac{1}{8} \| \widehat f_{\mathcal{A}_{nm}}- \overline f_{\mathcal{A}_{nm}}\| _{nm}^2   
+ C_{\gamma,3} a_n\bigl\{\mathcal A_{nm} C_{\eta_n}\phi_{\frac{nm}{S}} \log^{5.3}\!\bigl(\frac{nm}{S}\bigr)+ \mathcal A^2_{nm}\beta(S)\bigr\}
\\
&\quad +   C_{\gamma,2} \sigma_\gamma \mathcal C_{n}  \phi_{nm}  \log^{5.3}(nm),
\end{align*}
where the first inequality follows from \Cref{coro:deviation bound for gamma with xi mn}, and the second inequality follows from \eqref{eq:norm comparison 1 in the temporal-spatial model}, for some universal constants \(C_{\gamma,2},C_{\gamma,3}\) satisfying \(C_{\gamma,2}>0\) and \(C_{\gamma,3}>0\).
\\
\\

{\bf Step 5.} Putting the previous steps together, it follows that 
\begin{align*}
&\frac{1}{2}\left\|   \widehat f _{\mathcal{A}_{nm}}- \overline f _{\mathcal{A}_{nm}}\right\|_{ nm} ^2
\le\;
 \frac{1}{32}\left\| \widehat f_{\mathcal{A}_{nm}}  - \overline f_{\mathcal{A}_{nm}}\right\|_{nm} ^2 + 16\phi_{nm}
+ C_{\epsilon,1}\sigma_\epsilon\,
{\color{black}{\Lambda_{n,S}\!\left(\widehat f_{\mathcal{A}_{nm}}-\overline f_{\mathcal{A}_{nm}}\right)}}
\\
&\quad + \frac{8\sigma_\gamma^2}{n} + \frac{1}{16 } \bigl\| \widehat f _{\mathcal{A}_{nm}}-\overline f _{\mathcal{A}_{nm}} \bigr\|_{nm} ^2   
+ C_{\gamma,2} \sigma_\gamma \mathcal C_{n}\phi_{nm} \log^{5.3}(nm)
\\
&\quad +\frac{1}{8} \bigl\| \widehat f_{\mathcal{A}_{nm}}- \overline f_{\mathcal{A}_{nm}}\bigr\| _{nm}^2   
+ C_{\star}a_n\bigl\{\mathcal A_{nm} C_{\eta_n}\phi_{\frac{nm}{S}} \log^{5.3}\!\bigl(\frac{nm}{S}\bigr) + \mathcal A^2_{nm}\beta(S) \bigr\},
\end{align*}
where \(C_{\star}\) is a universal constant satisfying \(C_{\star}>0\), and
\color{black}
\begin{align*}
&\Lambda_{n,S}\!\bigl(\widehat f_{\mathcal{A}_{nm}}-\overline f_{\mathcal{A}_{nm}}\bigr)   
={} \sum_{s=0}^{S-1}\frac{|\mathcal J_{e,s}|}{n} \bigl\|\widehat f_{\mathcal{A}_{nm}}-\overline f_{\mathcal{A}_{nm}}\bigr\|_{|\mathcal J_{e,s}|m}
\\
&\qquad\qquad\times   
\sqrt{ \phi_{|\mathcal J_{e,s}|m}\log^{4.3}(nm)\, \log\!\bigl( \bigl\|\widehat f_{\mathcal{A}_{nm}}-\overline f_{\mathcal{A}_{nm}}\bigr\|_{|\mathcal J_{e,s}|m}^{-1} \bigr) }
\\
&\quad+
\sum_{s=0}^{S-1}\frac{|\mathcal J_{o,s}|}{n}
\left\|\widehat f_{\mathcal{A}_{nm}}-\overline f_{\mathcal{A}_{nm}}\right\|_{|\mathcal J_{o,s}|m}
\\
&\qquad\qquad\times   
\sqrt{ \phi_{|\mathcal J_{o,s}|m}\log^{4.3}(nm)\, \log\!\bigl( \bigl\|\widehat f_{\mathcal{A}_{nm}}-\overline f_{\mathcal{A}_{nm}}\bigr\|_{|\mathcal J_{o,s}|m}^{-1} \bigr) }.
\end{align*}
\color{Black}
Since
\begin{align*}
&\frac{1}{32}+\frac{1}{16}+\frac{1}{8}=\frac{7}{32}<\frac{1}{4},
\end{align*}
it follows that
\begin{align*}
&\frac{1}{4} \left\| \widehat f_{\mathcal{A}_{nm}}- \overline f_{\mathcal{A}_{nm}} \right\|_{nm}^2
 \le\;
 C_{5} \Bigg\{
 \sigma_\epsilon\,
 {\color{black}{\Lambda_{n,S}\!\left(\widehat f_{\mathcal{A}_{nm}}-\overline f_{\mathcal{A}_{nm}}\right)}}
\\
&\qquad\qquad + a_n\mathcal A_{nm} C_{\eta_n}\phi_{\frac{nm}{S}}\log^{5.3}(nm)   
+ a_n\mathcal A_{nm}^2\beta(S) + \frac{\sigma_\gamma^2}{n} \Bigg\},
\end{align*}
where \(C_{5}\) is a universal constant satisfying \(C_{5}>0\), and where we have absorbed the lower-order terms \(16\phi_{nm}\) and \(C_{\gamma,2} \sigma_\gamma \mathcal C_{n}\phi_{nm} \log^{5.3}(nm)\) into the right-hand side, using \(\mathcal A_{nm}\ge 1\).

Equivalently,
\begin{align}
&\label{eq:iid conclude in the temporal-spatial model}
 \left\|   \widehat f _{\mathcal{A}_{nm}}- \overline f _{\mathcal{A}_{nm}}\right\|_{ nm} ^2  
 \le\;
 C_{6} \Bigg\{
 \mathcal A_{nm}\,
{\color{black}{\Lambda_{n,S}\!\left(\widehat f_{\mathcal{A}_{nm}}-\overline f_{\mathcal{A}_{nm}}\right)}}
\\
&\qquad\qquad + a_n\mathcal A_{nm} C_{\eta_n}\phi_{\frac{nm}{S}}\log^{5.3}(nm)   
+ a_n\mathcal A^2_{nm}\beta(S) + \frac{\mathcal A^2_{nm}}{n} \Bigg\},
\end{align}
where \(C_{6}\) is a universal constant satisfying \(C_{6}>0\).

\color{black}
We now handle the first term on the right-hand side of \eqref{eq:iid conclude in the temporal-spatial model}. Let
\begin{align*}
&g:=\widehat f_{\mathcal A_{nm}}-\overline f_{\mathcal A_{nm}}.
\end{align*}
If \(g=0\), there is nothing to prove. Hence, in what follows, assume \(g\neq 0\). Then, by \eqref{eq:iid conclude in the temporal-spatial model},
\begin{align}
&\|g\|_{nm}^2 \le C_{6} \bigl\{ \mathcal A_{nm}\, \Lambda_{n,S}(g) + a_n\mathcal A_{nm} C_{\eta_n}\phi_{\frac{nm}{S}}   
\log^{5.3}(nm) + a_n\mathcal A_{nm}^2\beta(S) + \frac{\mathcal A_{nm}^2}{n} \bigr\}.\label{eq:iid conclude temp rewrite corrected}
\end{align}
\medskip
\noindent
We now bound \(\Lambda_{n,S}(g)\). Define
\begin{align*}
&w_{e,s}:=\frac{|\mathcal J_{e,s}|}{n}, \qquad   
w_{o,s}:=\frac{|\mathcal J_{o,s}|}{n}, \qquad s=0,\ldots,S-1.
\end{align*}
Since
\begin{align*}
&\sum_{s=0}^{S-1}|\mathcal J_{e,s}|+\sum_{s=0}^{S-1}|\mathcal J_{o,s}|=n,
\end{align*}
we have
\begin{align*}
&\sum_{s=0}^{S-1}w_{e,s}+\sum_{s=0}^{S-1}w_{o,s}=1.
\end{align*}

Moreover, by \eqref{eq:effective sample size in the temporal-spatial model}, there exists a universal constant \(C_{\Lambda,1}\) satisfying \(C_{\Lambda,1}>0\) such that
\begin{align*}
&\phi_{|\mathcal J_{e,s}|m}\le C_{\Lambda,1}\,\phi_{\frac{nm}{S}}, \qquad \phi_{|\mathcal J_{o,s}|m}   
\le C_{\Lambda,1}\,\phi_{\frac{nm}{S}}, \qquad s=0,\ldots,S-1.
\end{align*}
Therefore,
\begin{align*}
&\Lambda_{n,S}(g) \le\; C_{\Lambda,2} \sqrt{\phi_{\frac{nm}{S}}\log^{4.3}(nm)} \Bigg[   
\sum_{s=0}^{S-1} w_{e,s}\, \|g\|_{|\mathcal J_{e,s}|m} \sqrt{\log\!\bigl(\|g\|_{|\mathcal J_{e,s}|m}^{-1}\bigr)}
\qquad\qquad\qquad \notag \\
&\quad + \sum_{s=0}^{S-1} w_{o,s}\, \|g\|_{|\mathcal J_{o,s}|m} \sqrt{\log\!\bigl(\|g\|_{|\mathcal J_{o,s}|m}^{-1}\bigr)} \Bigg],
\end{align*}
where \(C_{\Lambda,2}\) is a universal constant satisfying \(C_{\Lambda,2}>0\).

Since \(\widehat f_{\mathcal A_{nm}}\) and \(\overline f_{\mathcal A_{nm}}\) are both truncated at level \(\mathcal A_{nm}\), we have
\begin{align*}
&\|g\|_\infty \le 2\mathcal A_{nm}.
\end{align*}
Hence, for every \(s\),
\begin{align*}
&\|g\|_{|\mathcal J_{e,s}|m}\le 2\mathcal A_{nm},
\qquad
\|g\|_{|\mathcal J_{o,s}|m}\le 2\mathcal A_{nm}.
\end{align*}
Since \(2\mathcal A_{nm}\ge 1\), it follows that
\begin{align*}
&\log\!\left(\|g\|_{|\mathcal J_{e,s}|m}^{-1}\right)
\le
\log\!\left(\frac{2\mathcal A_{nm}}{\|g\|_{|\mathcal J_{e,s}|m}}\right),
\end{align*}
and similarly,
\begin{align*}
&\log\!\left(\|g\|_{|\mathcal J_{o,s}|m}^{-1}\right)
\le
\log\!\left(\frac{2\mathcal A_{nm}}{\|g\|_{|\mathcal J_{o,s}|m}}\right).
\end{align*}
Consequently,
\begin{align*}
&\Lambda_{n,S}(g) \le\; C_{\Lambda,2} \sqrt{\phi_{\frac{nm}{S}}\log^{4.3}(nm)} \Bigg[   
\sum_{s=0}^{S-1} w_{e,s}\, \|g\|_{|\mathcal J_{e,s}|m} \sqrt{\log\!\bigl(\frac{2\mathcal A_{nm}}{\|g\|_{|\mathcal J_{e,s}|m}}\bigr)}
\qquad\qquad\qquad \notag \\
&\quad + \sum_{s=0}^{S-1} w_{o,s}\, \|g\|_{|\mathcal J_{o,s}|m} \sqrt{\log\!\bigl(\frac{2\mathcal A_{nm}}{\|g\|_{|\mathcal J_{o,s}|m}}\bigr)} \Bigg].
\end{align*}

Now define
\begin{align*}
&u_{e,s}:=\frac{\|g\|_{|\mathcal J_{e,s}|m}^2}{4\mathcal A_{nm}^2},
\qquad
u_{o,s}:=\frac{\|g\|_{|\mathcal J_{o,s}|m}^2}{4\mathcal A_{nm}^2}.
\end{align*}
Then \(u_{e,s},u_{o,s}\in[0,1]\), and
\begin{align*}
&\|g\|_{|\mathcal J_{e,s}|m}=2\mathcal A_{nm}\sqrt{u_{e,s}}, \qquad \log\!\bigl(\frac{2\mathcal A_{nm}}{\|g\|_{|\mathcal J_{e,s}|m}}\bigr)   
= \frac12 \log(u_{e,s}^{-1}).
\end{align*}
Hence,
\begin{align*}
&\|g\|_{|\mathcal J_{e,s}|m} \sqrt{\log\!\bigl(\frac{2\mathcal A_{nm}}{\|g\|_{|\mathcal J_{e,s}|m}}\bigr)}   
= \sqrt{2}\,\mathcal A_{nm}\sqrt{u_{e,s}\log(u_{e,s}^{-1})},
\end{align*}
and similarly for the odd blocks.

Consider the function
\begin{align*}
&\psi(u):=\sqrt{u\log(u^{-1})}, \qquad u\in(0,1],
\end{align*}
and define \(\psi(0):=0\). A direct computation shows that \(\psi\) is concave on \((0,1)\), hence concave on \([0,1]\) by continuity. Therefore, by Jensen's inequality,
\begin{align*}
&\sum_{s=0}^{S-1} w_{e,s}\sqrt{u_{e,s}\log(u_{e,s}^{-1})}
+
\sum_{s=0}^{S-1} w_{o,s}\sqrt{u_{o,s}\log(u_{o,s}^{-1})}
\le\;
\psi\!\left(
\sum_{s=0}^{S-1} w_{e,s}u_{e,s}
+
\sum_{s=0}^{S-1} w_{o,s}u_{o,s}
\right).
\end{align*}
Using
\begin{align*}
&\sum_{s=0}^{S-1} w_{e,s}u_{e,s} + \sum_{s=0}^{S-1} w_{o,s}u_{o,s}   
= \frac{1}{4\mathcal A_{nm}^2} \bigl( \sum_{s=0}^{S-1} w_{e,s}\|g\|_{|\mathcal J_{e,s}|m}^2 + \sum_{s=0}^{S-1} w_{o,s}\|g\|_{|\mathcal J_{o,s}|m}^2 \bigr)
\\
&=
\frac{\|g\|_{nm}^2}{4\mathcal A_{nm}^2},
\end{align*}
we conclude that
\begin{align*}
&\sum_{s=0}^{S-1} w_{e,s}\, \|g\|_{|\mathcal J_{e,s}|m} \sqrt{\log\!\bigl(\frac{2\mathcal A_{nm}}{\|g\|_{|\mathcal J_{e,s}|m}}\bigr)}   
+ \sum_{s=0}^{S-1} w_{o,s}\, \|g\|_{|\mathcal J_{o,s}|m} \sqrt{\log\!\bigl(\frac{2\mathcal A_{nm}}{\|g\|_{|\mathcal J_{o,s}|m}}\bigr)}
\\
&\le\;
\sqrt{2}\,\mathcal A_{nm}\,
\psi\!\left(\frac{\|g\|_{nm}^2}{4\mathcal A_{nm}^2}\right)
\\
&=\;
\sqrt{2}\,\mathcal A_{nm}
\sqrt{
\frac{\|g\|_{nm}^2}{4\mathcal A_{nm}^2}
\log\!\left(\frac{4\mathcal A_{nm}^2}{\|g\|_{nm}^2}\right)
}
\\
&=\;
\|g\|_{nm}
\sqrt{\log\!\left(\frac{2\mathcal A_{nm}}{\|g\|_{nm}}\right)}.
\end{align*}
Hence,
\begin{align*}
&\Lambda_{n,S}(g)
\le
C_{\Lambda,3} \|g\|_{nm}
\sqrt{
\phi_{\frac{nm}{S}}\log^{4.3}(nm)
\log\!\left(\frac{2\mathcal A_{nm}}{\|g\|_{nm}}\right)
},
\end{align*}
where \(C_{\Lambda,3}\) is a universal constant satisfying \(C_{\Lambda,3}>0\).

\medskip
\noindent
Substituting this bound into \eqref{eq:iid conclude temp rewrite corrected}, we obtain
\begin{align*}
&\|g\|_{nm}^2   
\le\; C_{7}\mathcal A_{nm}\|g\|_{nm} \sqrt{ \phi_{\frac{nm}{S}}\log^{4.3}(nm) \log\!\bigl(\frac{2\mathcal A_{nm}}{\|g\|_{nm}}\bigr) }
\\
&\quad+ C_{7}\Bigg\{ a_n\mathcal A_{nm} C_{\eta_n}\phi_{\frac{nm}{S}}\log^{5.3}(nm)   
+ a_n\mathcal A_{nm}^2\beta(S) + \frac{\mathcal A_{nm}^2}{n} \Bigg\}
\end{align*}
for some universal constant \(C_{7}\) satisfying \(C_{7}>0\).

\medskip
\noindent
For any \(\eta>0\), Young's inequality gives
\begin{align*}
&ab \le \eta a^2 + \frac{1}{4\eta} b^2.
\end{align*}
Applying this with
\begin{align*}
&a=\|g\|_{nm}, \qquad b   
=C_{7}\mathcal A_{nm} \sqrt{ \phi_{\frac{nm}{S}}\log^{4.3}(nm) \log\!\bigl(\frac{2\mathcal A_{nm}}{\|g\|_{nm}}\bigr) },
\end{align*}
and choosing \(\eta=\frac14\), we obtain
\begin{align*}
&C_{7}\mathcal A_{nm}\|g\|_{nm}
\sqrt{
\phi_{\frac{nm}{S}}\log^{4.3}(nm)
\log\!\left(\frac{2\mathcal A_{nm}}{\|g\|_{nm}}\right)
}
\\
\le\;&
\frac{1}{4}\|g\|_{nm}^2
+
C_{8} \mathcal A_{nm}^2
\phi_{\frac{nm}{S}}\log^{4.3}(nm)
\log\!\left(\frac{2\mathcal A_{nm}}{\|g\|_{nm}}\right)
\end{align*}
for some universal constant \(C_{8}\) satisfying \(C_{8}>0\). Therefore,
\begin{align*}
&\|g\|_{nm}^2
\le\;
\frac{1}{4}\|g\|_{nm}^2
+
C_{8} \mathcal A_{nm}^2
\phi_{\frac{nm}{S}}\log^{4.3}(nm)
\log\!\left(\frac{2\mathcal A_{nm}}{\|g\|_{nm}}\right)
\\
&+ C_{7}\Bigg\{ a_n\mathcal A_{nm} C_{\eta_n}\phi_{\frac{nm}{S}}\log^{5.3}(nm)   
+ a_n\mathcal A_{nm}^2\beta(S) + \frac{\mathcal A_{nm}^2}{n} \Bigg\}.
\end{align*}
Since \(\mathcal A_{nm}\ge 1\), the term
\begin{align*}
&a_n\mathcal A_{nm} C_{\eta_n}\phi_{\frac{nm}{S}}\log^{5.3}(nm)
\end{align*}
is bounded by a constant multiple of
\begin{align*}
&a_n\mathcal A_{nm}^2 C_{\eta_n}\phi_{\frac{nm}{S}}\log^{5.3}(nm).
\end{align*}
Hence, after enlarging constants if necessary,
\begin{align}
&\|g\|_{nm}^2 \le C_{9} \bigl\{ \mathcal A_{nm}^2 \phi_{\frac{nm}{S}}\log^{4.3}(nm) \log\!\bigl(\frac{2\mathcal A_{nm}}{\|g\|_{nm}}\bigr) \notag \\
&\quad + a_n\mathcal A_{nm}^2\phi_{\frac{nm}{S}}\log^{5.3}(nm) + a_n\mathcal A_{nm}^2\beta(S) + \frac{\mathcal A_{nm}^2}{n} \bigr\}\label{eq:pre-two-case temporal spatial}
\end{align}
for some universal constant \(C_{9}\) satisfying \(C_{9}>0\).

\medskip
\noindent
We now consider two cases.

\medskip
\noindent
If \(\|g\|_{nm}^2\le n^{-1}\), then
\begin{align*}
&\|g\|_{nm}^2 \le \frac{1}{n}\le \frac{\mathcal A_{nm}^2}{n},
\end{align*}
since \(\mathcal A_{nm}\ge 1\). Hence the desired bound follows immediately.

\medskip
\noindent
Now suppose that \(\|g\|_{nm}^2>n^{-1}\). Then \(\|g\|_{nm}>n^{-1/2}\), and therefore
\begin{align*}
&\log\!\left(\frac{2\mathcal A_{nm}}{\|g\|_{nm}}\right)
\le
\log\!\left(2\mathcal A_{nm}\sqrt n\right).
\end{align*}
Substituting this into \eqref{eq:pre-two-case temporal spatial}, we obtain
\begin{align*}
&\|g\|_{nm}^2 \le C_{9} \Bigg\{ \mathcal A_{nm}^2 \phi_{\frac{nm}{S}}\log^{4.3}(nm) \log\!\bigl(2\mathcal A_{nm}\sqrt n\bigr) \notag \\
&\quad + a_n\mathcal A_{nm}^2\phi_{\frac{nm}{S}}\log^{5.3}(nm) + a_n\mathcal A_{nm}^2\beta(S) + \frac{\mathcal A_{nm}^2}{n} \Bigg\}.
\end{align*}
Since \(\mathcal A_{nm}\) grows at most polynomially in \(nm\), there exists a universal constant \(C_{10}\) satisfying \(C_{10}>0\) such that
\begin{align*}
&\log\!\left(2\mathcal A_{nm}\sqrt n\right)\le C_{10}\log(nm),
\end{align*}
and hence
\begin{align*}
&\mathcal A_{nm}^2 \phi_{\frac{nm}{S}}\log^{4.3}(nm) \log\!\bigl(2\mathcal A_{nm}\sqrt n\bigr)   
\le C_{11}\mathcal A_{nm}^2 \phi_{\frac{nm}{S}}\log^{5.3}(nm)
\end{align*}
for some universal constant \(C_{11}\) satisfying \(C_{11}>0\). Therefore, in this case,
\begin{align*}
&\|g\|_{nm}^2 \le C_{12} \bigl\{ a_n\mathcal A_{nm}^2 \phi_{\frac{nm}{S}}\log^{5.3}(nm)   
+ a_n\mathcal A_{nm}^2\beta(S) + \frac{\mathcal A_{nm}^2}{n} \bigr\}
\end{align*}
for some universal constant \(C_{12}\) satisfying \(C_{12}>0\).

Combining the two cases, we conclude that
\begin{align}
&\bigl\| \widehat{f}_{\mathcal{A}_{nm}}-\overline{f}_{\mathcal{A}_{nm}} \bigr\|_{ nm} ^2 \le C_{13} \bigl\{ a_n\mathcal A_{nm}^2 \phi_{\frac{nm}{S}}\log^{5.3}(nm)   
+ a_n\mathcal A_{nm}^2\beta(S) + \frac{\mathcal A_{nm}^2}{n} \bigr\}\label{eq:iid conclude2 in the temporal-spatial model}
\end{align}
for some universal constant \(C_{13}\) satisfying \(C_{13}>0\).
\color{Black}

So the above implies that, up to a constant, the rate is 
\begin{align*}
&a_n(\mathcal{A}^2_{nm}+1)\max_{(p,K)\in \mathcal {P}}\bigl( \frac{S}{nm}\bigr)^{\frac{2p}{2p+K}}\log^{5.3}(nm)   
+ a_n\mathcal{A}_{nm}^2\beta(S) + \frac{\sigma_\gamma^2}{n}.
\end{align*}

Our rate is for the truncated estimator $\widehat f_{\mathcal{A}_{nm}}$.

{\bf a.} By choosing $\mathcal{A}_{nm} \asymp \max \left\{\sigma_\epsilon,\sigma_\gamma\right\}\cdot \sqrt{\log(nm)}$ based on \Cref{lemma_epsilon_gamma_infinity_bound};

{\bf b.} By Assumption 1, if the mixing rate $\beta(S) \lesssim \exp(-C_{\beta}S)$, then it suffices to choose $S = O(\log(n))$.

{\bf a} and {\bf b} imply the rate is
\begin{align*}
&a_n\bigl(\sigma_\epsilon^2+\sigma_\gamma^2 +1\bigr)\max_{(p,K)\in \mathcal {P}}\bigl( \frac{1}{nm}\bigr)^{\frac{2p}{2p+K}}\log^{7.3}(nm)   
+ a_n\frac{\bigl(\sigma_\epsilon^2+\sigma_\gamma^2\bigr)\log(nm)}{n} + \frac{\sigma_\gamma^2}{n},
\end{align*}

which gives the rate in \eqref{rate_in_theorem_temporal_spatial}. The rate holds by selecting the network architecture in \eqref{network_arch} to guarantee the approximation error.
 \end{proof}
 
 \begin{corollary}[Polynomial Decay for Mixing Coefficient] \label{mixing_rate_coefficient}
 Under the setting in Theorem 1 in the Main Manuscript, If $\beta(S) \lesssim  S^{-\alpha} $ for some $\alpha >0$. 
We can achieve the rate
\begin{align*}  
a_n\frac{\sigma_\epsilon^2 + \sigma_\gamma^2 +1}{n^\alpha }\log^{6.3}(nm) + \frac{\sigma_\gamma^2}{n},\end{align*} with $S=n$ and $ m\ge \max_{(p,K)\in \mathcal{P}} n^{\alpha  ( \frac{2p+K}{2p} )}.$
If $ m\le \max_{(p,K)\in \mathcal{P}}n^{\alpha  ( \frac{2p+K}{2p} )},$
we achieve the rate 
\begin{align*}
&a_n(nm)^{\frac{-2p}{2p+K}}\log^{6.3}(nm) + \frac{\sigma_\gamma^2}{n},
\end{align*} with $S = (nm)^{\frac{2p}{2p+2p\alpha + \alpha K}}$ and $\alpha\to \infty$. 
 \end{corollary}
 \begin{proof}
     {\bf Case a.}  If $\beta(S) \lesssim  S^{-\alpha} $ for some $\alpha >0$,  and  
\begin{align*}
&m\le n^{\alpha  ( \frac{2p+K}{2p} )}.
\end{align*}
Then  $S = \max_{(p,K)\in \mathcal{P}}(nm)^{\frac{2p}{2p+2p\alpha + \alpha K}} \le n  $ is a valid choice of $S$ and this implies the error bound is 
\begin{align*}
&a_n\left(\sigma_\epsilon^2+\sigma_\gamma^2 + 1\right)\max_{(p,K)\in \mathcal{P}}(nm)^{\frac{-2p\alpha }{2p+2p\alpha + \alpha K}}\log^{6.3}(nm) \\ &+ a_n\left(\sigma_\epsilon^2+\sigma_\gamma^2\right)\max_{(p,K)\in \mathcal{P}}(nm)^{\frac{-2p\alpha }{2p+2p\alpha + \alpha K}}\log(nm)+\frac{\sigma_\gamma^2}{n}.
\end{align*}  

When $\alpha\to \infty$, so that the data are independent, it follows that this rate goes to 
\begin{align*}
&\max_{(p,K)\in \mathcal{P}}(nm)^{\frac{-2p}{2p+K}}\log^{6.3}(nm) + \frac{\sigma_\gamma^2}{n}.
\end{align*}
\\
{\bf Case b.} If $\beta(S) \lesssim  S^{-\alpha} $ and  
\begin{align*}
&m\ge \max_{(p,K)\in \mathcal{P}} n^{\alpha  ( \frac{2p+K}{2p} )}.
\end{align*}
Then  $S =  n  $ is a valid choice and this implies the error bound is 
\begin{align*}
&a_n\bigl(\sigma_\epsilon^2 + \sigma_\gamma^2 +1\bigr) \cdot \max_{(p,K)\in \mathcal{P}}m^{\frac{-2p}{ 2p+K}}\log^{6.3}(nm)   
+ a_n\frac{\sigma_\epsilon^2 + \sigma_\gamma^2 }{n^\alpha}\log^2(nm) + \frac{\sigma_\gamma^2}{n}
\\
&\le a_n \bigl(\sigma_\epsilon^2 + \sigma_\gamma^2 +1\bigr) \cdot \bigl(\frac{1}{n^\alpha}\log^{6.3}(nm)   
+ \frac{1}{n^\alpha}\log^2(nm)\bigr) + \frac{\sigma_\gamma^2}{n}
\\
&\le a_n\frac{\sigma_\epsilon^2 + \sigma_\gamma^2 +1}{n^\alpha }\log^{6.3}(nm) + \frac{\sigma_\gamma^2}{n}
.
\end{align*}
 \end{proof}

\section{Additional Technical Results for Theorem 1}\label{beta mixing proof temporal-spatial model main}

\subsection{Deviation Bounds for Independent Case}
\label{proof temporal-spatial model indepdent}
The proof below assumes the data are independent, meaning the $\beta$ coefficients are all zero. We focus on the independent case here, and address the dependent case later in Section~\ref{proof temporal-spatial model depdent}.

Let $g \in \mathcal L_2$ and recall
\begin{align*}
&\left\|g\right\|_{nm}^2 =\frac{1}{n}\sum_{i=1}^n \frac{1}{m_i} \sum_{j=1}^{m_i} g^2 \left(x_{ij}\right)=\sum_{i=1}^n  \sum_{j=1}^{m_i} \frac{1}{nm_i}g^2 \left(x_{ij}\right).
\end{align*}

Let $m$ be the is the harmonic mean of $\left\{m_i\right\}_{i=1}^{n}$  defined as
\begin{align*}
&m=\left(\frac{1}{n} \sum_{i=1}^n \frac{1}{m_i}\right)^{-1}.
\end{align*} By the assumption that \begin{align*}
&c m \leq m_i \leq C m,
\end{align*} it follows that for any $i=1,2,\cdots, n$, 
\begin{align*}
&\frac{1}{Cm} \leq \frac{1}{m_i} \leq \frac{1}{cm}.
\end{align*}
 
\subsubsection{Deviation bounds for $\epsilon$}
In this section, we assume $\epsilon$ is a vector of  independent random variables. The general case is given in Section~\ref{deviation_bound_dependent_epsilon}. 
Denote 
\begin{align*}
&D_{nm}\left(g\right) =  \sum_{i=1}^{n}\sum_{j=1}^{m_i} \frac{1}{\sqrt{nm_i}}g \left(x_{ij}\right) \epsilon_{ij}.
\end{align*} 
In this section, the analysis is made conditioning on $ \left\{x_{ij} \right\}_{i=1,j=1}^{n,m_i}$. We denote the conditional probability conditioning on $ \left\{x_{ij} \right\}_{i=1,j=1}^{n,m_i}$ by $\p_{\cdot \mid \left\{x_{ij} \right\}_{i=1,j=1}^{n,m_i}}$, and we use    
$\p_{\cdot \mid x}$ as its shorthand notation.

\begin{corollary} 
\label{coro:sub-gaussian deviation mn}
Let  $g,f  \in \lt$ be two deterministic functions. Then there exists a constant $c>0$ such that  
\begin{align}
&\p_{\cdot \mid x} \left( D_{nm}(g) \ge t \right)  \le \exp   \left(  -      \frac{c t^2}{\sigma_\epsilon^2 \cdot \left\|g\right\|_{nm} ^2} \right),
\\
&\p_{\cdot \mid x} \bigl( D_{nm}(g) -D_{nm}(f) \ge t \bigr)   
\le \exp \bigl(- \frac{c t^2}{\sigma_\epsilon^2 \cdot \bigl\|f-g\bigr\|_{nm} ^2} \bigr) .\label{concentration_D_nm_dif}
\end{align} 
\end{corollary} 
\begin{proof}
Since $\left\{\epsilon_{ij}\right\}_{i=1, j=1}^{n, m_i}$ are i.i.d. from a sub-Gaussian distribution with sub-Gaussian parameter $\sigma_\epsilon$. Conditioning on $\left\{x_{ij}\right\}_{i=1, j=1}^{n, m_i}$, the first bound is a direct consequence of classical Hoeffding bound. To see this, notice that for $\left\{g(x_{ij})\epsilon_{ij}\right\}_{i=1, j=1}^{n, m_i}$ are independent sub-Gaussian variables with sub-Gaussian parameter $|g(x_{ij})|\sigma_\epsilon$. This implies that $D_{nm}(g)=\sum_{i=1}^{n}\sum_{j=1}^{m_i}\frac{1}{\sqrt{nm_i}} g \left(x_{ij}\right) \epsilon_{ij} $
is sub-Gaussian with sub-Gaussian parameters \begin{align*}
&\sqrt{\sigma_\epsilon^2\sum_{i=1}^n\sum_{j=1}^{m_i}\frac{1}{nm_i} g^2\left(x_{ij}\right)}=\sigma_\epsilon\left\|g\right\|_{nm}.
\end{align*}

The second bound holds from the linearity of $D_{nm}(g)$ with respect to g: for $i=1, \dots, n$, $j=1,\dots, m_i$, conditioning on $x_{ij}$, $\left[g \left(x_{ij}\right)-f \left(x_{ij}\right)\right] \epsilon_{ij}$ are independent sub-Gaussian variables with sub-Gaussian parameter $\left|\left(g-f\right)\left(x_{ij}\right)\right|\sigma_\epsilon$. Then conditioning on $\left\{x_{ij}\right\}_{i=1, j=1}^{n, m_i}$,
\begin{align*}
&D_{nm}(g)-D_{nm}(f)   
=\sum_{i=1}^{n}\sum_{j=1}^{m_i}\frac{1}{\sqrt{nm_i}} \bigl[g \bigl(x_{ij}\bigr)-f \bigl(x_{ij}\bigr)\bigr] \epsilon_{ij}
\end{align*}
is sub-Gaussian with parameter 
\begin{align*}
&\sqrt{\sigma_\epsilon^2\sum_{i=1}^n\sum_{j=1}^{m_i}\frac{1}{nm_i}\left[(g-f)(x_{ij})\right]^2} 
=\sigma_\epsilon\left\|g-f\right\|_{nm}.
\end{align*}

The desired inequalities now follow from standard sub-Gaussian tail bounds. In particular,
\begin{align*}
&\p_{\cdot \mid x} \bigl( D_{nm}(g) \ge t \bigr)   
\le \exp \bigl( - \frac{c_0t^2}{\sigma_\epsilon^2 \cdot \sum_{i=1}^{n}\sum_{j=1}^{m_i}\frac{1}{nm_i}g^2(x_{ij})} \bigr)
\\
&\le \exp   \left(  -      \frac{c_0t^2}{\sigma_\epsilon^2 \cdot \left\|g\right\|_{nm} ^2} \right),
\\
&\p_{\cdot \mid x} \bigl( D_{nm}(g) -D_{nm}(f) \ge t \bigr)   
\le \exp \bigl( - \frac{c_0t^2}{\sigma_\epsilon^2 \cdot \sum_{i=1}^{n}\sum_{j=1}^{m_i}\frac{1}{nm_i}\bigl[g(x_{ij})-f(x_{ij})\bigr]^2} \bigr)
\\
&\le \exp \left(- \frac{c_0t^2}{\sigma_\epsilon^2 \cdot \left\|f-g\right\|_{nm} ^2} \right).
\end{align*} 

\end{proof}

\
\\
Next, for $\lambda>0$, we introduce the notation
\begin{align*}
&B_{nm}\left(\lambda\right) = \left\{g : \left\|g \right\|_{nm} \le \lambda \right\}.
\end{align*}

\begin{corollary} 
\label{coro:complexity bound 1 mn} 
Let $\mathcal F$ be a function class such that
\begin{align}
&\log \mathcal N \bigl(\delta, \mathcal F, \bigl\| \cdot\bigr\|_{nm} \bigr)   
\le C^{\prime} nm \phi_{nm} \bigl( \log^{4.3}(nm) \log \bigl(\delta^{-1}\bigr) \bigr).\label{eq:metric entropy assumption in epsilon 1 mn}
\end{align}

For  
\begin{align*}
&t\ge  \lambda \sqrt {2C^{\prime}\phi_{nm} nm \left\{ \log^{4.3}(nm)\log \left(\lambda^{-1}\right)\right\}}
\end{align*} with $\lambda \leq \frac{1}{4}$,
it holds that 
\begin{align*}
&\p_{\cdot \mid x}\left(\sup_{g \in \mathcal F  \cap B_{nm} \left(\lambda\right) } D_{nm}(g) \ge \sigma_\epsilon \cdot t   \right)\le  \exp \left( -\frac{c't^2}{\lambda^2}\right),
\end{align*} for some positive constants $c'$ and $C'$.
\end{corollary}

\begin{proof}
For $ l\in \mathbf{Z}^+$, let $\mathcal{T}_l $ be the $2^{-l} \lambda  $ covering set of $ \mathcal{F} \cap B_{nm} \left(\lambda\right)$. Then by setting $\delta=2^{-l}\lambda$ in \eqref{eq:metric entropy assumption in epsilon 1 mn},  
\begin{align*}
&\bigl|\mathcal{T}_l \bigr|   
\le \exp\bigl(C^{\prime} nm \phi_{nm} \log^{4.3}(nm)\bigl\{ l\log(2) + \log(\lambda^{-1}) \bigr\} \bigr).
\end{align*}
Denote $ \psi_l (g) \in \mathcal T_l$ be such that 
\begin{align*}
&\left\| g-\psi_l(g) \right\|_{nm} \leq 2^{-l} \lambda .
\end{align*}
\
\\
{\bf Step 1.} 
Observe that by union bound, 
\begin{align}
&\p_{\cdot \mid x} \left( \sup_{ g \in\mathcal F \cap B_{nm} \left(\lambda\right) }  D_{nm}\left(g\right) \ge \sigma_\epsilon \cdot t  \right) 
\le \p_{\cdot \mid x}\left( \sup_{ g\in \mathcal F \cap B_{nm}\left(\lambda\right) }  D_{nm} \left( \psi_0 \left(g\right)\right) \ge \sigma_\epsilon \cdot t/2  \right) \nonumber
\\
&+ \p_{\cdot \mid x}\bigl( \sup_{g\in\mathcal F \cap B_{nm}(\lambda) } D_{nm}(g) -D_{nm} ( \psi_0(g))   
\ge \sigma_\epsilon \cdot t/2 \bigr).\label{part_1}
\end{align}
Since  $D_{nm}(\psi_0(g))$ is sub-Gaussian with parameter $\left\|\psi_0(g)\right\|_{nm}\sigma_\epsilon \le \lambda \sigma_\epsilon $, it follows that 
\begin{align}
&\p_{\cdot \mid x}\left( \sup_{ g\in \mathcal F \cap B_{nm}(\lambda) }  D_{nm}( \psi_0(g)) \ge \sigma_\epsilon \cdot t/2  \right)   
\le \p_{\cdot \mid x}\left( \sup_{ \psi_0(g) \in \mathcal T_0 }  D_{nm} ( \psi_0(g)) \ge \sigma_\epsilon \cdot t/2  \right) \nonumber \\
&\le   \sum_{\psi_0(g) \in \mathcal T_0 }\p_{\cdot \mid x} \left(D_{nm}( \psi_0(g)) \ge \sigma_\epsilon \cdot t/2  \right) \nonumber \\
&\le   \exp\left( C^{\prime} nm \phi_{nm} \left\{ \log^{4.3}(nm) \log(\lambda^{-1}) \right\}  - \frac{c_0t^2}{\lambda^2} \right) \nonumber \\
&\le    \exp\left( - \frac{c_0t^2}{2 \lambda^2} \right),\label{part_2}
\end{align}
where the last inequality follows if $t\ge  \lambda \sqrt{2C'\phi_{nm} nm\left\{ \log^{4.3}(nm) \log(\lambda^{-1})\right\}}.$

{\bf Step 2.}
Note that 
\begin{align*}
&D_{nm}(\psi_0(g) )-D_{nm}(g)   
= \sum_{l=1}^\infty D_{nm}( \psi _{l-1}(g) ) - D_{nm}(\psi _l (g) ).
\end{align*}
Therefore, for $\eta_l\ge 0 $ and $\sum_{l=1}^\infty \eta_l\le 1 $, 
\begin{align*}
&\p_{\cdot \mid x}\left(  \sup_{g\in\mathcal F \cap B_{nm}(\lambda) }    D_{nm}(g)   -D_{nm} ( \psi_0(g)) \ge \sigma_\epsilon \cdot t/2 \right)
\\
&\leq \p_{\cdot \mid x}\bigl(\exists l \in Z^{+}, \text{such that}\sup_{g\in\mathcal F \cap B_{nm}(\lambda) } D_{nm}(\psi_{l} (g))   
-D_{nm} ( \psi_{l-1}(g)) \ge \eta_l \sigma_\epsilon \cdot t/2 \bigr)
\\
&= \p_{\cdot \mid x}\bigl(\bigcup_{l=1}^{\infty}\bigl\{\sup_{g\in\mathcal F \cap B_{nm}(\lambda) } D_{nm}(\psi_{l} (g)) -D_{nm} ( \psi_{l-1}(g))   
\ge \eta_l\sigma_\epsilon \cdot t/2 \bigr\}\bigr)
\\
&\le \sum_{l=1}^\infty \p_{\cdot \mid x} \bigl( \sup_{g\in\mathcal F \cap B_{nm}(\lambda) } D_{nm}(\psi_{l} (g)) -D_{nm} ( \psi_{l-1}(g))   
\ge \eta_l\sigma_\epsilon \cdot t/2 \bigr)
\\
&= \sum_{l=1}^\infty\p_{\cdot \mid x} \bigl( \sup_{\psi_{l-1}(g) \in \mathcal{T}_{l-1}, \psi_{l}(g) \in \mathcal{T}_{l}} D_{nm}(\psi_{l} (g))   
-D_{nm} ( \psi_{l-1}(g)) \ge \eta_l\sigma_\epsilon \cdot t/2 \bigr)
\\
&\le \sum_{l=1}^\infty \sum_{\psi_{l-1}(g) \in \mathcal T_{l-1}, \psi_{l}(g) \in \mathcal T_{l}}\p_{\cdot \mid x} \bigl(D_{nm}(\psi_{l-1}(g)-\psi_{l}(g))   
\ge \eta_l\sigma_\epsilon \cdot t/2 \bigr).
\end{align*}
Since
\begin{align*}
&\|\psi_{l-1} (g) -\psi_l (g) \|_{nm} \le \| \psi_{l-1} (g)- g \|_{nm} + \| \psi_l (g)   
-g \|_{nm} \notag \\
&\quad \le 2^{-l+1} \lambda+ 2^{-l}\lambda =3\cdot 2^{-l}\lambda \leq 2^{-l+2}\lambda,
\end{align*}
and
\begin{align*}
&\bigl|\mathcal{T}_{l-1}\bigr| \le \bigl|\mathcal{T}_{l}\bigr|   
\le \exp \bigl(C^{\prime} nm \phi_{nm}\log^{4.3}(nm)\bigl\{l \log (2)+\log \bigl(\lambda^{-1}\bigr)\bigr\}\bigr),
\end{align*}
by plugging $\|\psi_{l-1} (g) -\psi_l (g) \|_{nm} \le 2^{-l+2} \lambda $ in \eqref{concentration_D_nm_dif}, 
\begin{align*}
&\p_{\cdot \mid x} \left(  \sup_{g\in\mathcal F \cap B_{nm}(\lambda) }    D_{nm}(g)   -D_{nm} ( \psi_0(g)) \ge \sigma_\epsilon \cdot  t/2  \right)
\\
&\le \sum_{l=1}^\infty \sum_{\psi_{l-1}(g) \in \mathcal T_{l-1}, \psi_{l}(g) \in \mathcal T_{l}}\p_{\cdot \mid x} \bigl(D_{nm} (\psi_{l-1}(g)-\psi_{l}(g))   
\ge \eta_l\sigma_\epsilon \cdot t/2 \bigr)
\\
&\le \sum_{l=1}^\infty2\exp \bigl(C^{\prime} nm \phi_{nm}\log^{4.3}(nm)\bigl\{l \log (2)   
+\log \bigl(\lambda^{-1}\bigr)\bigr\}\bigr)\exp \bigl(-\frac{c^{\prime \prime} 2^{2l}\eta_l^2 t^2}{\lambda^2}\bigr),
\end{align*} for some $c^{\prime \prime} >0$.\ 
 \\ 
If $C^{\prime} nm \phi_{nm} \log^{4.3}(nm)\left\{  l\log(2) + \log(\lambda^{-1}) \right\} \le  \frac{c^{\prime \prime} 2^{2l}  \eta_l^2t^2}{2\lambda^2} $, or simply
\begin{align}
\label{inequality_for_eta}
\eta_l^2 \ge \frac{2C^{\prime}\lambda^2 nm  \phi_{nm} \log^{4.3}(nm) \left\{ l\log(2) + \log(\lambda^{-1}) \right\}  }{ c^{\prime \prime}t^2 2^{2l }},\end{align}
following the previous derivation, we have 
\begin{align*}
&\p_{\cdot \mid x} \left(  \sup_{g\in\mathcal F \cap B_{nm}(\lambda) }    D_{nm}(g)   -D_{nm} ( \psi_0(g)) \ge \sigma_\epsilon \cdot  t/2  \right)
\\
&\le \sum_{l=1}^\infty2\exp \bigl(C^{\prime} nm \phi_{nm}\log^{4.3}(nm)\bigl\{l \log (2)   
+\log \bigl(\lambda^{-1}\bigr)\bigr\}\bigr)\exp \bigl(-\frac{c^{\prime \prime} 2^{2l}\eta_l^2 t^2}{\lambda^2}\bigr)
\\
&\le  \sum_{l=1}^\infty2\exp \left(-\frac{c^{\prime \prime} 2^{2l}\eta_l^2 t^2}{2\lambda^2}\right).
\end{align*}

Since $t\ge \lambda \sqrt{2C^{\prime}\phi_{nm} nm\log^{4.3}(nm) \log(\lambda^{-1})}$, we have 

\begin{align*}
  &  \frac{2C^{\prime}\lambda^2 nm  \phi_{nm} \log^{4.3}(nm) \left\{ l\log(2) + \log(\lambda^{-1}) \right\}  }{ c^{\prime \prime}t^2 2^{2l }} \\
  \leq  &\frac{1}{c^{\prime \prime}2^{2l}} + \frac{l \log(2)}{\log(\lambda^{-1}) c^{\prime \prime}2^{2l}} = \frac{l}{c^{\prime \prime}2^{2l}}\left(\frac{1}{l} + \frac{\log(2)}{\log(\lambda^{-1})}\right) \\
  \leq & \frac{C_1l}{c^{\prime \prime}2^{2l}},
\end{align*}
for some constant $\frac{1}{2} > C_1>0$, where the last inequality holds by $\lambda \leq \frac{1}{4}$.

Thus, to satisfy~\eqref{inequality_for_eta}, 
it suffices to choose 
$\eta_l^2 = C_1  l 2^{-2l} $ for some positive constant $C_1$. Furthermore, with such $\eta_l$, we have $\sum_{l=1}^\infty \eta_l = \sqrt{C_1} \sum_{l=1}^\infty l^{1/2} 2^{-l}.$ For large \(l\), \(2^{-l}\) decreases exponentially faster than \(l^{1/2}\) increases, so the series $\sum_{l=1}^\infty l^{1/2} 2^{-l}$ converges.  And further by $ \frac{1}{2} >C_1$, we have $\sum_{l=1}^\infty \eta_l \leq 1$. Thus 
\begin{align}
& \p_{\cdot \mid x}\left(  \sup_{g\in\mathcal F \cap B_{nm}(\lambda) }   D_{nm}(g)  -D_{nm} ( \psi_0(g)) \ge \sigma_\epsilon \cdot t/2  \right) \nonumber \\
&\le  \sum_{l=1}^\infty  2\exp\left(   - \frac{c^{\prime \prime}   l t^2}{2\lambda^2} \right) \le   \exp( -\frac{c^{\prime \prime \prime}t^2}{2\lambda^2}).\label{part_3}
\end{align} for some positive constant $c^{\prime \prime \prime}$.
By \eqref{part_1}, \eqref{part_2}, \eqref{part_3}, we have 
\begin{align*}
&\p_{\cdot \mid x}\left(\sup_{g \in \mathcal F  \cap B_{nm} \left(\lambda\right) } D_{nm}(g) \ge \sigma_\epsilon \cdot t   \right)\le  \exp \left( -\frac{c't^2}{\lambda^2}\right),
\end{align*} for some positive constant $c^{\prime}$.
\end{proof}

\begin{corollary} \label{coro:complexity bound 2 mn}
Let $\mathcal F$ be a function class such that
\begin{align}
&\log \mathcal N (\delta, \mathcal F , \| \cdot\|_{nm} )   
\le C^{\prime}nm\phi_{nm} \bigl( \log^{4.3}(nm)\log(\delta^{-1}) \bigr)\label{eq:metric entropy assumption in epsilon 2 mn}
\end{align}
and that $ \mathcal F \subset B_\infty\left(\tau\right)$, where $\tau \leq \frac{1}{2}$.
Then it holds that for any $s\ge 1$,
\begin{align*}
&\p_{\cdot \mid x}\bigl(\sup_{g \in \mathcal F } \frac{D_{mn}\bigl(g\bigr)}{\bigl\|g\bigr\|_{nm} \sqrt {nm\phi_{nm} \bigl\{ \log^{4.3}(nm) \log( \|g\|_{nm}^{-1})\bigr\}}}   
\ge C_1\sigma_\epsilon \cdot s \bigr)
\\
&\le  C_2\exp( - c s ^2 nm\phi_{nm}  \log^{4.3}(nm)  ).
\end{align*}
\end{corollary}
\begin{proof}
Since 
$ \mathcal F \subset B_\infty \left(\tau \right)\subset B_{nm}\left(\tau \right) $, it holds that 
\begin{align*}
&\mathcal F \cap B_{nm} \left(\tau\right) = \mathcal F .
\end{align*}
Therefore, it suffices to show that  for any $s\ge 1$,
\begin{align*}
&\p_{\cdot \mid x}\bigl(\sup_{g \in \mathcal F \cap B_{nm} \bigl(\tau\bigr) }   
\frac{D_{mn} \bigl(g\bigr)}{\bigl\|g\bigr\|_{nm} \sqrt {nm\phi_{nm} \bigl\{ \log^{4.3}(nm) \log\bigl(\|g\|_{nm}^{-1}\bigr)\bigr\}}}   
\ge C_1 \sigma_\epsilon \cdot s \bigr)
\\
&\le  C_2\exp( - c s ^2 nm\phi_{nm}  \log^{4.3}(nm)  )  .
\end{align*}

Choose $\lambda = 2^{-l+1} \tau$ in \Cref{coro:complexity bound 1 mn}. For a positive constant $C_1'$ that depends on $C'$, if $s\ge 1 $ and  
\begin{align*}
&t \ge C_1' s2^{-l+1}\tau   
\sqrt {nm\phi_{nm} \bigl\{ \log^{4.3}(nm) \log\bigl((2^{-l+1}\tau)^{-1}\bigr)\bigr\}},
\end{align*}
it holds that 
\begin{align*}
&\p_{\cdot \mid x}\bigl(\sup_{g \in \mathcal F \cap B_{nm}(2^{-l+1}\tau) } D_{nm}(g) \ge C_1' \sigma_\epsilon \cdot s2^{-l+1} \tau   
\sqrt {nm\phi_{nm} \bigl\{ \log^{4.3}(nm)\log(\tau^{-1}2^{l-1})\bigr\}} \bigr)
\\
&\le \exp\left( -c''s^2nm\phi_{nm} \left\{ \log^{4.3}(nm) \log(\tau^{-1}2^{l-1}) \right\}\right)
\\
&\le \exp\bigl( -c''s^2nm\phi_{nm} \log^{4.3}(nm)l \log(2)\bigr)   
:=\exp\bigl( -c'''s^2nm\phi_{nm} \log^{4.3}(nm)l\bigr),
\end{align*} for some positive constants $c''$ and $c'''$, 
where the second inequality holds since $\tau \leq \frac{1}{2}$, so $\tau^{-1}2^{l-1}\geq 2^l$. 
Thus,
\begin{align*}
&\p_{\cdot \mid x}\left(
\sup_{g \in \mathcal{F} \cap \left\{ f : 2^{-l}\tau \le \|f\|_{nm} \le 2^{-l+1}\tau \right\} } 
D_{nm}(g) \right.
\\
&\quad  \geq C_1' \sigma_\epsilon \cdot s 2^{-l+1} \tau   
\sqrt{nm \phi_{nm} \bigl( \log^{4.3}(nm) \log(\tau^{-1} 2^{l-1}) \bigr)} \bigr)
\\
&\leq \exp\left( -c''' s^2 nm \phi_{nm} \log^{4.3}(nm) l \right).
\end{align*}

Note that if  
$g\in \{  f: 2^{-l}\tau \le \|f\|_{nm}\le 2^{-l+1}\tau \}$, then  $\|g\|^{-1}_{nm} \ge \tau^{-1}2^{l-1}$ and $2^{-l+1}\tau \le 2\left\|g\right\|_{nm}$. It follows that 
\begin{align*}
&\p_{\cdot \mid x} \left(\sup_{g \in \mathcal F  \cap \{  f: 2^{-l}\tau \le \|f\|_{nm}\le 2^{-l+1}\tau \} } D_{nm}(g) \right.
\\
&\quad  \ge C_1' \sigma_\epsilon \cdot s2^{-l+1}\tau   
\sqrt {nm\phi_{nm} \bigl\{ \log^{4.3}(nm) \log(\tau^{-1}2^{l-1})\bigr\}} \bigr)
\\
&= \p_{\cdot \mid x} \bigl(\sup_{g \in \mathcal F \cap \{ f: 2^{-l}\tau \le \|f\|_{nm}\le 2^{-l+1}\tau \} }   
\frac{D_{mn}\bigl(g\bigr)}{2^{-l+1}\tau \sqrt {nm\phi_{nm} \bigl\{ \log^{4.3}(nm) \log(\tau^{-1}2^{l-1})\bigr\}}}   
\ge C_1' \sigma_\epsilon \cdot s \bigr)
\\
&\geq \p_{\cdot \mid x} \bigl(\sup_{g \in \mathcal F \cap \{ f: 2^{-l}\tau \le \|f\|_{nm}\le 2^{-l+1}\tau \} }   
\frac{D_{mn}\bigl(g\bigr)}{2\bigl\|g\bigr\|_{nm} \sqrt {nm\phi_{nm} \bigl\{ \log^{4.3}(nm) \log(\|g\|^{-1}_{nm})\bigr\}}}   
\ge C_1' \sigma_\epsilon \cdot s \bigr)
\\
&= \p_{\cdot \mid x} \bigl(\sup_{g \in \mathcal F \cap \{ f: 2^{-l}\tau \le \|f\|_{nm}\le 2^{-l+1}\tau \} }   
\frac{D_{mn}\bigl(g\bigr)}{\bigl\|g\bigr\|_{nm} \sqrt {nm\phi_{nm} \bigl\{ \log^{4.3}(nm) \log(\|g\|^{-1}_{nm})\bigr\}}}   
\ge 2C_1' \sigma_\epsilon \cdot s \bigr).
\end{align*}
Thus, 
\begin{align*}
&\p_{\cdot \mid x}\bigl(\sup_{g \in \mathcal F \cap \{ f: 2^{-l}\tau \le \|f\|_{nm}\le 2^{-l+1}\tau \} }   
\frac{D_{mn}\bigl(g\bigr)}{\|g\|_{nm} \sqrt {nm\phi_{nm} \bigl\{ \log^{4.3}(nm)\log\bigl( \| g\|_{nm}^{-1}\bigr)\bigr\}}}   
\ge C_1 \sigma_\epsilon \cdot s \bigr)
\\
&\le \exp\left( -c''''s^2nm\phi_{nm}   \log^{4.3}(nm)l\right).
\end{align*}
Therefore,
\begin{align*}
&\p_{\cdot \mid x}\bigl(\sup_{g \in \mathcal F \cap B_{nm} \bigl(\tau\bigr) }   
\frac{D_{nm}(g)}{\|g\|_{nm} \sqrt {nm\phi_{nm} \bigl\{ \log^{4.3}(nm)\log(\|g\|_{nm}^{-1})\bigr\}}}   
\ge C_1 \sigma_\epsilon \cdot s \bigr)
\\
&= \ \p_{\cdot \mid x}\Bigg(\exists l \in Z^{+}, \text{such that}
\\
&\,\sup_{g \in \mathcal F \cap \bigl\{ f: 2^{-l}\tau \le \|f\|_{nm}\le 2^{-l+1}\tau \bigr\} }   
\frac{D_{nm}(g)}{\|g\|_{nm} \sqrt {nm\phi_{nm} \bigl\{ \log^{4.3}(nm)\log(\|g\|_{nm}^{-1})\bigr\}}}   
\ge C_1 \sigma_\epsilon \cdot s \Bigg)
\\
&\le \sum_{l=1}^\infty \p_{\cdot \mid x}\bigl(\sup_{g \in \mathcal F \cap \{ f: 2^{-l}\tau\le \|f\|_{nm}\le 2^{-l+1}\tau \} }   
\frac{D_{nm}(g)}{\|g\|_{nm} \sqrt {nm\phi_{nm} \bigl\{ \log^{4.3}(nm)\log(\|g\|_{nm}^{-1})\bigr\}}}
\\
&\quad \left. \ge  C_1 \sigma_\epsilon \cdot  s    \right)
\\
&\le \sum_{l=1}^\infty \exp\left( -c''''s^2nm\phi_{nm}   \log^{4.3}(nm)l\right)
\\
&\le  C_2\exp( - c s ^2 nm\phi_{nm}  \log^{4.3}(nm)  ),
\end{align*}
where the first identity follows from a peeling argument, noticing that $\bigcup_{l=1}^{\infty}\mathcal F  \cap \{  f: 2^{-l}\tau \le \|f\|_{nm}\le 2^{-l+1}\tau  \}=\mathcal F  \cap \bigcup_{l=1}^{\infty}\{f: 2^{-l}\tau \le \|f\|_{nm}\le 2^{-l+1}\tau \}=\mathcal F  \cap B_{nm}(\tau)$.
\end{proof}

\newpage 
\subsubsection{Deviation bounds  for  $\xi $}
Suppose $\{\xi_{ij}\}_{i=1, j=1}^{n, m_i}$ are i.i.d. Rademacher   random variables and denote 
\begin{align*}
&H(g) = \sum_{i=1}^n \sum_{j=1}^{m_i} \frac{1}{\sqrt{n m_i}} g^2 (x_{ij}) \xi_{ij}  .
\end{align*} 
In \Cref{coro:sub-gaussian deviation sigma mn}, \Cref{coro:sigma complexity bound 1 mn} and \Cref{coro:sigma complexity bound 2 mn}, the  analysis is  conditioning on  $ \{x_{ij}\}_{i=1, j=1}^{n, m_i}$. We denote the conditional probability conditioning on $\left\{x_{ij}\right\}_{i=1, j=1}^{n, m_i}$ by $\p_{\cdot \mid x}$.
\begin{corollary} 
\label{coro:sub-gaussian deviation sigma mn}
Let  $g,f  \in \lt$ be two deterministic functions such that $\|g\|_\infty  \le \tau, \|f\|_\infty   \le \tau$.  Then there exists a constant $c>0$ such that  
\begin{align*} 
& \p_{\cdot \mid x} \left( H \left(g\right)     \ge    t        \right)  \le   \exp   \left(  -      \frac{c t^2}{\tau^2\left\|g\right\|_{mn}^2 }    \right),
\\ 
&  \p_{\cdot \mid x} \left( H \left(g\right)  -H  \left(f\right)     \ge      t       \right)  \le  \exp\left(- \frac{c t^2}{\tau^2\|f-g\|_{mn}^2} \right) .
\end{align*} 
\end{corollary} 
\begin{proof}
For the first bound,  it suffices to note that 
$\left\{g^2 (x_{ij}) \xi_{ij}\right\}_{i=1, j=1}^{n, m_i}$'s are independent sub-Gaussian with parameter $g^2(x_{ij})   $. So  $H(g)$ is sub-Gaussian with parameter 
\begin{align*}
&\sqrt{\sum_{i=1}^n \sum_{j=1}^{m_i} \frac{1}{n m_i}  g^4(x_{ij})} \le \tau \left\|g\right\|_{mn},
\end{align*}
since $\left\|g\right\|_\infty  \le \tau$. 
For the second bound, it suffices to note that 
$\left\{\left(f^2 (x_{ij}) -g^2(x_{ij})\right) \xi_{ij}\right\}_{i=1, j=1}^{n, m_i}$ are independent sub-Gaussian with parameter $\left\{|f^2(x_{ij}) -g^2 (x_{ij})|\right\}_{i=1, j=1}^{n, m_i}$. So  $H(f ) -H(g)$ is sub-Gaussian with parameter 
\begin{align*}
&\sqrt{\sum_{i=1}^n \sum_{j=1}^{m_i} \frac{1}{n m_i} (f^2(x_{ij}) -g^2 (x_{ij}))^2}   
= \sqrt{\sum_{i=1}^n \sum_{j=1}^{m_i} \frac{1}{n m_i} (f (x_{ij}) -g (x_{ij}))^2 (f (x_{ij}) +g (x_{ij}))^2}
\\
&\le  \sqrt{\sum_{i=1}^n \sum_{j=1}^{m_i} \frac{4\tau^2}{n m_i}  (f(x_{ij}) -g(x_{ij}))^2} \le 2\tau \|f-g\|_{mn}.
\end{align*}  
Since $\|g\|_\infty  \le \tau , \|f\|_\infty \le \tau$, the claim follows.
\end{proof}

\begin{corollary} 
\label{coro:sigma complexity bound 1 mn} Let $\mathcal F$ be a function class such that
\begin{align*}
&\log \mathcal N (\delta, \mathcal F , \| \cdot\|_{nm} )   
\le C^{\prime} nm \phi_{nm} \bigl( \log^{4.3}(nm)\log(\delta^{-1}) \bigr),
\end{align*}
and suppose that $g \in \mathcal{F}$ implies that $\|g\|_\infty \le \tau$.
For  \begin{align*}
&t\ge C'' \lambda \tau \sqrt {nm\phi_{nm} \left\{ \log^{4.3}(nm) \log(\lambda^{-1})\right\}}
\end{align*} with $\tau \leq \frac{1}{2}$, 
it holds that 
\begin{align*}
&\p_{\cdot \mid x}\left(\sup_{g \in \mathcal F  \cap B_{nm}(\lambda) } H_{nm}(g) \ge t   \right)\le  \exp( -\frac{c't^2}{\tau^2\lambda^2}),
\end{align*}
for some positive constants $c'$ and $C''$. 
\end{corollary}

\begin{proof}
The proof of the above bound follows from the same argument as in \Cref{coro:complexity bound 1 mn}.
\end{proof}

\begin{corollary} 
\label{coro:sigma complexity bound 2 mn} 
Let $\mathcal F$ be a function class such that
\begin{align*}
&\log \mathcal N (\delta, \mathcal F , \| \cdot\|_{nm} )   
\le C'nm \phi_{nm} \bigl( \log^{4.3}(nm)\log(\delta^{-1}) \bigr),
\end{align*}
and suppose that $g \in \mathcal{F}$ implies that $\|g\|_\infty \le \tau$, where $\tau \leq \frac{1}{2}$.
For any $s\ge 1$, it holds that 
\begin{align*}
&\p_{\cdot \mid x} \bigl(\sup_{g \in \mathcal F } \frac{H_{nm}(g)}{\|g\|_{nm} \sqrt {nm\phi_{nm} \bigl\{ \log^{4.3}(nm) \log(\|g\|_{nm}^{-1})\bigr\}} }   
\ge C_1 \tau s \bigr)
\\
&\le C_2\exp( - c s^2nm\phi_{nm}  \log^{4.3}(nm)  )  .
\end{align*}
\end{corollary}

\begin{proof}
Since 
$ \mathcal F \subset B_\infty \left(\tau \right)\subset B_{nm} \left(\tau\right) $, it holds that 
\begin{align*}
&\mathcal F \cap B_{nm}\left(\tau\right) = \mathcal F .
\end{align*}
Therefore it suffices to show that  for any $s\ge 1$,
\begin{align*}
&\p_{\cdot \mid x} \bigl(\sup_{g \in \mathcal F \cap B_{nm}\bigl(\tau\bigr) }   
\frac{H_{nm}(g)}{ \bigl\|g\bigr\|_{nm} \sqrt {nm \phi_{nm} \bigl\{ \log^{4.3}(nm) \log \bigl(\tau \bigl\|g\bigr\|_{nm}^{-1} \bigr)\bigr\}}} \ge C_1 \tau s \bigr)
\\
&\le C_2\exp \left( - c s^2nm\phi_{nm}  \log^{4.3} \left(n\right) \right).
\end{align*} 
Choose $\lambda = 2^{-l+1}\tau$ in \Cref{coro:sigma complexity bound 1 mn}. For any $s\ge 1 $, and  
\begin{align*}
&t \ge C_1' s2^{-l+1} \tau^2   
\sqrt {nm\phi_{nm} \bigl\{ \log^{4.3}(nm) \log\bigl((2^{-l+1}\tau)^{-1}\bigr)\bigr\}}
\end{align*}
it holds that 
\begin{align*}
&\p_{\cdot \mid x} \bigl(\sup_{g \in \mathcal F \cap B_{nm}(2^{-l+1}\tau) } H_{nm}(g) \ge C_1' s2^{-l+1} \tau^2   
\sqrt {nm\phi_{nm} \bigl\{ \log^{4.3}(nm)\tau^{-1}\log(\tau^{-1}2^{l-1})\bigr\}} \bigr)
\\
&\le \exp\left( -c''s^2nm\phi_{nm} \left\{ \log^{4.3}(nm) \log(2^{l-1}\tau^{-1}) \right\}\right)
\\
&\le \exp\left( -c''s^2nm\phi_{nm} \left\{ \log^{4.3}(nm) \log(2^{l}) \right\}\right)
\\
&\le \exp\left( -c''s^2nm\phi_{nm}   \log^{4.3}(nm)l \log\left(2\right)\right),
\end{align*} for a positive constant $c''$.
Thus, for a positive constant $c'''$, we have 
\begin{align*}
&\p_{\cdot \mid x} \bigl(\sup_{g \in \mathcal F \cap \{ f: 2^{-l} \tau \le \|f\|_{nm}\le 2^{-l+1}\tau \} } H_{nm}(g) \notag \\
&\quad \ge C_1' s2^{-l+1}\tau^2 \sqrt {nm\phi_{nm} \bigl\{ \log^{4.3}(nm) \log(\tau^{-1}2^{l-1})\bigr\}} \bigr)
\\
&\le \exp\left( -c'''s^2nm\phi_{nm} \log^{4.3}(nm)l\right).
\end{align*}
Note that if  
$g\in \{  f: 2^{-l}\tau \le \|f\|_{nm}\le 2^{-l+1}\tau \}$, then  $\|g\|^{-1}_{nm} \ge \tau^{-1}2^{l-1}$ and $2^{-l+1}\tau \le 2\left\|g\right\|_{nm}$.It follows that 
\begin{align*}
&\p_{\cdot \mid x} \left(\sup_{g \in \mathcal F  \cap \{  f: 2^{-l} \tau^2\le \|f\|_{nm}\le 2^{-l+1}\tau \} } H_{nm}(g) \right.
\notag \\
&\quad \ge C_1' s2^{-l+1}\tau^2 \sqrt {nm\phi_{nm} \bigl\{ \log^{4.3}(nm) \log(\tau^{-1}2^{l-1})\bigr\}} \bigr)
\\
&= \p_{\cdot \mid x} \bigl(\sup_{g \in \mathcal F \cap \{ f: 2^{-l} \tau^2\le \|f\|_{nm}\le 2^{-l+1}\tau \} }   
\frac{H_{nm}(g)}{2^{-l+1}\tau \sqrt {nm\phi_{nm} \bigl\{ \log^{4.3}(nm) \log(\tau^{-1}2^{l-1})\bigr\}}} \ge C_1' \tau s \bigr)
\\
&\geq \p_{\cdot \mid x} \bigl(\sup_{g \in \mathcal F \cap \{ f: 2^{-l} \tau^2\le \|f\|_{nm}\le 2^{-l+1}\tau \} }   
\frac{H_{nm}(g)}{2\bigl\|g\bigr\|_{nm} \sqrt {nm \phi_{nm} \bigl\{ \log^{4.3}(nm) \log(\|g\|_{nm}^{-1})\bigr\}}} \ge C_1' \tau s \bigr).
\end{align*}
Thus,  
\begin{align*}
&\p_{\cdot \mid x} \bigl(\sup_{g \in \mathcal F \cap \{ f: 2^{-l}\tau \le \|f\|_{nm}\le 2^{-l+1}\tau \} }   
\frac{H_{nm}(g)}{\bigl\|g\bigr\|_{nm} \sqrt {nm\phi_{nm} \bigl\{ \log^{4.3}(nm)\log(\bigl\| g\bigr\|_{nm}^{-1})\bigr\}} } \ge C_1 \tau s \bigr)
\\
&\le \exp\left( -c''''s^2nm\phi_{nm}   \log^{4.3}(nm)l\right).
\end{align*}
Therefore,
\begin{align*}
&\p_{\cdot \mid x} \bigl(\sup_{g \in \mathcal F \cap B_{nm}(\tau) }   
\frac{H_{nm}(g)}{\bigl\|g\bigr\|_{nm} \sqrt {nm\phi_{nm} \bigl\{ \log^{4.3}(nm)\log(\|g\|_{nm}^{-1})\bigr\}}} \ge C_1 \tau s \bigr)
\\
&= \ \p_{\cdot \mid x} \Bigg(\exists l \in Z^{+}, \text{such that}
\\
&\,\sup_{g \in \mathcal F \cap \bigl\{ f: 2^{-l}\tau\le \|f\|_{nm}\le 2^{-l+1}\tau \bigr\} }   
\frac{H_{nm}(g)}{\|g\|_{nm} \sqrt {nm\phi_{nm} \bigl\{ \log^{4.3}(nm) \log(\| g\|_{nm}^{-1})\bigr\}}} \ge C_1 \tau s \Bigg)
\\
&\le \sum_{l=1}^\infty \p_{\cdot \mid x}\bigl(\sup_{g \in \mathcal F \cap \{ f: 2^{-l}\tau \le \|f\|_{nm}\le 2^{-l+1}\tau \} }   
\frac{H_{nm}(g)}{\bigl\|g\bigr\|_{nm} \sqrt {nm\phi_{nm} \bigl\{ \log^{4.3}(nm) \log(\| g\|_{nm}^{-1})\bigr\}}} \ge C_1 \tau s \bigr)
\\
&\le 
\sum_{l=1}^\infty \exp\left( -c''''s^2nm\phi_{nm}   \log^{4.3}(nm)l\right)
\\
&\le  C_2\exp( - c s^2nm\phi_{nm}  \log^{4.3}(nm)  ),
\end{align*}
where the first identity follows from a peeling argument that $\bigcup_{l=1}^{\infty}\mathcal F  \cap \{  f: 2^{-l}\tau \le \|f\|_{nm}\le 2^{-l+1}\tau  \}=\mathcal F  \cap \bigcup_{l=1}^{\infty}\{  f: 2^{-l}\tau \le \|f\|_{nm}\le 2^{-l+1}\tau  \}= \mathcal F  \cap B_{nm}(\tau)$.
\end{proof}

For \Cref{coro:deviation for norms mn}, we assume $x$ is a vector of  independent random variables. The general case is given in Section~\ref{section_dependent_x}, \Cref{lemma:rate in the temporal-spatial model}. 

\begin{corollary} \label{coro:deviation for norms mn}
Let $n$ be sufficiently large. For any $\eta \in (0, \frac{1}{4})$, there exists a  constant $C_\eta $  only depending on $\eta $ such that 
\begin{align*}
    \E\left( \sup_{ g \in \mathcal F } \left| \frac{\|g\|_{nm}^2  - \|g\|_\lt ^2  }{ \eta  \|g\|_\lt ^2 +  C_\eta\phi_{nm} \log^{5.3}(nm)}\right| \right) \le 12.
\end{align*}
Consequently,  there exists a constant $C'_\eta>0$ such that 
\begin{align} \label{eq:norm localization probability bound mn} 
& \p \left(\sup_{ g \in \mathcal F  } \frac{\big |\|g\|_{nm}^2  - \|g\|_\lt ^2|}{ \|g\|_\lt ^2 +  C'_\eta\phi_{nm} \log^{5.3}(nm)} \le \frac{1}{2} \right) \ge  1-24\eta . 
\end{align}
\end{corollary}

\begin{proof}
 By the symmetrization argument,
 \begin{align*}
&\E\left( \sup_{ g \in \mathcal F  }\big |  \frac{\|g\|_{nm}^2  - \|g\|_\lt ^2 }{ \eta \|g\|_\lt ^2 + C_\eta\phi_{nm} \log^{5.3}(nm)} \big | \right)
\\
&=  \E\left( \sup_{ g \in \mathcal F  }\big |  \frac{\sum_{i=1}^{n}\sum_{j=1}^{m_i}\frac{1}{nm_i} g^2(x_{ij})  - \E(g^2(x_{ij})) }{ \eta \|g\|_\lt ^2 + C_\eta\phi_{nm} \log^{5.3}(nm)} \big | \right)
\\
&=  \E\left( \sup_{ g \in \mathcal F  }\big |  \frac{\sum_{i=1}^{n}\sum_{j=1}^{m_i}\frac{1}{nm_i} {\E}_{x' \mid x}(g^2(x_{ij})  - g^2(x'_{ij}) ) }{ \eta \|g\|_\lt ^2 + C_\eta\phi_{nm} \log^{5.3}(nm)} \big | \right)
\\
&\le \E \bigl[ {\E}_{x' \mid x}\bigl( \sup_{ g \in \mathcal F }\big |   
\frac{\sum_{i=1}^{n}\sum_{j=1}^{m_i}\frac{1}{nm_i} (g^2(x_{ij}) - g^2(x'_{ij})) }{ \eta \|g\|_\lt ^2 + C_\eta\phi_{nm} \log^{5.3}(nm)} \big | \bigr)\bigr]
\\
&=  \E \left( \sup_{ g \in \mathcal F  }\big |  \frac{\sum_{i=1}^{n}\sum_{j=1}^{m_i}\frac{1}{nm_i} (g^2(x_{ij})  - g^2(x'_{ij})) }{ \eta \|g\|_\lt ^2 + C_\eta\phi_{nm} \log^{5.3}(nm)} \big |\right)
\\
&=  \E \left( \sup_{ g \in \mathcal F  }\big |  \frac{\sum_{i=1}^{n}\sum_{j=1}^{m_i}\frac{1}{nm_i} \xi_{ij}(g^2(x_{ij})  - g^2(x'_{ij})) }{ \eta \|g\|_\lt ^2 + C_\eta\phi_{nm} \log^{5.3}(nm)} \big |\right)
\\
&\leq \E \bigl( \sup_{ g \in \mathcal F }\big | \frac{\sum_{i=1}^{n}\sum_{j=1}^{m_i}\frac{1}{nm_i} \xi_{ij} g^2(x_{ij}) }{ \eta \|g\|_\lt ^2 + C_\eta\phi_{nm} \log^{5.3}(nm)} \big |\bigr)   
+ \E \bigl( \sup_{ g \in \mathcal F }\big | \frac{\sum_{i=1}^{n}\sum_{j=1}^{m_i}\frac{1}{nm_i} \xi_{ij} g^2(x'_{ij}) }{ \eta \|g\|_\lt ^2 + C_\eta\phi_{nm} \log^{5.3}(nm)} \big |\bigr)
\\
&=  2 \E \left( \sup_{ g \in \mathcal F  }\big |  \frac{\sum_{i=1}^{n}\sum_{j=1}^{m_i}\frac{1}{nm_i} \xi_{ij} g^2(x_{ij}) }{ \eta \|g\|_\lt ^2 + C_\eta\phi_{nm} \log^{5.3}(nm)} \big |\right),
\end{align*}
where we use ${\E}_{x' \mid x}$ to indicate ${\E}_{\{x'_{ij}\}_{i=1, j=1}^{n, m_i} \mid \{x_{ij}\}_{i=1, j=1}^{n, m_i}}$, 
 where $\left\{x'_{ij}\right\}_{i=1, j=1}^{n, m_i}$ are ghost samples of $\left\{x_{ij}\right\}_{i=1, j=1}^{n, m_i}$, which means that $\left\{x_{ij}\right\}$ and $\left\{x'_{ij}\right\}$ are i.i.d drawn from the same distribution as $\left\{x'_{ij}\right\}_{i=1, j=1}^{n, m_i}$ and $\left\{x_{ij}\right\}_{i=1, j=1}^{n, m_i}$ are independent of $\left\{x'_{ij}\right\}_{i=1, j=1}^{n, m_i}$. In addition, we use $x_{ij}$ to denote a random variable that follows the same distribution as $x_{ij}$. $\left\{\xi_{ij}\right\}_{i=1, j=1}^{n, m_i}$ are i.i.d. Rademacher variables. The fourth identity follows from the fact that $\xi_{ij} (g(x_{ij})-g(x'_{ij}))$ and $g(x_{ij})-g(x'_{ij})$ have the same distribution.

We use ${\E}_{x}$ to indicate ${\E}_{\{x_{ij}\}_{i=1, j=1}^{n, m_i}}$, and use ${\E}_{\xi \mid x}$ to indicate ${\E}_{\{\xi_{ij}\}_{i=1, j=1}^{n, m_i} \mid \{x_{ij}\}_{i=1, j=1}^{n, m_i}}$.

Therefore, 
    \begin{align*}
&\E\left( \sup_{ g \in \mathcal F  }\left|  \frac{\|g\|_{nm}^2  - \|g\|_\lt ^2 }{ \eta \|g\|_\lt ^2 + C_\eta\phi_{nm} \log^{5.3}(nm)} \right| \right)
\le  2  \E\left( \sup_{ g \in \mathcal F  } \left| \frac{\sum_{i=1}^{n}\sum_{j=1}^{m_i}\frac{1}{nm_i} g^2(x_{ij}) \xi_{ij}    }{ \eta \|g\|_\lt ^2 +  C_\eta\phi_{nm} \log^{5.3}(nm)} \right| \right)
\\
&=   2  {\E}_{x}\left[ {\E}_{\xi \mid x}\left(\sup_{ g \in \mathcal F  } \left| \frac{ \sum_{i=1}^{n}\sum_{j=1}^{m_i}\frac{1}{nm_i} g^2(x_{ij}) \xi_{ij} }{ \eta \|g\|_\lt ^2 + C_\eta\phi_{nm} \log^{5.3}(nm)} \right| \right)\right]
\\
&= 2 {\E}_{x}\bigl[ {\E}_{\xi \mid x}\bigl(\sup_{ g \in \mathcal F } \bigl| \frac{ \sum_{i=1}^{n}\sum_{j=1}^{m_i}\frac{1}{nm_i} g^2(x_{ij}) \xi_{ij} }{ \eta \|g\|_{nm} ^2 + C_\eta\phi_{nm} \log^{5.3}(nm)} \bigr|   
\cdot \frac{\eta \|g\|_{nm} ^2 + C_\eta\phi_{nm} \log^{5.3}(nm) }{\eta \|g\|_\lt ^2 + C_\eta\phi_{nm} \log^{5.3}(nm)}\bigr)\bigr]
\\
&\le 2{\E}_{x}\bigl[{\E}_{\xi \mid x}\bigl(\sup_{ g \in \mathcal F } \bigl| \frac{\sum_{i=1}^{n}\sum_{j=1}^{m_i}\frac{1}{nm_i} g^2(x_{ij}) \xi_{ij} }{ \eta \|g\|_{nm} ^2 + C_\eta\phi_{nm} \log^{5.3}(nm)} \bigr|   
\cdot \sup_{ g \in \mathcal F } \frac{\eta \|g\|_{nm} ^2 + C_\eta\phi_{nm} \log^{5.3}(nm) }{\eta \|g\|_\lt ^2 + C_\eta\phi_{nm} \log^{5.3}(nm) } \bigr)\bigr]
\\
&\le 2{\E}_{x}\bigl[\sup_{ g \in \mathcal F } \frac{\eta \|g\|_{nm} ^2 + C_\eta\phi_{nm} \log^{5.3}(nm) }{\eta \|g\|_\lt ^2 + C_\eta\phi_{nm} \log^{5.3}(nm) } {\E}_{\xi \mid x}   
\bigl(\sup_{ g \in \mathcal F } \bigl| \frac{ \sum_{i=1}^{n}\sum_{j=1}^{m_i}\frac{1}{nm_i} g^2(x_{ij}) \xi_{ij} }{ \eta \|g\|_{nm} ^2 + C_\eta\phi_{nm} \log^{5.3}(nm)}\bigr| \bigr)\bigr].
\end{align*}
\
\\
{\bf Step 1.} 
Denote 
\begin{align*}
&B_{nm}(a, b) =\{ g: a\le \|g\|_{nm} \le b \}.
\end{align*}
Observe that 
\begin{align*}
&{\E}_{\xi \mid x} \left(\sup_{ g \in \mathcal F } \left| \frac{ 
     \sum_{i=1}^{n}\sum_{j=1}^{m_i}\frac{1}{nm_i} g^2(x_{ij}) \xi_{ij} }{ \eta \|g\|_{nm} ^2 + C_\eta\phi_{nm} \log^{5.3}(nm)} \right|  \right)
\\
&\le {\E}_{\xi \mid x} \bigl(\sup_{ g \in \mathcal F \cap B_{nm}\bigl( (nm)^{-1}\bigr)} \bigl|   
\frac{ \sum_{i=1}^{n}\sum_{j=1}^{m_i}\frac{1}{nm_i} g^2(x_{ij}) \xi_{ij} }{ \eta \|g\|_{nm} ^2 + C_\eta\phi_{nm} \log^{5.3}(nm)} \bigr| \bigr)
\\
&+ {\E}_{\xi \mid x} \bigl(\sup_{ g \in \mathcal F \cap B_{nm}((nm)^{-1}, \infty ) } \big |   
\frac{ \sum_{i=1}^{n}\sum_{j=1}^{m_i}\frac{1}{nm_i} g^2(x_{ij}) \xi_{ij} }{ \eta \|g\|_{nm} ^2 + C_\eta\phi_{nm} \log^{5.3}(nm)}\big | \bigr).
\end{align*}
Since $\sum_{i=1}^{n}\sum_{j=1}^{m_i}\frac{1}{nm_i} g^2(x_{ij}) \xi_{ij} \le \|g\|_{nm}^2$,
\begin{align*}
&{\E}_{\xi \mid x} \bigl(\sup_{ g \in \mathcal F \cap B_{nm}( (nm)^{-1})} \bigl|   
\frac{ \sum_{i=1}^{n}\sum_{j=1}^{m_i}\frac{1}{nm_i} g^2(x_{ij}) \xi_{ij} }{ \eta \|g\|_{nm} ^2 + C_\eta\phi_{nm} \log^{5.3}(nm)} \bigr| \bigr)
\\
&\le {\E}_{\xi \mid x} \bigl(\sup_{ g \in \mathcal F \cap B_{nm}((nm)^{-1})} \bigl|   
\frac{ \frac{1}{(nm)^2} }{ \eta \|g\|_{nm} ^2 + C_\eta\phi_{nm} \log^{5.3}(nm)}\bigr| \bigr)
\\
&\le   \frac{1}{C_\eta (nm)^2\phi_{nm}\log^{5.3}(nm)}.
\end{align*}
In addition, 
\begin{align*}
&{\E}_{\xi \mid x} \bigl(\sup_{ g \in \mathcal F \cap B_{nm} \bigl((nm)^{-1}, \infty \bigr) } \bigl|   
\frac{ \sum_{i=1}^{n}\sum_{j=1}^{m_i}\frac{1}{nm_i} g^2(x_{ij}) \xi_{ij} }{ \eta \|g\|_{nm} ^2 + \tau C_\eta\phi_{nm} \log^{5.3}(nm)} \bigr| \bigr)
\\
&\le {\E}_{\xi \mid x} \bigl(\sup_{ g \in \mathcal F \cap B_{nm}((nm)^{-1}, \infty ) } \bigl|   
\frac{ \sum_{i=1}^{n}\sum_{j=1}^{m_i}\frac{1}{nm_i} g^2(x_{ij}) \xi_{ij} }{ \eta \|g\|_{nm} ^2 + \tau C_\eta \phi_{nm} \{ \log^{4.3}(nm)\log(\|g\|_{nm}^{-1} ) \}} \bigr| \bigr)
\\
&\le {\E}_{\xi \mid x} \bigl(\sup_{ g \in \mathcal F} \bigl|   
\frac{ \sum_{i=1}^{n}\sum_{j=1}^{m_i}\frac{1}{nm_i} g^2(x_{ij}) \xi_{ij} }{ \eta \|g\|_{nm} ^2 + C_\eta \phi_{nm} \{ \log^{4.3}(nm) \log(\|g\|_{nm}^{-1} ) \}} \bigr| \bigr)
\\
&\le \int_0^\infty \p_{\cdot \mid x}\bigl( \sup_{ g \in \mathcal F } \bigl|   
\frac{ \sum_{i=1}^{n}\sum_{j=1}^{m_i}\frac{1}{nm_i} g^2(x_{ij}) \xi_{ij} }{ \eta \|g\|_{nm} ^2 + C_\eta\phi_{nm} \{ \log^{4.3}(nm) \log(\|g\|_{nm}^{-1} ) \}} \bigr| \ge s\bigr) ds
\\
&\le \int_0^1 1 ds + \int_1^\infty \p_{\cdot \mid x}\bigl( \sup_{ g \in \mathcal F } \bigl|   
\frac{ \sum_{i=1}^{n}\sum_{j=1}^{m_i}\frac{1}{nm_i} g^2(x_{ij}) \xi_{ij} }{ \eta \|g\|_{nm} ^2 + C_\eta \phi_{nm} \{ \log^{4.3}(nm) \log(\|g\|_{nm}^{-1} ) \}} \bigr| \ge s\bigr) ds,
\end{align*}
where the first inequality follows from the fact that $g \in B_{nm}((nm)^{-1}, \infty )$ so that $\log(\|g\|_{nm}^{-1} ) \le \log(nm)$.

Note that 
\begin{align*}
&\int_1^\infty \p_{\cdot \mid x}\bigl( \sup_{ g \in \mathcal F } \bigl|   
\frac{ \sum_{i=1}^{n}\sum_{j=1}^{m_i}\frac{1}{nm_i} g^2(x_{ij}) \xi_{ij} }{ \eta \|g\|_{nm} ^2 + C_\eta \phi_{nm} \{ \log^{4.3}(nm)\log(\|g\|_{nm}^{-1} ) \}} \bigr| \ge s\bigr) ds
\\
&\le \int_1^\infty \p_{\cdot \mid x}\bigl( \sup_{ g \in \mathcal F } \bigl|   
\frac{ \sum_{i=1}^{n}\sum_{j=1}^{m_i}\frac{1}{nm_i} g^2(x_{ij}) \xi_{ij} }{\|g\|_{nm} \sqrt { C_\eta \phi_{nm} \{ \log^{4.3}(nm)\log(\|g\|_{nm}^{-1} ) \}}} \bigr| \ge s\bigr) ds
\\
&\le \int_1^\infty \p_{\cdot \mid x}\bigl( \sup_{ g \in \mathcal F }   
\frac{ \sum_{i=1}^{n}\sum_{j=1}^{m_i}\frac{1}{nm_i} g^2(x_{ij}) \xi_{ij} }{\|g\|_{nm} \sqrt { C_\eta \phi_{nm} \{ \log^{4.3}(nm)\log(\|g\|_{nm}^{-1} ) \}}} \ge s\bigr) ds
\\
&+ \int_1^\infty \p_{\cdot \mid x}\bigl( \sup_{ g \in \mathcal F }   
\frac{ - \sum_{i=1}^{n}\sum_{j=1}^{m_i}\frac{1}{nm_i} g^2(x_{ij}) \xi_{ij} }{\|g\|_{nm} \sqrt { C_\eta \phi_{nm} \{ \log^{4.3}(nm)\log(\|g\|_{nm}^{-1} ) \}}} \ge s\bigr) ds
\\
&\le  2C_2 \int_1^\infty  \exp(-cs^2 nm\phi_{nm} \log^{4.3}(nm))ds
\\
&\le 2C_2 \frac{1}{\sqrt {cnm \phi_{nm} \log^{4.3}(nm) }} \int_0^\infty \exp(-u^2) du   
= \frac{\sqrt \pi C_2 }{ \sqrt { cnm \phi_{nm} \log^{4.3}(nm) } },
\end{align*}
where the third inequality follows from  \Cref{coro:sigma complexity bound 2 mn} for sufficiently large positive constant $C_\eta$ and the fact that $\{-\xi_{ij}\}_{i=1}^n $ are i.i.d. Rademacher random variables.
So for sufficiently large $n$,
\begin{align*}
&{\E}_{\xi \mid x} \left(\sup_{ g \in \mathcal F  } \left| \frac{ 
     \sum_{i=1}^{n}\sum_{j=1}^{m_i}\frac{1}{nm_i} g^2(x_{ij}) \xi_{ij}    }{ \eta \left\|g\right\|_{nm} ^2 +C_\eta\phi_{nm} \log^{5.3}(nm)}\right|  \right)
\\
&\le {\E}_{\xi \mid x} \bigl(\sup_{ g \in \mathcal F\cap B_{nm} \bigl( (nm)^{-1}\bigr) } \bigl|   
\frac{ \sum_{i=1}^{n}\sum_{j=1}^{m_i}\frac{1}{nm_i} g^2(x_{ij}) \xi_{ij} }{ \eta \bigl\|g\bigr\|_{nm}^2 +C_\eta\phi_{nm} \log^{5.3}(nm)}\bigr| \bigr) +
\\
&{\E}_{\xi \mid x} \bigl(\sup_{ g \in \mathcal F\cap B_{nm}\bigl((nm)^{-1}, \infty\bigr) } \bigl|   
\frac{ \sum_{i=1}^{n}\sum_{j=1}^{m_i}\frac{1}{nm_i} g^2(x_{ij}) \xi_{ij} }{ \eta \bigl\|g\bigr\|_{nm} ^2 +C_\eta\phi_{nm} \log^{5.3}(nm)} \bigr| \bigr)
\\
&\le \frac{1}{C_\eta (nm)^2\phi_{nm}\log^{5.3}(nm)}+1   
+ \frac{\sqrt \pi C_2 }{ \sqrt {cnm \phi_{nm} \log^{4.3}(nm) } }\le\frac{3}{2} .
\end{align*}
This implies that 
\begin{align*}
&\E \left( \sup_{ g \in \mathcal F  } \left|  \frac{ \left\|g\right\|_{nm}^2  - \left\|g\right\|_\lt ^2 }{ \eta \left\|g\right\|_\lt ^2 + C_\eta\phi_{nm} \log^{5.3}(nm)} \right| \right)
\\
&\leq 2 \mathbb{E}_x\bigl[\sup _{g \in \mathcal{F}} \frac{\eta \bigl\|g\bigr\|_{nm}^2+ C_\eta \phi_{nm} \log^{5.3}(nm)}{\eta \bigl\|g\bigr\|_{\mathcal{L}_2}^2+ C_\eta \phi_{nm} \log^{5.3}(nm)} \mathbb{E}_{\xi \mid x}   
\bigl(\sup_{g \in \mathcal{F}}\bigl|\frac{\frac{1}{n} \sum_{i=1}^n g^2\bigl(x_{ij}\bigr) \xi_{ij}}{\eta \bigl\|g\bigr\|_{nm}^2+ C_\eta \phi_{nm} \log^{5.3}(nm)}\bigr|\bigr)\bigr]
\\
&\leq  2 \mathbb{E}_x\left(\sup _{g \in \mathcal{F}} \frac{\eta \left\|g\right\|_{nm}^2+ C_\eta \phi_{nm} \log ^5(nm)}{\eta \left\|g\right\|_{\mathcal{L}_2}^2+ C_\eta \phi_{nm} \log ^5(nm)} \cdot \frac{3}{2}\right)
\\
&\le 3{\E}_{x} \left(\sup_{ g \in \mathcal F  }  \frac{\eta \|g\|_{nm} ^2 + \tau C_\eta\phi_{nm} \log^{5.3}(nm) }{\eta \|g\|_\lt  ^2 +C_\eta\phi_{nm} \log^{5.3}(nm) } \right).
\end{align*}
\
\\
{\bf Step 2.} 
Note that 
\begin{align}\label{eq:rademacher middle steps mn}  \nonumber 
    & {\E}_{x}\left(\sup_{ g \in \mathcal F  }  \frac{\eta \left\|g\right\|_{nm} ^2 +  C_\eta\phi_{nm} \log^{5.3}(nm) }{\eta \left\|g\right\|_\lt  ^2 + C_\eta\phi_{nm} \log^{5.3}(nm) }      \right)
    \\ =&   {\E}_{x}\left(\sup_{ g \in \mathcal F  }     \frac{\eta \left\|g\right\|_{nm} ^2 + C_\eta\phi_{nm} \log^{5.3}(nm) }{\eta \left\|g\right\|_\lt  ^2 + C_\eta\phi_{nm} \log^{5.3}(nm) }-1 \right) + 1   \nonumber
    \\
     \le  &  {\E}_{x}\left(\sup_{ g \in \mathcal F  }  \frac{ \left| \eta \left\|g\right\|_{nm} ^2   - \eta \left\|g\right\|_\lt^2 \right| }{\eta \left\|g\right\|_\lt  ^2 + C_\eta\phi_{nm} \log^{5.3}(nm) }       \right) +1   \nonumber
    \\ = &  \eta{\E}_{x} \left(\sup_{ g \in \mathcal F  } \left|  \frac{ \left\|g\right\|_{nm} ^2 - \left\|g\right\|_\lt^2   }{\eta \left\|g\right\|_\lt  ^2 + C_\eta\phi_{nm} \log^{5.3}(nm) } \right| \right) + 1 .
\end{align}
Thus,  
\begin{align*}
&\E \left( \sup_{ g \in \mathcal F  } \left|  \frac{\|g\|_{nm}^2  - \left\|g\right\|_\lt ^2 }{ \eta \left\|g\right\|_\lt ^2 + C_\eta\phi_{nm} \log^{5.3}(nm)} \right| \right)  
\le  3 \mathbb{E}_x\left(\sup _{g \in \mathcal{F}} \frac{\eta \left\|g\right\|_{nm}^2+ C_\eta \phi_{nm} \log ^5(nm)}{\eta \|g \|_{\mathcal{L}_2}^2+ C_\eta \phi_{nm} \log ^5(nm)}\right)\\
\le & 3 \eta{\E}_{x} \left(\sup_{ g \in \mathcal F  } \left|  \frac{ \left\|g\right\|_{nm} ^2 - \left\|g\right\|_\lt^2   }{\eta \left\|g\right\|_\lt^2 + C_\eta\phi_{nm} \log^{5.3}(nm) }  \right| \right) + 3 ,\\
\text{i.e. } & {\E}_{x} \left(\sup_{ g \in \mathcal F  } \left|  \frac{  \left\|g\right\|_{nm} ^2 - \left\|g\right\|_\lt^2   }{\eta \left\|g\right\|_\lt  ^2 + C_\eta\phi_{nm} \log^{5.3}(nm) } \right| \right) \le \frac{3}{1-3\eta}
\end{align*} for any $ 0\le \eta<1/3$.
Since $ 0\le \eta<1/4$, this implies that 
\begin{align*}
    &\E \left( \sup_{ g \in \mathcal F  } \left|  \frac{ \left\|g\right\|_{nm}^2  - \left\|g\right\|_\lt ^2 }{ \eta \left\|g\right\|_\lt ^2 + C_\eta\phi_{nm} \log^{5.3}(nm)} \right| \right) 
     \le 12.
\end{align*}
    \
    \\

\color{black}
{\bf Step 3.}
Let \(C_\eta' := C_\eta/\eta\), so that \(\eta C_\eta' = C_\eta\). Then, for every \(g \in \mathcal F\),
\begin{align*}
&\|g\|_\lt^2 + C_\eta' \phi_{nm} \log^{5.3}(nm)   
= \frac{1}{\eta}\Big( \eta \|g\|_\lt^2 + C_\eta \phi_{nm} \log^{5.3}(nm) \Big).
\end{align*}
Hence,
\begin{align*}
&\frac{ \bigl| \|g\|_{nm}^2 - \|g\|_\lt^2 \bigr| }{ \|g\|_\lt^2 + C_\eta' \phi_{nm} \log^{5.3}(nm) }   
= \eta\, \frac{ \bigl| \|g\|_{nm}^2 - \|g\|_\lt^2 \bigr| }{ \eta \|g\|_\lt^2 + C_\eta \phi_{nm} \log^{5.3}(nm) }.
\end{align*}
Taking the supremum over \(g \in \mathcal F\),
\begin{align*}
&\sup_{ g \in \mathcal F } \bigl| \frac{ \|g\|_{nm}^2 - \|g\|_\lt^2 } { \|g\|_\lt^2 + C_\eta' \phi_{nm} \log^{5.3}(nm)} \bigr|   
= \eta\, \sup_{ g \in \mathcal F } \bigl| \frac{ \|g\|_{nm}^2 - \|g\|_\lt^2 } { \eta \|g\|_\lt^2 + C_\eta \phi_{nm} \log^{5.3}(nm)} \bigr|.
\end{align*}
Therefore,
\begin{align*}
&\p\left( 
\sup_{ g \in \mathcal F }
\left|
\frac{ \|g\|_{nm}^2 - \|g\|_\lt^2 }
{ \|g\|_\lt^2 + C_\eta' \phi_{nm} \log^{5.3}(nm)}
\right|
\ge \frac{1}{2}
\right)
\\
&=
\p\left(
\sup_{ g \in \mathcal F }
\left|
\frac{ \|g\|_{nm}^2 - \|g\|_\lt^2 }
{ \eta \|g\|_\lt^2 + C_\eta \phi_{nm} \log^{5.3}(nm)}
\right|
\ge \frac{1}{2\eta}
\right).
\end{align*}

Since the random variable inside the probability is nonnegative, by Markov's inequality,
\begin{align*}
&\p\left(
\sup_{ g \in \mathcal F }
\left|
\frac{ \|g\|_{nm}^2 - \|g\|_\lt^2 }
{ \eta \|g\|_\lt^2 + C_\eta \phi_{nm} \log^{5.3}(nm)}
\right|
\ge \frac{1}{2\eta}
\right)
\\
&\le
2\eta\,
\E\left(
\sup_{ g \in \mathcal F }
\left|
\frac{ \|g\|_{nm}^2 - \|g\|_\lt^2 }
{ \eta \|g\|_\lt^2 + C_\eta \phi_{nm} \log^{5.3}(nm)}
\right|
\right).
\end{align*}

Using the expectation bound established previously,
\begin{align*}
&\E\left(
\sup_{ g \in \mathcal F }
\left|
\frac{ \|g\|_{nm}^2 - \|g\|_\lt^2 }
{ \eta \|g\|_\lt^2 + C_\eta \phi_{nm} \log^{5.3}(nm)}
\right|
\right)
\le 12,
\end{align*}
we conclude that
\begin{align*}
&\p\left( 
\sup_{ g \in \mathcal F }
\left|
\frac{ \|g\|_{nm}^2 - \|g\|_\lt^2 }
{ \|g\|_\lt^2 + C_\eta' \phi_{nm} \log^{5.3}(nm)}
\right|
\ge \frac{1}{2}
\right)
\le 24\eta.
\end{align*}
Equivalently,
\begin{align*}
&\p\left( 
\sup_{ g \in \mathcal F }
\frac{ \big| \|g\|_{nm}^2 - \|g\|_\lt^2 \big| }
{ \|g\|_\lt^2 + C_\eta' \phi_{nm} \log^{5.3}(nm)}
\le \frac{1}{2}
\right)
\ge 1 - 24\eta.
\end{align*}
\color{Black}

\end{proof}

\newpage
\subsubsection{Deviation bounds  for  $\gamma $}
\label{sec: Deviation bounds  for gamma in temporal-spatial mn}
Conditioning on $\{\gamma_i\}_{i=1}^n$,
we aim to bound 
\begin{align*}
&\sup_{g \in \mathcal F \cap B_{nm}(\lambda)}  \sum_{i=1}^n \sum_{j=1}^{m_i}\frac{1}{nm_i } \left\{ g(x_{ij}) \gamma_i(x_{ij}) - \left\langle g, \gamma_i \right\rangle_{\lt} \right\}.
\end{align*}

\
\\

Suppose $\left\{\xi_{ij}\right\}_{i=1, j=1}^{n, m_i}$ are i.i.d. Rademacher random variables and denote 
\begin{align*}
&J_{nm}(g) 
 =  \sum_{i=1}^n \sum_{j=1}^{m_i}\frac{1}{\sqrt{nm_i}}  g  \left(x_{ij}\right) \gamma_i \left(x_{ij}\right) \xi_{ij} .
\end{align*} 
In \Cref{coro:sub-gaussian deviation gamma mn} and \Cref{coro:gamma complexity bound mn}, the  analysis is  conditioning on  $ \left\{x_{ij}\right\}_{i=1, j=1}^{n, m_i}$, as well as conditioning on $\left\{\gamma_i\right\}_{i=1}^n$ as mentioned above. 
We denote the conditional probability conditioning on $\left\{x_{ij}\right\}_{i=1, j=1}^{n, m_i}$ by $\p_{\cdot \mid \left\{x_{ij}\right\}_{i=1,j=1}^{n,m_i}}$, and conditioning on $\left\{x_{ij}\right\}_{i=1, j=1}^{n, m_i}$ and $\left\{\gamma_i\right\}_{i=1}^n$ by $\p_{\cdot \mid \left\{x_{ij}\right\}_{i=1,j=1}^{n,m_i}, \left\{\gamma_i\right\}_{i=1}^n}$, and we use $\p_{\cdot \mid x}$ and $\p_{\cdot \mid x, \gamma}$ as their shorthand notations. 

\begin{corollary} 
\label{coro:sub-gaussian deviation gamma mn}
Let  $g,f  \in \lt$ be two deterministic functions, 
furthermore, let $\|\gamma\|_{\infty,nm}  \le \tau$. Here  \(
\left\|\gamma_i\right\|_{\infty,nm}:=\max_{1\le j\le m_i}\left|\gamma_i\left(x_{ij}\right)\right|.
\)  Then conditioning on $ \{x_{ij}\}_{i=1, j=1}^{n, m_i}$, and $\left\{\gamma_i\right\}_{i=1}^n$, there exists a constant $c>0$ such that  
\begin{align*} 
& \p_{\cdot \mid x, \gamma} \left( J_{nm}(g)     \ge    t        \right)  \le   \exp   \left(  -      \frac{c t^2}{\tau^2\|g\|_{nm}^2 }    \right),
\\ 
&  \p_{\cdot \mid x, \gamma} \left( J_{nm}  (g)  -J_{n}  (f)     \ge t \right)  \le  \exp\left(  -     \frac{c t^2}{\tau^2\|f-g\|_{nm} ^2}    \right) .
\end{align*} 
\end{corollary} 
\begin{proof}
For the first bound,  it suffices to note that 
$g  (x_{ij})\gamma_i(x_{ij})  \xi_{ij}$ is sub-Gaussian with parameter $|g (x_{ij}) \gamma_i(x_{ij})|\ $. So  $J(g)$ is sub-Gaussian with parameter 
\begin{align*}
&\sqrt{\sum_{i=1}^n \sum_{j=1}^{m_i} \frac{1}{nm_i} g^ 2 (x_{ij}) \gamma_i^ 2 (x_{ij})} \le \tau\|g\|_{nm}.
\end{align*}
The second bound follows from the fact that $J(g)$ is linear in $g$. The second part of this lemma follows by the same argument.
\end{proof}

\begin{corollary} 
\label{coro:gamma complexity bound mn} Let $\mathcal F$ be a function class such that
\begin{align*}
&\log \mathcal N \left(\delta,  \mathcal F  ,  \left\| \cdot \right\|_{nm} \right) \le nm \phi_{nm} \left( \log^{4.3}(nm)\log(\delta^{-1}) \right)
\end{align*}
and that $\mathcal F\subset B_\infty(\tau)$.
For  
\begin{align*}
&t\ge C \lambda \tau \sqrt {nm\phi_{nm} \left\{ \log^{4.3}(nm) \log(\lambda^{-1})\right\}},
\end{align*}
it holds that 
\begin{align*} 
\p_{\cdot \mid x, \gamma}\left(\sup_{g \in \mathcal F  \cap B_n(\lambda) } J_{nm}(g) \ge t   \right)\le  \exp( -\frac{c't^2}{\tau^2\lambda^2}) .  
\end{align*}
In addition, for any $s\ge 1$, it holds that 
\begin{align*}
&\p_{\cdot \mid x, \gamma}\bigl(\sup_{g \in \mathcal F } \frac{J_{nm}(g)}{\|g\|_{nm} \sqrt {nm\phi_{nm} \bigl\{ \log^{4.3}(nm)\log(\|g\|_{nm}^{-1})\bigr\}}}   
\ge C_1 \tau s \bigr)
\\
&\le  C_2\exp( - c s ^2 nm\phi_{nm} \log^{4.3}(nm)  )  .
\end{align*}
\end{corollary}
\begin{proof}
The proof of the first bound follows from the same argument as in \Cref{coro:sigma complexity bound 1 mn} and the second bound follows from the same argument as in \Cref{coro:sigma complexity bound 2 mn}.
\end{proof}

\begin{corollary} \label{coro:deviation bound for gamma with xi mn}
Let $n$ be sufficiently large and suppose that $\|\gamma_i\|_{\infty,nm}:=\max_{1\le j\le m_i}\left|\gamma_i\left(x_{ij}\right)\right|\le \tau$, 
there exists a  constant $C_\eta $  only depending on $\eta $ such that 
\begin{align*}
&\E\bigl( \sup_{ g \in \mathcal F } \big | \frac{ \sum_{i=1}^n \sum_{j=1}^{m_i} \frac{1}{nm_i} \bigl[ g(x_{ij}) \gamma_i(x_{ij}) - \langle g, \gamma_i\rangle_{\lt} \bigr]}{\eta \|g\|_\lt ^2 + \tau C_\eta\phi_{nm} \log^{5.3}(nm)}\big | \bigr)   
\le 12.
\end{align*}
    Consequently, for any $ \eta \in (0,\frac{1}{4})$,  there exists a constant $C'_\eta>0$ such that 
    \begin{align}
&\p \bigl(\sup_{ g \in \mathcal F } \bigl|\frac{ \sum_{i=1}^n \sum_{j=1}^{m_i} \frac{1}{nm_i} \bigl[g(x_{ij}) \gamma_i(x_{ij}) - \langle g, \gamma_i\rangle_{\lt} \bigr]}{ \|g\|_\lt ^2 + \tau C_\eta'\phi_{nm} \log^{5.3}(nm)}\bigr|   
\le \frac{1}{16} \bigr) \ge 1- 192\eta .\label{eq:norm localization probability bound with xi mn}
\end{align}
\end{corollary}
\begin{proof}
 By symmetrization argument,
    \begin{align*}
&\E\left( \sup_{ g \in \mathcal F  }\big |  \frac{\sum_{i=1}^n \sum_{j=1}^{m_i} \frac{1}{nm_i} \left[ g(x_{ij}) \gamma_i(x_{ij}) - \langle g, \gamma_i\rangle_{\lt} \right] }{ \eta \|g\|_\lt ^2 + \tau C_\eta\phi_{nm} \log^{5.3}(nm)} \big | \right)
\\
&\le  2  \E\left( \sup_{ g \in \mathcal F  } \big  | \frac{ 
      \sum_{i=1}^n \sum_{j=1}^{m_i} \frac{1}{nm_i} g(x_{ij}) \gamma_i(x_{ij}) \xi_{ij}  }{ \eta \|g\|_\lt ^2 + \tau C_\eta\phi_{nm} \log^{5.3}(nm)}\big | \right)
\\
&= 2 {\E}_{x, \gamma} \bigl[ {\E}_{\xi|x, \gamma } \bigl(\sup_{ g \in \mathcal F } \bigl|   
\frac{ \sum_{i=1}^n \sum_{j=1}^{m_i} \frac{1}{nm_i} g (x_{ij})\gamma_i(x_{ij}) \xi_{ij} }{ \eta \|g\|_\lt ^2 + \tau C_\eta\phi_{nm} \log^{5.3}(nm)}\bigr |\bigr) \bigr]
\\
&= 2 {\E}_{x, \gamma} \bigl[ {\E}_{\xi|x, \gamma } \bigl(\sup_{ g \in \mathcal F } \bigl| \frac{ \sum_{i=1}^n \sum_{j=1}^{m_i} \frac{1}{nm_i} g(x_{ij}) \gamma_i(x_{ij}) \xi_{ij} }{ \eta \|g\|_{nm} ^2 + \tau C_\eta\phi_{nm} \log^{5.3}(nm)} \bigr|   
\cdot \frac{\eta \|g\|_{nm} ^2 + \tau C_\eta\phi_{nm} \log^{5.3}(nm) }{\eta \|g\|_\lt ^2 + \tau C_\eta\phi_{nm} \log^{5.3}(nm) } \bigr)\bigr]
\\
&\le 2{\E}_{x, \gamma}\bigl[ {\E}_{\xi|x, \gamma } \bigl(\sup_{ g \in \mathcal F } \big | \frac{ \sum_{i=1}^n \sum_{j=1}^{m_i} \frac{1}{nm_i} g(x_{ij}) \gamma_i(x_{ij}) \xi_{ij} }{ \eta \|g\|_{nm} ^2 + \tau C_\eta\phi_{nm} \log^{5.3}(nm)}\big |   
\sup_{ g \in \mathcal F } \frac{\eta \|g\|_{nm} ^2 + \tau C_\eta\phi_{nm} \log^{5.3}(nm) }{\eta \|g\|_\lt ^2 + \tau C_\eta\phi_{nm} \log^{5.3}(nm) }\bigr) \bigr]
\\
&\le 2{\E}_{x, \gamma}\bigl[\sup_{ g \in \mathcal F } \frac{\eta \|g\|_{nm} ^2 + \tau C_\eta\phi_{nm} \log^{5.3}(nm) }{\eta \|g\|_\lt ^2 + \tau C_\eta\phi_{nm} \log^{5.3}(nm) } \notag \\
&\quad \cdot {\E}_{\xi|x, \gamma }\bigl(\sup_{ g \in \mathcal F } \bigl| \frac{ \sum_{i=1}^n \sum_{j=1}^{m_i} \frac{1}{nm_i} g(x_{ij}) \gamma_i(x_{ij}) \xi_{ij} }{ \eta \|g\|_{nm} ^2 + \tau C_\eta\phi_{nm} \log^{5.3}(nm)} \bigr| \bigr) \bigr].
\end{align*}
\
\\
{\bf Step 1.} 
Denote 
\begin{align*}
&B_{nm}(a, b) =\{ g: a\le \|g\|_{nm} \le b \}.
\end{align*}
Observe that 
\begin{align*}
&{\E}_{\xi|x, \gamma } \left(\sup_{ g \in \mathcal F } \left| \frac{ 
     \sum_{i=1}^n \sum_{j=1}^{m_i} \frac{1}{nm_i}  g(x_{ij}) \gamma_i(x_{ij}) \xi_{ij} }{ \eta \|g\|_{nm} ^2 + \tau C_\eta\phi_{nm} \log^{5.3}(nm)}\right| \right)
\\
&\le {\E}_{\xi|x, \gamma } \bigl(\sup_{ g \in \mathcal F \cap B_{nm}((nm)^{-1})} \bigl|   
\frac{ \sum_{i=1}^n \sum_{j=1}^{m_i} \frac{1}{nm_i} g(x_{ij}) \gamma_i(x_{ij}) \xi_{ij} }{ \eta \|g\|_{nm} ^2 + \tau C_\eta\phi_{nm} \log^{5.3}(nm)} \bigr| \bigr) +
\\
&{\E}_{\xi|x, \gamma } \bigl(\sup_{ g \in \mathcal F \cap B_{nm}((nm)^{-1}, \infty ) } \bigl|   
\frac{ \sum_{i=1}^n \sum_{j=1}^{m_i} \frac{1}{nm_i} g(x_{ij}) \gamma_i(x_{ij}) \xi_{ij} }{ \eta \|g\|_{nm} ^2 + \tau C_\eta\phi_{nm} \log^{5.3}(nm)} \bigr| \bigr).
\end{align*}
Since 
\begin{align*}
&\bigl|\sum_{i=1}^n \sum_{j=1}^{m_i} \frac{1}{nm_i} g (x_{ij})\gamma_i(x_{ij}) \xi_{ij} \bigr|   
\le \sum_{i=1}^n \sum_{j=1}^{m_i} \frac{1}{nm_i} \bigl| g (x_{ij})\gamma_i(x_{ij}) \bigr|
\\
&\le \sqrt{\sum_{i=1}^n \sum_{j=1}^{m_i} \frac{1}{nm_i} g^2 (x_{ij})} \sqrt{\sum_{i=1}^n \sum_{j=1}^{m_i} \frac{1}{nm_i} \gamma_i^2 (x_{ij})}   
\le \tau \|g\|_{nm} ,
\end{align*}
\begin{align*}
&{\E}_{\xi|x, \gamma } \bigl(\sup_{ g \in \mathcal F \cap B_{nm} \bigl((nm)^{-1}\bigr)} \big |   
\frac{ \sum_{i=1}^n \sum_{j=1}^{m_i} \frac{1}{nm_i} g (x_{ij}) \gamma_i(x_{ij}) \xi_{ij} }{ \eta \|g\|_{nm} ^2 + \tau C_\eta\phi_{nm} \log^{5.3}(nm)}\big | \bigr)
\\
&\le {\E}_{\xi|x, \gamma } \bigl(\sup_{ g \in \mathcal F \cap B_{nm}((nm)^{-1})} \big |   
\frac{ \tau \|g\|_{nm}}{ \eta \|g\|_{nm} ^2 + \tau C_\eta\phi_{nm} \log^{5.3}(nm)}\big | \bigr)
\\
&\le {\E}_{\xi|x, \gamma } \bigl(\sup_{ g \in \mathcal F \cap B_{nm}((nm)^{-1})} \bigl|   
\frac{ \frac{\tau}{nm} }{ \eta \|g\|_{nm} ^2 + \tau C_\eta\phi_{nm} \log^{5.3}(nm)}\bigr| \bigr)
\\
&\le {\E}_{\xi|x, \gamma } \bigl(\sup_{ g \in \mathcal F \cap B_{nm}((nm)^{-1})} \big | \frac{ \frac{\tau}{nm} }{\tau C_\eta\phi_{nm} \log^{5.3}(nm)}\big | \bigr)   
\le \frac{1}{C_\eta nm \phi_{nm}\log^{5.3}(nm)}.
\end{align*}
In addition, 
\begin{align*}
&{\E}_{\xi|x, \gamma } \bigl( \sup_{ g \in \mathcal F \cap B_{nm} \bigl((nm)^{-1}, \infty \bigr) } \bigl|   
\frac{ \sum_{i=1}^n \sum_{j=1}^{m_i} \frac{1}{nm_i} g (x_{ij}) \gamma_i(x_{ij}) \xi_{ij} }{ \eta \|g\|_{nm} ^2 + \tau C_\eta\phi_{nm} \log^{5.3}(nm)} \bigr| \bigr)
\\
&\le {\E}_{\xi|x, \gamma } \bigl(\sup_{ g \in \mathcal F \cap B_{nm} \bigl((nm)^{-1}, \infty \bigr) } \bigl|   
\frac{ \sum_{i=1}^n \sum_{j=1}^{m_i} \frac{1}{nm_i} g (x_{ij}) \gamma_i(x_{ij}) \xi_{ij} }{ \eta \|g\|_{nm} ^2 + \tau C_\eta \phi_{nm} \{ \log^{4.3}(nm)\log(\|g\|_{nm}^{-1} ) \}} \bigr| \bigr)
\\
&\le {\E}_{\xi|x, \gamma } \bigl(\sup_{ g \in \mathcal F } \big |   
\frac{ \sum_{i=1}^n \sum_{j=1}^{m_i} \frac{1}{nm_i} g (x_{ij}) \gamma_i(x_{ij}) \xi_{ij} }{ \eta \|g\|_{nm} ^2 + \tau C_\eta \phi_{nm} \{ \log^{4.3}(nm)\log(\|g\|_{nm}^{-1} ) \}}\big | \bigr)
\\
&\le \int_0^\infty \p_{\cdot \mid x, \gamma} \bigl( \sup_{ g \in \mathcal F } \big |   
\frac{ \sum_{i=1}^n \sum_{j=1}^{m_i} \frac{1}{nm_i} g (x_{ij}) \gamma_i(x_{ij}) \xi_{ij} }{ \eta \|g\|_{nm} ^2 + \tau C_\eta\phi_{nm} \{ \log^{4.3}(nm)\log(\|g\|_{nm}^{-1} ) \}}\big | \ge s\bigr) ds
\\
&\le \int_0^1 1 ds + \int_1^\infty \p_{\cdot \mid x, \gamma} \bigl( \sup_{ g \in \mathcal F } \big|   
\frac{ \sum_{i=1}^n \sum_{j=1}^{m_i} \frac{1}{nm_i} g (x_{ij}) \gamma_i(x_{ij}) \xi_{ij} }{ \eta \|g\|_{nm} ^2 + \tau C_\eta \phi_{nm} \{ \log^{4.3}(nm)\log(\|g\|_{nm}^{-1} ) \}} \big| \ge s\bigr) ds.
\end{align*}
where the first inequality follows from the restriction on the function class that $\|g\|_{nm} \ge \frac{1}{nm}$, which implies that $ \|g\|^{-1}_{nm} \le nm$. 

Note that 
\begin{align*}
&\int_1^\infty \p_{\cdot \mid x, \gamma}\bigl( \sup_{ g \in \mathcal F } \bigl|   
\frac{ \sum_{i=1}^n \sum_{j=1}^{m_i} \frac{1}{nm_i} g (x_{ij}) \gamma_i(x_{ij}) \xi_{ij} }{ \eta \|g\|_{nm} ^2 + \tau C_\eta \phi_{nm} \{ \log^{4.3}(nm)\log(\|g\|_{nm}^{-1} ) \}} \bigr| \ge s\bigr) ds
\\
&\le \int_1^\infty \p_{\cdot \mid x, \gamma} \bigl( \sup_{ g \in \mathcal F } \big|   
\frac{ \sum_{i=1}^n \sum_{j=1}^{m_i} \frac{1}{nm_i} g (x_{ij}) \gamma_i(x_{ij}) \xi_{ij} }{\|g\|_{nm}\sqrt{\tau C_\eta \phi_{nm} \{ \log^{4.3}(nm)\log(\|g\|_{nm}^{-1} ) \}}} \big| \ge s\bigr) ds
\\
&\le \int_1^\infty \p_{\cdot \mid x, \gamma} \bigl( \sup_{ g \in \mathcal F }   
\frac{ \sum_{i=1}^n \sum_{j=1}^{m_i} \frac{1}{nm_i} g (x_{ij}) \gamma_i(x_{ij}) \xi_{ij} }{\|g\|_{nm}\sqrt{\tau C_\eta \phi_{nm} \{ \log^{4.3}(nm) \log(\|g\|_{nm}^{-1} ) \}}} \ge s\bigr) ds
\\
&+ \int_1^\infty \p_{\cdot \mid x, \gamma}\bigl( \sup_{ g \in \mathcal F }   
\frac{ - \sum_{i=1}^n \sum_{j=1}^{m_i} \frac{1}{nm_i} g (x_{ij}) \gamma_i(x_{ij}) \xi_{ij} }{\|g\|_{nm}\sqrt{\tau C_\eta \phi_{nm} \{ \log^{4.3}(nm) \log(\|g\|_{nm}^{-1} ) \}}} \ge s\bigr) ds
\\
&\le  2C_2 \int_1^\infty  \exp(-cs^2 nm \phi_{nm} \log^{4.3}(nm))ds
\\
&\le 2C_2 \frac{1}{\sqrt {cnm \phi_{nm} \log^{4.3}(nm) }} \int_0^\infty \exp(-u^2) du   
= \frac{\sqrt \pi C_2 }{ \sqrt { cnm \phi_{nm} \log^{4.3}(nm) } },
\end{align*}
where the second inequality follows from  \Cref{coro:gamma complexity bound mn} for sufficiently large constant $C_\eta$ and the fact that $\{-\xi_{ij}\}_{i=1, j=1}^{n, m_i} $ are also i.i.d. Rademacher random variables.
Since $r=o\left({nm \phi_{nm}\log^{5.3}(nm)}\right)$, for sufficiently large $n$,
\begin{align*}
&{\E}_{\xi|x, \gamma } \left(\sup_{ g \in \mathcal F  } \big  | \frac{ 
     \sum_{i=1}^n \sum_{j=1}^{m_i} \frac{1}{nm_i} g (x_{ij}) \gamma_i(x_{ij}) \xi_{ij}   }{ \eta \|g\|_{nm} ^2 +C_\eta\phi_{nm} \log^{5.3}(nm)}\big |  \right)
\\
&\le {\E}_{\xi|x, \gamma }\bigl(\sup _{g \in \mathcal{F} \cap B_{nm}\bigl((nm)^{-1}\bigr)}\bigl|   
\frac{\sum_{i=1}^n \sum_{j=1}^{m_i} \frac{1}{nm_i} g\bigl(x_{ij}\bigr) \gamma_i\bigl(x_{ij}\bigr) \xi_{ij}}{\eta\|g\|_{nm}^2+ \tau C_\eta \phi_{nm} \log^{5.3}(nm)}\bigr|\bigr)+
\\
&{\E}_{\xi|x, \gamma }\bigl(\sup _{g \in \mathcal{F} \cap B_{nm}\bigl((nm)^{-1}, \infty\bigr)}\bigl|   
\frac{\sum_{i=1}^n \sum_{j=1}^{m_i} \frac{1}{nm_i} g\bigl(x_{ij}\bigr) \gamma_i\bigl(x_{ij}\bigr) \xi_{ij}}{\eta\|g\|_{nm}^2+\tau C_\eta \phi_{nm} \log^{5.3}(nm)}\bigr|\bigr)
\\
&\le \frac{\tau^2}{C_\eta (nm)^2\phi_{nm}\log^{5.3}(nm)} + 1   
+\frac{\sqrt \pi C_2}{\sqrt{ c nm \phi_{nm} \log^{4.3}(nm)}} \le \frac{3}{2}.
\end{align*}
This implies that 
\begin{align*}
&2\E\left( \sup_{ g \in \mathcal F  }\big |  \frac{ \sum_{i=1}^n g(x_{ij}) \gamma_i(x_{ij}) \xi_{ij} }{ \eta \|g\|_\lt ^2 +C_\eta\phi_{nm} \log^{5.3}(nm)} \big | \right)
\\
&\le 2 {\E}_{x, \gamma}\bigl[\sup _{g \in \mathcal{F}} \frac{\eta\|g\|_{nm}^2+ \tau C_\eta \phi_{nm} \log^{5.3}(nm)}{\eta\|g\|_{\mathcal{L}_2}^2+ \tau C_\eta \phi_{nm} \log^{5.3}(nm)} {\E}_{\xi|x, \gamma }   
\bigl(\sup _{g \in \mathcal{F}}\bigl|\frac{\sum_{i=1}^n \sum_{j=1}^{m_i} \frac{1}{nm_i} g\bigl(x_{ij}\bigr) \gamma_i\bigl(x_{ij}\bigr) \xi_{ij}}{\eta\|g\|_{nm}^2+ \tau C_\eta \phi_{nm} \log^{5.3}(nm)}\bigr| \bigr) \bigr]
\\
&\le  3{\E}_{x}\left(\sup_{ g \in \mathcal F }  \frac{\eta \|g\|_{nm} ^2 +C_\eta\phi_{nm} \log^{5.3}(nm) }{\eta \|g\|_\lt  ^2 +C_\eta\phi_{nm} \log^{5.3}(nm) }      \right).
\end{align*}
\
\\
{\bf Step 2.} 
Note that since $0\le \eta \le 1/4$
\begin{align*}
&{\E}_{x}\left(\sup_{ g \in \mathcal F  }  \frac{\eta \|g\|_{nm} ^2 +C_\eta\phi_{nm} \log^{5.3}(nm) }{\eta \|g\|_\lt  ^2 +C_\eta\phi_{nm} \log^{5.3}(nm) }      \right)
\\
&\le \eta{\E}_{x}\bigl(\sup_{ g \in \mathcal F } \big| \frac{ \|g\|_{nm} ^2 -\|g\|_\lt^2 }{\eta \|g\|_\lt ^2 +C_\eta\phi_{nm} \log^{5.3}(nm) } \big| \bigr) + 1   
\le 12\eta+1 =4.
\end{align*}
where the first inequality follows from \eqref{eq:rademacher middle steps mn}, and the second inequality follows from  \Cref{coro:deviation for norms mn}.
Thus 
\begin{align*}
   &\E\left( \sup_{ g \in \mathcal F  }\big |  \frac{\sum_{i=1}^n \sum_{j=1}^{m_i} \frac{1}{nm_i} \left[ g(x_{ij}) \gamma_i(x_{ij}) - \langle g, \gamma_i\rangle_{\lt} \right] }{ \eta \|g\|_\lt ^2 + \tau C_\eta\phi_{nm} \log^{5.3}(nm)} \big | \right) \\ 
   &\leq  2\E\left( \sup_{ g \in \mathcal F  }\big |  \frac{ \sum_{i=1}^n g(x_{ij}) \gamma_i(x_{ij}) \xi_{ij} }{ \eta \|g\|_\lt ^2 +C_\eta\phi_{nm} \log^{5.3}(nm)} \big | \right)       \le 12. 
     \end{align*}
  \
  \\

    {\bf Step 3.} 
   By Markov's inequality,
    \begin{align*}
&\p\bigl( \sup_{ g \in \mathcal F } \bigl|\frac{ \sum_{i=1}^n \sum_{j=1}^{m_i} \frac{1}{nm_i} \bigl[g(x_{ij}) \gamma_i(x_{ij}) - \langle g, \gamma_i\rangle_{\lt} \bigr]}{ \|g\|_\lt ^2 + \tau C_\eta'\phi_{nm} \log^{5.3}(nm)}\bigr|   
\ge \frac{1}{16}\bigr)
\\
&=\p\bigl( \sup_{ g \in \mathcal F } \bigl|\frac{ \sum_{i=1}^n \sum_{j=1}^{m_i} \frac{1}{nm_i} \bigl[g(x_{ij}) \gamma_i(x_{ij}) - \langle g, \gamma_i\rangle_{\lt} \bigr]}{ \eta\|g\|_\lt ^2 + \tau \eta C_\eta'\phi_{nm} \log^{5.3}(nm)}\bigr|   
\ge \frac{1}{16\eta }\bigr)
\\
&\le 16\eta \E\bigl( \sup_{ g \in \mathcal F } \bigl|   
\frac{ \sum_{i=1}^n \sum_{j=1}^{m_i} \frac{1}{nm_i} \bigl[g(x_{ij}) \gamma_i(x_{ij}) - \langle g, \gamma_i\rangle_{\lt} \bigr]}{ \|g\|_\lt ^2 + \tau C_\eta\phi_{nm} \log^{5.3}(nm)}\bigr| \bigr) \le 192 \eta,
\end{align*}
where $\eta C_\eta' = C_\eta$.
\end{proof}

\
\\

\begin{lemma} \label{lemma:sub-Gaussian conditioning}
Suppose $X$ is a sub-Gaussian random variable and that for some constant $c$, it holds that
\begin{align*}
&\E(\exp(cX^2)) < \infty .
\end{align*} For any $t>0$,  let 
$\mathcal E = \{|X|\le t  \}$  and suppose that $ \p( \mathcal E  )\ge 1/2$. Then 
$ X| \mathcal E    $ is also a sub-Gaussian random variable.
    
\end{lemma}
\begin{proof}
For any $A \subset \mathbb{R}$.
Observe that 
\begin{align*}
&\p(X \in A | \mathcal E) = \frac{ \p( X\in A \text{ and } |X|\le t ) }{\p(|X|\le t)} \le 2 \p(X\in A).
\end{align*}
    \begin{align*}
      \text{Therefore } & \E(\exp(cX^2) |\mathcal E ) =\int_0^\infty \p(\exp(cX^2)  \ge s | \mathcal E)  \\ \le &\int_0^\infty 2 \p(\exp(cX^2)  \ge s  ) ds  = 2   \E(\exp(cX^2)   )     < \infty  .
    \end{align*}
\end{proof}

\newpage 
\subsection{Deviation Bounds for Temporal Dependence with $\beta$-mixing Conditions}
\label{proof temporal-spatial model depdent}

Now we address the general case of $\epsilon$ and $x$. The proof strategy involves constructing 'ghost' versions of $\epsilon$ and $x$ that are independent but closely mimic the behavior of the original $\epsilon$ and $x$. The details are provided below.

\subsubsection{The $\beta$-mixing Lemmas for Theorem 1}\label{beta mixing proof temporal-spatial model}

We begin by recalling the notation for sigma-fields and \(\beta\)-mixing that will be used throughout this subsection.

Let \(X\) be a random variable taking values in \(\mathbb R^d\). We denote by \(\sigma(X)\) the sigma-field generated by \(X\). More generally, for a collection of random variables \(\{X_t:t\in I\}\), where \(I\subseteq \mathbf Z^+\), we write
\begin{align*}
&\sigma(X_t,\ t\in I)
\end{align*}
for the sigma-field they generate.

For two sigma-fields \(\mathcal A\) and \(\mathcal B\), define the \(\beta\)-mixing coefficient by
\begin{align*}
&\beta(\mathcal A,\mathcal B)
=
\mathbb E\Big[\sup_{C\in \mathcal B}\big|\mathbb P(C)-\mathbb P(C\mid \mathcal A)\big|\Big].
\end{align*}
A sequence of sigma-fields \(\{\mathcal F_t\}_{t\in\mathbf Z^+}\) is said to be \(\beta\)-mixing if
\begin{align*}
&\beta_s = \sup_{t\in\mathbf Z^+} \beta\!\bigl( \sigma(\mathcal F_\ell,\,\ell\le t),\, \sigma(\mathcal F_\ell,\,\ell   
\ge t+s) \bigr) \longrightarrow 0 \qquad\text{as } s\to\infty.
\end{align*}
Accordingly, a sequence of random variables \(\{X_t\}_{t\in\mathbf Z^+}\) is called \(\beta\)-mixing if the associated sequence of sigma-fields \(\{\sigma(X_t)\}_{t\in\mathbf Z^+}\) is \(\beta\)-mixing.

A important Theorem we will use is present below.

\color{black}

\begin{theorem}[{\cite[Theorem 1, Section 1.2.1]{doukhan2012mixing}}]
\label{lemma:coupling in the temporal-spatial model}
Let $E$ and $F$ be two Polish spaces, and let $(X,Y)$ be an $E\times F$-valued random element defined on a probability space $(\Omega,\mathcal F,\mathbb P)$. Set
\begin{align*}
&\beta=\beta(\sigma(X),\sigma(Y)).
\end{align*}
Assume that the probability space is rich enough to support a random variable
\begin{align*}
&\Delta \sim \mathrm{Unif}[0,1]
\end{align*}
that is independent of $\sigma(X,Y)$.

Then there exists an $F$-valued random element $Y^*$ such that:
\begin{enumerate}
    \item $Y^*$ has the same distribution as $Y$;
    \item $Y^*$ is independent of $X$;
    \item
    \begin{align*}
&\mathbb P(Y\neq Y^*)=\beta.
\end{align*}
\end{enumerate}
Moreover, there exists a measurable function
\begin{align*}
&f:E\times F\times [0,1]\to F
\end{align*}
such that
\begin{align*}
&Y^*=f(X,Y,\Delta).
\end{align*}
\end{theorem}

Throughout this subsection, for any Polish space \(E\), we write \(\mathcal B(E)\) for its Borel sigma-field. In particular, \(\mathcal B(\mathbb R^d)\) denotes the Borel sigma-field on \(\mathbb R^d\).

Although Theorem~\ref{lemma:coupling in the temporal-spatial model} only requires the existence of a single auxiliary random variable $\Delta$ independent of $\sigma(X,Y)$, our proof applies this coupling argument repeatedly at different indices. To ensure that each application uses fresh independent randomness, we enlarge the probability space, if necessary, so that it carries an i.i.d.\ sequence
\begin{align*}
&\{\Delta_i\}_{i\ge 1}
\end{align*}
of $\mathrm{Unif}[0,1]$ random variables that are independent of the original process.

One convenient construction is to work on the product space
\begin{align*}
&\bigl(\widetilde\Omega,\widetilde{\mathcal F},\widetilde{\mathbb P}\bigr)
=
\bigl(\Omega\times [0,1]^{\mathbb N},
\ \mathcal F\otimes \mathcal B([0,1]^{\mathbb N}),
\ \mathbb P\otimes \lambda^{\otimes \mathbb N}\bigr),
\end{align*}
where $\lambda$ denotes Lebesgue measure on $[0,1]$, and define
\begin{align*}
&\Delta_i(\omega,u_1,u_2,\ldots)=u_i.
\end{align*}
Then $\{\Delta_i\}_{i\ge1}$ are i.i.d.\ $\mathrm{Unif}[0,1]$ and are independent of the sigma-field generated by the original coordinate, namely
\begin{align*}
&\mathcal F\otimes \{\emptyset,[0,1]^{\mathbb N}\}.
\end{align*}

For integers \(a,b\) and \(m\ge1\), we write
\begin{align*}
&a\equiv b \pmod m
\end{align*}
to mean that \(m\) is a divisor \(a-b\).

For any $S \in \{1,\ldots, n\}$ and
\begin{align*}
&K=\lfloor n/S \rfloor,
\end{align*}
define
\[
\mathcal I_k =
\begin{cases}
((k-1)S,\ kS] & \text{if } 1\le k\le K,\\[0.3em]
(KS,\ n] & \text{if } k=K+1.
\end{cases}
\]
For $0\le s\le S-1$, define
\begin{align*}
&\mathcal J_{e,s}
=
\{\, i\in \mathcal I_k : k \text{ is even and } i \equiv s \ (\mathrm{mod}\ S)\,\},
\end{align*}
and
\begin{align*}
&\mathcal J_{o,s}
=
\{\, i\in \mathcal I_k : k \text{ is odd and } i \equiv s \ (\mathrm{mod}\ S)\,\}.
\end{align*}
Then
\begin{align}
&\frac{K-1}{2}\le |\mathcal J_{e,s}| \le \frac{K+2}{2}, \qquad \frac{K-1}{2}   
\le |\mathcal J_{o,s}| \le \frac{K+2}{2},\label{eq:effective sample size in the temporal-spatial model}
\end{align}
and
\begin{align*}
&\bigcup_{s=0}^{S-1}\mathcal J_{e,s}\ \cup\ \bigcup_{s=0}^{S-1}\mathcal J_{o,s}
=
\{1,\ldots,n\}.
\end{align*}

We first apply the coupling theorem to the covariate process. The next lemma shows that, on each subcollection \(\mathcal J_{e,s}\) or \(\mathcal J_{o,s}\), one can construct coupled covariate blocks that preserve the marginal distribution of the original blocks, are jointly independent, and differ from the originals only with probability at most \(\beta(S)\).
\begin{lemma}
\label{lemma:beta mixing x}
Let
\begin{align*}
&\sigma_x(i):=\sigma\bigl((x_{ij})_{j=1}^{m_i}\bigr),\qquad i=1,\ldots,n,
\end{align*}
and suppose that the sequence of sigma-fields \(\{\sigma_x(i)\}_{i=1}^n\) is \(\beta\)-mixing.

Fix \(s\in\{0,\ldots,S-1\}\). Then there exists a collection
\begin{align*}
&\Bigl\{(x_{ij}^*)_{j=1}^{m_i}: i\in \mathcal J_{e,s}\Bigr\}
\end{align*}
such that:
\begin{enumerate}
    \item for each \(i\in \mathcal J_{e,s}\), the vector \((x_{ij}^*)_{j=1}^{m_i}\) has the same distribution as \((x_{ij})_{j=1}^{m_i}\);
    \item the family
    \begin{align*}
&\Bigl\{(x_{ij}^*)_{j=1}^{m_i}: i\in \mathcal J_{e,s}\Bigr\}
\end{align*}
    is jointly independent;
    \item for each \(i\in \mathcal J_{e,s}\),
    \begin{align*}
&\mathbb P\bigl(x_{ij}^*\neq x_{ij}\text{ for some }j\in\{1,\ldots,m_i\}\bigr)
    \le \beta(S).
\end{align*}
\end{enumerate}
The same conclusion holds with \(\mathcal J_{e,s}\) replaced by \(\mathcal J_{o,s}\).
\end{lemma}

\begin{proof}
We prove the result for \(\mathcal J_{e,s}\); the proof for \(\mathcal J_{o,s}\) is identical.

For each \(i\in\{1,\ldots,n\}\), define the past sigma-field
\begin{align*}
&\mathcal F_{i-S}^x
=
\sigma\bigl((x_{tj})_{j=1}^{m_t}: t\le i-S\bigr),
\end{align*}
with the convention that \(\mathcal F_{i-S}^x\) is trivial when \(i-S\le 0\).

For each \(i\), let
\begin{align*}
&E_i^x:=\prod_{t\le i-S}\mathbb R^{m_t},
\qquad
F_i^x:=\mathbb R^{m_i},
\end{align*}
and define the random elements
\begin{align*}
&\mathbf X_i^{\mathrm{past}} := \Bigl((x_{tj})_{j=1}^{m_t}\Bigr)_{t\le i-S}\in E_i^x, \qquad \mathbf Y_i^x    
:= (x_{ij})_{j=1}^{m_i}\in F_i^x.
\end{align*}
Since \(E_i^x\) is a countable product of Polish spaces and \(F_i^x\) is a finite-dimensional Euclidean space, both \(E_i^x\) and \(F_i^x\) are Polish spaces. Moreover,
\begin{align*}
&\sigma(\mathbf X_i^{\mathrm{past}})=\mathcal F_{i-S}^x,
\qquad
\sigma(\mathbf Y_i^x)=\sigma\bigl((x_{ij})_{j=1}^{m_i}\bigr).
\end{align*}

Apply Theorem~\ref{lemma:coupling in the temporal-spatial model} to the \(E_i^x\times F_i^x\)-valued random element
\begin{align*}
&(\mathbf X_i^{\mathrm{past}},\mathbf Y_i^x).
\end{align*}
Since the sequence \(\{\sigma_x(i)\}\) is \(\beta\)-mixing, we have
\begin{align*}
&\beta\bigl(\mathcal F_{i-S}^x,\sigma((x_{ij})_{j=1}^{m_i})\bigr)\le \beta(S).
\end{align*}
Also, by construction of the enlarged probability space, \(\Delta_i\sim \mathrm{Unif}[0,1]\) and \(\Delta_i\) is independent of
\begin{align*}
&\sigma(\mathbf X_i^{\mathrm{past}},\mathbf Y_i^x).
\end{align*}
Therefore, Theorem~\ref{lemma:coupling in the temporal-spatial model} yields an \(F_i^x\)-valued random element
\begin{align*}
&\mathbf Y_i^{x,*}:=(x_{ij}^*)_{j=1}^{m_i}
\end{align*}
and a measurable map
\begin{align*}
&f_i^x:E_i^x\times F_i^x\times [0,1]\to F_i^x
\end{align*}
such that
\begin{align*}
&(x_{ij}^*)_{j=1}^{m_i}
=
f_i^x\bigl(\mathbf X_i^{\mathrm{past}},(x_{ij})_{j=1}^{m_i},\Delta_i\bigr),
\end{align*}
and:
\begin{enumerate}
    \item \((x_{ij}^*)_{j=1}^{m_i}\) has the same distribution as \((x_{ij})_{j=1}^{m_i}\);
    \item \((x_{ij}^*)_{j=1}^{m_i}\) is independent of \(\mathbf X_i^{\mathrm{past}}\), hence independent of \(\mathcal F_{i-S}^x\);
    \item
    \begin{align*}
&\mathbb P\bigl((x_{ij}^*)_{j=1}^{m_i}\neq (x_{ij})_{j=1}^{m_i}\bigr)   
= \beta\bigl(\mathcal F_{i-S}^x,\sigma((x_{ij})_{j=1}^{m_i})\bigr) \le \beta(S).
\end{align*}
\end{enumerate}
In particular,
\begin{align}
&\mathbb P\bigl(x_{ij}^*\neq x_{ij}\text{ for some }j\in\{1,\ldots,m_i\}\bigr)\le \beta(S).\label{eq:coupling error in the temporal-spatial model}
\end{align}
Because the auxiliary variables \(\{\Delta_i\}_{i\ge1}\) are i.i.d.\ and independent of the original process, each coupled block is generated using fresh external randomness.

We now prove joint independence on \(\mathcal J_{e,s}\).

Write the elements of \(\mathcal J_{e,s}\) in increasing order as
\begin{align*}
&i_1<i_2<\cdots<i_M.
\end{align*}
By construction of \(\mathcal J_{e,s}\),
\begin{align*}
&i_{r+1}-i_r\ge 2S,
\qquad r=1,\ldots,M-1.
\end{align*}

We prove by induction on \(m\in\{1,\ldots,M\}\) that the family
\begin{align*}
&\Bigl\{(x_{i_rj}^*)_{j=1}^{m_{i_r}}:r=1,\ldots,m\Bigr\}
\end{align*}
is jointly independent.

The case \(m=1\) is immediate.

Assume the claim holds for some \(m<M\). We prove it for \(m+1\).
Let
\begin{align*}
&A_r\in\mathcal B(\mathbb R^{m_{i_r}}),
\qquad r=1,\ldots,m+1,
\end{align*}
and write, for brevity,
\begin{align*}
&x_{i_r}^*:=(x_{i_rj}^*)_{j=1}^{m_{i_r}}.
\end{align*}

Since
\begin{align*}
&i_m\le i_{m+1}-2S<i_{m+1}-S,
\end{align*}
all original variables involved in the construction of
\begin{align*}
&x_{i_1}^*,\ldots,x_{i_m}^*
\end{align*}
belong to the past up to time \(i_{m+1}-S\). Therefore,
\begin{align*}
&\sigma(x_{i_1}^*,\ldots,x_{i_m}^*)
\subseteq
\mathcal G_m^x,
\end{align*}
where
\begin{align*}
&\mathcal G_m^x
:=
\sigma\bigl(\mathcal F_{i_{m+1}-S}^x,\Delta_{i_1},\ldots,\Delta_{i_m}\bigr).
\end{align*}

On the other hand,
\begin{align*}
&x_{i_{m+1}}^*
=
f_{i_{m+1}}^x\bigl(\mathbf X_{i_{m+1}}^{\mathrm{past}},(x_{i_{m+1}j})_{j=1}^{m_{i_{m+1}}},\Delta_{i_{m+1}}\bigr).
\end{align*}
By the coupling theorem, \(x_{i_{m+1}}^*\) is independent of \(\mathcal F_{i_{m+1}-S}^x\).

Moreover, since the sequence \(\{\Delta_i\}_{i\ge1}\) is i.i.d.\ and independent of the original process, the sigma-field
\begin{align*}
&\sigma(\Delta_{i_1},\ldots,\Delta_{i_m})
\end{align*}
is independent of
\begin{align*}
&\sigma\bigl(\mathbf X_{i_{m+1}}^{\mathrm{past}},(x_{i_{m+1}j})_{j=1}^{m_{i_{m+1}}},\Delta_{i_{m+1}}\bigr).
\end{align*}
Since 
\begin{align*}
&x_{i_{m+1}}^*
=
f_{i_{m+1}}^x\bigl(\mathbf X_{i_{m+1}}^{\mathrm{past}},(x_{i_{m+1}j})_{j=1}^{m_{i_{m+1}}},\Delta_{i_{m+1}}\bigr),
\end{align*}
it is measurable with respect to
\begin{align*}
&\sigma\bigl(\mathbf X_{i_{m+1}}^{\mathrm{past}},(x_{i_{m+1}j})_{j=1}^{m_{i_{m+1}}},\Delta_{i_{m+1}}\bigr).
\end{align*}
By construction, this sigma-field is independent of
\begin{align*}
&\sigma(\Delta_{i_1},\ldots,\Delta_{i_m}),
\end{align*}
and Theorem~\ref{lemma:coupling in the temporal-spatial model} ensures independence from 
\(\mathcal F_{i_{m+1}-S}^x\). Therefore,
\(x_{i_{m+1}}^*\) is independent of
\begin{align*}
&\mathcal G_m^x
=
\sigma\bigl(\mathcal F_{i_{m+1}-S}^x,\Delta_{i_1},\ldots,\Delta_{i_m}\bigr),
\end{align*}
and hence independent of \(\sigma(x_{i_1}^*,\ldots,x_{i_m}^*)\).

Therefore,
\begin{align*}
&\mathbb P\Bigl(
x_{i_1}^*\in A_1,\ldots,x_{i_m}^*\in A_m,x_{i_{m+1}}^*\in A_{m+1}
\Bigr)
\\
&=
\mathbb P\Bigl(
x_{i_1}^*\in A_1,\ldots,x_{i_m}^*\in A_m
\Bigr)
\,
\mathbb P\bigl(x_{i_{m+1}}^*\in A_{m+1}\bigr).
\end{align*}
By the induction hypothesis,
\begin{align*}
&\mathbb P\Bigl(
x_{i_1}^*\in A_1,\ldots,x_{i_m}^*\in A_m
\Bigr)
=
\prod_{r=1}^m \mathbb P\bigl(x_{i_r}^*\in A_r\bigr),
\end{align*}
and hence
\begin{align*}
&\mathbb P\Bigl(
x_{i_1}^*\in A_1,\ldots,x_{i_{m+1}}^*\in A_{m+1}
\Bigr)
=
\prod_{r=1}^{m+1}\mathbb P\bigl(x_{i_r}^*\in A_r\bigr).
\end{align*}
This proves that
\begin{align*}
&\Bigl\{(x_{i_rj}^*)_{j=1}^{m_{i_r}}:r=1,\ldots,m+1\Bigr\}
\end{align*}
is jointly independent, completing the induction.

Thus the family
\begin{align*}
&\Bigl\{(x_{ij}^*)_{j=1}^{m_i}: i\in \mathcal J_{e,s}\Bigr\}
\end{align*}
is jointly independent, and for each \(i\in\mathcal J_{e,s}\), the block \((x_{ij}^*)_{j=1}^{m_i}\) has the same distribution as \((x_{ij})_{j=1}^{m_i}\). Moreover,
\begin{align*}
&\mathbb P\bigl(x_{ij}^*\neq x_{ij}\text{ for some }j\in\{1,\ldots,m_i\}\bigr)\le \beta(S).
\end{align*}
This proves the lemma for \(\mathcal J_{e,s}\). The proof for \(\mathcal J_{o,s}\) is identical.
\end{proof}


\begin{lemma}
\label{Ind-cop-measerrors}
Let
\begin{align*}
&\sigma_\epsilon(i):=\sigma\bigl((\epsilon_{ij})_{j=1}^{m_i}\bigr), \qquad i=1,\ldots,n,
\end{align*}
and suppose that the sequence of sigma-fields \(\{\sigma_\epsilon(i)\}_{i=1}^n\) is \(\beta\)-mixing.

Fix \(s\in\{0,\ldots,S-1\}\). Then there exists a collection
\begin{align*}
&\Bigl\{(\epsilon_{ij}^*)_{j=1}^{m_i}: i\in \mathcal J_{e,s}\Bigr\}
\end{align*}
such that:
\begin{enumerate}
    \item for each \(i\in \mathcal J_{e,s}\), the vector \((\epsilon_{ij}^*)_{j=1}^{m_i}\) has the same distribution as \((\epsilon_{ij})_{j=1}^{m_i}\);
    \item the family
    \begin{align*}
&\Bigl\{(\epsilon_{ij}^*)_{j=1}^{m_i}: i\in \mathcal J_{e,s}\Bigr\}
\end{align*}
    is jointly independent;
    \item for each \(i\in \mathcal J_{e,s}\),
    \begin{align*}
&\mathbb P\bigl(\epsilon_{ij}^*\neq \epsilon_{ij}\text{ for some }j\in\{1,\ldots,m_i\}\bigr)
    \le \beta(S).
\end{align*}
\end{enumerate}
The same conclusion holds with \(\mathcal J_{e,s}\) replaced by \(\mathcal J_{o,s}\).
\end{lemma}

\begin{proof}
We prove the result for \(\mathcal J_{e,s}\); the proof for \(\mathcal J_{o,s}\) is identical.

For each \(i\in\{1,\ldots,n\}\), define the past sigma-field
\begin{align*}
&\mathcal F_{i-S}^{\epsilon}
=
\sigma\bigl((\epsilon_{tj})_{j=1}^{m_t}: t\le i-S\bigr),
\end{align*}
with the convention that \(\mathcal F_{i-S}^{\epsilon}\) is trivial when \(i-S\le 0\).

For each \(i\), let
\begin{align*}
&E_i^{\epsilon}:=\prod_{t\le i-S}\mathbb R^{m_t},
\qquad
F_i^{\epsilon}:=\mathbb R^{m_i},
\end{align*}
and define the random elements
\begin{align*}
&\mathbf X_i^{\epsilon,\mathrm{past}} := \Bigl((\epsilon_{tj})_{j=1}^{m_t}\Bigr)_{t\le i-S}\in E_i^{\epsilon}, \qquad \mathbf Y_i^{\epsilon}    
:= (\epsilon_{ij})_{j=1}^{m_i}\in F_i^{\epsilon}.
\end{align*}
Since \(E_i^{\epsilon}\) is a countable product of Polish spaces and \(F_i^{\epsilon}\) is a finite-dimensional Euclidean space, both \(E_i^{\epsilon}\) and \(F_i^{\epsilon}\) are Polish spaces.
Moreover,
\begin{align*}
&\sigma(\mathbf X_i^{\epsilon,\mathrm{past}})=\mathcal F_{i-S}^{\epsilon},
\qquad
\sigma(\mathbf Y_i^{\epsilon})=\sigma\bigl((\epsilon_{ij})_{j=1}^{m_i}\bigr).
\end{align*}

Apply Theorem~\ref{lemma:coupling in the temporal-spatial model} to the \(E_i^{\epsilon}\times F_i^{\epsilon}\)-valued random element
\begin{align*}
&(\mathbf X_i^{\epsilon,\mathrm{past}},\mathbf Y_i^{\epsilon}).
\end{align*}
Since the sequence of sigma-fields \(\{\sigma_\epsilon(i)\}_{i=1}^n\) is \(\beta\)-mixing, we have
\begin{align*}
&\beta\bigl(\mathcal F_{i-S}^{\epsilon},\sigma((\epsilon_{ij})_{j=1}^{m_i})\bigr)\le \beta(S).
\end{align*}
Also, by construction of the enlarged probability space, \(\Delta_i\sim \mathrm{Unif}[0,1]\) and \(\Delta_i\) is independent of
\begin{align*}
&\sigma(\mathbf X_i^{\epsilon,\mathrm{past}},\mathbf Y_i^{\epsilon}).
\end{align*}
Therefore, Theorem~\ref{lemma:coupling in the temporal-spatial model} yields an \(F_i^{\epsilon}\)-valued random element
\begin{align*}
&\mathbf Y_i^{\epsilon,*}:=(\epsilon_{ij}^*)_{j=1}^{m_i}
\end{align*}
and a measurable map
\begin{align*}
&f_i^{\epsilon}:E_i^{\epsilon}\times F_i^{\epsilon}\times [0,1]\to F_i^{\epsilon}
\end{align*}
such that
\begin{align*}
&(\epsilon_{ij}^*)_{j=1}^{m_i}
=
f_i^{\epsilon}\bigl(\mathbf X_i^{\epsilon,\mathrm{past}},(\epsilon_{ij})_{j=1}^{m_i},\Delta_i\bigr),
\end{align*}
and:
\begin{enumerate}
    \item \((\epsilon_{ij}^*)_{j=1}^{m_i}\) has the same distribution as \((\epsilon_{ij})_{j=1}^{m_i}\);
    \item \((\epsilon_{ij}^*)_{j=1}^{m_i}\) is independent of \(\mathbf X_i^{\epsilon,\mathrm{past}}\), hence independent of \(\mathcal F_{i-S}^{\epsilon}\);
    \item
    \begin{align*}
&\mathbb P\bigl((\epsilon_{ij}^*)_{j=1}^{m_i}\neq (\epsilon_{ij})_{j=1}^{m_i}\bigr) 
=\beta\bigl(\mathcal F_{i-S}^{\epsilon},\sigma((\epsilon_{ij})_{j=1}^{m_i})\bigr)
    \le \beta(S).
\end{align*}
\end{enumerate}
In particular,
\begin{align*}
&\mathbb P\bigl(\epsilon_{ij}^*\neq \epsilon_{ij}\text{ for some }j\in\{1,\ldots,m_i\}\bigr)\le \beta(S).
\end{align*}

Because the auxiliary variables \(\{\Delta_i\}_{i\ge1}\) are i.i.d.\ and independent of the original process, each coupled block is generated using fresh external randomness.

We now prove joint independence on \(\mathcal J_{e,s}\).

Write the elements of \(\mathcal J_{e,s}\) in increasing order as
\begin{align*}
&i_1<i_2<\cdots<i_M.
\end{align*}
By construction of \(\mathcal J_{e,s}\),
\begin{align*}
&i_{r+1}-i_r\ge 2S,
\qquad r=1,\ldots,M-1.
\end{align*}

We prove by induction on \(m\in\{1,\ldots,M\}\) that the family
\begin{align*}
&\Bigl\{(\epsilon_{i_rj}^*)_{j=1}^{m_{i_r}}:r=1,\ldots,m\Bigr\}
\end{align*}
is jointly independent.

The case \(m=1\) is immediate.

Assume the claim holds for some \(m<M\). We prove it for \(m+1\).
Let
\begin{align*}
&A_r\in\mathcal B(\mathbb R^{m_{i_r}}),
\qquad r=1,\ldots,m+1,
\end{align*}
and write, for brevity,
\begin{align*}
&\epsilon_{i_r}^*:=(\epsilon_{i_rj}^*)_{j=1}^{m_{i_r}}.
\end{align*}

Since
\begin{align*}
&i_m\le i_{m+1}-2S<i_{m+1}-S,
\end{align*}
all original variables involved in the construction of
\begin{align*}
&\epsilon_{i_1}^*,\ldots,\epsilon_{i_m}^*
\end{align*}
belong to the past up to time \(i_{m+1}-S\). Therefore,
\begin{align*}
&\sigma(\epsilon_{i_1}^*,\ldots,\epsilon_{i_m}^*)
\subseteq
\mathcal G_m^{\epsilon},
\end{align*}
where
\begin{align*}
&\mathcal G_m^{\epsilon}
:=
\sigma\bigl(\mathcal F_{i_{m+1}-S}^{\epsilon},\Delta_{i_1},\ldots,\Delta_{i_m}\bigr).
\end{align*}

On the other hand,
\begin{align*}
&\epsilon_{i_{m+1}}^*
=
f_{i_{m+1}}^{\epsilon}\bigl(\mathbf X_{i_{m+1}}^{\epsilon,\mathrm{past}},(\epsilon_{i_{m+1}j})_{j=1}^{m_{i_{m+1}}},\Delta_{i_{m+1}}\bigr).
\end{align*}
By the coupling theorem, \(\epsilon_{i_{m+1}}^*\) is independent of \(\mathcal F_{i_{m+1}-S}^{\epsilon}\).

Moreover, since the sequence \(\{\Delta_i\}_{i\ge1}\) is i.i.d.\ and independent of the original process, the sigma-field
\begin{align*}
&\sigma(\Delta_{i_1},\ldots,\Delta_{i_m})
\end{align*}
is independent of
\begin{align*}
&\sigma\bigl(\mathbf X_{i_{m+1}}^{\epsilon,\mathrm{past}},(\epsilon_{i_{m+1}j})_{j=1}^{m_{i_{m+1}}},\Delta_{i_{m+1}}\bigr).
\end{align*}
Since 
\begin{align*}
&\epsilon_{i_{m+1}}^*
=
f_{i_{m+1}}^{\epsilon}\bigl(\mathbf X_{i_{m+1}}^{\epsilon,\mathrm{past}},(\epsilon_{i_{m+1}j})_{j=1}^{m_{i_{m+1}}},\Delta_{i_{m+1}}\bigr),
\end{align*}
it is measurable with respect to
\begin{align*}
&\sigma\bigl(\mathbf X_{i_{m+1}}^{\epsilon,\mathrm{past}},(\epsilon_{i_{m+1}j})_{j=1}^{m_{i_{m+1}}},\Delta_{i_{m+1}}\bigr).
\end{align*}
By construction, this sigma-field is independent of
\begin{align*}
&\sigma(\Delta_{i_1},\ldots,\Delta_{i_m}),
\end{align*}
and Theorem~\ref{lemma:coupling in the temporal-spatial model} ensures independence from 
\(\mathcal F_{i_{m+1}-S}^{\epsilon}\). Therefore,
\(\epsilon_{i_{m+1}}^*\) is independent of
\begin{align*}
&\mathcal G_m^{\epsilon}
=
\sigma\bigl(\mathcal F_{i_{m+1}-S}^{\epsilon},\Delta_{i_1},\ldots,\Delta_{i_m}\bigr),
\end{align*}
and hence independent of \(\sigma(\epsilon_{i_1}^*,\ldots,\epsilon_{i_m}^*)\).

Therefore,
\begin{align*}
&\mathbb P\Bigl(
\epsilon_{i_1}^*\in A_1,\ldots,\epsilon_{i_m}^*\in A_m,\epsilon_{i_{m+1}}^*\in A_{m+1}
\Bigr)
\\
&=
\mathbb P\Bigl(
\epsilon_{i_1}^*\in A_1,\ldots,\epsilon_{i_m}^*\in A_m
\Bigr)
\,
\mathbb P\bigl(\epsilon_{i_{m+1}}^*\in A_{m+1}\bigr).
\end{align*}
By the induction hypothesis,
\begin{align*}
&\mathbb P\Bigl(
\epsilon_{i_1}^*\in A_1,\ldots,\epsilon_{i_m}^*\in A_m
\Bigr)
=
\prod_{r=1}^m \mathbb P\bigl(\epsilon_{i_r}^*\in A_r\bigr),
\end{align*}
and hence
\begin{align*}
&\mathbb P\Bigl(
\epsilon_{i_1}^*\in A_1,\ldots,\epsilon_{i_{m+1}}^*\in A_{m+1}
\Bigr)
=
\prod_{r=1}^{m+1}\mathbb P\bigl(\epsilon_{i_r}^*\in A_r\bigr).
\end{align*}
This proves that
\begin{align*}
&\Bigl\{(\epsilon_{i_rj}^*)_{j=1}^{m_{i_r}}:r=1,\ldots,m+1\Bigr\}
\end{align*}
is jointly independent, completing the induction.

Thus the family
\begin{align*}
&\Bigl\{(\epsilon_{ij}^*)_{j=1}^{m_i}: i\in \mathcal J_{e,s}\Bigr\}
\end{align*}
is jointly independent, and for each \(i\in\mathcal J_{e,s}\), the block \((\epsilon_{ij}^*)_{j=1}^{m_i}\) has the same distribution as \((\epsilon_{ij})_{j=1}^{m_i}\). Moreover,
\begin{align*}
&\mathbb P\bigl(\epsilon_{ij}^*\neq \epsilon_{ij}\text{ for some }j\in\{1,\ldots,m_i\}\bigr)\le \beta(S).
\end{align*}
\end{proof}

Before stating the corresponding lemma for the spatial noise, we introduce the following definition.

\begin{definition}[Independent and identically distributed random functions]
Let \(\mathcal H\) be a Polish space of real-valued functions on \([0,1]^d\), equipped with its Borel sigma-field \(\mathcal B(\mathcal H)\). A collection of \(\mathcal H\)-valued random elements
\begin{align*}
&\delta_1,\ldots,\delta_n
\end{align*}
is said to be independent and identically distributed if:
\begin{enumerate}
    \item the sigma-fields \(\sigma(\delta_1),\ldots,\sigma(\delta_n)\) are independent;
    \item for every \(A\in\mathcal B(\mathcal H)\) and every \(i,j\in\{1,\ldots,n\}\),
    \begin{align*}
&\mathbb P(\delta_i\in A)=\mathbb P(\delta_j\in A).
\end{align*}
\end{enumerate}
\end{definition}
\begin{lemma}
\label{Ind-cop-func}
Let \(\{\delta_i\}_{i=1}^n\) be identically distributed random functions taking values in a Polish function space \((\mathcal H,\mathcal B(\mathcal H))\), and suppose that the sequence of sigma-fields
\begin{align*}
&\sigma_\delta(i):=\sigma(\delta_i), \qquad i=1,\ldots,n,
\end{align*}
is \(\beta\)-mixing.

Fix \(s\in\{0,\ldots,S-1\}\). Then there exists a collection
\begin{align*}
&\{\delta_i^*\}_{i\in \mathcal J_{e,s}}
\end{align*}
such that:
\begin{enumerate}
    \item for each \(i\in \mathcal J_{e,s}\), the random function \(\delta_i^*\) has the same distribution as \(\delta_i\);
    \item the family
    \begin{align*}
&\{\delta_i^*\}_{i\in \mathcal J_{e,s}}
\end{align*}
    is jointly independent;
    \item for each \(i\in \mathcal J_{e,s}\),
    \begin{align*}
&\mathbb P(\delta_i^*\neq \delta_i)\le \beta(S).
\end{align*}
\end{enumerate}
The same conclusion holds with \(\mathcal J_{e,s}\) replaced by \(\mathcal J_{o,s}\).
\end{lemma}

\begin{proof}
We prove the result for \(\mathcal J_{e,s}\); the proof for \(\mathcal J_{o,s}\) is identical.

For each \(i\in\{1,\ldots,n\}\), define the past sigma-field
\begin{align*}
&\mathcal F_{i-S}^{\delta}
=
\sigma(\delta_t:t\le i-S),
\end{align*}
with the convention that \(\mathcal F_{i-S}^{\delta}\) is trivial when \(i-S\le 0\).

For each \(i\), let
\begin{align*}
&E_i^{\delta}:=\prod_{t\le i-S}\mathcal H,
\qquad
F_i^{\delta}:=\mathcal H,
\end{align*}
and define the random elements
\begin{align*}
&\mathbf X_i^{\delta,\mathrm{past}}
:=
(\delta_t)_{t\le i-S}\in E_i^{\delta},
\qquad
\mathbf Y_i^{\delta}
:=
\delta_i\in F_i^{\delta}.
\end{align*}
Since \(E_i^{\delta}\) is a countable product of Polish spaces and \(F_i^{\delta}=\mathcal H\) is Polish, both \(E_i^{\delta}\) and \(F_i^{\delta}\) are Polish spaces. Here the index \(t\le i-S\) indexes the coordinates of the product, so that
\begin{align*}
&E_i^{\delta}
=
\left\{ (h_t)_{t\le i-S} : h_t \in \mathcal H \text{ for all } t\le i-S \right\}.
\end{align*} Moreover,
\begin{align*}
&\sigma(\mathbf X_i^{\delta,\mathrm{past}})=\mathcal F_{i-S}^{\delta},
\qquad
\sigma(\mathbf Y_i^{\delta})=\sigma(\delta_i).
\end{align*}

Apply Theorem~\ref{lemma:coupling in the temporal-spatial model} to the \(E_i^{\delta}\times F_i^{\delta}\)-valued random element
\begin{align*}
&(\mathbf X_i^{\delta,\mathrm{past}},\mathbf Y_i^{\delta}).
\end{align*}
Since the sequence \(\{\sigma_\delta(i)\}\) is \(\beta\)-mixing, we have
\begin{align*}
&\beta\bigl(\mathcal F_{i-S}^{\delta},\sigma(\delta_i)\bigr)\le \beta(S).
\end{align*}
Also, by construction of the enlarged probability space, \(\Delta_i\sim \mathrm{Unif}[0,1]\) and \(\Delta_i\) is independent of
\begin{align*}
&\sigma(\mathbf X_i^{\delta,\mathrm{past}},\mathbf Y_i^{\delta}).
\end{align*}
Therefore, Theorem~\ref{lemma:coupling in the temporal-spatial model} yields an \(F_i^{\delta}\)-valued random element $\delta_i^*$, and a measurable map
\begin{align*}
&f_i^{\delta}:E_i^{\delta}\times F_i^{\delta}\times [0,1]\to F_i^{\delta}
\end{align*}
such that
\begin{align*}
&\delta_i^*
=
f_i^{\delta}\bigl(\mathbf X_i^{\delta,\mathrm{past}},\delta_i,\Delta_i\bigr),
\end{align*}
and:
\begin{enumerate}
    \item \(\delta_i^*\) has the same distribution as \(\delta_i\);
    \item \(\delta_i^*\) is independent of \(\mathbf X_i^{\delta,\mathrm{past}}\), hence independent of \(\mathcal F_{i-S}^{\delta}\);
    \item
    \begin{align*}
&\mathbb P(\delta_i^*\neq \delta_i)
    =
    \beta\bigl(\mathcal F_{i-S}^{\delta},\sigma(\delta_i)\bigr)
    \le \beta(S).
\end{align*}
\end{enumerate}

Because the auxiliary variables \(\{\Delta_i\}_{i\ge1}\) are i.i.d.\ and independent of the original process, each coupled random function is generated using fresh external randomness.

We now prove joint independence on \(\mathcal J_{e,s}\).

Write the elements of \(\mathcal J_{e,s}\) in increasing order as
\begin{align*}
&i_1<i_2<\cdots<i_M.
\end{align*}
By construction of \(\mathcal J_{e,s}\),
\begin{align*}
&i_{r+1}-i_r\ge 2S,
\qquad r=1,\ldots,M-1.
\end{align*}

We prove by induction on \(m\in\{1,\ldots,M\}\) that the family
\begin{align*}
&\{\delta_{i_1}^*,\ldots,\delta_{i_m}^*\}
\end{align*}
is jointly independent.

The case \(m=1\) is immediate.

Assume the claim holds for some \(m<M\). We prove it for \(m+1\).
Let
\begin{align*}
&B_r\in\mathcal B(\mathcal H),
\qquad r=1,\ldots,m+1.
\end{align*}

Since
\begin{align*}
&i_m\le i_{m+1}-2S<i_{m+1}-S,
\end{align*}
all original random functions involved in the construction of
\begin{align*}
&\delta_{i_1}^*,\ldots,\delta_{i_m}^*
\end{align*}
belong to the past up to time \(i_{m+1}-S\). Therefore,
\begin{align*}
&\sigma(\delta_{i_1}^*,\ldots,\delta_{i_m}^*)
\subseteq
\mathcal G_m^{\delta},
\end{align*}
where
\begin{align*}
&\mathcal G_m^{\delta}
:=
\sigma\bigl(\mathcal F_{i_{m+1}-S}^{\delta},\Delta_{i_1},\ldots,\Delta_{i_m}\bigr).
\end{align*}

On the other hand,
\begin{align*}
&\delta_{i_{m+1}}^*
=
f_{i_{m+1}}^{\delta}\bigl(\mathbf X_{i_{m+1}}^{\delta,\mathrm{past}},\delta_{i_{m+1}},\Delta_{i_{m+1}}\bigr).
\end{align*}
By the coupling theorem, \(\delta_{i_{m+1}}^*\) is independent of \(\mathcal F_{i_{m+1}-S}^{\delta}\).

Moreover, since the sequence \(\{\Delta_i\}_{i\ge1}\) is i.i.d.\ and independent of the original process, the sigma-field
\begin{align*}
&\sigma(\Delta_{i_1},\ldots,\Delta_{i_m})
\end{align*}
is independent of
\begin{align*}
&\sigma\bigl(\mathbf X_{i_{m+1}}^{\delta,\mathrm{past}},\delta_{i_{m+1}},\Delta_{i_{m+1}}\bigr).
\end{align*}
Since 
\begin{align*}
&\delta_{i_{m+1}}^*
=
f_{i_{m+1}}^{\delta}\bigl(\mathbf X_{i_{m+1}}^{\delta,\mathrm{past}},\delta_{i_{m+1}},\Delta_{i_{m+1}}\bigr),
\end{align*}
it is measurable with respect to
\begin{align*}
&\sigma\bigl(\mathbf X_{i_{m+1}}^{\delta,\mathrm{past}},\delta_{i_{m+1}},\Delta_{i_{m+1}}\bigr).
\end{align*}
By construction, this sigma-field is independent of
\begin{align*}
&\sigma(\Delta_{i_1},\ldots,\Delta_{i_m}),
\end{align*}
and Theorem~\ref{lemma:coupling in the temporal-spatial model} ensures independence from 
\(\mathcal F_{i_{m+1}-S}^{\delta}\). Therefore,
\(\delta_{i_{m+1}}^*\) is independent of
\begin{align*}
&\mathcal G_m^{\delta}
=
\sigma\bigl(\mathcal F_{i_{m+1}-S}^{\delta},\Delta_{i_1},\ldots,\Delta_{i_m}\bigr),
\end{align*}
and hence independent of \(\sigma(\delta_{i_1}^*,\ldots,\delta_{i_m}^*)\).

Therefore,
\begin{align*}
&\mathbb P\bigl(
\delta_{i_1}^*\in B_1,\ldots,\delta_{i_m}^*\in B_m,\delta_{i_{m+1}}^*\in B_{m+1}
\bigr)
\\
&=
\mathbb P\bigl(
\delta_{i_1}^*\in B_1,\ldots,\delta_{i_m}^*\in B_m
\bigr)
\,
\mathbb P\bigl(\delta_{i_{m+1}}^*\in B_{m+1}\bigr).
\end{align*}
By the induction hypothesis,
\begin{align*}
&\mathbb P\bigl(
\delta_{i_1}^*\in B_1,\ldots,\delta_{i_m}^*\in B_m
\bigr)
=
\prod_{r=1}^m \mathbb P\bigl(\delta_{i_r}^*\in B_r\bigr),
\end{align*}
and hence
\begin{align*}
&\mathbb P\bigl(
\delta_{i_1}^*\in B_1,\ldots,\delta_{i_{m+1}}^*\in B_{m+1}
\bigr)
=
\prod_{r=1}^{m+1}\mathbb P\bigl(\delta_{i_r}^*\in B_r\bigr).
\end{align*}
This proves that
\begin{align*}
&\{\delta_{i_1}^*,\ldots,\delta_{i_{m+1}}^*\}
\end{align*}
is jointly independent, completing the induction.

Thus the family
\begin{align*}
&\{\delta_i^*\}_{i\in \mathcal J_{e,s}}
\end{align*}
is jointly independent, each \(\delta_i^*\) has the same distribution as the corresponding \(\delta_i\), and the coupling error satisfies
\begin{align*}
&\mathbb P(\delta_i^*\neq \delta_i)\le \beta(S).
\end{align*}
This proves the lemma for \(\mathcal J_{e,s}\). The proof for \(\mathcal J_{o,s}\) is identical.

\end{proof}
\color{Black}

\newpage

\color{black}
\subsubsection{Deviation bounds for \(\epsilon\)}
\label{deviation_bound_dependent_epsilon}

Now we handle the general case, where \(\epsilon^*\) is the ghost version of \(\epsilon\).

Let
\begin{align*}
&D'_{mn}(g)=\sum_{i=1}^{n}\sum_{j=1}^{m_i}\frac{1}{nm_i} g(x_{ij}) \epsilon_{ij}.
\end{align*}

To formalize the effective sample size arising from the blocking argument, for any subset
\(I\subseteq \{1,\ldots,n\}\), define the block empirical norm
\begin{align*}
&\|g\|_{I,m}^2
:=
\frac{1}{|I|}\sum_{i\in I}\frac{1}{m_i}\sum_{j=1}^{m_i} g^2(x_{ij}).
\end{align*}
For each block \(\mathcal J_{e,s}\) and \(\mathcal J_{o,s}\), we use the shorthand
\begin{align*}
&\|g\|_{|\mathcal J_{e,s}|m}^2 := \frac{1}{|\mathcal J_{e,s}|}\sum_{i\in\mathcal J_{e,s}}\frac{1}{m_i}\sum_{j=1}^{m_i} g^2(x_{ij}),   
\qquad \|g\|_{|\mathcal J_{o,s}|m}^2 := \frac{1}{|\mathcal J_{o,s}|}\sum_{i\in\mathcal J_{o,s}}\frac{1}{m_i}\sum_{j=1}^{m_i} g^2(x_{ij}).
\end{align*}
Here \(0\le s\le S-1\), 
Since the families \(\{\mathcal J_{e,s}\}_{s=0}^{S-1}\) and \(\{\mathcal J_{o,s}\}_{s=0}^{S-1}\) form a partition of \(\{1,\dots,n\}\), it follows that
\begin{align*}
&\|g\|_{nm}^2
=
\frac{1}{n}\sum_{s=0}^{S-1}
\Bigl(
|\mathcal J_{e,s}|\,\|g\|_{|\mathcal J_{e,s}|m}^2
+
|\mathcal J_{o,s}|\,\|g\|_{|\mathcal J_{o,s}|m}^2
\Bigr).
\end{align*}

\begin{corollary}\label{coro:complexity bound 2 beta mn}
Let \(\mathcal F\) be a function class such that
\begin{align}\label{eq:metric entropy assumption in epsilon 2 beta mn1}
\log \mathcal N(\delta,\mathcal F,\|\cdot\|_{nm})
\le
nm\,\phi_{nm}\Bigl(\log^{4.3}(nm)\log(\delta^{-1})\Bigr),
\end{align}
and suppose that \(\mathcal F\subset B_\infty(\tau)\), where \(\tau\ge 1\). Let
\begin{align*}
&S\ge \frac{2}{C_\beta}\log n.
\end{align*}
Define
\begin{align*}
&\Lambda_{n,S}(g) :={} \sum_{s=0}^{S-1} \frac{|\mathcal J_{e,s}|}{n}\, \|g\|_{|\mathcal J_{e,s}|m}   
\sqrt{ \phi_{|\mathcal J_{e,s}|m}\log^{4.3}(nm)\, \log\!\bigl(\|g\|_{|\mathcal J_{e,s}|m}^{-1}\bigr) }
\\
&\quad+ \sum_{s=0}^{S-1} \frac{|\mathcal J_{o,s}|}{n}\, \|g\|_{|\mathcal J_{o,s}|m}   
\sqrt{ \phi_{|\mathcal J_{o,s}|m}\log^{4.3}(nm)\, \log\!\bigl(\|g\|_{|\mathcal J_{o,s}|m}^{-1}\bigr) }.
\end{align*}
Then for any \(t\ge 1\), there exist positive constants \(C_1^{\prime\prime}\) and \(C_2\) such that
\begin{align*}
&\p_{\cdot\mid x}\left(
\sup_{g\in\mathcal F}
\frac{D'_{mn}(g)}{C_1^{\prime\prime}\sigma_\epsilon\, t\,\Lambda_{n,S}(g)}
\ge 1
\right)
\le
2S\,C_2\exp\!\left(
-c t^2 \frac{nm}{S}\phi_{nm/S}\log^{4.3}\!\left(\frac{nm}{S}\right)
\right)
+
n\,\beta(S).
\end{align*}
\end{corollary}

\begin{proof}
Observe that
\begin{align}
&\label{eq:coupling applied to localization in the temporal-spatial model beta corrected}
D'_{mn}(g)
=
\sum_{i=1}^{n}\sum_{j=1}^{m_i}\frac{1}{nm_i} g(x_{ij}) \epsilon_{ij}
\nonumber
\\
&=
\sum_{i=1}^n \sum_{j=1}^{m_i} \frac{1}{nm_i}
\left\{g(x_{ij}) \epsilon_{ij} - g(x_{ij}) \epsilon_{ij}^* \right\}
\nonumber
\\
&\quad+ \sum_{s=0}^{S-1} \sum_{i \in \J_{e,s}}\sum_{j=1}^{m_i}\frac{1}{nm_i} g(x_{ij}) \epsilon_{ij}^*   
+ \sum_{s=0}^{S-1} \sum_{i \in \J_{o,s}}\sum_{j=1}^{m_i}\frac{1}{nm_i} g(x_{ij}) \epsilon_{ij}^* .
\end{align}

Since
\begin{align*}
&\mathcal F \subset B_\infty(\tau)\subset B_{nm}(\tau),
\end{align*}
it holds that
\begin{align*}
&\mathcal F\cap B_{nm}(\tau)=\mathcal F.
\end{align*}

\medskip
\noindent
{\bf Step 1.}
By Lemma~\ref{Ind-cop-measerrors},
\begin{align*}
&\p_{\cdot \mid x }\bigg( \sup_{ g \in \mathcal F } \sum_{i=1}^n \sum_{j=1}^{m_i} \frac{1}{nm_i} g(x_{ij}) \epsilon_{ij}   
\not = \sup_{ g \in \mathcal F } \sum_{i=1}^n \sum_{j=1}^{m_i} \frac{1}{nm_i} g(x_{ij}) \epsilon_{ij}^* \bigg)
\\
&\le\; \p _{\cdot \mid x } \bigg( \epsilon_{ij}   
\not = \epsilon^*_{ij } \text{ for some } i\in\{1,\ldots,n\},\,j\in\{1,\ldots,m_i\} \bigg)
\\
&\le\;
\sum_{i=1}^{n}\p _{\cdot \mid x } \bigg(
\epsilon_{ij} \not = \epsilon^*_{ij }
\text{ for some } j\in\{1,\ldots,m_i\}
\bigg)
\\
&\le\;
n\,\beta(S).
\end{align*}

\medskip
\noindent
{\bf Step 2.}
We rewrite the even and odd block terms as
\begin{align*}
&\sum_{s=0}^{S-1} \sum_{i \in \J_{e,s}}\sum_{j=1}^{m_i}\frac{1}{nm_i} g(x_{ij}) \epsilon_{ij}^*   
= \sum_{s=0}^{S-1} \frac{\sqrt{|\mathcal J_{e,s}|}}{n\sqrt{m}} \sum_{i \in \J_{e,s}}\sum_{j=1}^{m_i} \frac{1}{\sqrt{|\mathcal J_{e,s}|m_i}} g(x_{ij}) \frac{\sqrt{m}}{\sqrt{m_i}}\epsilon_{ij}^*,
\\
&\sum_{s=0}^{S-1} \sum_{i \in \J_{o,s}}\sum_{j=1}^{m_i}\frac{1}{nm_i} g(x_{ij}) \epsilon_{ij}^*   
= \sum_{s=0}^{S-1} \frac{\sqrt{|\mathcal J_{o,s}|}}{n\sqrt{m}} \sum_{i \in \J_{o,s}}\sum_{j=1}^{m_i} \frac{1}{\sqrt{|\mathcal J_{o,s}|m_i}} g(x_{ij}) \frac{\sqrt{m}}{\sqrt{m_i}}\epsilon_{ij}^*.
\end{align*}

Define
\begin{align}
&\widetilde{\epsilon}_{ij}:=\frac{\sqrt{m}}{\sqrt{m_i}}\epsilon_{ij}^*.\label{T4ildevariables corrected}
\end{align}
This rescaling is introduced only to rewrite each block sum in the normalized form required by Corollary~\ref{coro:complexity bound 2 mn}. By Assumption~1{\bf d} and {\bf f}, the variables \(\widetilde{\epsilon}_{ij}\) are sub-Gaussian with parameter \(\frac{1}{\sqrt c}\sigma_\epsilon\).

For each \(s\), we apply Corollary~\ref{coro:complexity bound 2 mn} to the even block and to the odd block. Thus, for all \(g\in\mathcal F\subset B_\infty(\tau)\), any \(t>1\), and all \(s\),
\begin{align*}
&\sum_{i \in \J_{e,s}}\sum_{j=1}^{m_i}
\frac{1}{\sqrt{|\mathcal J_{e,s}|m_i}} g(x_{ij}) \widetilde{\epsilon}_{ij}
\\
&\le\; \|g\|_{|\mathcal J_{e,s}|m}   
\sqrt{|\mathcal J_{e,s}|m\,\phi_{|\mathcal J_{e,s}|m} \Bigl\{\log^{4.3}(nm)\log(\|g\|_{|\mathcal J_{e,s}|m}^{-1})\Bigr\}}   
\frac{C_1}{c}\sigma_\epsilon\, t,
\end{align*}
with probability at least
\begin{align*}
&1-C_2\exp\!\left(
-c t^2 |\mathcal J_{e,s}|m\,\phi_{|\mathcal J_{e,s}|m}\log^{4.3}(|\mathcal J_{e,s}|m)
\right).
\end{align*}
A completely analogous bound holds with \(\J_{e,s}\) replaced by \(\J_{o,s}\).

Using a union bound over the \(2S\) block events, with probability at least
\begin{align*}
&1-2S\,C_2\exp\!\left(
-c t^2 \frac{nm}{S}\phi_{nm/S}\log^{4.3}\!\left(\frac{nm}{S}\right)
\right),
\end{align*}
we have
\begin{align*}
&\sum_{s=0}^{S-1}  \frac{\sqrt{|\mathcal J_{e,s}|}}{n\sqrt{m}}
   \sum_{i \in \J_{e,s}}\sum_{j=1}^{m_i}
   \frac{1}{\sqrt{|\mathcal J_{e,s}|m_i}} g(x_{ij}) \widetilde{\epsilon}_{ij}
+ \sum_{s=0}^{S-1}  \frac{\sqrt{|\mathcal J_{o,s}|}}{n\sqrt{m}}
   \sum_{i \in \J_{o,s}}\sum_{j=1}^{m_i}
   \frac{1}{\sqrt{|\mathcal J_{o,s}|m_i}} g(x_{ij}) \widetilde{\epsilon}_{ij}
\\
&\leq \frac{C_1}{c}\sigma_\epsilon t \sum_{s=0}^{S-1}\frac{|\mathcal J_{e,s}|}{n} \Biggl[ \|g\|_{|\mathcal J_{e,s}|m}   
\sqrt{\phi_{|\mathcal J_{e,s}|m} \Bigl\{\log^{4.3}(nm)\log(\|g\|_{|\mathcal J_{e,s}|m}^{-1})\Bigr\}} \Biggr]
\\
&\quad+ \frac{C_1}{c}\sigma_\epsilon t \sum_{s=0}^{S-1}\frac{|\mathcal J_{o,s}|}{n} \Biggl[ \|g\|_{|\mathcal J_{o,s}|m}   
\sqrt{\phi_{|\mathcal J_{o,s}|m} \Bigl\{\log^{4.3}(nm)\log(\|g\|_{|\mathcal J_{o,s}|m}^{-1})\Bigr\}} \Biggr]
\\
&= \frac{C_1}{c}\sigma_\epsilon t\,\Lambda_{n,S}(g).
\end{align*}
Hence, for a suitable positive constant \(C_1^{\prime\prime}\),
\begin{align*}
&\sum_{s=0}^{S-1} \sum_{i \in \J_{e,s}}\sum_{j=1}^{m_i}\frac{1}{nm_i} g(x_{ij}) \epsilon_{ij}^*   
+ \sum_{s=0}^{S-1} \sum_{i \in \J_{o,s}}\sum_{j=1}^{m_i}\frac{1}{nm_i} g(x_{ij}) \epsilon_{ij}^*
\\
&\leq
 C_1^{\prime\prime}\sigma_\epsilon\, t\,\Lambda_{n,S}(g),
\end{align*}
with probability at least
\begin{align*}
&1-2S\,C_2\exp\!\left(
-c t^2 \frac{nm}{S}\phi_{nm/S}\log^{4.3}\!\left(\frac{nm}{S}\right)
\right).
\end{align*}

Combining this with Step 1 yields
\begin{align*}
&\p_{\cdot  \mid x  }\left(
\sup_{g \in \mathcal F}
\frac{D'_{mn}(g)}
{C_1^{\prime\prime}\sigma_\epsilon\, t\,\Lambda_{n,S}(g)}
\ge 1
\right)
\le
2S\,C_2\exp\!\left(
-c t^2 \frac{nm}{S}\phi_{nm/S}\log^{4.3}\!\left(\frac{nm}{S}\right)
\right)
+n\,\beta(S).
\end{align*}
This completes the proof.
\end{proof}

\color{Black}
\subsubsection{Coupling Error in Temporal-spatial Model}
\label{section_dependent_x}
Now we handle the general case, where $x^*$ is the ghost version of $x$.
Let
\begin{align*}
&\left\| g \right\|_{nm}^2 = \sum_{i=1}^n \sum _{j=1}^{m_i} \frac{1}{nm_i}   g^2(x_{ij}).
\end{align*}

\begin{lemma}\label{lemma:rate in the temporal-spatial model}
Suppose  $g\in \mathcal F $ implies $\left\|g\right\|_\infty \leq \tau$.
For any $\eta \in (0, \frac{1}{4})$, there exists a  constant $C_\eta $  only depending on $\eta $ such that 
\begin{align*}
&\E\bigl( \sup_{ g \in \mathcal F } \bigl|   
\frac{ \bigl\|g\bigr\|_{nm}^2 - \bigl\|g\bigr\|_\lt ^2 }{ \eta \bigl\|g\bigr\|_\lt^2 + \tau C_\eta\phi_{nm/S} \log^{5.3}(nm/S) + \tau^2\beta(S)}\bigr| \bigr) \le 12.
\end{align*}
\end{lemma}
\begin{proof}
 Observe that  
 \begin{align}
&\bigl\|g\bigr\|_{nm}^2 - \bigl\|g\bigr\|_{\lt}^2 = \sum_{i=1}^n \sum_{j=1}^{m_i} \frac{1}{nm_i} \bigl\{g^2(x_{ij})   
- g^2(x_{ij}^*) \bigr\} + \sum_{i=1}^n \sum_{j=1}^{m_i} \frac{1}{nm_i} \bigl\{g^2(x_{ij}^*) - \bigl\|g\bigr\|_\lt^2 \bigr\}\label{eq:coupling applied to localization in the temporal-spatial model}
\\
&= \sum_{i=1}^n \sum_{j=1}^{m_i} \frac{1}{nm_i} \bigl\{g^2(x_{ij}) - g^2(x_{ij}^*) \bigr\}   
+ \sum_{s=0}^{S-1} \sum_{i \in \J_{e,s}}\sum_{j=1}^{m_i}\frac{1}{nm_i} \bigl\{g^2(x_{ij}^*) - \|g\|_\lt^2 \bigr\} \nonumber
\\
&+ \sum_{s=0}^{S-1} \sum_{i \in \J_{o,s}}\sum_{j=1}^{m_i}\frac{1}{nm_i}\left\{g^2\left(x_{ij}^*\right) - \left\|g\right\|_\lt^2 \right\}.  \nonumber
\end{align}
{\bf Step 1.} Note that 
\begin{align*}
&\E\bigl( \sup_{ g \in \mathcal F } \big |   
\frac{ \sum_{i=1}^n \sum_{j=1}^{m_i} \frac{1}{nm_i} \bigl\{g^2(x_{ij}) - g^2(x_{ij}^*) \bigr\} }{ \eta \bigl\|g\bigr\|_\lt ^2 + \tau C_\eta\phi_{nm/S} \log^{5.3}(nm/S) + \tau^2\beta(S)}\big | \bigr)
\\
&\le \sum_{i=1}^n \sum_{j=1}^{m_i} \frac{1}{nm_i} \E\bigl( \sup_{ g \in \mathcal F }   
\frac{ \big | g^2(x_{ij}) - g^2(x_{ij}^*) \big | }{ \eta \bigl\|g\bigr\|_\lt ^2 + \tau C_\eta\phi_{nm/S} \log^{5.3}(nm/S) + \tau^2\beta(S)} \bigr)
\\
&= \sum_{i=1}^n \sum_{j=1}^{m_i} \frac{1}{nm_i} \E\bigl( \mathbb I_{ \bigl\{ x_{ij}=x_{ij}^* \bigr\} } \sup_{ g \in \mathcal F }   
\frac{ \bigl| g^2 \bigl(x_{ij}\bigr) - g^2 \bigl(x_{ij}^*\bigr) \bigr| }{ \eta \bigl\|g\bigr\|_\lt ^2 + \tau C_\eta\phi_{nm/S} \log^{5.3}(nm/S) + \tau^2\beta(S)} \bigr)
\\
&+ \sum_{i=1}^n \sum_{j=1}^{m_i} \frac{1}{nm_i} \E\bigl( \mathbb I_{ \{ x_{ij}\not =x_{ij}^* \} } \sup_{ g \in \mathcal F }   
\frac{ \bigl| g^2(x_{ij}) - g^2(x_{ij}^*) \bigr| }{ \eta \|g\|_\lt ^2 + \tau C_\eta\phi_{nm/S} \log^{5.3}(nm/S) + \tau^2\beta(S)} \bigr)
\\
&\le   0+ \sum_{i=1}^n \sum_{j=1}^{m_i} \frac{1}{nm_i} \E\left(  \mathbb I_{ \{ x_{ij}\not =x_{ij}^* \} } \frac{2\tau^2}{\tau^2\beta(S)} \right)\le   2,
\end{align*}
where the last inequality follows from \eqref{eq:coupling error in the temporal-spatial model}.\\
{ \bf Step 2.} Observe that
\begin{align*}
&\E\bigl( \sup_{ g \in \mathcal F } \bigl|   
\frac{ \sum_{s=0}^{S-1} \sum_{i \in \J_{e,s}} \sum_{j=1}^{m_i}\frac{1}{nm_i} \bigl\{g^2(x_{ij}^*) - \bigl\|g\bigr\|_\lt^2 \bigr\} }{ \eta \bigl\|g\bigr\|_\lt ^2 + \tau C_\eta\phi_{nm/S} \log^{5.3}(nm/S) + \tau^2\beta(S)}\bigr| \bigr)
\\
&\le \sum_{s=0}^{S-1} \frac{|\mathcal J_{e,s}|}{n} \E\bigl( \sup_{ g \in \mathcal F } \bigl|   
\frac{ \sum_{i \in \J_{e,s}} \sum_{j=1}^{m_i} \frac{1}{|\mathcal J_{e,s}| m_i}\bigl\{g^2(x_{ij}^*) - \|g\|_\lt^2 \bigr\} }{ \eta \|g\|_\lt ^2 + \tau C_\eta\phi_{nm/S} \log^{5.3}(nm/S) + \tau^2\beta(S)}\bigr| \bigr)
\\
&\le \sum_{s=0}^{S-1} \frac{|\mathcal J_{e,s}|}{n} \E\bigl( \sup_{ g \in \mathcal F } \bigl |   
\frac{ \sum_{i \in \J_{e,s}}\sum_{j=1}^{m_i} \frac{1}{|\mathcal J_{e,s}|m_i}\bigl\{g^2(x_{ij}^*) - \|g\|_\lt^2 \bigr\} }{ \eta \|g\|_\lt ^2 + \tau C_\eta\phi_{nm/S} \log^{5.3}(nm/S) }\bigr | \bigr)
\\
&\le \sum_{s=0}^{S-1} \frac{|\mathcal J_{e,s}|}{n} \E\bigl( \sup_{ g \in \mathcal F } \bigl |   
\frac{ \sum_{i \in \J_{e,s}}\sum_{j=1}^{m_i} \frac{1}{|\mathcal J_{e,s}|m_i}\bigl\{g^2(x_{ij}^*) - \|g\|_\lt^2 \bigr\} }{ \eta \|g\|_\lt ^2 + \tau C_\eta\phi_{|\mathcal J_{e,s}|m} \log^{5.3}(|\mathcal J_{e,s}|m) }\bigr | \bigr)   
\le \sum_{s=0}^{S-1}\frac{|\mathcal J_{e,s}|12 }{n} ,
\end{align*}

where the second last inequality follows from \eqref{eq:effective sample size in the temporal-spatial model} which implies $ | \J_{e,s}| \asymp n/S$ as well as $ | \J_{o,s}| \asymp n/S$, indeed, we have $|\mathcal J_{e,s}| \leq \frac{K+2}{2} \leq \frac{n}{2S} + 1 \leq \frac{n}{S}$, and the last inequality follows from applying
\Cref{coro:deviation for norms mn} in $\mathcal J_{e,s}$ block.
So 
\begin{align*}
&\E\bigl( \sup_{ g \in \mathcal F } \big |   
\frac{ \sum_{i=1}^n \sum_{j=1}^{m_i} \frac{1}{nm_i} \bigl\{g^2(x_{ij} ^*) - \|g\|_\lt^2 \bigr\} }{ \eta \|g\|_\lt ^2 + \tau C_\eta\phi_{nm/S} \log^{5.3}(nm/S) + \tau^2\beta(S)}\big | \bigr)   
\le \sum_{s=0}^{S-1}\frac{|\mathcal J_{e,s}|12 }{n} +\sum_{s=0}^{S-1}\frac{|\mathcal J_{o,s}|12 }{n}
\\
&\le 12.
\end{align*}
  \\
  {\bf Step 3.} Therefore 
      \eqref{eq:coupling applied to localization in the temporal-spatial model}  and Step 1 gives 
      \begin{align*}
&\E\bigl( \sup_{ g \in \mathcal F } \bigl|   
\frac{ \|g\|_{nm}^2 - \bigl\|g\bigr\|_\lt ^2 }{ \eta \bigl\|g\bigr\|_\lt ^2 + \tau C_\eta\phi_{nm/S} \log^{4.3}(nm/S) + \tau^2\beta(S)} \bigr| \bigr) \le 14.
\end{align*}
\end{proof}

\section{Sketch of Proof Of Neural Network Approximation with Independent Observations (Corollary 3 in the Main Manuscript)} 
\label{functional_estimator_proof}

The proof is similar to Step 1 to Step 4 of proof of Theorem~1 in Section~\ref{proof of Theorem thm:main thm in the temporal-spatial model}, except no terms for $\beta$-mixing needed. 
By the same argument in Section~\ref{proof of Theorem thm:main thm in the temporal-spatial model} \Cref{lemma_epsilon_gamma_infinity_bound}, 
we can choose $\mathcal{A}_{nm}= \max\left\{\sigma_\epsilon, \sigma_\gamma \right\}o \left(\log(nm)\right)$.
Thus, it is true that

\begin{align*}
\| \widehat f_{\mathcal{A}_n} - f^*\|_\lt^2  
&\le \| \overline  f_{\mathcal{A}_{nm}} - f^*\|_\lt ^2 + \| \overline  f_{\mathcal{A}_{nm}} - \widehat f_{\mathcal{A}_{nm}} \|_\lt ^2 \\
&\le  4 \phi_{nm} + 2\| \overline  f_{\mathcal{A}_{nm}} - \widehat f_{\mathcal{A}_{nm}} \|_\lt ^2  , 
\end{align*}
 
\begin{align*}
&\bigl\| \widehat{f}_{\mathcal{A}_{nm}} -\overline{f}_{\mathcal{A}_{nm}} \bigr\|_\lt ^2 \le 2 \bigl\| \widehat f_{\mathcal{A}_{nm}} -\overline{f}_{\mathcal{A}_{nm}} \bigr\|_{nm} ^2 + C a_n\mathcal{A}_{nm}\phi_{nm} \log^{5.3}(nm) ,
\end{align*}

\begin{align*}
&\bigl\| \widehat{f}_{\mathcal{A}_{nm}} - \overline{f}_{\mathcal{A}_{nm}} \bigr\|_{nm} ^2 \le \bigl(2C_3+1\bigr) \bigl\{a_n\mathcal{A}_{nm} \phi_{nm}\log^{5.3} \bigl(nm\bigr) + \frac{1}{n}\bigr\},
\end{align*}
   hold with high probability with respect to the distribution of $\{x_{ij}\}^{n,m_i}_{i=1,j=1}$, which gives  
\begin{align*}
&a_n\left(\sigma_\epsilon^2+\sigma_\gamma^2 +1\right) \max_{(p,K)\in \mathcal {P}}\left( \frac{1}{nm}\right)^{\frac{2p}{2p+K}}\log^{6.3}(nm)    +  \frac{\sigma_\gamma^2}{n}
\end{align*} holds with probability approaches to one. 


\newpage 


\section{Sketch of Proof of Neural Network Approximation with Temporal Dependence (Corollary 4 in the Main Manuscript)}
\label{proof_temporal_model}

The proof of Corollary 4 in the Main Manuscript adheres to a similar trajectory, with the incorporation of simpler step into the analysis of Theorem 1. 
The $m_i = 1$, in which the set of designs $\left\{x_i\right\}_{i=1}^{n}$ is replaced by the set of designs $\left\{x_{ij}\right\}_{i=1, j=1}^{n, m_i}$
For $1\le i \le n $  , let 
\begin{align*}
&y_{i} = f^*(x_{i} )  +\epsilon_{i} .
\end{align*}\

We utilize the following block scheme for $\beta$-mixing conditions. 
Let $\{ x_i\}_{i=1}^n$ be a collection of stationary $\beta$-mixing time series with mixing  coefficient $\beta(s)$. 
For any  $S \in \{1,\ldots, n\} $ and 
$ K=\lfloor n/S \rfloor$,  denote 
$$ \I_k = \begin{cases}
    (0+(k-1) S, \ k  S] & \text{when } 1 \le k \le K,
    \\
    ( K  S, n] &\text{when } k=K+1  .
\end{cases}$$
By \Cref{lemma:coupling in the temporal-spatial model}, for any $\{x_i\}_{i\in \I_k} $, there exists $ \{ x_i ^* \}_{i \in \I_k} $ such that 

\begin{align} 
\p(x_i\not =x^*_i) =\beta(S)
\end{align} 
and that $ x_i^*$ is independent of $ \sigma(\{x_j\}_{j\le i-1 }   ).$
Denote for $ 0\le s \le S-1$,
\begin{align*}
&\mathcal J_{e,s} =\{\, i \in \mathcal I_k \text{ for } k \equiv 0 \pmod 2 \text{ and } i \equiv s \pmod S \,\} \quad \text{and}
\\
&\mathcal J_{o,s} =\{\, i \in \mathcal I_k \text{ for } k \equiv 1 \pmod 2 \text{ and } i \equiv s \pmod S \,\}.
\end{align*}
   Then 
   \begin{align} \label{eq:effective sample size in the temporal model}\frac{k-1}{2 }\le |\mathcal J_{e,s}|  \le \frac{k+2 }{2} \text{ and }
   \frac{k-1}{2 }\le |\mathcal J_{o,s}|  \le \frac{k+2 }{2}
   \end{align} 
   and that 
   \begin{align*}
&\bigcup _{s=0}^{S-1}\mathcal J_{e,s} \cup  \bigcup_{s=0}^{S-1}\mathcal J_{o,s} =[1,\ldots, n].
\end{align*}
   Therefore,
   $ \{ x_i^*\}_{i \in \J_{e, l}} $ are i.i.d. random variables.
   \\
   \\
\begin{remark}
    The same analysis in Supplementary Material~\ref{proof of Theorem thm:main thm in the temporal-spatial model} \Cref{lemma_epsilon_gamma_infinity_bound} shows that 
by choosing $ \mathcal{A}_n =o(\log(n))$, then
\begin{align*}
&\mathbb{P}\left(\left|\epsilon_{i}\right| \le \frac{\mathcal{A}_n}{4} \text { for all } 1 \le i \le n\right)=1-o(1).
\end{align*}
Then, with the deviation result as a special case of Section \ref{proof temporal-spatial model indepdent} and \ref{proof temporal-spatial model depdent} by letting $m_i = 1$. We can show that 
\begin{align*}
&\left\| \overline f \right\|_\infty \le \left\|f^*\right\|_\infty  + \left\| f^* -\overline f \right\|_\infty \le  \frac{\mathcal{A}_n}{4} + \sqrt {2 \phi_{n}}  \le \mathcal{A}_n,
\end{align*}
 
\begin{align*}
&\| \widehat f _{\mathcal{A}_n} -\overline f_{\mathcal{A}_n} \|_\lt ^2 \le 2 \| \widehat f_{\mathcal{A}_n} -\overline f_{\mathcal{A}_n} \|_{n} ^2 + C a_n\An\phi_{\frac{n}{S}} \log^{5.3}(\frac{n}{S})+ a_n\mathcal A_n^2\beta(S) ,
\end{align*}  
  \begin{align*}
 \|   \widehat f _{\mathcal{A}_n}- \overline f _{\mathcal{A}_n}\|_{ n} ^2  
 \le   (C_3+1)a_n{\mathcal{A}_n} \phi_{\frac{n}{S}}\log^{5.3}(n),
 \end{align*}
 by combining the above steps. 

 Note that 
 $\phi_{n}=
\max _{(p, K) \in \mathcal{P}}\left(n^{-\frac{2 p}{2 p+K}}\right)
 \le \phi_{\frac{n}{S}}=
\max _{(p, K) \in \mathcal{P}}(\frac{n}{S})^{-\frac{2 p}{2 p+K}}
 $, 
 and by by Assumption 1 {\bf a} for $\beta(S)$ and replacing $S=O(\log(n))$, 
 with high probability with respect to the probability measure on $\{x_{i}\}^{n}_{i=1}$, gives  error rate
 \begin{align*}
&\bigl\|\widehat{f}_{\mathcal{A}_n}-f^*\bigr\|_{\mathcal{L}_2}^2 \lesssim a_n\mathcal{A}_n\phi_{\frac{n}{S}} \log^{5.3}(n)+ a_n\mathcal{A}_n^2\beta(S) \lesssim a_n(\sigma^2_\epsilon+1) \bigl( \frac{1}{n}\bigr)^{\frac{2p}{2p+K}}\log^{7.3}(n).
\end{align*}

\end{remark}
\newpage

\section{Proof of Theorem 5}

\subsection{Preliminary Results for Proof of Theorem 5}

\begin{lemma}[Neural Network Approximation Result on Manifold, Theorem 2 of  \cite{kohler2023estimation}]
\label{lemma_approximation_manifold} Let $f^*$ be $(p,C)$-smooth 
function on a $d^*$-dimensional Lipschitz-manifold.
There 
exists a neural network $f$  with the the choice of $L=L_{nm}$ and $r= r_{nm}$, where
\begin{align}
&L=L_{nm}     \asymp \log(nm) \quad \text{and} \quad r= r_{nm} \asymp \left(nm\right)^{\frac{d^*}{2(2p+d^*)}} \nonumber
\\
&\text{ or} , L=L_{nm} \asymp \log(nm)\cdot \bigl(nm\bigr)^{\frac{d^*}{2(2p+d^*)}} \quad \text{and} \quad r= r_{nm} \asymp 1\label{manifold_network}
\end{align}
 satisfying that \begin{align}
\label{approximation_error_manifold}
\inf_{f \in \mathcal F(L , r  )} \|f-f^*\|_{\infty, \mathcal{M}} \le \sqrt {\phi_{nm, \mathcal{M}}},  
\end{align}
where 
\begin{align*}
&\phi_{nm, \mathcal{M}} =(nm)^{\frac{-2p}{ (2p+d^*)}},
\end{align*}
when $n$ is sufficiently large.

Thus, by the choice of $L$ and $r$ in \eqref{manifold_network}, we have
\begin{align}
& Lr \asymp \log(nm)(nm)^{\frac{d^*}{2(2p+d^*)}}, \nonumber \\
& Lr^2 \lesssim \log(nm)\left[\left(nm\right)^{\frac{d^*}{2(2p+d^*)}}\right]^2 \lesssim nm. \label{l_r_manifold}
\end{align}

\end{lemma}



\begin{lemma}
\label{lemma:threshold space entropy manifold}
    For any fixed $\mathcal{A}_{nm}$, it holds that 
    
\begin{align*}
&\log \mathcal N(\delta,\mathcal F_{\mathcal{A}_{nm}}(L, r ) , \|\cdot \|_{nm, \mathcal{M}} ) = \log \mathcal N(\delta,\mathcal F_{\mathcal{A}_{nm}}(L, r ) , \|\cdot \|_{nm} )
\\
&\lesssim 
L^2 r^2 \log(Lr^2)   \log(\mathcal{A}_{nm}\delta ^{-1} ).
\end{align*} Furthermore, with $L,r$ satisfying \eqref{manifold_network}, then 
\begin{align*}
&\log \mathcal N(\delta,\mathcal F_{\mathcal{A}_n} (L , r ) , \|\cdot \|_{nm, \mathcal{M}} ) \lesssim nm \phi_{nm, \mathcal{M}}\log^{3.3}(nm) \log(\mathcal{A}_{nm}\delta^{-1}).
\end{align*}
\end{lemma}
\begin{proof}
Recall the probability distribution is supported on a manifold $\mathcal{M}$, so we write $\|f\|_{nm}$ as $\|f\|_{nm, \mathcal{M}}$.
  Then  by the same argument as in the Euclidean space from \Cref{lemma:threshold space entropy}, it follows that 

\begin{align*}
&\log \mathcal N(\delta,\mathcal F_{\mathcal{A}_{nm}}(L, r ) , \|\cdot \|_{nm, \mathcal{M}} ) = \log \mathcal N(\delta,\mathcal F_{\mathcal{A}_{nm}}(L, r ) , \|\cdot \|_{nm} )
\\
&\lesssim 
L^2 r^2 \log(Lr^2)   \log(\mathcal{A}_{nm}\delta ^{-1} ).
\end{align*}  Then, with the choice of $L$ and $r$ in \eqref{l_r_manifold} leads to 
\begin{align*}
&\log \mathcal N(\delta,\mathcal F_{\mathcal{A}_n} (L , r ) , \|\cdot \|_{nm, \mathcal{M}} ) \lesssim nm \phi_{nm, \mathcal{M}}\log^{3.3}(nm) \log(\mathcal{A}_{nm}\delta^{-1}).
\end{align*}
\end{proof}

\subsection{Proof of Theorem 5}
\begin{proof}
Let 
$  \overline f  \in \mathcal F(L, r)   $  be such that 
\begin{align}
\label{f star minus overline f}
\left\|f^* - \overline f \right\|_{\infty, \mathcal{M}}^2 \le 2\phi_{nm, \mathcal{M}},
\end{align}
where the existence of $\overline f$ is guaranteed by \Cref{lemma_approximation_manifold}. 

The same analysis in Supplementary Material~\ref{proof of Theorem thm:main thm in the temporal-spatial model} \Cref{lemma_epsilon_gamma_infinity_bound} shows that 
by choosing \\
$\mathcal{A}_{nm} \asymp \max \left\{\sigma_\epsilon \cdot \mathcal{B}_{nm}, \sigma_\gamma \cdot \mathcal{C}_{n}\right\}\asymp \max \left\{\sigma_\epsilon \cdot \sqrt{\log(nm)}, \sigma_\gamma \cdot \sqrt{\log(n)} \right\}$, or $\max\{\sigma_\epsilon, \sigma_\gamma \}\cdot o(\log(nm))$, it holds that
\begin{align*}
&\mathbb{P}\bigl( \bigl|\epsilon_{i j}\bigr| \le \frac{\mathcal{A}_{nm}}{4} \text{ for all } 1 \le i \le n,\, 1 \le j \le m_i\,\text{and}\,\bigl\|\gamma_i\bigr\|_{\infty,nm} \le \frac{\mathcal{A}_{nm}}{4} \text{ for all } 1 \le i \le n \bigr)
\\
&=\ 1 - o(1).
\end{align*}
Since $\mathcal{A}_{nm}$ can be chosen appropriately large with respect to $n$ while $\phi_{nm,\mathcal{M}}$ decays with respect to $n$ and $m$, assume that $\sqrt{\phi_{nm, \mathcal{M}}} \le \frac{3}{4}\mathcal{A}_{nm}$. 
Then,
\begin{align*}
&\bigl\| \overline f \bigr\|_{\infty, \mathcal{M}} \le \bigl\|f^* \bigr\|_{\infty, \mathcal{M}} + \bigl\| f^* -\overline f \bigr\|_{\infty,\mathcal{M}} \le \frac{{\mathcal{A}_{nm}}}{4} + \sqrt {2 \phi_{nm, \mathcal{M}}} \le {\mathcal{A}_{nm}}.
\end{align*}
It follows that $ \overline f=\overline f_{\mathcal{A}_{nm}}$. 
Since 
\begin{align}
&\bigl\| \widehat f_{\mathcal{A}_{nm}} - f^* \bigr\|_\lt^2 \le 2 \bigl\| \overline f_{\mathcal{A}_{nm}} - f^*\bigr\|_\lt ^2 + 2 \bigl\| \overline{f}_{\mathcal{A}_{nm}} - \widehat f_{\mathcal{A}_{nm}} \bigr\|_\lt ^2 \nonumber\label{eq:norm comparison 1 in the temporal-spatial model on manifold}
\\
&\le 4 \phi_{nm, \mathcal{M}} + 2 \left\| \overline  f_{\mathcal{A}_{nm}} - \widehat f_{\mathcal{A}_{nm}} \right\|_\lt ^2 ,
\end{align}
it suffices to show 
\begin{align*}
&\bigl\| \overline f_{\mathcal{A}_{nm}} - \widehat f_{\mathcal{A}_{nm}} \bigr\|_\lt ^2 \lesssim a_n {\mathcal{A}^2_{nm}} \phi_{\frac{nm}{S}, \mathcal{M}}\log^{5.3}(nm) + a_n {\mathcal{A}^2_{nm}}\beta(S) + \frac{1}{ n },
\end{align*}
where $S\asymp \log(n)$.
\\
With \Cref{lemma:threshold space entropy manifold} established, we follow a similar approach as in Steps 1 to 4 from Supplementary Material Section \ref{proof of Theorem thm:main thm in the temporal-spatial model}, adapting the deviation bounds to the manifold setting. For positive constants $C$ and $C_3$, we obtain
\begin{align}
&\bigl\| \widehat{f}_{\mathcal{A}_{nm}} -\overline{f}_{\mathcal{A}_{nm}} \bigr\|_\lt ^2 \leq 2 \bigl\| \widehat{f}_{\mathcal{A}_{nm}} -\overline{f}_{\mathcal{A}_{nm}} \bigr\|_{nm, \mathcal{M}} ^2 \notag \\
&\quad + C a_n\bigl\{\phi_{\frac{nm}{S}, \mathcal{M}} \log^{5.3}(nm)+ \mathcal{A}^2_{nm} \beta(S)\bigr\},\label{eq:norm comparison 1 in the temporal-spatial model 2}
\end{align}  
and
\begin{align}
&\bigl\| \widehat{f}_{\mathcal{A}_{nm}} - \overline{f}_{\mathcal{A}_{nm}} \bigr\|_{nm, \mathcal{M}} ^2 \leq (2C_3+1) \bigl\{a_n\mathcal{A}^2_{nm} \phi_{\frac{nm}{S}, \mathcal{M}} \log^{5.3}(nm) \notag \\
&\quad +a_n \mathcal{A}^2_{nm} \beta(S) + \frac{1}{n}\bigr\}.\label{eq:iid conclude2 in the temporal-spatial model on manifold}
\end{align}
Combining \eqref{eq:norm comparison 1 in the temporal-spatial model on manifold}, \eqref{eq:norm comparison 1 in the temporal-spatial model 2} and \eqref{eq:iid conclude2 in the temporal-spatial model on manifold}, which hold with high probability, 
we establish Equation (12) in the Main Manuscript.

\end{proof}

\section{Proof of  Lemma 4 in the Main Manuscript}\label{ProofLemma1}
\begin{proof} First, we consider the simpler model where $\gamma_i(x) = 0$ for all $i\in[n],$ and $x\in[0,1]$, i.e., there is no spatial noise. This is, 
\begin{align*}
&y_{ij}=f^*(x_{ij})+\epsilon_{ij}.
\end{align*}

Recall that  $\mathcal{H}(l,\mathcal{P})$ denotes the space of hierarchical composition
model in Definition 3 in the Main Manuscript. Consider the case 
 \begin{align*}
&\widetilde{\mathcal{P}}=\left\{\left(p_1, K_1\right),\left(p_2, K_2\right), \ldots,\left(p_l, K_l\right)\right\},
\end{align*}
where
\textcolor{Maroon}{
\(p_i\in[1,\infty)\) and \(K_i\in\mathbf Z^+\) for every
\(i\in\{1,\ldots,l\}\), with
}
\(K_l \leq K_{l-1} \leq \cdots \leq K_1 \leq d\).
 By \Cref{lb_epsilon}, there exists $C_{\text{opt,1}}$, a positive constant, such that
\begin{align*}
&\limsup _{n \rightarrow \infty} \inf _{\widetilde{f}} \sup _{f^*\in \mathcal{H}(l,\widetilde{\mathcal{P}})} \mathbb{P}\bigl(\bigl\|\widetilde{f}-f^*\bigr\|_\lt ^2 >C_{\text{opt,1}}\max_{(p,K)\in\mathcal{P}}(nm)^{-2p/(2p+K)}\bigr)>0.
\end{align*}
 It now suffices to show that there exist a constant $C_{\text{opt,2}}>0,$ such that
\begin{align*}
&\limsup _{n \rightarrow \infty} \inf _{\widetilde{f}} \sup _{f^*\in \mathcal{H}(l,\widetilde{\mathcal{P}})} \mathbb{P}\left(\left\|\widetilde{f}-f^*\right\|_{\mathcal{L}_2}^2>C_{\text{opt,2}} n^{-1}\right)>0.
\end{align*}
To this end, we assume that $\epsilon_{ij}=0, \ \forall i=1,...,n;j=1,...,m_i.$ Therefore, we obtain a simple functional mean model given by,
\begin{align}
&y_{ij}(x_{ij})=X(x_{ij}) =f^*(x_{ij})+\gamma_i(x_{ij}), i\in[n],\,j\in [m_i].\label{model-only-spnoi}
\end{align}
Here $\{\gamma_{i}\}_{i=1}^{n}$ denotes the spatial noise and follow Assumption 1 in the Main Manuscript. 

Consider $S=\frac{1}{C_{\beta}}\log^{(2p+K)/2p}(n)$ with $C_{\beta}>0$ a constant stated in Assumption 1 in the Main Manuscript. Take into account the copies of the design points and spatial noises, denoted as $\{x_{ij}^*\}_{i=1,j=1}^{n,m_i}$ and $\{\gamma_{i}^*\}_{i=1}^{n}$, which were generated in Supplementary Material \ref{proof temporal-spatial model depdent}.  During the formation of these copies in Supplementary Material \ref{proof temporal-spatial model depdent}, the quantities $ | \mathcal J_{e,s}|$ and $ | \mathcal J_{o,s}|$, which represent the count of even and odd blocks respectively, satisfy $ | \mathcal J_{e,s}|,| \mathcal J_{o,s}|  \asymp n/S$. Thus, there exist positive constants $c_1$ and $c_2$ such that  
 \begin{align}
&c_1n/S\le| \mathcal J_{e,s}|,| \mathcal J_{o,s}| \le  c_2n/S.\label{BoundBlocksize-K-ap_delta}
\end{align}
By Lemma \ref{lemma:beta mixing x} and Lemma \ref{Ind-cop-measerrors}, we have that $\{x_{ij}^*\}_{i\in\mathcal{J}_{e,s},j\in[m_i]}$ and $\{\gamma_{i}^*\}_{i\in\mathcal{J}_{e,s}}$ are independent. Further, for any $i\in[n]$ it follows that $\mathbb{P}(x_{ij}\neq x_{ij}^*\ \text{for some}\ j\in[m_i])\le \beta_{x}(S).$ Moreover, by Assumption 1 in the Main Manuscript and that $S=\frac{1}{C_\beta}\log^{(2p+K)/2p}(n)\ge \frac{2}{C_{\beta}}\log(n)$, we note that $\mathbb{P}(x_{ij}\neq x_{ij}^*\ \text{for some}\ j\in[m_i])\le \beta_{x}(S)\le \frac{1}{n^2}.$  Therefore, the event $\Omega_x=\bigcap_{i=1}^{n}\{x_{ij}=x_{ij}^*\ \forall\ j\in[m_i]\}$ happens with probability at least $\frac{1}{n}$. The same is satisfied for  $\{\gamma_{i}^*\}_{i=1}^{n}$ with $\Omega_\gamma=\bigcap_{i=1}^{n}\{\gamma_i=\gamma_{i}^*\}$. We denote by $\Omega$ the event $\Omega_x\cap\Omega_\gamma.$

Let $\widetilde{f}$ be an estimator for the model described in Equation (\ref{model-only-spnoi}) based on the observations $\{(x_{ij},y_{ij})\}_{i=1,j=1}^{n,m_i}$. 
Let $s\in[S]$. We notice that $\widetilde{f}$ is an estimator for the model
\begin{align*}
&y_{ij}=f^*(x_{ij})+\gamma_i(x_{ij}),\ i\in{\mathcal{J}_{e,s}},\ j\in[m_i].
\end{align*}

Denote by $\widetilde{f}^{\mathcal{J}_{e,s}}$, the estimator such that, 
\begin{align*}
&\widetilde{f}^{\mathcal{J}_{e,s}} = \mathop{\mathrm{arginf}}_{\widehat{f}}  \| \widehat{f} - f^* \|_{\mathcal{L}_2}^2,\ i\in{\mathcal{J}_{e,s}},\ j\in[m_i],
\end{align*} 
which applies to all estimators of the model
\begin{align}
&y_{ij}=f^*(x_{ij}^*)+\gamma_i^*(x_{ij}^*),\ i\in{\mathcal{J}_{e,s}},\ j\in[m_i].\label{aux-model-lemma1}
\end{align}
Specifically, $\widetilde{f}^{\mathcal{J}_{e,s}}$ is the estimator derived from the data $x_{ij}^*$ and $\gamma_{i}^*$ for $f^*$,  with $(i,j) \in \mathcal{J}_{e,s} \times [m_i]$.
Under the event $\Omega$, $\widetilde{f}$ is an estimator for the model
\begin{align*}
&y_{ij}=f^*(x_{ij}^*)+\gamma_i^*(x_{ij}^*),\ i\in{\mathcal{J}_{e,s}},\ j\in[m_i].
\end{align*} 

Now observe that when looking for the minimax lower bound for the model described in Equation (\ref{aux-model-lemma1}), for each time $i$ only one observation among the $m_i$ available observations contributes to the determination of such minimax lower bound. To see this, suppose that conditioning in $x_{ij}^*$, we examine the distributions
$P_1$ and $P_2$ originating from $Z_{i,1}=(z_{i,1},...,z_{i,1})^{t}\in \mathbb{R}^{m_i}$ and $Z_{i,2}=(z_{i,1},...,z_{i,1})^{t}\in \mathbb{R}^{m_i}$, where $z_{i,l} \sim \mathcal{N}(\mu_l,1)$ for $l=1$ and $l=2.$ Note that the Kullback-Leibler divergence  $KL(P_1,P_2)$ is solely dependent on $KL(\mathcal{N}(\mu_1,1),\mathcal{N}(\mu_2,1))$. By invoking Le Cam's lemma, the problem reduces to the estimation of the mean based on $\vert\mathcal{J}_{e,s}\vert$ observations. Given the independence of the copies of the design points and spatial noises, it is established that the minimax lower bound for this task is $\frac{1}{\vert \mathcal{J}_{e,s}\vert}$. Thus,
\begin{align*}
&\underset{ n \rightarrow \infty}{\lim\sup}\,\underset{f^* \in \mathcal{H}(l,\widetilde{\mathcal{P}}) }{ \sup }\,\mathbb{P}\bigl( \| \widetilde{f}^{\mathcal{J}_{e,s}}-f^* \|_\lt^2 \, \geq \, \frac{C_{\text{opt,2},\mathcal{J}_{e,s}}}{\vert \mathcal{J}_{e,s}\vert } \bigr)\,>0,\,
\end{align*}
   for a positive constant $C_{\text{opt,2},\mathcal{J}_{e,s}}.$ 
   Assume that $Cm\ge m_i\ge cm,$ for a positive constant $c,C$, and $m=\Big(\frac{1}{n}\sum_{i=1}^n\frac{1}{m_i}\Big)^{-1}$. From this fact and Inequality (\ref{BoundBlocksize-K-ap_delta}) we obtain that 
   \begin{align}
&\underset{ n \rightarrow \infty}{\lim\sup}\,\underset{f^* \in \mathcal{H}(l,\widetilde{\mathcal{P}}) }{ \sup }\,\mathbb{P}\bigl( \| \widetilde{f}^{\mathcal{J}_{e,s}}-f^* \|_\lt^2 \, \geq \, \frac{C_{\text{opt,2},\mathcal{J}_{e,s}}S}{ Cc_2n } \bigr)\,>C_0>0.\,\label{AppendixK-eq1-lemma1}
\end{align}
   for a constant $C_0.$

Similarly, let 
$\widetilde{f}^{\mathcal{J}_{o,s}}$ be the estimator such that,  
\begin{align*}
&\widetilde{f}^{\mathcal{J}_{o,s}} = \mathop{\mathrm{arginf}}_{\widehat{f}} \| \widehat{f} - f^* \|_{\mathcal{L}_2}^2,\ i\in{\mathcal{J}_{o,s}},\ j\in[m_i],
\end{align*}
 which applies to  all estimators of the model
\begin{align*}
&y_{ij}=f^*(x_{ij}^*)+\gamma_i^*(x_{ij}^*),\ i\in{\mathcal{J}_{o,s}},\ j\in[m_i].
\end{align*}

Similarly, under the event $\Omega$, $\widetilde{f}$ is an estimator for the model
\begin{align*}
&y_{ij}=f^*(x_{ij}^*)+\gamma_i^*(x_{ij}^*),\ i\in{\mathcal{J}_{o,s}},\ j\in[m_i],
\end{align*} 
with
\begin{align}
&\underset{ n \rightarrow \infty}{\lim\sup}\,\underset{f^* \in \mathcal{H}(l,\widetilde{\mathcal{P}}) }{ \sup }\,\mathbb{P}\bigl( \| \widetilde{f}^{\mathcal{J}_{o,s}}-f^* \|_\lt^2 \, \geq \, \frac{C_{\text{opt,2},\mathcal{J}_{o,s}}S}{ Cc_2n } \bigr)\,>C_0>0,\,\label{AppendixK-eq2-lemma1}
\end{align}
   for a positive constant $C_{\text{opt,2},\mathcal{J}_{o,s}}.$ 

   We are now going to elaborate on the final observations of this discussion. Let $C_{\text{opt,2}}=\frac{1}{Cc_2}\min_{s\in[S]}\{{C_{\text{opt,2},\mathcal{J}_{e,s}},C_{\text{opt,2},\mathcal{J}_{o,s}}}\}$. Then,
\begin{align*}
&\underset{ n \rightarrow \infty}{\lim\sup}\,\underset{f^* \in \mathcal{H}(l,\widetilde{\mathcal{P}})      }{ \sup     }\,\mathbb{P}\left(  \| \widetilde{f}-f^* \|_\lt ^2 \,\geq \,       \frac{C_{\text{opt,2}}}{  n } \right)
\\
&\ge \underset{ n \rightarrow \infty}{\lim\sup}\,\underset{f^* \in \mathcal{H}(l,\widetilde{\mathcal{P}}) }{ \sup }\,\mathbb{P}\bigl( \Big\{ \|\widetilde{f}-f^* \|_\lt ^2 \, \geq \, \frac{C_{\text{opt,2}}}{ n } \Big\}\cap\Omega\bigr) >0.
\end{align*}
The logic behind the second inequality is as follows. For any $f^*\in\mathcal{F},$
\begin{align}
&\label{AppendixK-eq3-lemma1}
      \mathbb{P}\left( \Big\{ \| \widetilde{f}-f^* \|_\lt ^2 \,\geq \,       \frac{C_{\text{opt,2}}}{  n } \Big\}\cap\Omega\right)\nonumber
\\
&\ge \mathbb{P}\left( \Big\{ \| \widetilde{f}-f^* \|_\lt ^2 \,\geq \,       \frac{SC_{\text{opt,2}}}{  n } \Big\}\cap\Omega\right)\nonumber
\\
&\ge\mathbb{P}\bigl(\bigcup_{s=1}^S \Big\{\Big\{ \| \widetilde{f}^{\mathcal{J}_{e,s}}-f^* \|_\lt^2 \, \geq \, \frac{SC_{\text{opt,2},\mathcal{J}_{e,s}}}{ Cc_2n } \Big\} \notag \\
&\quad \bigcup \Big\{ \| \widetilde{f}^{\mathcal{J}_{o,s}}-f^* \|_\lt^2 \, \geq \, \frac{SC_{\text{opt,2},\mathcal{J}_{o,s}}}{ Cc_2n } \Big\}\Big\}\cap\Omega\bigr) \nonumber
\\
&>C_0,
\end{align}
where the first inequality is followed by the fact that $S=\frac{1}{C_\beta}\log^{(2p+K)/2p}(n)\ge1.$ Moreover the third inequality is achieved from Inequality (\ref{AppendixK-eq1-lemma1}) and (\ref{AppendixK-eq2-lemma1}). The derivation of the second inequality is explained below.

Under the event $\Omega$ the relation, 
\begin{align*}
&\| \widetilde{f}-f^*  \|_\lt ^2\ge \Big\{\| \widetilde{f}^{\mathcal{J}_{e,s}}-f^* \|_\lt ^2,\| \widetilde{f}^{\mathcal{J}_{o,s}}-f^* \|_\lt ^2\Big\},
\end{align*} holds for any $s\in[S].$ These specific conditions allow us to conclude that the event $\{\| \widetilde{f}-f^* \|_\lt ^2< \frac{SC_{\text{opt,2}}}{  n } \}$ is contained in the event
\begin{align*}
&\bigcap_{s=1}^S \Big\{\Big\{ \| \widetilde{f}^{\mathcal{J}_{e,s}}-f^* \|_\lt^2 \,< \, \frac{SC_{\text{opt,2},\mathcal{J}_{e,s}}}{ Cc_2n } \Big\}\bigcap \Big\{ \| \widetilde{f}^{\mathcal{J}_{o,s}} -f^* \|_\lt^2 \,< \, \frac{SC_{\text{opt,2},\mathcal{J}_{o,s}}}{ Cc_2n } \Big\}\Big\},
\end{align*}
and consequently the second inequality in Inequality (\ref{AppendixK-eq3-lemma1}) is satisfied. Specifically, suppose that  $\|\widetilde{f}-f^*  \|_{2}^2 \,< \,       \frac{SC_{\text{opt,2}}}{  n } $. For any $s\in[S]$ by the definition of $C_{\text{opt,2}}$ we have that $\| \widetilde{f}-f^*  \|_{2}^2 \,< \,       \frac{SC_{\text{opt,2},\mathcal{J}_{e,s}}}{  Cc_2n }$. It follows that $\| \widetilde{f}^{\mathcal{J}_{e,s}}-f^* \|_\lt^2 \,< \,       \frac{SC_{\text{opt,2},\mathcal{J}_{e,s}}}{  Cc_2n }$. Similarly, $\| \widetilde{f}^{\mathcal{J}_{o,s}}-f^{*}\|_{2}^2 \,< \,       \frac{SC_{\text{opt,2},\mathcal{J}_{o,s}}}{  Cc_2n }$, concluding the contention of the aforementioned events.

The claim is then followed by taking $C_{\text{opt}}=\frac{1}{2}\min\{C_{\text{opt,1}},C_{\text{opt,2}}\}$. 
To see this, we think two cases. 
For case 1 where $\frac{1}{n} \geq (nm)^{-2p/(2p+K)}$, then 
\begin{align*}
&\frac{C_{\text{opt}}}{(nm)^{2p/(2p+K)}} \leq \frac{C_{\text{opt}}}{n} \leq \frac{C_{\text{opt,2}}}{2n}, \text{ so } C_{\text{opt}}\bigl( \frac{1}{n } \,+\, \bigl( \frac{1}{nm}\bigr)^{\frac{2p}{2p+K}} \bigr) \leq \frac{C_{\text{opt,2}}}{n},
\end{align*}
which leads to 
\begin{align*}
&\mathbb{P}\bigl( \| \widetilde{f}-f^* \|_\lt^2 \,\geq \, C_{\text{opt}}\bigl( \frac{1}{n } \,+\, \bigl( \frac{1}{nm}\bigr)^{\frac{2p}{2p+K}} \bigr) \bigr) \geq \mathbb{P}\bigl( \| \widetilde{f}-f^* \|_\lt ^2 \,\geq \, \frac{C_{\text{opt,2}}}{ n } \bigr)>0.
\end{align*}
For case 2 where $\frac{1}{n} < (nm)^{-2p/(2p+K)}$, then 
\begin{align*}
&\frac{C_{\text{opt}}}{n} \leq \frac{C_{\text{opt}}}{(nm)^{2p/(2p+K)}} \leq \frac{C_{\text{opt,1}}}{2(nm)^{2p/(2p+K)}},
\end{align*}

so 
\begin{align*}
&C_{\text{opt}}\left(    \frac{1}{n }  \,+\, \left( \frac{1}{nm}\right)^{\frac{2p}{2p+K}} \right) \leq \frac{C_{\text{opt,1}}}{(nm)^{2p/(2p+K)}},
\end{align*}

which leads to 
\begin{align*}
&\mathbb{P}\left(  \| \widetilde{f} - f^* \|_\lt^2 \,\geq\,  C_{\text{opt}}\left( \frac{1}{n} \,+\, \left( \frac{1}{nm} \right)^{\frac{2p}{2p+K}} \right) \right) \\
&\geq \mathbb{P}\left(  \| \widetilde{f} - f^* \|_\lt^2 \,\geq\, \frac{C_{\text{opt,1}}}{(nm)^{2p/(2p+K)}} \right) > 0.
\end{align*}
  Since the above condition holds for any $(p,k) \in \mathcal{P}$, it leads to 
\begin{align*}
&\mathbb{P}\bigl(\|\widetilde{f}-f^*\|_\lt^2\,\geq\,
 C_{\text{opt}}\bigl(\frac{1}{n}\, +\,\max_{(p,K)\in\mathcal{P}}\left(\frac{1}{nm}\right)^{\frac{2p}{2p+K}}\bigr)\bigr)>0.
\end{align*}
\end{proof}

\begin{lemma}
\label{lb_epsilon}
    Consider the simpler model where $\delta_i(x) = 0$ for all $i\in[n],$ and $x\in[0,1]$, i.e., there is no spatial noise. This is, 
\begin{align*}
&y_{ij}=f^*(x_{ij})+\epsilon_{ij}.
\end{align*}
Recall that  $\mathcal{H}(l,\mathcal{P})$ denotes the space of hierarchical composition
model in Definition 3 in the Main Manuscript. \textcolor{Maroon}{Consider the case
\begin{align*}
&\mathcal P=\widetilde{\mathcal P}
:=
\left\{
\left(p_1,K_1\right),
\left(p_2,K_2\right),
\ldots,
\left(p_l,K_l\right)
\right\},
\end{align*}
where $p_i\in[1,\infty)$ and $K_i\in\mathbf Z^+$ for every
$i\in\{1,\ldots,l\}$, with
$K_l\leq K_{l-1}\leq\cdots\leq K_1\leq d$}. There exists $C_{\text{opt,1}}$, a positive constant, such that
\begin{align*}
&\limsup _{n \rightarrow \infty} \inf _{\widetilde{f}} \sup _{f^*\in \mathcal{H}(l,\widetilde{\mathcal{P}})} \mathbb{P}\bigl(\bigl\|\widetilde{f}-f^*\bigr\|_\lt ^2 >C_{\text{opt,1}}\max_{(p,K)\in\mathcal{P}}(nm)^{-2p/(2p+K)}\bigr)>0.
\end{align*}
\end{lemma}
\begin{proof}
    A direct consequence of \cite{schmidt2020nonparametric} is the following. Let $\{x_{ij}\}_{i=1,j=1}^{n,m_i}\subset [0,1]^d$ denoting the random locations from Assumption 1 in the Main Manuscript, where the temporal-spatial data $\{y_{ij}\}_{i=1,j=1}^{n,m_i},$ defined by
\begin{align}
&y_{ij}=f^*(x_{ij})+\epsilon_{ij},\label{Model-K-appendix}
\end{align}
is observed. Here $\{\epsilon_{ij}\}_{i=1,j=1}^{n,m_i}$ denotes the measurements error and follow Assumption 1 in the Main Manuscript. Consider $S=\frac{1}{C_{\beta}}\log^{(2p+K)/2p}(n)$ with $C_{\beta}>0$ a constant stated in Assumption 1 in the Main Manuscript. Take into account the copies of the design points and measurement errors, denoted as $\{x_{ij}^*\}_{i=1,j=1}^{n,m_i}$ and $\{\epsilon_{ij}^*\}_{i=1,j=1}^{n,m_i}$, which were generated in Supplementary Material \ref{proof temporal-spatial model depdent}.  During the formation of these copies in Supplementary Material \ref{proof temporal-spatial model depdent}, the quantities $ | \mathcal J_{e,s}|$ and $ | \mathcal J_{o,s}|$, which represent the count of even and odd blocks respectively, satisfy $ | \mathcal J_{e,s}|,| \mathcal J_{o,s}|  \asymp n/S$. Thus, there exist positive constants $c_1$ and $c_2$ such that  
 \begin{align}
&c_1n/S\le| \mathcal J_{e,s}|,| \mathcal J_{o,s}| \le  c_2n/S.\label{BoundBlocksize-K-ap}
\end{align}
By Lemma \ref{lemma:beta mixing x} and Lemma \ref{Ind-cop-measerrors}, we have that $\{x_{ij}^*\}_{i\in\mathcal{J}_{e,s},j\in[m_i]}$ and $\{\epsilon_{ij}^*\}_{i\in\mathcal{J}_{o,s},j\in[m_i]}$ are independent. Further, for any $i\in[n]$ it follows that $\mathbb{P}(x_{ij}\neq x_{ij}^*\ \text{for some}\ j\in[m_i])\le \beta_{x}(S).$ Moreover, by Assumption 1 in the Main Manuscript and that $S=\frac{1}{C_\beta}\log^{(2p+K)/2p}(n)\ge \frac{2}{C_{\beta}}\log(n)$, we note that $\mathbb{P}(x_{ij}\neq x_{ij}^*\ \text{for some}\ j\in[m_i])\le \beta_{x}(S)\le \frac{1}{n^2}.$  Therefore, the event $\Omega_x=\bigcap_{i=1}^{n}\{x_{ij}=x_{ij}^*\ \forall\ j\in[m_i]\}$ happens with probability at least $\frac{1}{n}$. The same is satisfied for  $\{\epsilon_{ij}^*\}_{i=1,j=1}^{n,m_i}$ with $\Omega_\epsilon=\bigcap_{i=1}^{n}\{\epsilon_{ij}=\epsilon_{ij}^*\ \forall\ j\in[m_i]\}$. We denote by $\Omega$ the event $\Omega_x\cap\Omega_\epsilon.$

Let $\widetilde{f}$ be an estimator for the model described in Equation (\ref{Model-K-appendix}) based on the observations $\{(x_{ij},y_{ij})\}_{i=1,j=1}^{n,m_i}$. 


Let $s\in[S]$. We notice that $\widetilde{f}$ is an estimator for the model
\begin{align*}
&y_{ij}=f^*(x_{ij})+\epsilon_{ij},\ i\in{\mathcal{J}_{e,s}},\ j\in[m_i].
\end{align*}

Denote by $\widetilde{f}^{\mathcal{J}_{e,s}}$, the estimator such that, 
\begin{align*}
&\widetilde{f}^{\mathcal{J}_{e,s}} = \mathop{\mathrm{arginf}}_{\widehat{f}}  \| \widehat{f} - f^* \|_{\mathcal{L}_2}^2,\ i\in{\mathcal{J}_{e,s}},\ j\in[m_i],
\end{align*} 
which applies to all estimators of the model
\begin{align*}
&y_{ij}=f^*(x_{ij}^*)+\epsilon_{ij}^*,\ i\in{\mathcal{J}_{e,s}},\ j\in[m_i].
\end{align*} 
Specifically, $\widetilde{f}^{\mathcal{J}_{e,s}}$ is the estimator derived from the data $x_{ij}^*$ and $\epsilon_{ij}^*$ for $f^*$, 
 with $(i,j) \in \mathcal{J}_{e,s} \times [m_i]$.
Under the event $\Omega$, $\widetilde{f}$ is an estimator for the model
\begin{align*}
&y_{ij}=f^*(x_{ij}^*)+\epsilon_{ij}^*,\ i\in{\mathcal{J}_{e,s}},\ j\in[m_i],
\end{align*} 

Using the independence of the copies of the design points and measurement errors, a direct consequence of \cite{schmidt2020nonparametric} is that 
\begin{align*}
&\underset{ n \rightarrow \infty}{\lim\sup}\,\underset{f^* \in \mathcal{H}(l,\widetilde{\mathcal{P}}) }{ \sup }\,\mathbb{P}\bigl( \| \widetilde{f}^{\mathcal{J}_{e,s}}-f^* \|_\lt ^2 \, \geq \, C_{\text{opt,1},\mathcal{J}_{e,s}} \max_{(p,K)\in \mathcal{P}} \Big(\sum_{i\in\mathcal{J}_{e,s}}m_i\Big)^{-2p/(2p+K)} \bigr)\,>0,\,
\end{align*}
   for a positive constant $C_{\text{opt,1},\mathcal{J}_{e,s}}.$ Thus, for any $(p,k) \in \mathcal{P}$, 
\begin{align*}
&\underset{ n \rightarrow \infty}{\lim\sup}\,\underset{f^* \in \mathcal{H}(l,\widetilde{\mathcal{P}}) }{ \sup }\,\mathbb{P}\bigl( \| \widetilde{f}^{\mathcal{J}_{e,s}}-f^* \|_\lt ^2 \, \geq \, C_{\text{opt,1},\mathcal{J}_{e,s}} \Big(\sum_{i\in\mathcal{J}_{e,s}}m_i\Big)^{-2p/(2p+K)} \bigr)\,>0.
\end{align*}   
   Assume that $Cm\ge m_i\ge cm,$ for a positive constant $c,C$, and $m=\Big(\frac{1}{n}\sum_{i=1}^n\frac{1}{m_i}\Big)^{-1}$. From this fact and Inequality (\ref{BoundBlocksize-K-ap}) we obtain that 
   \begin{align}
&\underset{ n \rightarrow \infty}{\lim\sup}\,\underset{f^* \in \mathcal{H}(l,\widetilde{\mathcal{P}}) }{ \sup }\,\mathbb{P}\bigl( \| \widetilde{f}^{\mathcal{J}_{e,s}}-f^* \|_\lt ^2 \, \geq \, \frac{C_{\text{opt,1},\mathcal{J}_{e,s}}S^{2p/(2p+K)}}{ (Cc_2nm)^{2p/(2p+K)} } \bigr)\,>C_0>0.\,\label{AppendixK-eq1}
\end{align}
   for a constant $C_0.$
Similarly, let 
$\widetilde{f}^{\mathcal{J}_{o,s}}$ be the estimator such that,  
\begin{align*}
&\widetilde{f}^{\mathcal{J}_{o,s}} = \mathop{\mathrm{arginf}}_{\widehat{f}} \| \widehat{f} - f^* \|_{\mathcal{L}_2}^2,\ i\in{\mathcal{J}_{o,s}},\ j\in[m_i],
\end{align*}
 which applies to  all estimators of the model
\begin{align*}
&y_{ij}=f^*(x_{ij}^*)+\epsilon_{ij}^*,\ i\in{\mathcal{J}_{o,s}},\ j\in[m_i],
\end{align*} 
with 
\begin{align}
&\underset{ n \rightarrow \infty}{\lim\sup}\,\underset{f^* \in \mathcal{H}(l,\widetilde{\mathcal{P}}) }{ \sup }\,\mathbb{P}\bigl( \| \widetilde{f}^{\mathcal{J}_{o,s}}-f^* \|_\lt ^2 \, \geq \, \frac{C_{\text{opt,1},\mathcal{J}_{o,s}}S^{2p/(2p+K)}}{ (Cc_2nm)^{2p/(2p+K)} } \bigr)\,>C_0>0,\,\label{AppendixK-eq2}
\end{align}
   for a positive constant $C_{\text{opt,1},\mathcal{J}_{o,s}}.$ 

Under the event $\Omega$, 
$\widetilde{f}$ is an estimator for the model
\begin{align*}
&y_{ij}=f^*(x_{ij}^*)+\epsilon_{ij}^*,\ i\in{\mathcal{J}_{o,s}},\ j\in[m_i],
\end{align*}

   We are now going to elaborate on the final observations of this discussion. Let \begin{align*}
&C_{\text{opt,1}}=\frac{1}{(Cc_2)^{2p/(2p+K)}}\min_{s\in[S]}\{{C_{\text{opt,1},\mathcal{J}_{e,s}},C_{\text{opt,1},\mathcal{J}_{o,s}}}\}.
\end{align*} Then,
\begin{align}
&\label{lower_bound_basic}
     \underset{ n \rightarrow \infty}{\lim\sup}\,\underset{f^* \in \mathcal{H}(l,\widetilde{\mathcal{P}})      }{  \sup     }\,\mathbb{P}\left(  \| \widetilde{f}-f^* \|_\lt ^2 \,\geq \,       \frac{C_{\text{opt,1}}}{  (nm)^{2p/(2p+K)} } \right)
\\
&\ge \underset{ n \rightarrow \infty}{\lim\sup}\,\underset{f^* \in \mathcal{H}(l,\widetilde{\mathcal{P}}) }{ \sup }\,\mathbb{P}\bigl( \Big\{ \|\widetilde{f}-f^* \|_\lt ^2 \, \geq \, \frac{C_{\text{opt,1}}}{ (nm)^{2p/(2p+K)} } \Big\}\cap\Omega\bigr) >0. \nonumber
\end{align}
The logic behind the second inequality is as follows. For any $f^*\in\mathcal{F},$
\begin{align}
&\label{AppendixK-eq3}
      \mathbb{P}\left( \Big\{ \| \widetilde{f}-f^* \|_\lt ^2 \,\geq \,       \frac{C_{\text{opt,1}}}{  (nm)^{2p/(2p+K)} } \Big\}\cap\Omega\right)\nonumber
\\
&\ge \mathbb{P}\left( \Big\{ \| \widetilde{f}-f^* \|_\lt ^2 \,\geq \,       \frac{S^{2p/(2p+K)}C_{\text{opt,1}}}{  (nm)^{2p/(2p+K)} } \Big\}\cap\Omega\right)\nonumber
\\
&\ge\mathbb{P}\bigl(\bigcup_{l=1}^S \Big\{\Big\{ \| \widetilde{f}^{\mathcal{J}_{e,s}}-f^* \|_\lt ^2 \, \geq \, \frac{S^{2p/(2p+K)}C_{\text{opt,1},\mathcal{J}_{e,s}}}{ (Cc_2nm)^{2p/(2p+K)} } \Big\}\bigcup \Big\{ \| \widetilde{f}^{\mathcal{J}_{o,s}}-f^* \|_\lt ^2
\\
&\left. \,\geq \,       \frac{S^{2p/(2p+K)}C_{\text{opt,1},\mathcal{J}_{o,s}}}{  (Cc_2nm)^{2p/(2p+K)} } \Big\}\Big\}\cap\Omega\right)\nonumber
\\
&>C_0,
\end{align}
where the first inequality is followed by the fact that $S=\frac{1}{C_\beta}\log^{(2p+K)/2p}(n)\ge1.$ Moreover the third inequality is achieved from Inequality (\ref{AppendixK-eq1}) and (\ref{AppendixK-eq2}). The derivation of the second inequality is explained below.

Under the event $\Omega$ the relation, 
\begin{align*}
&\| \widetilde{f}-f^*  \|_\lt ^2\ge \Big\{\| \widetilde{f}^{\mathcal{J}_{e,s}}-f^* \|_\lt ^2,\| \widetilde{f}^{\mathcal{J}_{o,s}}-f^* \|_\lt ^2\Big\},
\end{align*} holds for any $s\in[S].$ These specific conditions allow us to conclude that the event $\{\| \widetilde{f}-f^* \|_\lt ^2< \frac{S^{2p/(2p+K)}C_{\text{opt,1}}}{  (nm)^{2p/(2p+K)} } \}$ is contained in the event
\begin{align*}
\bigcap_{s=1}^S \Big\{\, 
  &\Big\{ \| \widetilde{f}^{\mathcal{J}_{e,s}} - f^* \|_\lt^2 \,<\, \frac{S^{2p/(2p+K)} C_{\text{opt,1},\mathcal{J}_{e,s}}}{(Cc_2nm)^{2p/(2p+K)}} \Big\} \\
  &\bigcap\ \Big\{ \| \widetilde{f}^{\mathcal{J}_{o,s}} - f^* \|_\lt^2 \,<\, \frac{S^{2p/(2p+K)} C_{\text{opt,1},\mathcal{J}_{o,s}}}{(Cc_2nm)^{2p/(2p+K)}} \Big\} 
\,\Big\}
\end{align*} 
and consequently the second inequality in Inequality (\ref{AppendixK-eq3}) is satisfied. Specifically, suppose that  $\|\widetilde{f}-f^*  \|_\lt ^2 \,< \,       \frac{S^{2p/(2p+K)}C_{\text{opt,1}}}{  (nm)^{2p/(2p+K)} } $. For any $s\in[S]$ by the definition of $C_{\text{opt,1}}$ we have that $\| \widetilde{f}-f^*  \|_\lt ^2 \,< \,       \frac{S^{2p/(2p+K)}C_{\text{opt,1},\mathcal{J}_{e,s}}}{  (Cc_2nm)^{2p/(2p+K)} }$. It follows that $\| \widetilde{f}^{\mathcal{J}_{e,s}}-f^* \|_\lt ^2 \,< \,       \frac{S^{2p/(2p+K)}C_{\text{opt,1},\mathcal{J}_{e,s}}}{  (Cc_2nm)^{2p/(2p+K)} }$. Similarly, $\| \widetilde{f}^{\mathcal{J}_{o,s}}-f^*\|_\lt ^2 \,< \,       \frac{S^{2p/(2p+K)}C_{\text{opt,1},\mathcal{J}_{o,s}}}{  (Cc_2nm)^{2p/(2p+K)} }$, concluding the contention of the aforementioned events. 

Since \eqref{lower_bound_basic} holds for any $(p,k)\in \mathcal{P}$, it holds that
\begin{align*}
&\underset{ n \rightarrow \infty}{\lim\sup}\,\underset{f^* \in \mathcal{H}(l,\widetilde{\mathcal{P}}) }{ \sup }\,\mathbb{P}\bigl( \| \widetilde{f}-f^* \|_\lt ^2 \, \geq \, C_{\text{opt,1}} \max_{(p,K)\in \mathcal{P}} (nm)^{-2p/(2p+K)} \bigr) >0.
\end{align*}

\end{proof}

\section{Additional Technical Lemmas}


\begin{lemma} \label{lemma:threshold space entropy}
For any fixed $\mathcal{A}_{nm}$, it holds that 
\begin{align*}
&\log \mathcal N(\delta,\mathcal F_{\mathcal{A}_{nm}}(L, r ) , \|\cdot \|_{nm} ) \lesssim L^2 r^2 \log(Lr^2) \log(\mathcal{A}_{nm}\delta ^{-1} )
\end{align*} Furthermore, with $L,r$ satisfying \eqref{L_r_size}
$Lr \asymp \log \left(nm\right)\max_{(p, K) \in \mathcal P  } \left(nm\right)^{\frac{K}{2(2p+K)}}$ and \\ 
$ Lr^2 \lesssim \log\left(nm\right)\left[\max_{(p, K) \in \mathcal P  }\left(nm\right)^{\frac{K}{2(2p+K)}}\right]^2 \lesssim nm$, we have 
\begin{align*}
&\log \mathcal N(\delta,\mathcal F_{\mathcal{A}_{nm}} (L , r ) , \|\cdot \|_{nm} ) \lesssim nm\phi_{nm}\log^{3.3} (nm) \log(\delta^{-1}).
\end{align*}
\end{lemma}

\begin{proof}
It follows from Theorem 6 of 
\cite{bartlett2019nearly}  (see also the proof of Lemma 19 on page 78 of \cite{kohler2021rate}) that 
the VC-dimension of $\mathcal{F}(L, r )$ satisfies 
\begin{align*}
&\text{VC}(\mathcal F(L, r  )) \lesssim L^2 r^2 \log(Lr^2) .
\end{align*}

By a slight variant of Lemma 19 of \cite{kohler2021rate}, where $\mathcal L_1$ norm is replaced by $\mathcal L_2$ norm (which is implied by Lemma 9.2 (\Cref{lem_9_2}) and Theorem 9.4 (\Cref{lem_9_4}) of 
\cite{gyorfi2002distribution}), it follows that 
\begin{align*}
&\log \mathcal N(\delta,\mathcal F_{\mathcal{A}_{nm}}(L, r ) , \|\cdot \|_{nm} ) \lesssim L^2 r^2 \log(Lr^2) \bigl\{ \log(\mathcal{A}_{nm}^2\delta ^{-2} ) + \log\log(\mathcal{A}_{nm}^2\delta^{-2}) \bigr\}
\\
&\lesssim 
L^2 r^2 \log(Lr^2)   \log(\mathcal{A}_{nm}\delta ^{-1} ).
\end{align*} 
Furthermore, with the choice of $L$ and $r$,
\begin{align*}
&\log \mathcal N(\delta,\mathcal F_{\mathcal{A}_{nm}} (L , r ) , \|\cdot \|_{nm} ) \lesssim L^2 r^2 \log(Lr^2) \log({\mathcal{A}_{nm}}\delta ^{-1} )
\\
&\lesssim   nm\phi_{nm}\log^3 (nm) \log(\mathcal{A}_{nm}\delta^{-1})
\\
&=   nm\phi_{nm}\log^3 (nm) \left[\log(\mathcal{A}_{nm})+\log(\delta^{-1})\right]
\\
&\lesssim   nm\phi_{nm}\log^3 (nm) \log(\mathcal{A}_{nm})\log(\delta^{-1})
\\
&\lesssim   nm\phi_{nm}\log^3 (nm) \log(\log(nm))\log(\delta^{-1})
\\
&\lesssim   nm\phi_{nm}\log^{3.3} (nm) \log(\delta^{-1}).
\end{align*}


\end{proof}

\begin{lemma} \label{lemma:threshold space entropy 2}
For  non-random function $g^*$, denote 
\begin{align*}
&\mathcal G_{\mathcal{A}_{nm}} = \left\{ f_{\mathcal{A}_{nm}} -g^* : f \in \mathcal{F}(L,r) \right\}.
\end{align*}
Then it holds that 
\begin{align*} 
& \log  \mathcal N(\delta,\mathcal G_{\mathcal{A}_{nm}} , \|\cdot  \|_{nm} )  \lesssim L^2 r^2 \log(Lr^2)   \log(\mathcal{A}_{nm}\delta ^{-1} )   .
\end{align*}

Furthermore, if the probability distribution 
is supported on some manifold $\mathcal M$, then it also holds that 
\begin{align*}
&\log \mathcal N(\delta,\mathcal{G}_{\mathcal{A}_{nm}} , \|\cdot \|_{nm, \mathcal{M}} ) \lesssim L^2 r^2 \log(Lr^2) \log(\mathcal{A}_{nm}\delta ^{-1} ) .
\end{align*}
\end{lemma}

\begin{proof}
 It suffices to observe that 
 $ (f_{\mathcal{A}_{nm}} -g^*) -(\tilde{f}_{\mathcal{A}_{nm}} -g^*) = f_{\mathcal{A}_{nm}} - \tilde{f}_{\mathcal{A}_{nm}}$, for any $f$ and $\tilde{f}$ in $\mathcal{F}(L,r)$. Thus, the distance between any two functions in $\mathcal{G}_{\mathcal{A}_{nm}}$ is the same as the distance between the corresponding functions in $\mathcal F_{\mathcal{A}_{nm}}(L, r  )$. 
Then it holds that 
\begin{align*}
& \log  \mathcal N(\delta,\mathcal{G}_{\mathcal{A}_{nm}} , \|\cdot  \|_{nm} ) = \log  \mathcal N(\delta,\mathcal{F}_{\mathcal{A}_{nm}}(L,r) , \|\cdot  \|_{nm} )
\\ 
& \lesssim L^2 r^2 \log(Lr^2)   \log({\mathcal{A}_{nm}}\delta ^{-1} ),
\end{align*}
where the last inequality follows from the proof in \Cref{lemma:threshold space entropy}.

Furthermore, recall the probability distribution is supported on a manifold $\mathcal{M}$, we write $\|f\|_{nm}$ as $\|f\|_{nm, \mathcal{M}}$. 

Then it holds that 
\begin{align*} 
\log  \mathcal N(\delta,\mathcal{G}_{\mathcal{A}_{nm}}  , \|\cdot  \|_{nm, \mathcal{M}} ) 
& = \log  \mathcal N(\delta,\mathcal{G}_{\mathcal{A}_{nm}} , \|\cdot  \|_{nm} ) \\ 
& \lesssim L^2 r^2 \log(Lr^2)   \log({\mathcal{A}_{nm}}\delta ^{-1} ).
\end{align*}

\end{proof}

\begin{lemma}[Lemma 9.2, \cite{gyorfi2002distribution}]
\label{lem_9_2}
Let $\mathcal{G}$ be a class of functions on $\mathbb{R}^d$ and let $\nu$ be a probability measure on $\mathbb{R}^d, p \geq 1$ and $\epsilon>0$. Then
\begin{align*}
&\mathcal{M}\bigl(2 \epsilon, \mathcal{G},\|\cdot\|_{L_p(\nu)}\bigr) \leq \mathcal{N}\bigl(\epsilon, \mathcal{G},\|\cdot\|_{L_p(\nu)}\bigr) \leq \mathcal{M}\bigl(\epsilon, \mathcal{G},\|\cdot\|_{L_p(\nu)}\bigr) .
\end{align*}
In particular,
\begin{align*}
&\mathcal{M}_p\left(2 \epsilon, \mathcal{G}, z_1^n\right) \leq \mathcal{N}_p\left(\epsilon, \mathcal{G}, z_1^n\right) \leq \mathcal{M}_p\left(\epsilon, \mathcal{G}, z_1^n\right)
\end{align*}
for all $z_1, \ldots, z_n \in \mathbb{R}^d$.
\end{lemma}

\begin{lemma}[Theorem 9.4, \cite{gyorfi2002distribution}]
\label{lem_9_4}
 Let $\mathcal{G}$ be a class of functions $g: \mathbb{R}^d \rightarrow[0, B]$ with $V_{\mathcal{G}^{+}} \geq 2$, let $p \geq 1$, let $\nu$ be a probability measure on $\mathbb{R}^d$, and let $0<\epsilon<\frac{B}{4}$. Then
\begin{align*}
&\mathcal{N}\left(\delta, \mathcal{G},\|\cdot\|_{\mathcal L_2(\nu)}\right) \leq 3\left(\frac{2 e B^2}{\epsilon^2} \log \frac{3 e B^2}{\epsilon^2}\right)^{V_{\mathcal{G}^{+}}}.
\end{align*}	
\end{lemma}


\begin{lemma}
\label{lemma:covariance inequality for Hilbert-valued random variables}
Suppose Assumption 1 holds. Let $\gamma_0$ and $\gamma_i$ are identically distributed. Then 
\begin{align*}
&\big|  \E(   \langle \gamma_0, \gamma_i\rangle    )  \big| \le 2\beta^{1/r}(i) \E(\|\gamma_0\|_\lt^{2q})^{1/q}
\end{align*}
for all $ i\in \mathbb Z^+$ and $q,r$ such that $q^{-1} + r^{-1} =1$.
\end{lemma}
\begin{proof}
By \Cref{lemma:coupling in the temporal-spatial model}, there exists  $\gamma_i^*$ such that $ \gamma_0 $ and $\gamma_i^*$ are independent, that  $\gamma_i^*$   are identically distributed to    $\gamma_i $, and that 
\begin{align*}
&\p( \gamma_i^* \not = \gamma_i) \le \beta(i).
\end{align*}
Observe that 
\begin{align*}
   \big| \E(\langle \gamma_0, \gamma_i\rangle )\big|  = \big| \E(\langle \gamma_0, \gamma_i-\gamma_i^*\rangle ) +  \E(\langle \gamma_0  ,  \gamma_i^*\rangle )  \big| .
\end{align*}
Note that  by independence 
\begin{align*}
&\E( \langle \gamma_0 , \gamma_i^*\rangle)=0,
\end{align*}
and that 
\begin{align*}
&\big| \E(\langle \gamma_0, \gamma_i-\gamma_i^*\rangle ) \big| = \big| \E \big( \mathbb I_{ \{ \gamma_i \neq\gamma_i^* \} }\langle \gamma_0, \gamma_i-\gamma_i^*\rangle \big) \big| \le \E \big( \mathbb I_{ \{ \gamma_i \neq\gamma_i^* \} }\| \gamma_0\|_\lt \| \gamma_i-\gamma_i^*\|_\lt \big)
\\
&\le 2 \E \big( \mathbb I_{ \{ \gamma_i \neq\gamma_i^* \} }\| \gamma_0\|_\lt \| \gamma_i \|_\lt \big) \le 2\p(\gamma_i^* \not = \gamma_i ) ^{1/r} \E(\|\gamma_0\|_\lt^{2q})^{1/q} \le 2\beta^{1/r}(i) \E(\|\gamma_0\|_\lt^{2q})^{1/q},
\end{align*}
where the second inequality follows by Triangle inequality and the fact that $\delta_i$ and $\delta_i^*$ have the same distribution, the third inequality follows by Hölder's inequality and that $\gamma_0$ and $\gamma_i$ have the same distribution.
\end{proof}

\begin{lemma}\label{lemma:bound on the expected norm of gamma}
Suppose that Assumption 1 holds. Let $r, q>1$ be such that $\frac{1}{q}+\frac{1}{r}=1$. In particular, suppose $\sum_{l=1}^{\infty} \beta^{1 / r}(l)<\infty$ and $\sup_{i\in \mathbb Z^+}\mathbb{E}\left(\|\gamma_i\|_{\mathcal{L}_2}^{2q}\right)<\infty$ as stated by Assumption 1.
Then
\begin{align*}
&\mathbb{E}\left(\left\|\frac{1}{n} \sum_{i=1}^n \gamma_i\right\|_{\mathcal{L}_2}^2\right) \lesssim \frac{\sigma_\gamma^2}{n}.
\end{align*}
\end{lemma}

\begin{proof}
Note that
\begin{align*}
&\mathbb{E}\bigl(\bigl\|\frac{1}{n} \sum_{i=1}^n \gamma_i\bigr\|_{\mathcal{L}_2}^2\bigr) =\frac{1}{n^2} \sum_{i=1}^n \mathbb{E}\bigl\{\bigl\|\gamma_i\bigr\|_{\mathcal{L}_2}^2\bigr\}+\frac{2}{n^2} \sum_{i<j} \mathbb{E}\bigl\{\bigl\langle\gamma_i, \gamma_j\bigr\rangle\bigr\} .
\end{align*}
Without loss of generality, we assume $\sigma_\gamma^2$ = 1, so that $\frac{1}{n^2} \sum_{i=1}^n \mathbb{E}\left\{\left\|\gamma_i\right\|_{\mathcal{L}_2}^2\right\} \lesssim \frac{1}{n}$. In addition,
\begin{align*}
&\frac{1}{n^2} \sum_{i<j} \mathbb{E}\left\{\left\langle\gamma_i, \gamma_j\right\rangle\right\} \le \frac{1}{n^2} \sum_{i=1}^n \sum_{j=i+1}^{\infty}\left|\mathbb{E}\left\{\left\langle\gamma_i, \gamma_j\right\rangle\right\}\right|
\\
&\lesssim \frac{1}{n^2} \sum_{i=1}^n \sum_{j=i+1}^{\infty} \beta^{1 / r}(j-i) \mathbb{E}\bigl(\bigl\|\gamma_i\bigr\|_{\mathcal{L}_2}^{2q}\bigr)^{1 / q} \lesssim \frac{1}{n^2} \sum_{i=1}^n \mathbb{E}\bigl(\bigl\|\gamma_i\bigr\|_{\mathcal{L}_2}^{2q}\bigr)^{1 / q} \lesssim \frac{1}{n},
\end{align*}
where the second inequality follows from \Cref{lemma:covariance inequality for Hilbert-valued random variables} and the third inequality by the fact that from Assumption 1 we have $ \sum_{l=1}^\infty \beta^{1/r}(l) < \infty $.
\end{proof}

\section{Additional Experiment Results}
\label{table_result_box_plots}
In this section, we present additional experiments from Section V in the Main Manuscript, which further support our previous findings and demonstrate that our method consistently outperforms the competitors.

\begin{figure*}[!t]
    
    \centering

\begin{minipage}[c]{\textwidth}
    \centering
    \begin{minipage}[c]{0.05\textwidth}
        \centering \rotatebox{90}{\textbf{Sce. 1}}
    \end{minipage}
    \begin{minipage}[c]{0.3\textwidth}
        \centering \includegraphics[width=\linewidth]{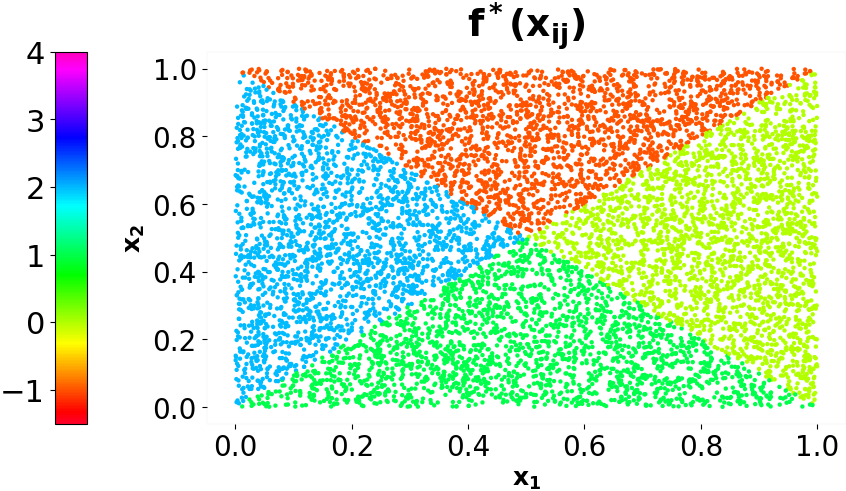}
    \end{minipage}
    \begin{minipage}[c]{0.3\textwidth}
        \centering \includegraphics[width=\linewidth]{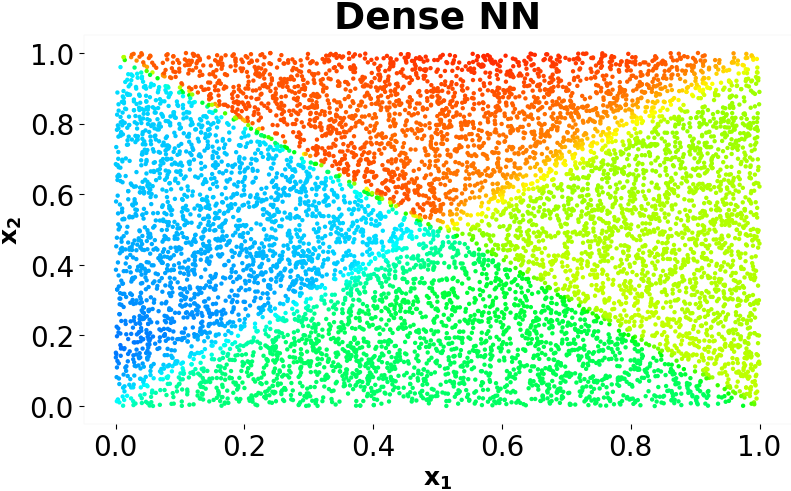}
    \end{minipage}
    \begin{minipage}[c]{0.3\textwidth}
        \centering \includegraphics[width=\linewidth]{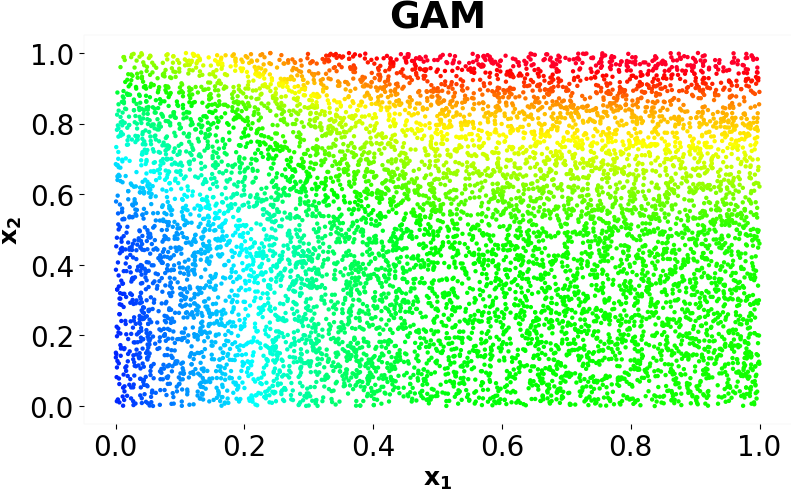}
    \end{minipage}
\end{minipage}\vspace{-1pt}

\begin{minipage}[c]{\textwidth}
    \centering
    \begin{minipage}[c]{0.05\textwidth}
        \centering \rotatebox{90}{\textbf{Sce. 1}}
    \end{minipage}
    \begin{minipage}[c]{0.3\textwidth}
        \centering \includegraphics[width=\linewidth]{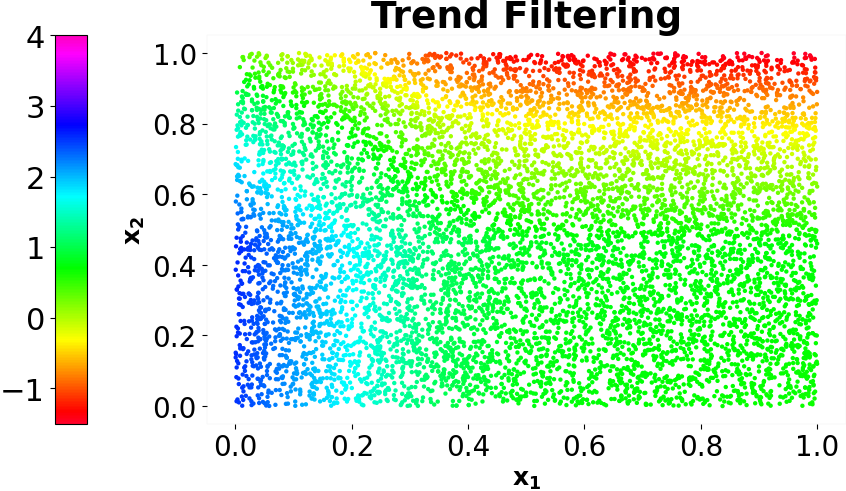}
    \end{minipage}
    \begin{minipage}[c]{0.3\textwidth}
        \centering \includegraphics[width=\linewidth]{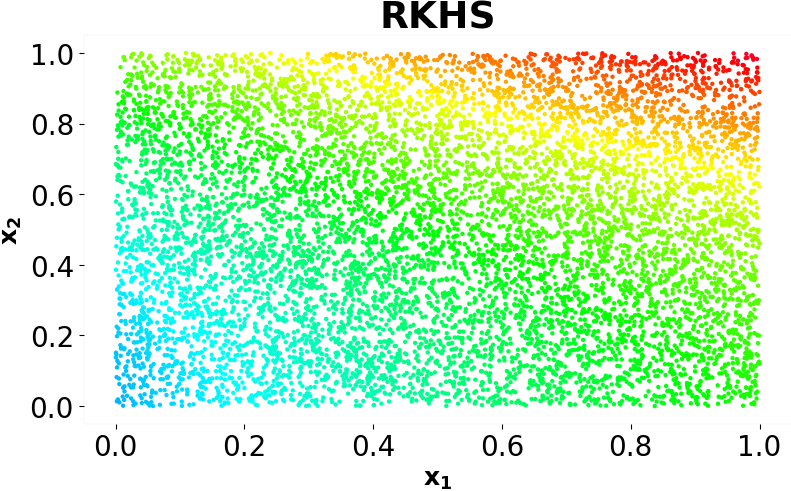}
    \end{minipage}
    \begin{minipage}[c]{0.3\textwidth}
        \centering \includegraphics[width=\linewidth]{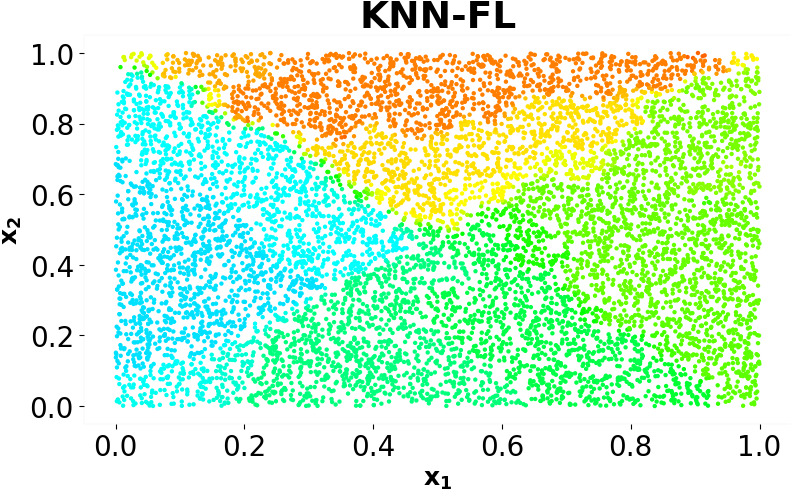}
    \end{minipage}
\end{minipage}\vspace{-1pt}

    \begin{minipage}[c]{\textwidth}
    \centering
    \begin{minipage}[c]{0.05\textwidth}
        \centering \rotatebox{90}{\textbf{Sce. 2}}
    \end{minipage}
    \begin{minipage}[c]{0.3\textwidth}
        \centering \includegraphics[width=\linewidth]{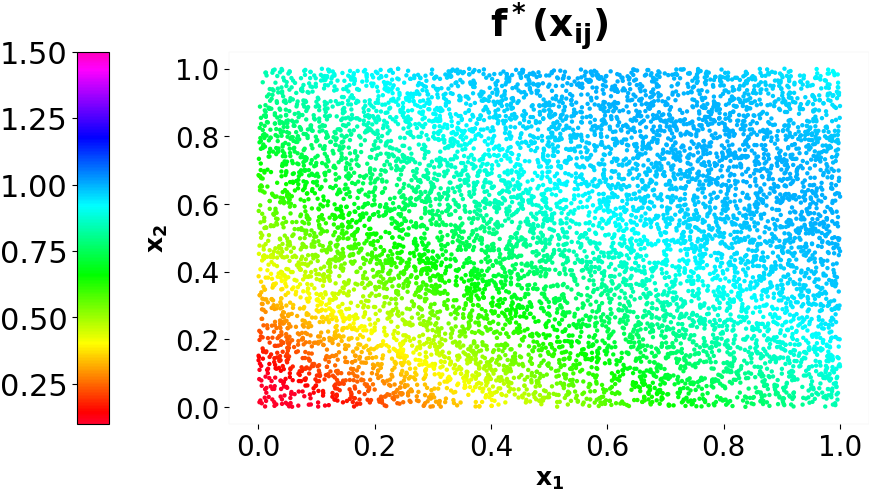}
    \end{minipage}
    \begin{minipage}[c]{0.3\textwidth}
        \centering \includegraphics[width=\linewidth]{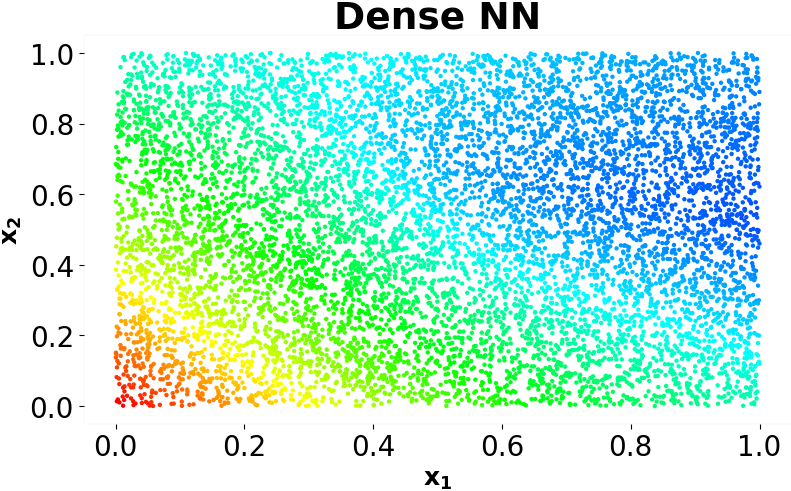}
    \end{minipage}
    \begin{minipage}[c]{0.3\textwidth}
        \centering \includegraphics[width=\linewidth]{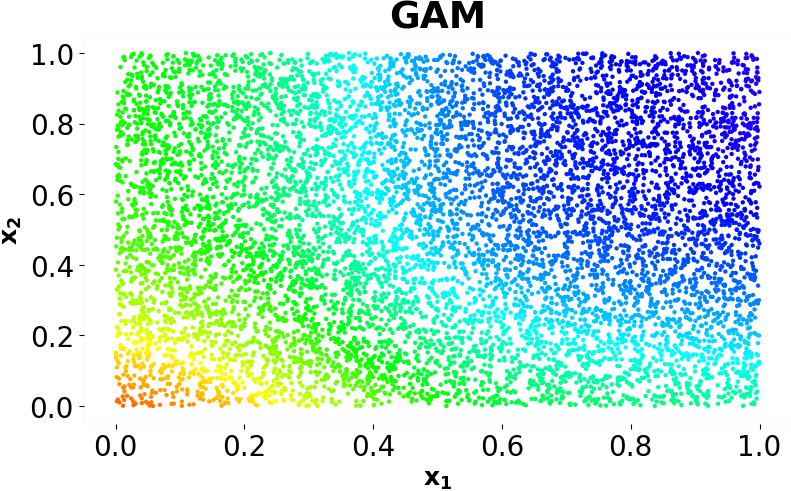}
    \end{minipage}
\end{minipage}\vspace{-1pt}

\begin{minipage}[c]{\textwidth}
    \centering
    \begin{minipage}[c]{0.05\textwidth}
        \centering \rotatebox{90}{\textbf{Sce. 2}}
    \end{minipage}
    \begin{minipage}[c]{0.3\textwidth}
        \centering \includegraphics[width=\linewidth]{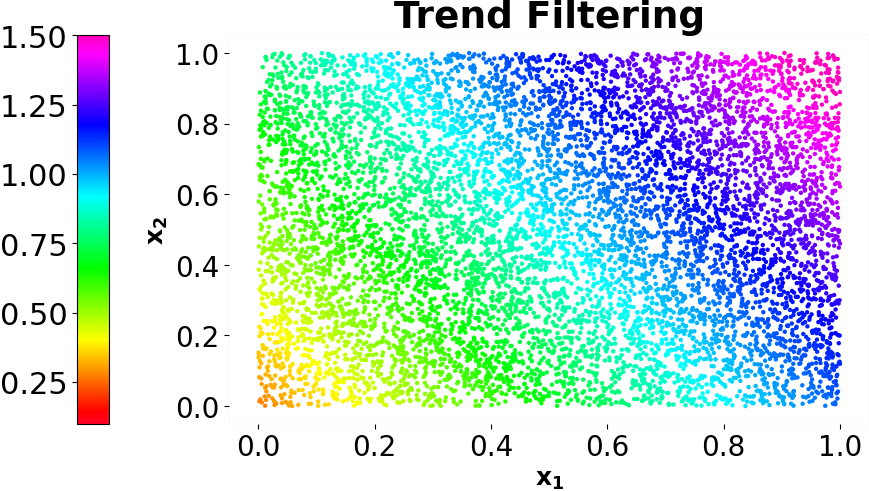}
    \end{minipage}
    \begin{minipage}[c]{0.3\textwidth}
        \centering \includegraphics[width=\linewidth]{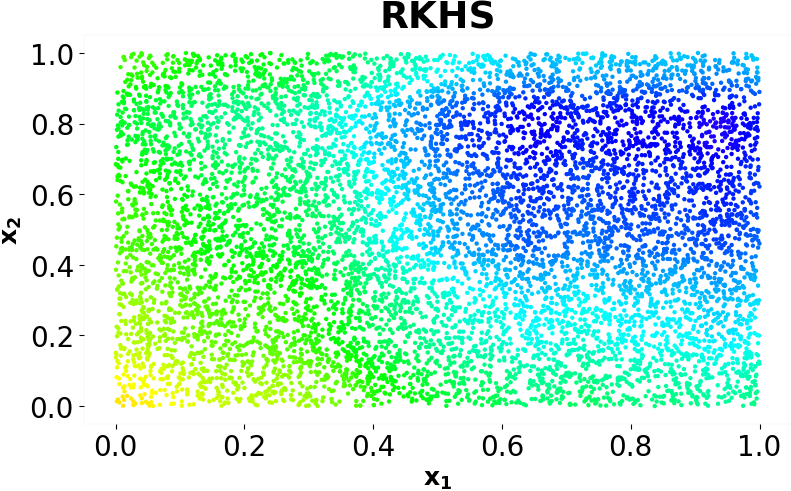}
    \end{minipage}
    \begin{minipage}[c]{0.3\textwidth}
        \centering \includegraphics[width=\linewidth]{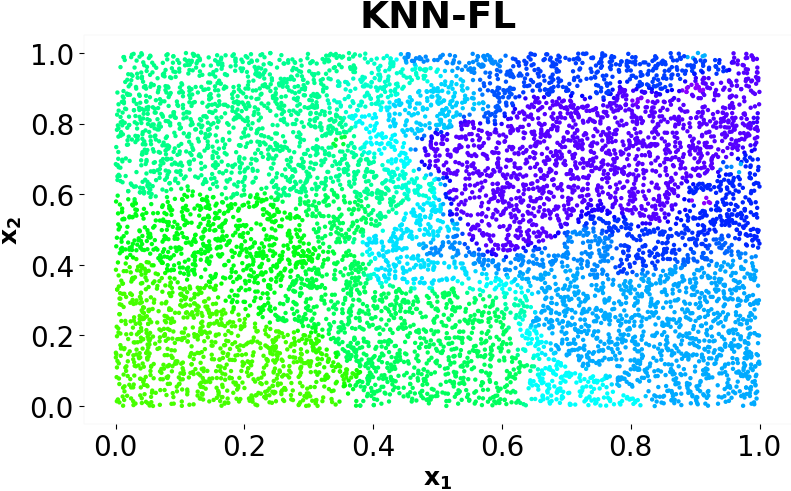}
    \end{minipage}
\end{minipage}\vspace{-1pt}

\begin{minipage}[c]{\textwidth}
    \centering
    \begin{minipage}[c]{0.05\textwidth}
        \centering \rotatebox{90}{\textbf{Sce. 3}}
    \end{minipage}
    \begin{minipage}[c]{0.3\textwidth}
        \centering \includegraphics[width=\linewidth]{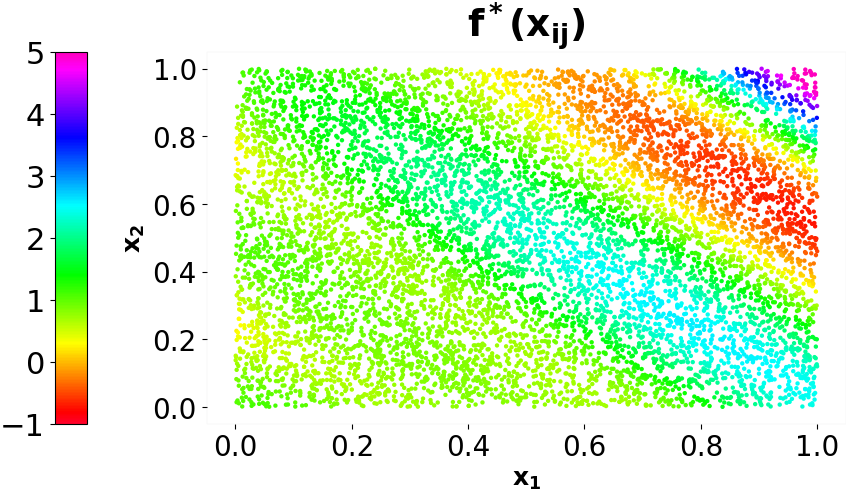}
    \end{minipage}
    \begin{minipage}[c]{0.3\textwidth}
        \centering \includegraphics[width=\linewidth]{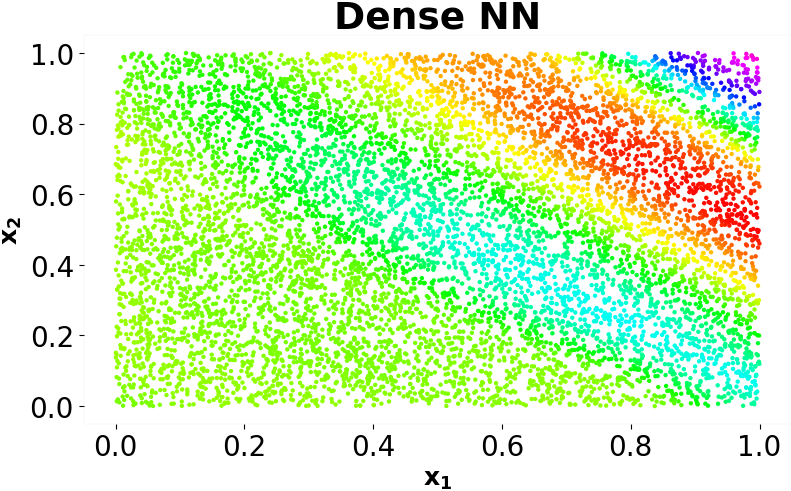}
    \end{minipage}
    \begin{minipage}[c]{0.3\textwidth}
        \centering \includegraphics[width=\linewidth]{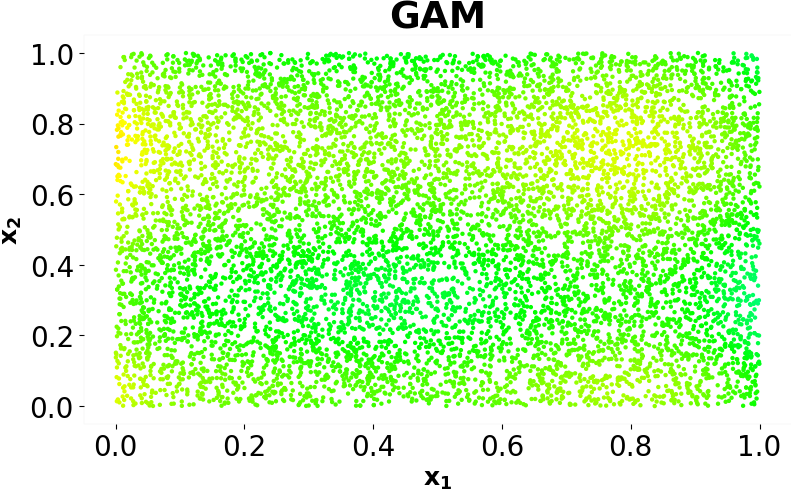}
    \end{minipage}
\end{minipage}\vspace{-1pt}

\begin{minipage}[c]{\textwidth}
    \centering
    \begin{minipage}[c]{0.05\textwidth}
        \centering \rotatebox{90}{\textbf{Sce. 3}}
    \end{minipage}
    \begin{minipage}[c]{0.3\textwidth}
        \centering \includegraphics[width=\linewidth]{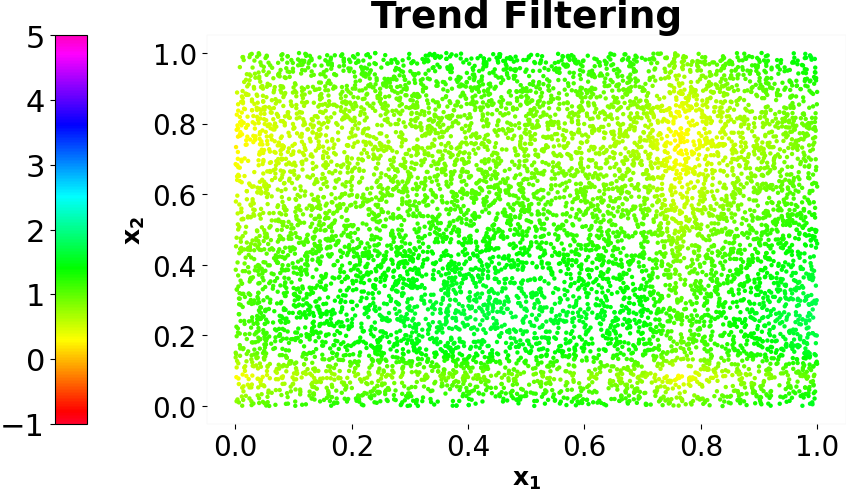}
    \end{minipage}
    \begin{minipage}[c]{0.3\textwidth}
        \centering \includegraphics[width=\linewidth]{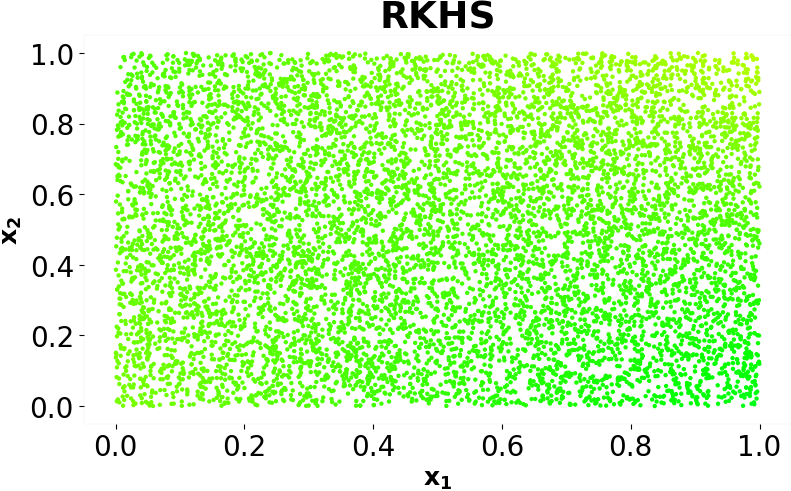}
    \end{minipage}
    \begin{minipage}[c]{0.3\textwidth}
        \centering \includegraphics[width=\linewidth]{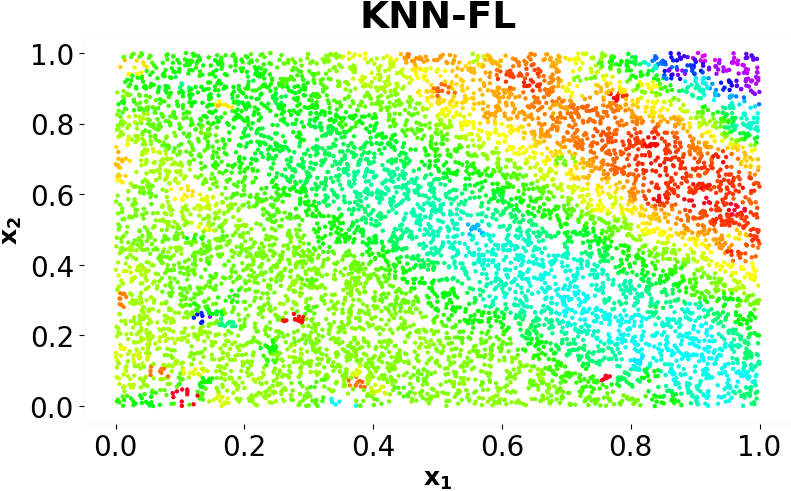}
    \end{minipage}
\end{minipage}\vspace{-1pt}


    \caption{The true mean $f^*(x_{ij})$ and the predicted $\hat{f}(x_{ij})$ for Scenarios 1, 2, 3 with $n=500$, $m_{mult}=1$, $d=2$, as estimated by Dense NN, GAM, Trend Filtering, RKHS, and KNN-FL, respectively.}
    \label{fig:comparison_scenarios}
\end{figure*}

\begin{figure*}[!t]
    
    \centering

\begin{minipage}[c]{\textwidth}
    \centering
    \begin{minipage}[c]{0.05\textwidth}
        \centering \rotatebox{90}{\textbf{Sce. 4}}
    \end{minipage}
    \begin{minipage}[c]{0.3\textwidth}
        \centering \includegraphics[width=\linewidth]{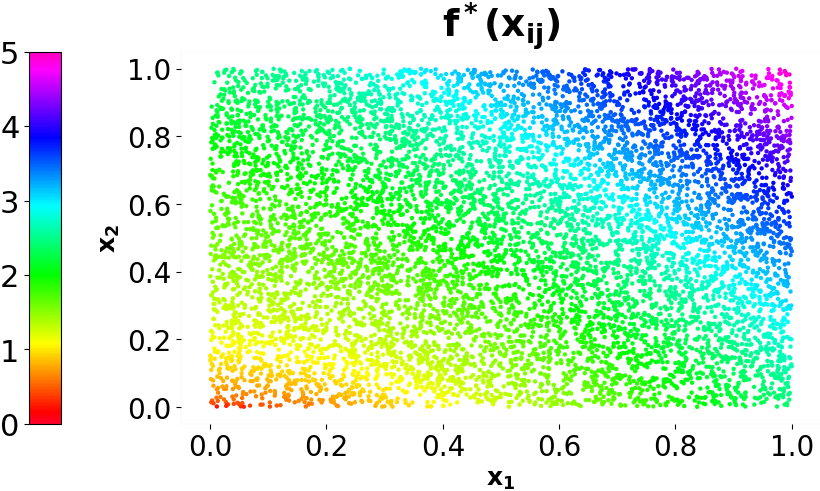}
    \end{minipage}
    \begin{minipage}[c]{0.3\textwidth}
        \centering \includegraphics[width=\linewidth]{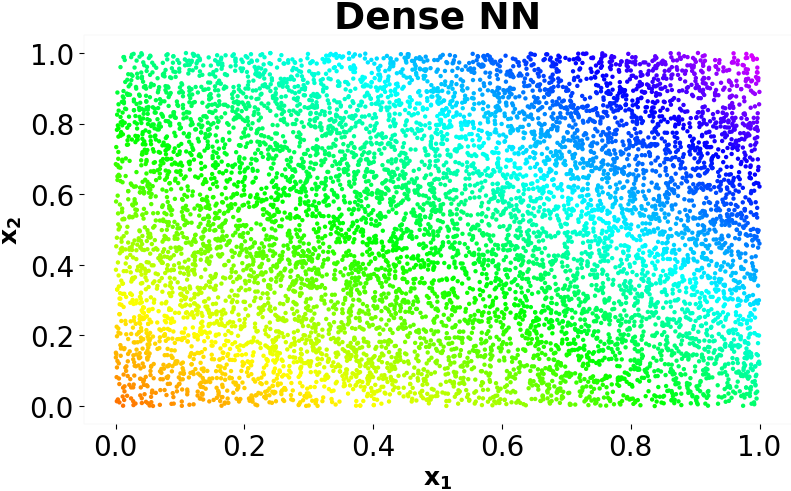}
    \end{minipage}
    \begin{minipage}[c]{0.3\textwidth}
        \centering \includegraphics[width=\linewidth]{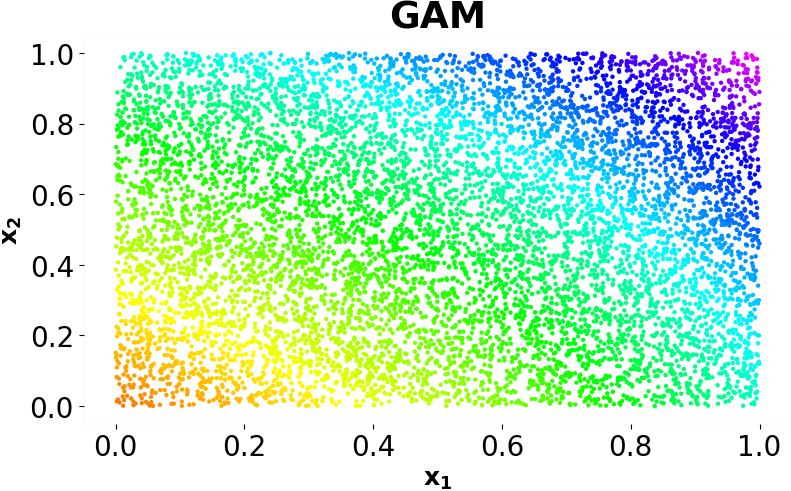}
    \end{minipage}
\end{minipage}\vspace{-1pt}

\begin{minipage}[c]{\textwidth}
    \centering
    \begin{minipage}[c]{0.05\textwidth}
        \centering \rotatebox{90}{\textbf{Sce. 4}}
    \end{minipage}
    \begin{minipage}[c]{0.3\textwidth}
        \centering \includegraphics[width=\linewidth]{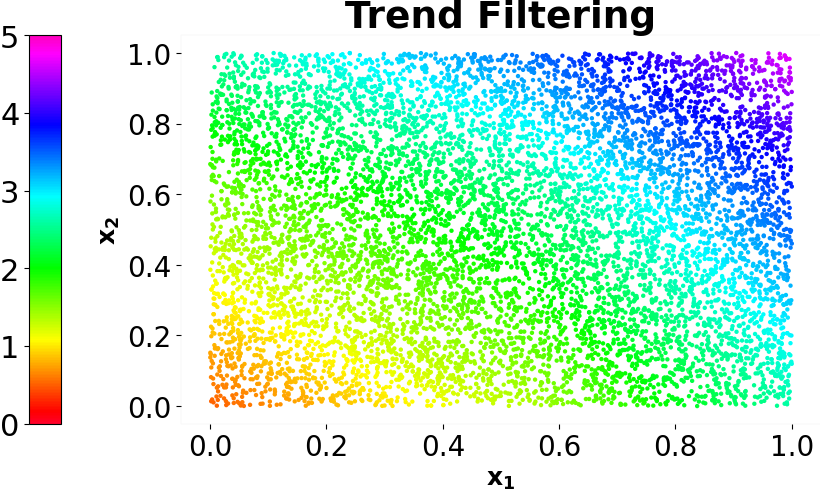}
    \end{minipage}
    \begin{minipage}[c]{0.3\textwidth}
        \centering \includegraphics[width=\linewidth]{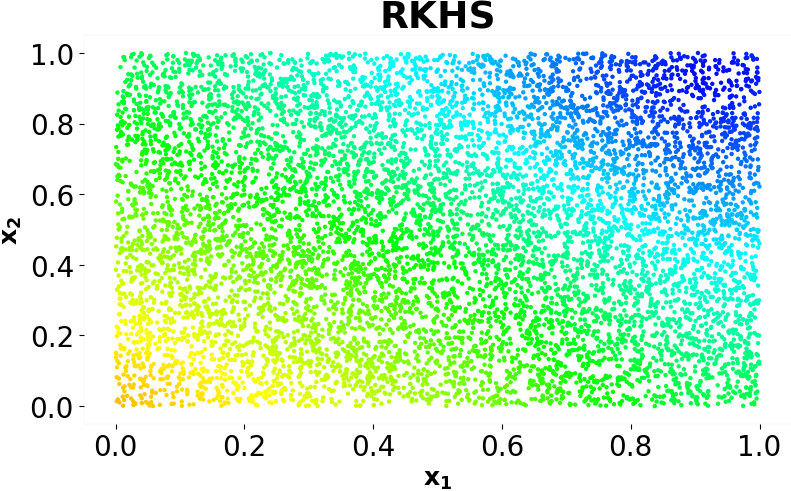}
    \end{minipage}
    \begin{minipage}[c]{0.3\textwidth}
        \centering \includegraphics[width=\linewidth]{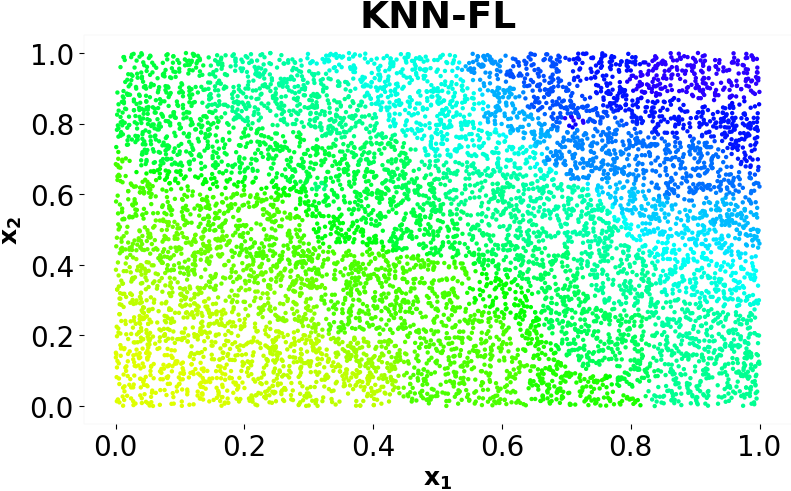}
    \end{minipage}
\end{minipage}\vspace{-1pt}

    \begin{minipage}[c]{\textwidth}
    \centering
    \begin{minipage}[c]{0.05\textwidth}
        \centering \rotatebox{90}{\textbf{Sce. 5}}
    \end{minipage}
    \begin{minipage}[c]{0.3\textwidth}
        \centering \includegraphics[width=\linewidth]{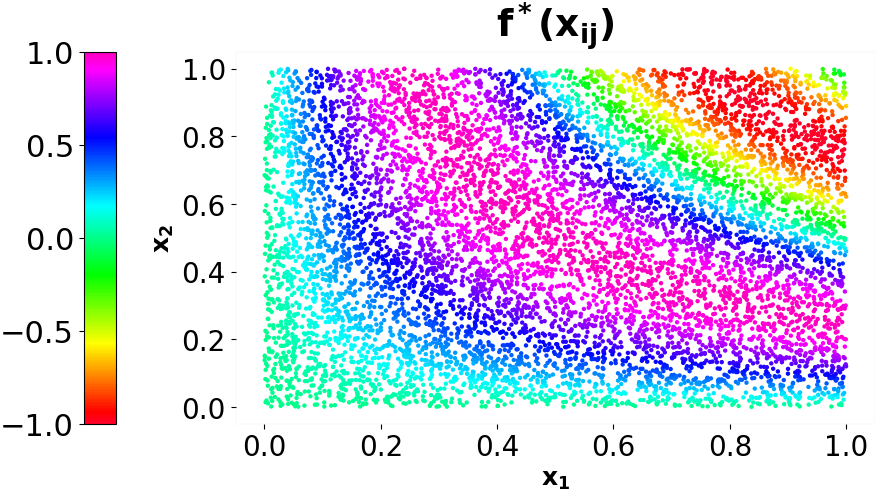}
    \end{minipage}
    \begin{minipage}[c]{0.3\textwidth}
        \centering \includegraphics[width=\linewidth]{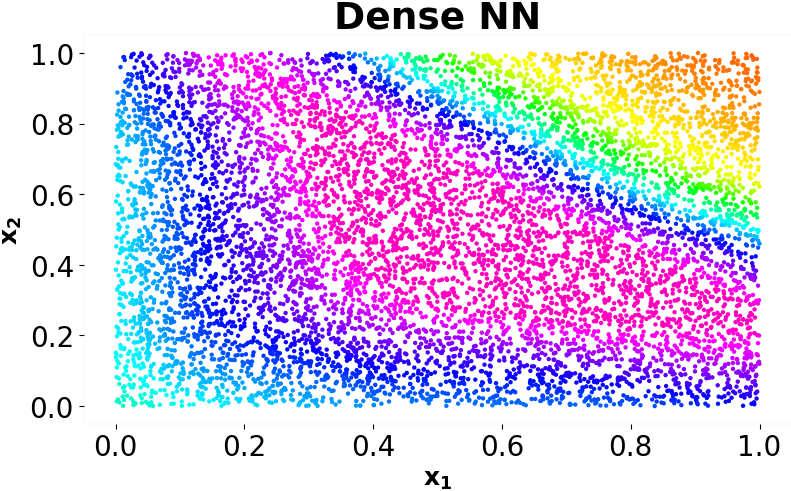}
    \end{minipage}
    \begin{minipage}[c]{0.3\textwidth}
        \centering \includegraphics[width=\linewidth]{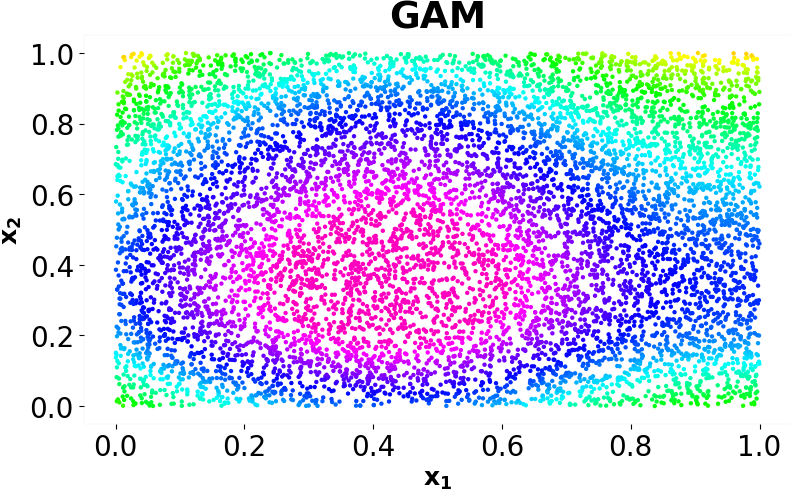}
    \end{minipage}
\end{minipage}\vspace{-1pt}

\begin{minipage}[c]{\textwidth}
    \centering
    \begin{minipage}[c]{0.05\textwidth}
        \centering \rotatebox{90}{\textbf{Sce. 5}}
    \end{minipage}
    \begin{minipage}[c]{0.3\textwidth}
        \centering \includegraphics[width=\linewidth]{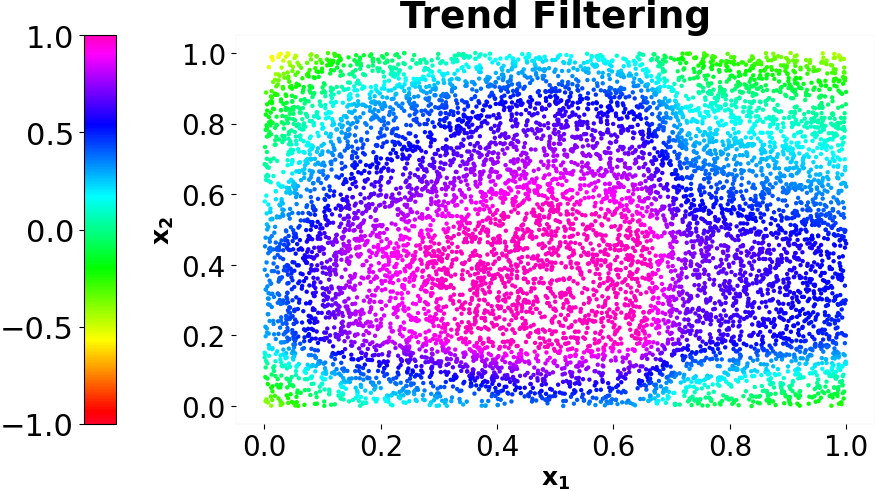}
    \end{minipage}
    \begin{minipage}[c]{0.3\textwidth}
        \centering \includegraphics[width=\linewidth]{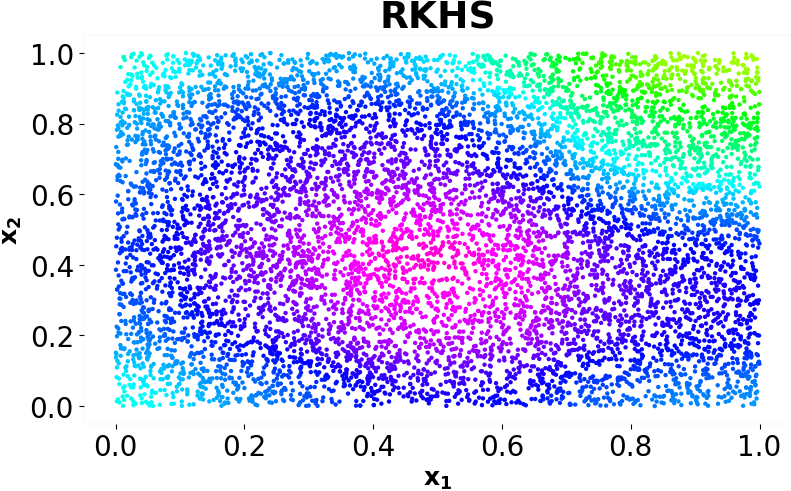}
    \end{minipage}
    \begin{minipage}[c]{0.3\textwidth}
        \centering \includegraphics[width=\linewidth]{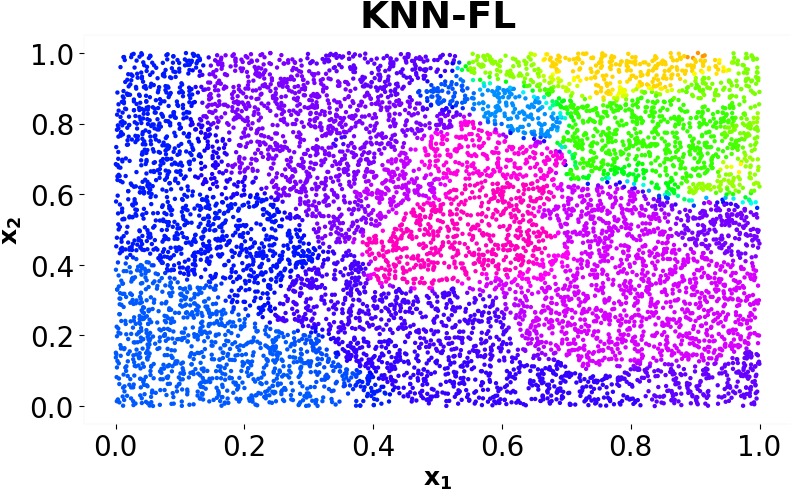}
    \end{minipage}
\end{minipage}\vspace{-1pt}

\begin{minipage}[c]{\textwidth}
    \centering
    \begin{minipage}[c]{0.05\textwidth}
        \centering \rotatebox{90}{\textbf{Sce. 6}}
    \end{minipage}
    \begin{minipage}[c]{0.3\textwidth}
        \centering \includegraphics[width=\linewidth]{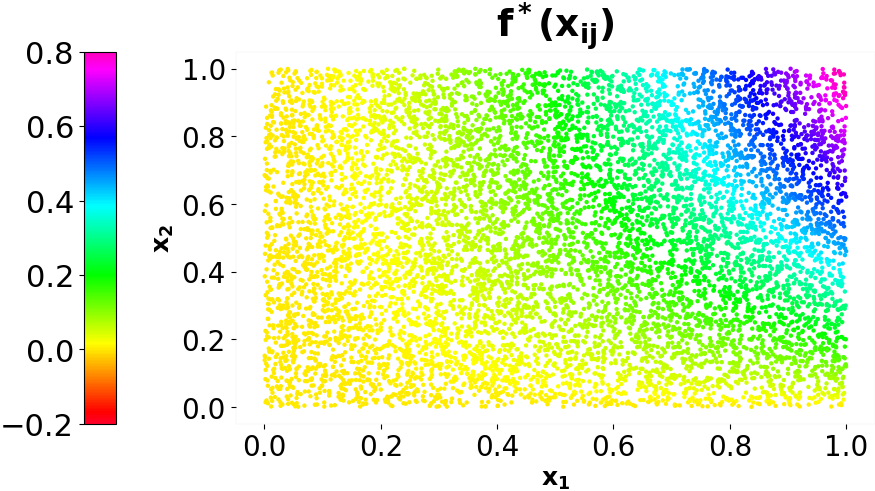}
    \end{minipage}
    \begin{minipage}[c]{0.3\textwidth}
        \centering \includegraphics[width=\linewidth]{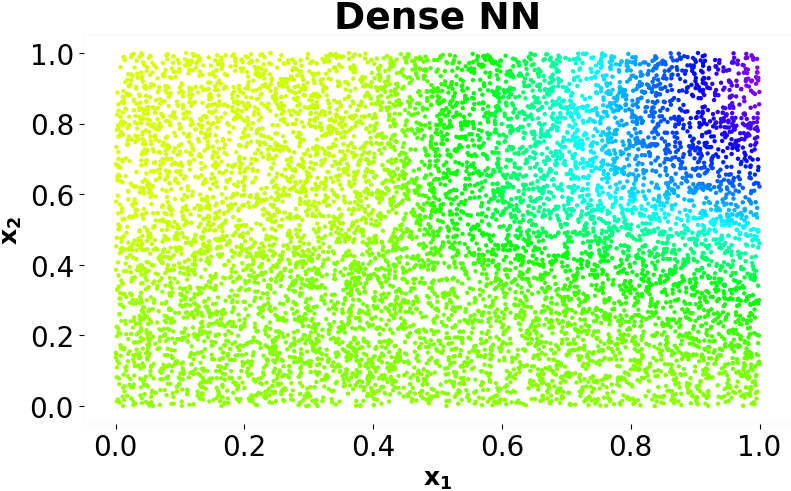}
    \end{minipage}
    \begin{minipage}[c]{0.3\textwidth}
        \centering \includegraphics[width=\linewidth]{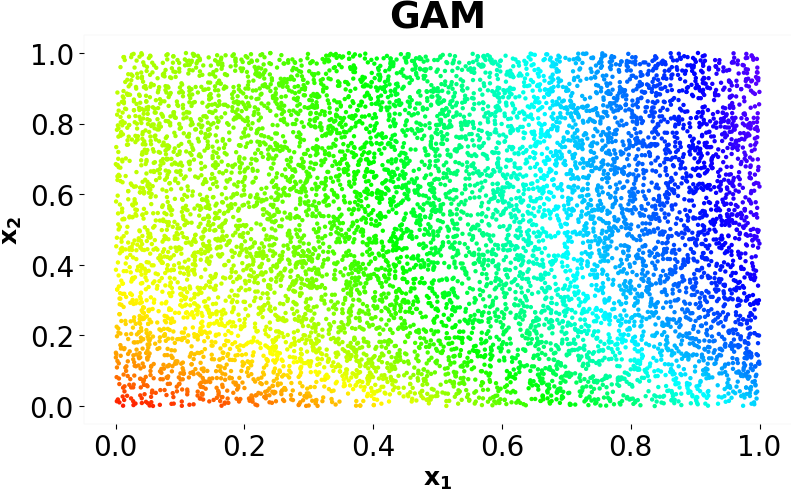}
    \end{minipage}
\end{minipage}\vspace{-1pt}

\begin{minipage}[c]{\textwidth}
    \centering
    \begin{minipage}[c]{0.05\textwidth}
        \centering \rotatebox{90}{\textbf{Sce. 6}}
    \end{minipage}
    \begin{minipage}[c]{0.3\textwidth}
        \centering \includegraphics[width=\linewidth]{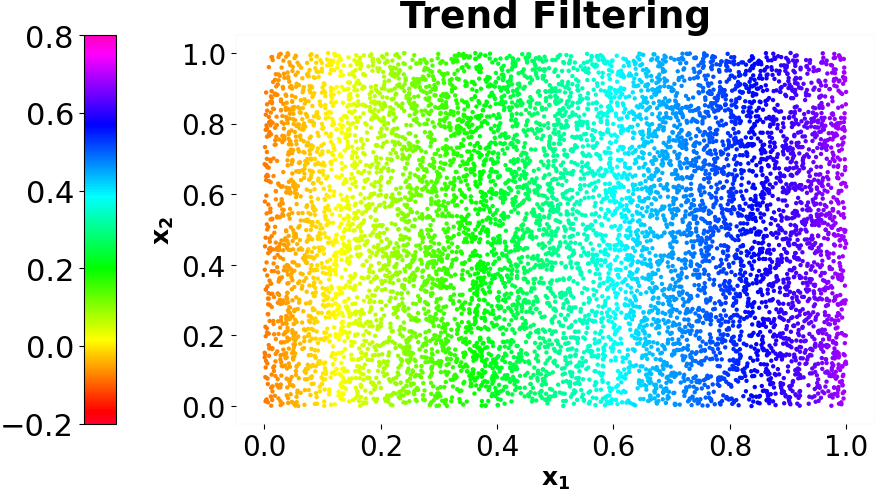}
    \end{minipage}
    \begin{minipage}[c]{0.3\textwidth}
        \centering \includegraphics[width=\linewidth]{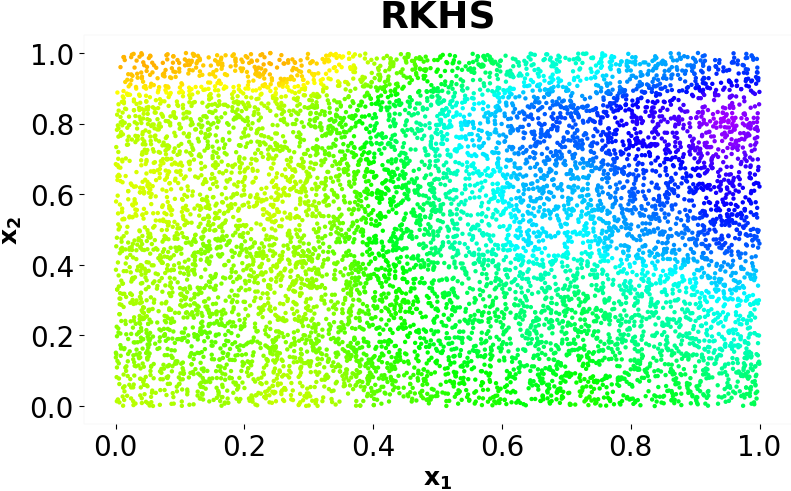}
    \end{minipage}
    \begin{minipage}[c]{0.3\textwidth}
        \centering \includegraphics[width=\linewidth]{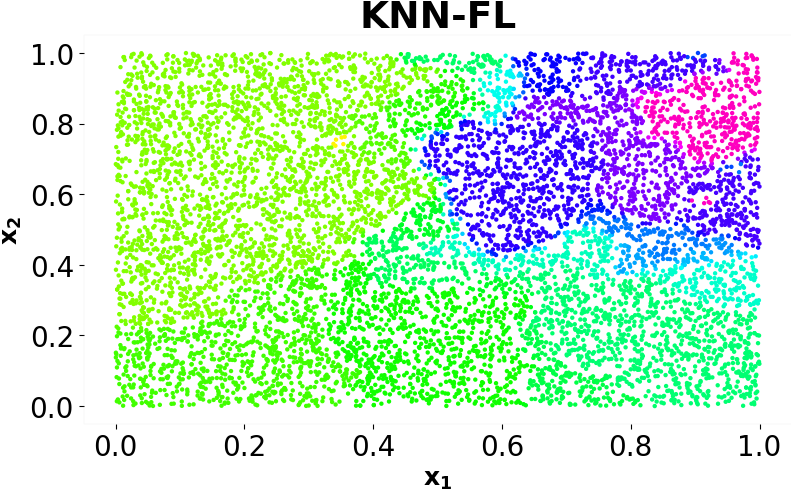}
    \end{minipage}
\end{minipage}\vspace{-1pt}


    \caption{The true mean $f^*(x_{ij})$ and the predicted $\hat{f}(x_{ij})$ for Scenarios 4, 5, 6 with $n=500$, $m_{mult}=1$, $d=2$, as estimated by Dense NN, GAM, Trend Filtering, RKHS, and KNN-FL, respectively.}
    \label{fig:comparison_scenarios_2}
\end{figure*}

\FloatBarrier

\begin{table*}[!t]
\centering
\caption{Results of relative error for different methods, $d$, $n$ and $m_{mult}$: Mean (Std). We highlight the best result in \textbf{bold} and
        the second best result in \textbf{\upshape\textit{bold and italic}}.}
\resizebox{\textwidth}{!}{
\begin{tabular}{@{}llccccc@{}}
    \toprule
    \multirow{2}{*}{d} & \multirow{2}{*}{scenario} & \multicolumn{5}{c}{\textcolor{blue}{$n$: 500, $m_{mult}$: 1}} \\ \cmidrule(lr){3-7}
                        &                          & Dense NN         & GAM             & KNN-FL           & RKHS             & Trend Filtering  \\ \midrule
\multirow{6}{*}{2}  & 1        & \bf 0.0676 (0.012) & \itbf 0.3087 (0.044)  & 0.1249 (0.008)     & 1.2926 (0.135)     & 0.3196 (0.058)  \\
                    & 2        & \bf 0.0133 (0.002) & \itbf 0.0322 (0.002)  & 0.0491 (0.003)     & 0.1678 (0.039)     & 0.0370 (0.002)  \\
                    & 3        & \bf 0.0061 (0.002) & 0.3547 (0.004)  & \itbf 0.0616 (0.004)     & 0.4069 (0.004)     & 0.3547 (0.004)  \\
                    & 4        & \bf 0.0008 (0.000) & \itbf 0.0013 (0.000)  & 0.0068 (0.001)     & 0.2007 (0.005)     & 0.0014 (0.000)  \\
                    & 5        & \bf 0.0640 (0.005) & 0.4844 (0.142)  & \itbf 0.2419 (0.016)     & 0.9635 (0.196)     & 0.6408 (0.042)  \\
                    & 6        & \bf 0.2315 (0.030) & \itbf 0.4059 (0.026)  & 0.4684 (0.031)     & 1.8777 (0.122)     & 0.6153 (0.040)    \\ \midrule
    \multirow{6}{*}{5}  & 1        & \bf 0.3779 (0.006) & 0.7387 (0.048)  & 0.9350 (0.135)     & 1.0916 (0.111)     & \itbf 0.7370 (0.048)  \\
                    & 2        & \bf 0.4711 (0.013) & \itbf 0.4645 (0.030)  & 1.0571 (0.069)     & 1.3723 (0.089)     & 0.4938 (0.032)  \\
                    & 3        & \bf 0.0235 (0.002) & 0.7988 (0.007)  & 0.3254 (0.021)     & 0.8077 (0.006)     & \itbf 0.3640 (0.024)  \\
                    & 4        & \bf 0.0002 (0.000) & \itbf 0.0015 (0.000)  & 0.0048 (0.000)     & 0.1652 (0.004)     & 0.0016 (0.000)  \\
                    & 5        & \bf 0.3719 (0.075) & \itbf 0.6536 (0.043)  & 0.8695 (0.194)     & 1.0597 (0.270)     & 0.6688 (0.044)  \\
                    & 6        & \bf 0.3700 (0.010) & 0.7022 (0.046)  & 0.7266 (0.047)     & 1.1502 (0.075)     & \itbf 0.3778 (0.025)
\\\midrule \multirow{6}{*}{7}  & 1        & \bf 0.7516 (0.022) & \itbf  0.9580 (0.062)  & 1.2894 (0.084)     & 1.9203 (0.125)     & 1.2933 (0.084)  \\
                    & 2        & \bf 0.7684 (0.082) & \itbf 1.3039 (0.085)  & 1.4174 (0.092)     & 5.7272 (0.647)     & 1.4344 (0.094)  \\
                    & 3        & \bf 0.0765 (0.006) & 0.8906 (0.003)  & 0.5716 (0.007)     & 0.8972 (0.002)     & \itbf 0.2159 (0.063)  \\
                    & 4        & \bf 0.0002 (0.000) & \itbf 0.0011 (0.000)  & 0.0053 (0.000)     & 0.1354 (0.001)     &  0.0011 (0.000)  \\
                    & 5        & \bf 0.6064 (0.152) & 0.7518 (0.049)  & 0.8755 (0.057)     & 1.3453 (0.092)     & 0.8789 (0.057)  \\
                    & 6        & \bf 0.5227 (0.011) & 0.7781 (0.051)  & \itbf 0.5223 (0.034)     & 2.6748 (0.516)     & 0.5302 (0.035)  \\
                       \midrule
    \multirow{6}{*}{10} & 1        & \bf 0.7547 (0.036) & 19.7128 (2.571) & \itbf 11.0163 (0.718)    & 12.6360 (1.260)    & 11.2345 (0.733) \\
                    & 2        & \bf 0.4922 (0.038) & 50.0359 (5.814) & \itbf 25.5730 (1.668)    & 49.3684 (3.220)    & 26.1330 (1.704) \\
                    & 3        & \bf 0.0808 (0.007) & 0.9363 (0.003)  & 0.7631 (0.014)     & 0.9487 (0.002)     & \itbf 0.1026 (0.016)  \\
                    & 4        & \bf 0.0001 (0.000) & \itbf 0.0006 (0.000)  & 0.0058 (0.000)     & 0.0998 (0.002)     &  0.0006 (0.000)  \\
                    & 5        & \bf 0.6247 (0.154) & 6.2090 (0.658)  & \itbf 3.5148 (0.229)     & 4.5779 (0.952)     & 3.5807 (0.234)  \\
                    & 6        & \bf 0.6397 (0.082) & 20.4280 (2.210) & \itbf 10.9305 (0.713)    & 19.9058 (1.298)    & 11.1575 (0.728) \\
                       
    \bottomrule
\end{tabular}}
\label{table_1}
\end{table*}

\begin{table*}[!t]
\centering
\caption{Results of relative error for different methods, $d$, $n$ and $m_{mult}$: Mean (Std). 
        }
\resizebox{\textwidth}{!}{
\begin{tabular}{@{}llccccc@{}}
\toprule
    \multirow{2}{*}{d} & \multirow{2}{*}{scenario} & \multicolumn{5}{c}{\textcolor{blue}{$n$: 500, $m_{mult}$: 2}} \\ \cmidrule(lr){3-7}
                        &                          & Dense NN         & GAM             & KNN-FL           & RKHS             & Trend Filtering  \\ \midrule 
    \multirow{6}{*}{2}  & 1        & \textbf{0.0550 (0.016)} & \textit{\textbf{0.3043 (0.032)}} & 0.1291 (0.021)                   & 1.2267 (0.060)  & 0.3075 (0.018)                   \\
                    & 2        & \textbf{0.0066 (0.000)} & \textit{\textbf{0.0178 (0.001)}} & 0.0306 (0.002)                   & 0.2019 (0.013)  & 0.0316 (0.002)                   \\
                    & 3        & \textbf{0.0049 (0.000)} & 0.3598 (0.002)                   & \textit{\textbf{0.0512 (0.003)}} & 0.4146 (0.003)  & 0.3576 (0.003)                   \\
                    & 4        & \textbf{0.0005 (0.000)} & \textit{\textbf{0.0010 (0.000)}} & 0.0045 (0.000)                   & 0.2320 (0.001)  & 0.0012 (0.000)                   \\
                    & 5        & \textbf{0.0817 (0.006)} & 0.4716 (0.145)                   & \textit{\textbf{0.2413 (0.016)}} & 0.9443 (0.027)  & 0.4912 (0.079)                   \\
                    & 6        & \textbf{0.1060 (0.011)} & \textit{\textbf{0.2307 (0.015)}} & 0.4460 (0.029)                   & 1.4758 (0.096)  & 0.4034 (0.026)                   \\
                    \midrule
\multirow{6}{*}{5}  & 1        & \textbf{0.3019 (0.004)} & \textit{\textbf{0.6902 (0.045)}} & 0.9929 (0.127)                   & 1.3277 (0.354)  & 0.7211 (0.047)                   \\
                    & 2        & \textbf{0.2840 (0.029)} & \textit{\textbf{0.3831 (0.025)}} & 0.8573 (0.056)                   & 1.1290 (0.074)  & 0.7937 (0.052)                   \\
                    & 3        & \textbf{0.0130 (0.002)} & 0.7899 (0.003)                   & \textit{\textbf{0.1895 (0.008)}} & 0.8003 (0.004)  & 0.3138 (0.037)                   \\
                    & 4        & \textbf{0.0003 (0.000)} & 0.0016 (0.000)                   & 0.0044 (0.000)                   & 0.2129 (0.001)  & \textit{\textbf{0.0015 (0.000)}} \\
                    & 5        & \textbf{0.2613 (0.005)} & \textit{\textbf{0.6197 (0.040)}} & 0.9543 (0.181)                   & 1.3243 (0.086)  & 0.6741 (0.044)                   \\
                    & 6        & \textit{\textbf{0.3101 (0.027)}} & 0.3281 (0.021)                   & 0.6005 (0.039)                   & 1.3460 (0.088)  & \textbf{0.2935 (0.019)} \\
                    \midrule
\multirow{6}{*}{7}  & 1        & \textbf{0.7837 (0.039)} & \textit{\textbf{2.2619 (0.148)}} & 2.7335 (0.178)                   & 2.9930 (0.248)  & 2.7797 (0.181)                   \\
                    & 2        & \textbf{0.7465 (0.008)} & \textit{\textbf{5.8878 (0.384)}} & 7.4242 (0.484)                   & 9.3845 (0.375)  & 7.6122 (0.496)                   \\
                    & 3        & \textbf{0.0222 (0.001)} & 0.8901 (0.004)                   & 0.4956 (0.010)                   & 0.8981 (0.003)  & \textit{\textbf{0.1964 (0.022)}} \\
                    & 4        & \textbf{0.0001 (0.000)} & 0.0011 (0.000)                   & 0.0051 (0.000)                   & 0.1793 (0.002)  & \textit{\textbf{0.0010 (0.000)}} \\
                    & 5        & \textbf{0.5901 (0.083)} & \textit{\textbf{1.5960 (0.104)}} & 1.9633 (0.128)                   & 2.3866 (0.167)  & 1.9978 (0.130)                   \\
                    & 6        & \textbf{0.4855 (0.114)} & 5.3070 (0.786)                   & \textit{\textbf{3.2943 (0.215)}} & 3.9682 (0.266)  & 3.3838 (0.221)                   \\
                    \midrule
\multirow{6}{*}{10} & 1        & \textbf{0.7751 (0.046)} & 11.7275 (0.700)                  & \textit{\textbf{5.0616 (0.237)}} & 8.7578 (1.560)  & 5.7709 (0.360)                   \\
                    & 2        & \textbf{0.3382 (0.014)} & \textit{\textbf{1.2430 (0.081)}} & 11.6736 (1.027)                  & 30.0460 (1.563) & 13.3159 (0.953)                  \\
                    & 3        & \textbf{0.0469 (0.010)} & 0.9439 (0.001)                   & 0.7622 (0.016)                   & 0.9498 (0.002)  & \textit{\textbf{0.1181 (0.028)}} \\
                    & 4        & \textbf{0.0001 (0.000)} & 0.0007 (0.000)                   & 0.0053 (0.000)                   & 0.1385 (0.002)  & \textit{\textbf{0.0006 (0.000)}} \\
                    & 5        & \textbf{0.6918 (0.053)} & 2.4092 (0.468)                   & \textit{\textbf{1.6978 (0.089)}} & 2.3645 (0.257)  & 2.0060 (0.223)                   \\
                    & 6        & \textbf{0.5265 (0.048)} & 10.8527 (1.239)                  & \textit{\textbf{5.0386 (0.308)}} & 12.4444 (0.350) & 5.7689 (0.422)                  \\
    \bottomrule
\end{tabular}}
\label{table_1_2}
\end{table*}

\begin{table*}[!t]
\centering
\caption{Results of relative error for different methods, $d$, $n$ and $m_{mult}$: Mean (Std)}
\resizebox{\textwidth}{!}{
\begin{tabular}{@{}llccccc@{}}
    \toprule
    \multirow{2}{*}{d} & \multirow{2}{*}{scenario} & \multicolumn{5}{c}{\textcolor{blue}{$n$: 1000, $m_{mult}$: 1}} \\ \cmidrule(lr){3-7}
                        &                          & Dense NN         & GAM             & KNN-FL           & RKHS             & Trend Filtering  \\ \midrule
\multirow{6}{*}{2}  & 1        & \bf 0.0350 (0.008) & \itbf 0.2867 (0.008) & 0.0981 (0.003)     & 1.2588 (0.108)     & 0.2891 (0.004)  \\
                    & 2        & \bf 0.0131 (0.001) & \itbf 0.0260 (0.006) & 0.0595 (0.004)     & 0.1517 (0.039)     & 0.0331 (0.002)  \\
                    & 3        & \bf 0.0069 (0.001) & 0.3588 (0.001) & \itbf 0.0443 (0.003)     & 0.4139 (0.001)     & 0.3599 (0.002)  \\
                    & 4        & \bf 0.0006 (0.000) & \itbf 0.0011 (0.000) & 0.0032 (0.001)     & 0.2202 (0.005)     & 0.0012 (0.000)  \\
                    & 5        & \bf 0.0516 (0.004) & 0.4033 (0.010) & \itbf 0.2538 (0.052)     & 0.7392 (0.039)     & 0.4147 (0.014)  \\
                    & 6        & \bf 0.0656 (0.005) & \itbf 0.3225 (0.066) & 0.3565 (0.023)     & 0.9116 (0.182)     & 0.3865 (0.081)  \\ \midrule
\multirow{6}{*}{5}  & 1        & \bf 0.2631 (0.032) & 0.5340 (0.122) & 0.8096 (0.025)     & 1.4061 (0.301)     & \itbf 0.4722 (0.131)  \\
                    & 2        & \bf 0.1284 (0.005) & \itbf 0.2231 (0.015) & 0.5461 (0.099)     & 0.7956 (0.103)     & 0.3762 (0.038)  \\
                    & 3        & \bf 0.0150 (0.001) & 0.7916 (0.006) & \itbf 0.2102 (0.043)     & 0.8019 (0.004)     & 0.4813 (0.031)  \\
                    & 4        & \bf 0.0003 (0.000) & \itbf 0.0015 (0.000) & 0.0047 (0.000)     & 0.1856 (0.001)     & 0.0015 (0.000)  \\
                    & 5        & \bf 0.2067 (0.003) & \itbf 0.5033 (0.033) & 0.7078 (0.038)     & 0.9709 (0.134)     & 0.5190 (0.034)  \\
                    & 6        & \bf 0.0709 (0.008) & 0.2036 (0.013) & 0.3800 (0.088)     & 0.6871 (0.045)     & \itbf 0.1541 (0.010)  \\ \midrule
\multirow{6}{*}{7}  & 1        & \bf 0.7408 (0.029) & 1.9282 (0.126) & 1.6383 (0.334)     & 2.6112 (0.170)     & \itbf 1.5761 (0.353)  \\
                    & 2        & \bf 0.7625 (0.015) & 4.2984 (0.280) & 4.3439 (0.283)     & 5.5496 (0.362)     & \itbf 4.1335 (0.270)  \\
                    & 3        & \bf 0.0211 (0.002) & 0.8849 (0.001) & 0.4988 (0.009)     & 0.8911 (0.001)     & \itbf 0.2511 (0.043)  \\
                    & 4        & \bf 0.0002 (0.000) &  0.0011 (0.000) & 0.0052 (0.000)     & 0.1532 (0.001)     & \itbf 0.0010 (0.000)  \\
                    & 5        & \bf 0.6646 (0.134) & 1.3807 (0.090) & 1.1592 (0.251)     & 1.7023 (0.111)     & \itbf 1.1182 (0.237)  \\
                    & 6        & \bf 0.3698 (0.028) & 2.1499 (0.140) & 1.8774 (0.122)     & 3.8011 (0.248)     & \itbf 1.7805 (0.116)  \\ \midrule 
\multirow{6}{*}{10} & 1        & \bf 0.8129 (0.053) & 3.6842 (0.815) & 3.4751 (0.150)     & 4.3302 (0.282)     & \itbf 1.4579 (0.095)  \\
                    & 2        & \bf 0.3389 (0.022) & 9.0955 (2.813) & 7.2269 (0.412)     & 6.3497 (0.414)     & \itbf 1.5146 (0.031)  \\
                    & 3        & \bf 0.0647 (0.005) & 0.9443 (0.003) & 0.7506 (0.005)     & 0.9485 (0.002)     & \itbf 0.1001 (0.011)  \\
                    & 4        & \bf 0.0001 (0.000) &  0.0007 (0.000) & 0.0054 (0.000)     & 0.1144 (0.001)     & \itbf 0.0006 (0.000)  \\
                    & 5        & \bf 0.3867 (0.032) & 1.3115 (0.335) & 1.1788 (0.055)     & 1.4374 (0.094)     & \itbf 0.5920 (0.039)  \\
                    & 6        & \bf 0.5579 (0.036) & 3.8700 (0.961) & 3.3235 (0.169)     & 3.5196 (0.230)     & \itbf 0.9957 (0.065)  \\        
    \bottomrule
\end{tabular}}
\label{table_2}
\end{table*}

\begin{table*}[!t]
\centering
\caption{Results of relative error for different methods, $d$, $n$ and $m_{mult}$: Mean (Std)}
\resizebox{\textwidth}{!}{
\begin{tabular}{@{}llccccc@{}}
\toprule
    \multirow{2}{*}{d} & \multirow{2}{*}{scenario} & \multicolumn{5}{c}{\textcolor{blue}{$n$: 1000, $m_{mult}$: 2}} \\ \cmidrule(lr){3-7}
                        &                          & Dense NN         & GAM             & KNN-FL           & RKHS             & Trend Filtering  \\ \midrule 
    \multirow{6}{*}{2}  & 1        & \bf 0.0317 (0.002) & 0.2918 (0.009) & \itbf 0.0591 (0.004)     & 1.3259 (0.033)     & 0.2969 (0.016)  \\
                    & 2        & \bf 0.0120 (0.001) & \itbf 0.0181 (0.001) & 0.0327 (0.002)     & 0.1602 (0.026)     & 0.0181 (0.001)  \\
                    & 3        & \bf 0.0058 (0.000) & 0.3578 (0.002) & \itbf 0.0366 (0.002)     & 0.4165 (0.003)     & 0.3584 (0.002)  \\
                    & 4        & \bf 0.0004 (0.000) & \itbf 0.0010 (0.000) & 0.0027 (0.000)     & 0.2328 (0.002)     & 0.0011 (0.000)  \\
                    & 5        & \bf 0.0403 (0.003) & 0.4226 (0.035) & \itbf 0.1005 (0.007)     & 1.0360 (0.044)     & 0.4420 (0.069)  \\
                    & 6        & \bf 0.1280 (0.010) & \itbf 0.2158 (0.014) & 0.3741 (0.024)     & 1.2701 (0.304)     & 0.2278 (0.015)  \\ \midrule
\multirow{6}{*}{5}  & 1        & \bf 0.2213 (0.018) & \itbf 0.3927 (0.014) & 0.8036 (0.023)     & 1.2202 (0.086)     & 0.4043 (0.031)  \\
                    & 2        & \bf 0.0663 (0.007) & \itbf 0.1049 (0.007) & 0.5283 (0.084)     & 1.0124 (0.178)     & 0.1239 (0.008)  \\
                    & 3        & \bf 0.0116 (0.001) & 0.7938 (0.004) & \itbf 0.1976 (0.011)     & 0.8045 (0.002)     & 0.4388 (0.085)  \\
                    & 4        & \bf 0.0002 (0.000) & \itbf 0.0015 (0.000) & 0.0085 (0.001)     & 0.2166 (0.002)     & 0.0015 (0.000)  \\
                    & 5        & \bf 0.1619 (0.036) & \itbf 0.3117 (0.017) & 0.7034 (0.032)     & 1.0090 (0.098)     & 0.3400 (0.055)  \\
                    & 6        & \bf 0.0548 (0.004) & \itbf 0.0833 (0.005) & 0.3636 (0.077)     & 0.5899 (0.065)     & 0.0971 (0.006)  \\ \midrule
\multirow{6}{*}{7}  & 1        & \bf 0.7386 (0.041) & \itbf 0.9086 (0.059) & 1.2393 (0.234)     & 1.6926 (0.034)     & 1.0924 (0.071)  \\
                    & 2        & \bf 0.7909 (0.038) & \itbf 1.2737 (0.083) & 1.6765 (0.109)     & 5.3951 (0.300)     & 1.4584 (0.095)  \\
                    & 3        & \bf 0.0256 (0.002) & 0.8901 (0.004) & 0.4828 (0.013)     & 0.8960 (0.004)     & \itbf 0.4100 (0.027)  \\
                    & 4        & \bf 0.0001 (0.000) & \itbf 0.0010 (0.000) & 0.0050 (0.000)     & 0.1840 (0.003)     & 0.0010 (0.000)  \\
                    & 5        & \bf 0.3978 (0.025) & \itbf 0.6856 (0.045) & 0.9178 (0.060)     & 1.5189 (0.084)     & 0.6985 (0.046)  \\
                    & 6        & \bf 0.2362 (0.016) & 0.7384 (0.048) & \itbf 0.6435 (0.042)     & 1.0598 (0.069)     & 0.6570 (0.043)  \\ \midrule 
\multirow{6}{*}{10} & 1        & \bf 0.7185 (0.008) & \itbf 1.4206 (0.093) & 2.2545 (0.147)     & 12.3521 (0.395)    & 1.8505 (0.121)  \\
                    & 2        & \bf 0.2877 (0.019) & 1.7947 (0.480) & \itbf 1.6678 (0.109)     & 32.1587 (2.097)    & 1.7717 (0.116)  \\
                    & 3        & \bf  0.0223 (0.001) & 0.9450 (0.006) & 
 0.8380 (0.047)     & 0.9494 (0.005)     & \itbf 0.1611 (0.009)  \\
                    & 4        & \bf 0.0001 (0.000) & \itbf 0.0006 (0.000) & 0.0056 (0.000)     & 0.1450 (0.001)     & 0.0007 (0.000)  \\
                    & 5        & \bf 0.2658 (0.043) & \itbf 0.5817 (0.038) & 0.8374 (0.055)     & 3.8153 (0.140)     & 0.7095 (0.046)  \\
                    & 6        & \bf 0.1417 (0.016) & \itbf 0.9445 (0.062) & 1.0597 (0.187)     & 16.9320 (1.104)    & 0.9663 (0.054)  \\
    \bottomrule
\end{tabular}}
\label{table_2_2}
\end{table*}

\begin{table*}[!t]
\centering
\caption{Results of relative error for different methods, $d$, $n$ and $m_{mult}$: Mean (Std)}
\resizebox{\textwidth}{!}{
\begin{tabular}{@{}llccccc@{}}
    \toprule
   \multirow{2}{*}{d} & \multirow{2}{*}{scenario} & \multicolumn{5}{c}{\textcolor{blue}{$n$: 2000, $m_{mult}$: 1}} \\ \cmidrule(lr){3-7}
                        &                          & Dense NN         & GAM             & KNN-FL           & RKHS             & Trend Filtering  \\ \midrule 
    \multirow{6}{*}{2}  & 1        & \textbf{0.0294 (0.005)}          & \textit{\textbf{0.2883 (0.004)}} & \textit{\textbf{0.0521 (0.004)}} & 1.2793 (0.047)  & 0.2888 (0.006)                   \\
                    & 2        & \textbf{0.0091 (0.001)}          & \textit{\textbf{0.0323 (0.002)}} & 0.0396 (0.003)                   & 0.1610 (0.020)  & \textit{\textbf{0.0231 (0.002)}} \\
                    & 3        & \textbf{0.0052 (0.000)}          & 0.3578 (0.002)                   & \textit{\textbf{0.0270 (0.003)}} & 0.4162 (0.003)  & 0.3571 (0.002)                   \\
                    & 4        & \textbf{0.0004 (0.000)}          & \textit{\textbf{0.0010 (0.000)}} & 0.0045 (0.000)                   & 0.2255 (0.001)  & 0.0011 (0.000)                   \\
                    & 5        & \textbf{0.0338 (0.008)}          & 0.4051 (0.013)                   & \textit{\textbf{0.0957 (0.013)}} & 0.9558 (0.043)  & 0.4108 (0.018)                   \\
                    & 6        & \textbf{0.1626 (0.012)}          & \textit{\textbf{0.3921 (0.026)}} & 0.4097 (0.027)                   & 1.2836 (0.170)  & \textit{\textbf{0.2851 (0.019)}} \\
                    \midrule
\multirow{6}{*}{5}  & 1        & \textbf{0.2704 (0.011)}          & \textit{\textbf{0.6361 (0.041)}} & 0.8273 (0.028)                   & 1.3507 (0.157)  & \textit{\textbf{0.5161 (0.084)}} \\
                    & 2        & \textbf{0.1286 (0.010)}          & \textit{\textbf{0.1746 (0.011)}} & 0.6055 (0.097)                   & 1.3393 (0.363)  & 0.5241 (0.034)                   \\
                    & 3        & \textbf{0.0129 (0.001)}          & 0.7955 (0.007)                   & \textit{\textbf{0.1993 (0.018)}} & 0.8052 (0.006)  & 0.4704 (0.031)                   \\
                    & 4        & \textbf{0.0002 (0.000)}          & 0.0016 (0.000)                   & 0.0051 (0.000)                   & 0.2000 (0.001)  & \textit{\textbf{0.0015 (0.000)}} \\
                    & 5        & \textbf{0.1967 (0.015)}          & \textit{\textbf{0.4542 (0.030)}} & 0.7335 (0.040)                   & 1.1670 (0.150)  & 0.4642 (0.079)                   \\
                    & 6        & \textit{\textbf{0.0939 (0.007)}} & \textbf{0.0805 (0.005)}          & 0.4274 (0.082)                   & 0.5648 (0.027)  & 0.2699 (0.018) \\
                    \midrule
\multirow{6}{*}{7}  & 1        & \textbf{0.7807 (0.051)}          & \textit{\textbf{0.8679 (0.057)}} & 1.3575 (0.089)                   & 1.9966 (0.130)  & 1.2397 (0.081)                   \\
                    & 2        & \textbf{0.4197 (0.027)}          & \textit{\textbf{1.0704 (0.070)}} & 1.6646 (0.109)                   & 3.8451 (0.251)  & 1.4231 (0.093)                   \\
                    & 3        & \textbf{0.0146 (0.002)}          & 0.8871 (0.002)                   & 0.4714 (0.002)                   & 0.8939 (0.002)  & \textit{\textbf{0.2543 (0.038)}} \\
                    & 4        & \textbf{0.0001 (0.000)}          & 0.0011 (0.000)                   & 0.0048 (0.000)                   & 0.1678 (0.002)  & \textit{\textbf{0.0010 (0.000)}} \\
                    & 5        & \textbf{0.1771 (0.015)}          & \textit{\textbf{0.6277 (0.041)}} & 0.9096 (0.059)                   & 1.5109 (0.099)  & 0.7869 (0.051)                   \\
                    & 6        & \textbf{0.3764 (0.025)}          & \textit{\textbf{0.5455 (0.036)}} & \textit{\textbf{0.6446 (0.042)}} & 1.8978 (0.124)  & 0.5641 (0.037)                   \\
                    \midrule
\multirow{6}{*}{10} & 1        & \textbf{0.7360 (0.048)}          & 1.3649 (0.089)                   & \textit{\textbf{1.1902 (0.025)}} & 8.0562 (1.239)  & 1.5662 (0.102)                   \\
                    & 2        & \textbf{0.2949 (0.019)}          & \textit{\textbf{0.7665 (0.050)}} & 0.8280 (0.054)                   & 21.0288 (1.928) & 1.0935 (0.071)                   \\
                    & 3        & \textbf{0.0330 (0.002)}          & 0.9460 (0.001)                   & 0.7523 (0.005)                   & 0.9497 (0.001)  & \textit{\textbf{0.1360 (0.023)}} \\
                    & 4        & \textbf{0.0001 (0.000)}          & 0.0007 (0.000)                   & 0.0054 (0.000)                   & 0.1290 (0.001)  & \textit{\textbf{0.0006 (0.000)}} \\
                    & 5        & \textbf{0.1383 (0.017)}          & \textit{\textbf{0.5327 (0.035)}} & \textit{\textbf{1.3057 (0.088)}} & 2.6521 (0.500)  & 0.6275 (0.041)                   \\
                    & 6        & \textbf{0.1737 (0.011)}          & \textit{\textbf{0.9733 (0.063)}} & \textit{\textbf{1.0557 (0.191)}} & 7.0305 (0.769)  & 1.1130 (0.073)                    \\ 
                       
    \bottomrule
\end{tabular}}
\label{table_3}
\end{table*}

\begin{table*}[!t]
\centering
\caption{Results of relative error for different methods, $d$, $n$ and $m_{mult}$: Mean (Std)}
\resizebox{\textwidth}{!}{
\begin{tabular}{@{}llccccc@{}}
\toprule
\multirow{2}{*}{d} & \multirow{2}{*}{scenario} & \multicolumn{5}{c}{\textcolor{blue}{$n$: 2000, $m_{mult}$: 2}} \\ \cmidrule(lr){3-7}
                        &                          & Dense NN         & GAM             & KNN-FL           & RKHS             & Trend Filtering  \\ \midrule
\multirow{6}{*}{2}  & 1        & \bf 0.0153 (0.003) & 0.2794 (0.002) & 0.0328 (0.002)     & 1.3257 (0.007)     & \itbf 0.2792 (0.001)  \\
                    & 2        & \bf 0.0031 (0.000) & \itbf 0.0095 (0.002) & 0.0144 (0.001)     & 0.1489 (0.001)     & 0.0103 (0.002)  \\
                    & 3        & \bf 0.0031 (0.000) & 0.3587 (0.002) & \itbf 0.0404 (0.011)     & 0.4207 (0.002)     & 0.3588 (0.002)  \\
                    & 4        & \bf 0.0003 (0.000) & \itbf 0.0010 (0.000) & 0.0039 (0.000)     & 0.2338 (0.001)     & 0.0010 (0.000)  \\
                    & 5        & \bf 0.0163 (0.002) & 0.3773 (0.004) & \itbf 0.0606 (0.002)     & 0.9943 (0.017)     & 0.3782 (0.004)  \\
                    & 6        & \bf 0.0234 (0.006) & 0.1252 (0.024) & \itbf 0.1098 (0.028)     & 1.4021 (0.078)     & 0.1363 (0.007)  \\ \midrule
\multirow{6}{*}{5}  & 1        & \bf 0.1300 (0.008) & \itbf 0.4416 (0.025) & 0.9427 (0.224)     & 1.6799 (0.110)     & 0.4437 (0.033)  \\
                    & 2        & \bf 0.1117 (0.003) & 0.7506 (0.049) & \itbf 0.7418 (0.048)     & 1.0230 (0.067)     & 0.7491 (0.049)  \\
                    & 3        & \bf 0.0066 (0.002) & 0.7931 (0.002) & \itbf 0.2419 (0.032)     & 0.8022 (0.002)     & 0.3959 (0.057)  \\
                    & 4        & \bf 0.0001 (0.000) & \itbf 0.0015 (0.000) & 0.0063 (0.000)     & 0.2195 (0.000)     & 0.0015 (0.000)  \\
                    & 5        & \bf 0.1594 (0.010) & \itbf 0.3620 (0.020) & 0.8878 (0.295)     & 1.4242 (0.093)     & 0.5221 (0.034)  \\
                    & 6        & \bf 0.0513 (0.004) & 0.2081 (0.014) & 0.4958 (0.032)     & 1.3786 (0.090)     & \itbf 0.1587 (0.010)  \\ \midrule
\multirow{6}{*}{7}  & 1        & \bf 0.1990 (0.007) & \itbf 0.6109 (0.034) & 1.3097 (0.085)     & 1.8931 (0.238)     & 0.7672 (0.050)  \\
                    & 2        & \bf 0.2983 (0.012) & \itbf 1.2123 (0.079) & 1.4433 (0.094)     & 5.0919 (0.332)     & 1.4532 (0.095)  \\
                    & 3        & \bf 0.0106 (0.002) & 0.8871 (0.001) & 0.4996 (0.009)     & 1.0042 (0.056)     & \itbf 0.2755 (0.064)  \\
                    & 4        & \bf 0.0001 (0.000) & \itbf 0.0010 (0.000) & 0.0077 (0.000)     & 0.1886 (0.001)     & 0.0010 (0.000)  \\
                    & 5        & \bf 0.1379 (0.007) & \itbf 0.4113 (0.023) & 0.8685 (0.057)     & 1.2727 (0.379)     & 0.5612 (0.037)  \\
                    & 6        & \bf 0.0970 (0.014) & \itbf 0.2485 (0.014) & 0.5549 (0.036)     & 1.2550 (0.082)     & 0.5591 (0.036)  \\
                    \midrule 
\multirow{6}{*}{10} & 1        & \bf 0.7102 (0.004) & 1.4528 (0.095) & \itbf 0.8447 (0.001)     & 6.4141 (0.958)     & 0.9805 (0.054)  \\
                    & 2        & \bf 0.2128 (0.008) & 3.0706 (0.200) & \itbf 0.3626 (0.012)     & 7.8720 (0.513)     & 0.4042 (0.019)  \\
                    & 3        & \bf 0.0105 (0.001) & 0.9462 (0.000) & 0.8499 (0.047)     & 0.9498 (0.001)     &\itbf  0.1627 (0.028)  \\
                    & 4        & \bf 0.0001 (0.000) & \itbf 0.0006 (0.000) & 0.0089 (0.000)     & 0.1477 (0.000)     & 0.0006 (0.000)  \\
                    & 5        & \bf 0.1217 (0.010) & 0.5133 (0.191) & \itbf 0.3477 (0.001)     & 1.8376 (0.102)     & 0.4034 (0.022)  \\
                    & 6        & \bf 0.1015 (0.007) & 0.9980 (0.065) & 0.5520 (0.002)     & 7.4778 (0.488)     & \itbf 0.6564 (0.036)    \\
    \bottomrule
\end{tabular}}
\label{table_3_2}
\end{table*}


\begin{figure*}[!t]
    
    \centering
    \begin{minipage}[c]{0.9\textwidth}
        \includegraphics[width=\linewidth, height=0.2\linewidth]{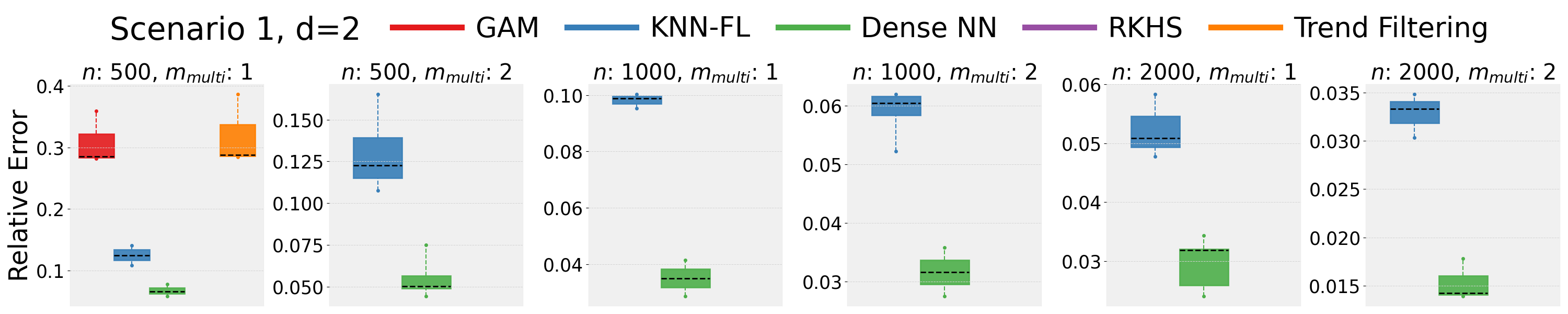}
    \end{minipage}\hspace{-4pt}

    \begin{minipage}[c]{0.9\textwidth}
        \includegraphics[width=\linewidth, height=0.2\linewidth]{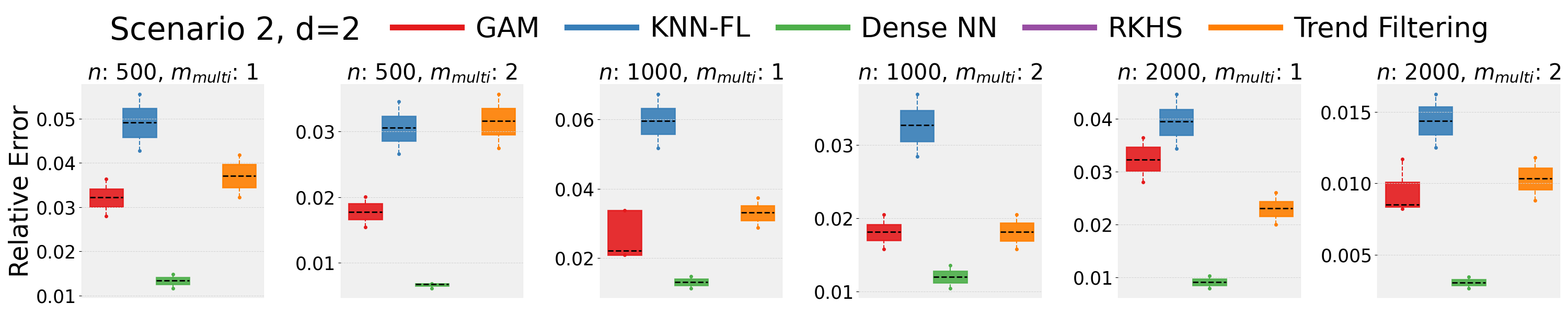}
    \end{minipage}\hspace{-4pt}

    \begin{minipage}[c]{0.9\textwidth}
        \includegraphics[width=\linewidth, height=0.2\linewidth]{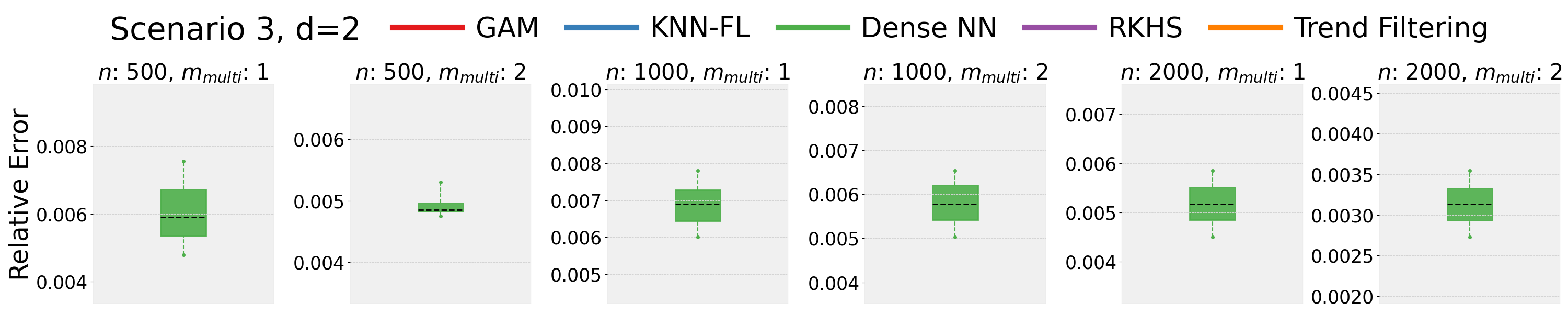}
    \end{minipage}\hspace{-4pt}

    \begin{minipage}[c]{0.9\textwidth}
        \includegraphics[width=\linewidth, height=0.2\linewidth]{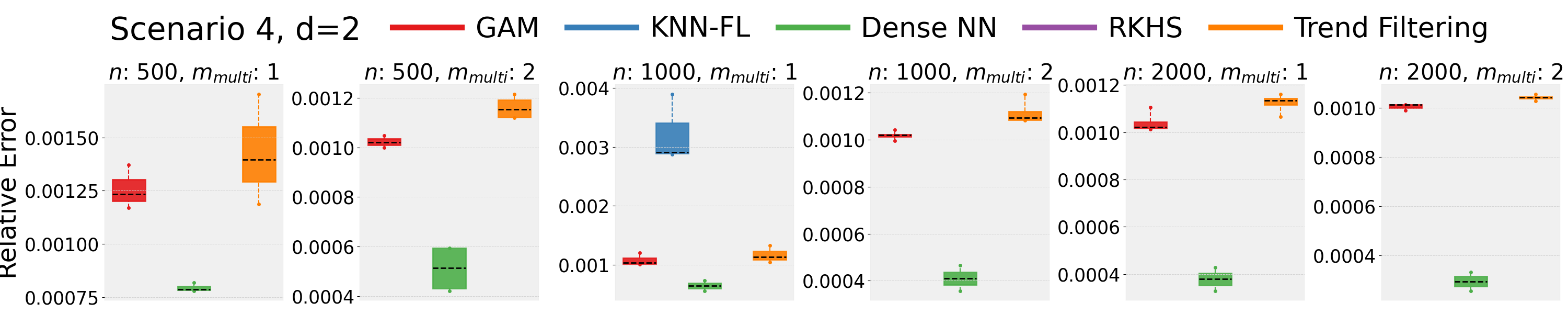}
    \end{minipage}\hspace{-4pt}

    \begin{minipage}[c]{0.9\textwidth}
        \includegraphics[width=\linewidth, height=0.2\linewidth]{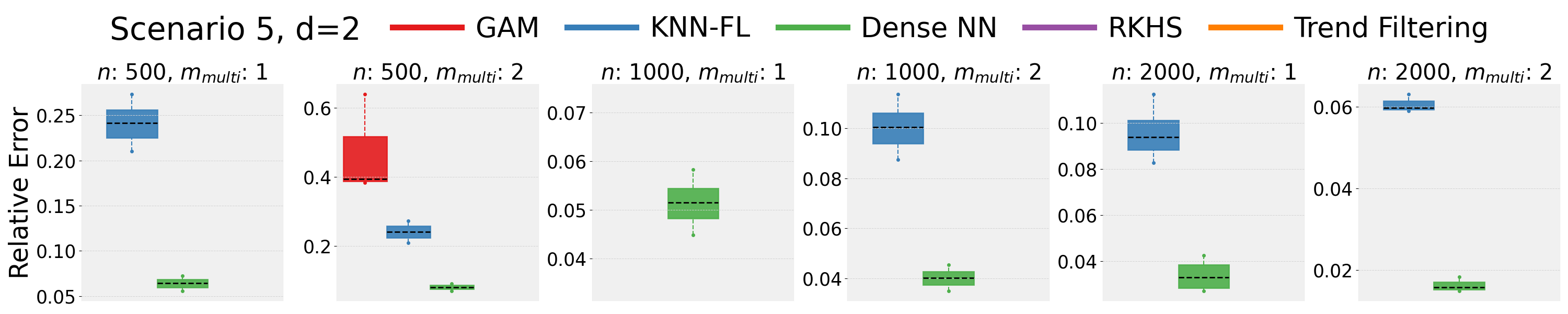}
    \end{minipage}\hspace{-4pt}

    \begin{minipage}[c]{0.9\textwidth}
        \includegraphics[width=\linewidth, height=0.2\linewidth]{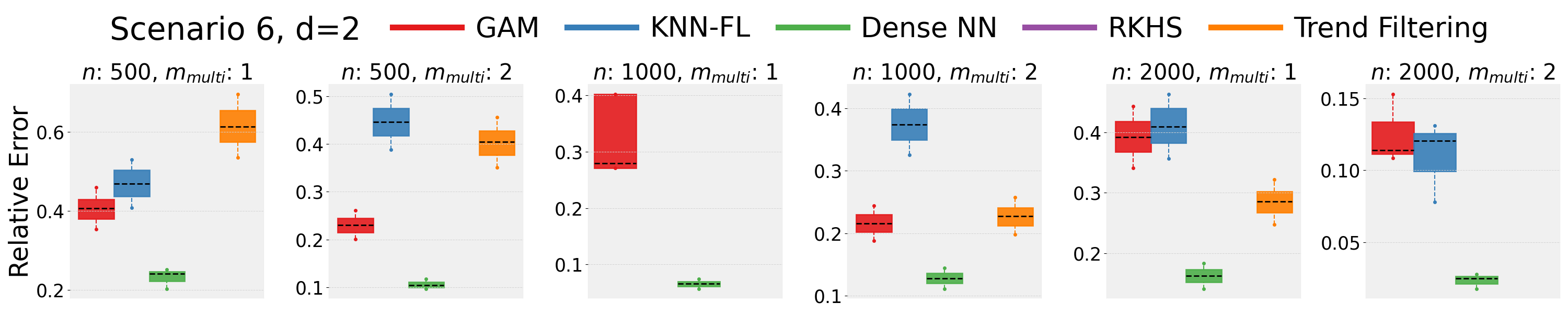}
    \end{minipage}\hspace{-4pt}

    \caption{Comparison of methods across different scenarios, $d = 2$, $n$, and $m_{mult}$ by box-plot. For each $(n,m_{mult})$ setting, we present the errors of the competitors in the range of within 5 times the smallest error.}
    \label{fig_box_d_2}
\end{figure*}

\begin{figure*}[!t]
    
    \centering
    \begin{minipage}[c]{0.9\textwidth}
        \includegraphics[width=\linewidth, height=0.2\linewidth]{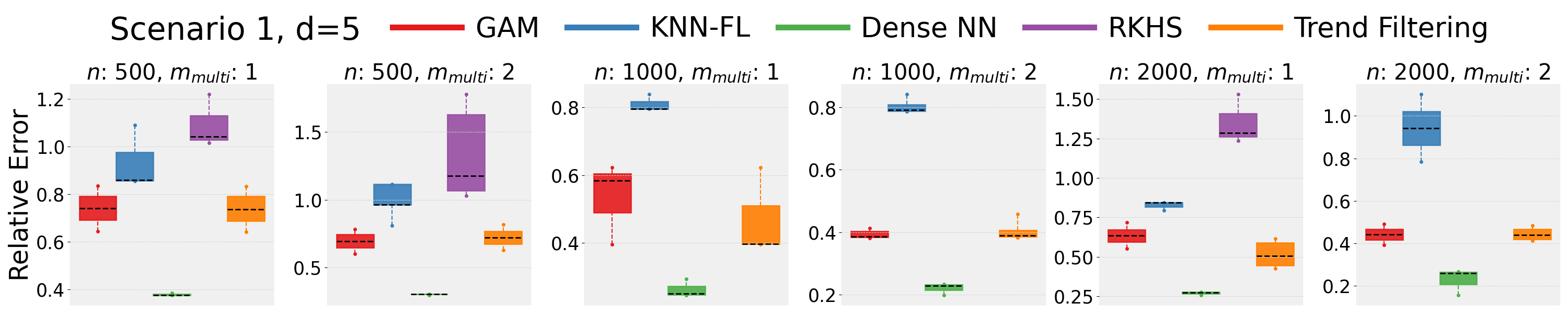}
    \end{minipage}\hspace{-4pt}

    \begin{minipage}[c]{0.9\textwidth}
        \includegraphics[width=\linewidth, height=0.2\linewidth]{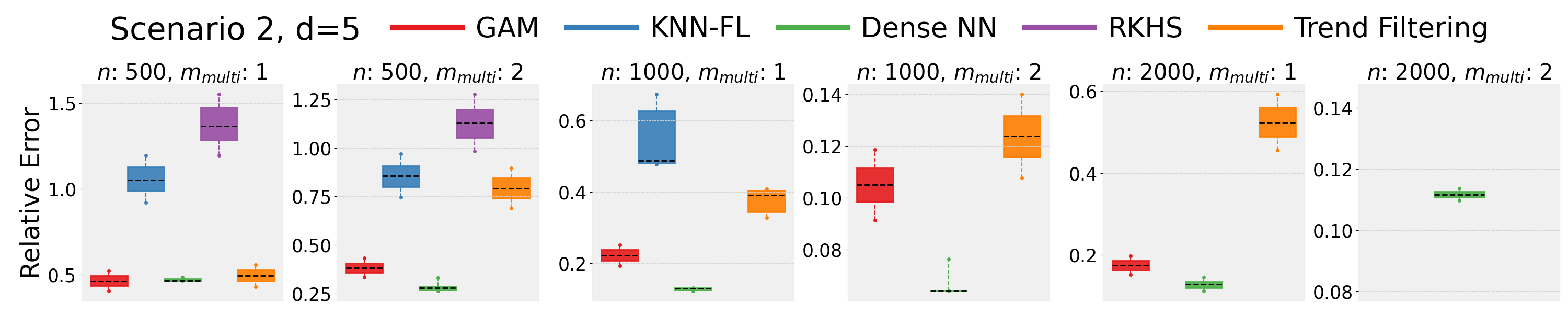}
    \end{minipage}\hspace{-4pt}

    \begin{minipage}[c]{0.9\textwidth}
        \includegraphics[width=\linewidth, height=0.2\linewidth]{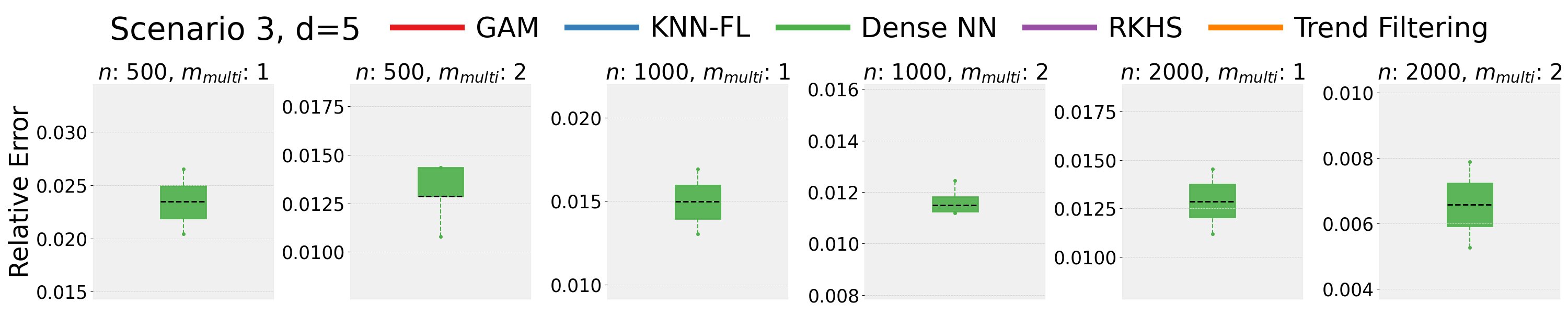}
    \end{minipage}\hspace{-4pt}

    \begin{minipage}[c]{0.9\textwidth}
        \includegraphics[width=\linewidth, height=0.2\linewidth]{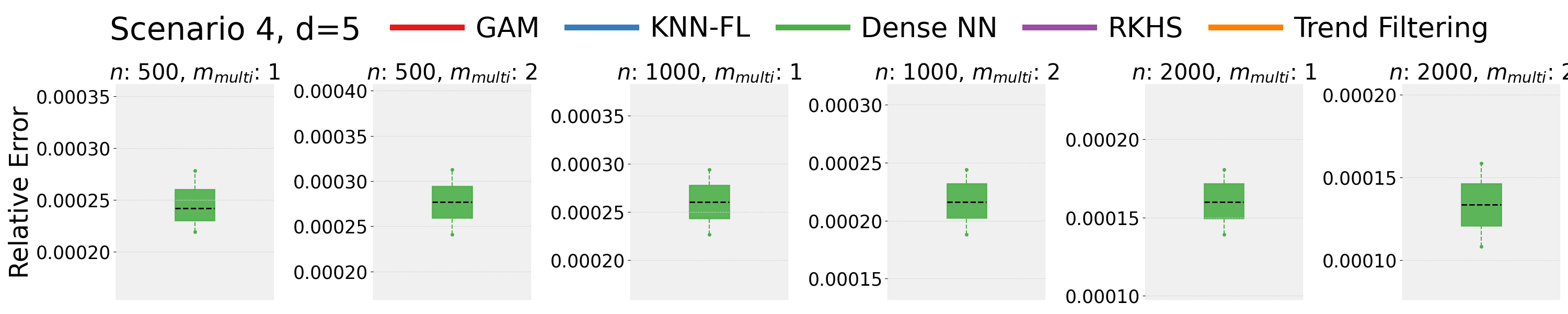}
    \end{minipage}\hspace{-4pt}

    \begin{minipage}[c]{0.9\textwidth}
        \includegraphics[width=\linewidth, height=0.2\linewidth]{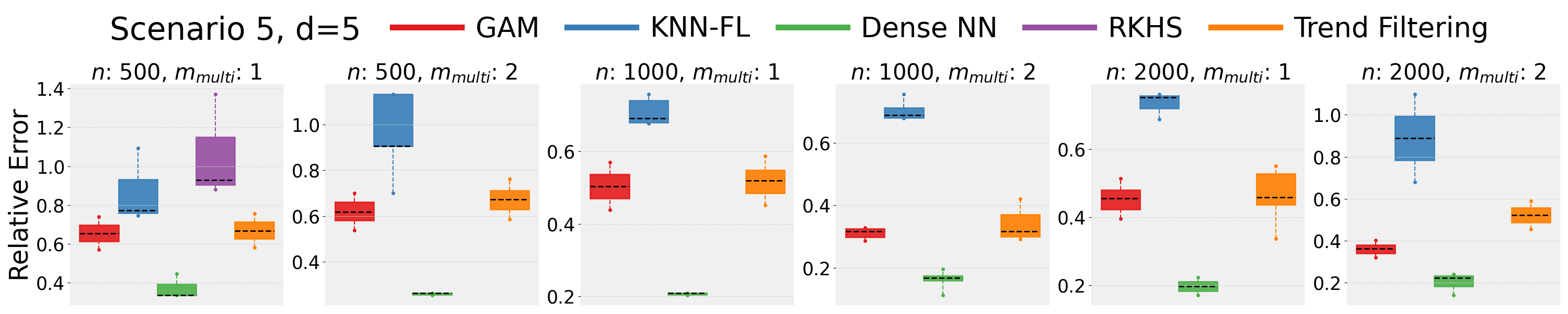}
    \end{minipage}\hspace{-4pt}

    \begin{minipage}[c]{0.9\textwidth}
        \includegraphics[width=\linewidth, height=0.2\linewidth]{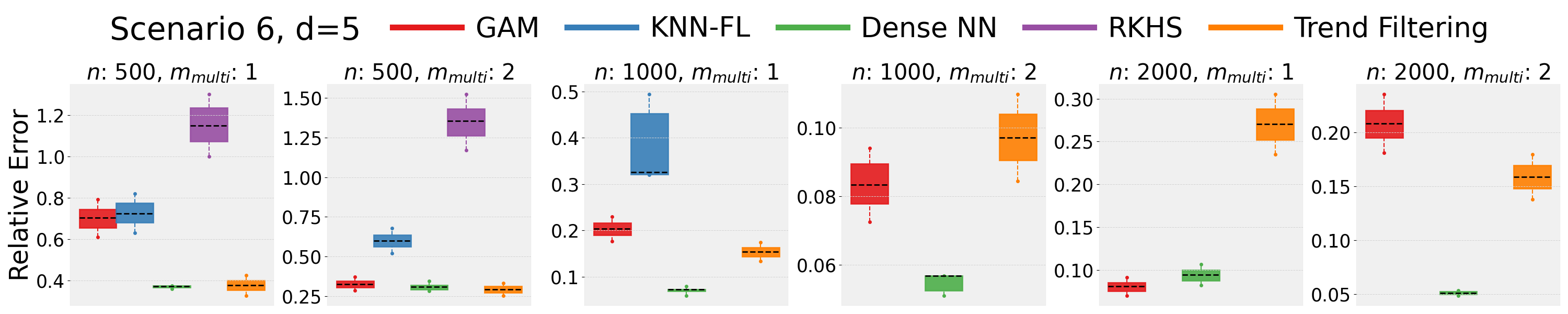}
    \end{minipage}\hspace{-4pt}

    \caption{Comparison of methods across different scenarios, $d = 5$, $n$, and $m_{mult}$ by box-plot. For each $(n,m_{mult})$ setting, we present the errors of the competitors in the range of within 5 times the smallest error.}
    \label{fig_box_d_5}
\end{figure*}


\clearpage 
\begin{table*}[!t]
    \centering
    \caption{Hyperparameters for Dense NN Selected by $5$-fold Cross-validation }
    \begin{tabular}{c|cccccc}
        \toprule
        \multirow{2}{*}{\textbf{Hyperparameter}} & \multicolumn{5}{c}{\textbf{Scenario}} \\
        \cmidrule(lr){2-7}
        & \textbf{1} & \textbf{2} & \textbf{3} & \textbf{4} & \textbf{5} & \textbf{6}\\
        \midrule
        \makecell{$L$ $\in \{1, 2, 3, 4, 5, 10\}$ \\ (Layers)} & 3 & 5 & 4 & 2 & 5  & 2 \\
        \midrule
        \makecell{$r$ $\in \{30, 60, 100, 300\}$ \\ (Neurons)} & 60 & 30 & 100 & 60 & 30 & 30\\
        \bottomrule
    \end{tabular}
    \label{table_hyper_nn}
\end{table*}

\begin{table*}[!t]
    \centering
    \setlength{\tabcolsep}{2pt} 
    \renewcommand{\arraystretch}{0.8}
    \caption{Hyperparameters for GAM Selected by $5$-fold Cross-validation}
    \begin{tabular}{c|cccccc}
        \toprule
        \multirow{2}{*}{\textbf{Hyperparameter}} & \multicolumn{5}{c}{\textbf{Scenario}} \\
        \cmidrule(lr){2-7}
        & \textbf{1} & \textbf{2} & \textbf{3} & \textbf{4} & \textbf{5} & \textbf{6}\\
        \midrule
        \makecell{basis \\ $\{\text{'ps': p-spline basis, 'cp': cyclic p-spline basis} \}$ \\ (Basis Type)} & \makecell{ps} & \makecell{ps} & \makecell{ps} & \makecell{ps} & \makecell{ps} & \makecell{ps} \\
        \midrule
        \makecell{constraint \\ $\{\text{None, 'convex', 'concave', \makecell{mono \\ inc}, \makecell{mono \\ dec}} \}$ \\ (Constraint Type)} & \makecell{mono \\ dec} & \makecell{mono \\ dec} & \makecell{convex} & \makecell{None} & \makecell{mono \\ dec} & \makecell{mono \\ inc}\\
        \midrule
        \makecell{$\lambda$, with $\log_{10}(\lambda) \in [-3,3]$\\ (Regularization Parameter)} & \makecell{63.1} & \makecell{0.001} & \makecell{0.25} & \makecell{3.98} & \makecell{1000} & \makecell{0.015}\\
        \midrule
        \makecell{n\_splines $\in \{6,7,8,9,10 \}$\\ (Number of Splines)} & \makecell{7} & \makecell{8} & \makecell{9} & \makecell{6} & \makecell{6} & \makecell{7}\\
        \midrule
        \makecell{spline $\in \{1, 2, 3, 4, 5\}$ \\ (Spline Order)} & \makecell{1} & \makecell{1} & \makecell{2} & \makecell{4} & \makecell{5} & \makecell{5}\\
        \bottomrule
    \end{tabular}
    \label{table_hyper_gam}
\end{table*}

\begin{table*}[!t]
    \centering
    \setlength{\tabcolsep}{2pt} 
    \caption{Hyperparameters for KNN-FL Selected by $5$-fold Cross-validation}
    \begin{tabular}{c|cccccc}
        \toprule
        \multirow{2}{*}{\textbf{Hyperparameter}} & \multicolumn{5}{c}{\textbf{Scenario}} \\
        \cmidrule(lr){2-7}
        & \textbf{1} & \textbf{2} & \textbf{3} & \textbf{4} & \textbf{5} & \textbf{6}\\
        \midrule
        \makecell{lambda $\in \{0.00001, 0.0001, 0.01, 0.1, 1, 10, 100\}$\\ (Regularization Parameter)} & 0.01 & 0.01 & 0.00001 & 0.0001 & 0.01 & 0.01\\
        \midrule
        \makecell{n\_neighbors $\in \{5,10\}$\\ (Number of Neighbors)} & 10 & 10 & 5 & 5 & 5 & 5 \\
        \bottomrule
    \end{tabular}
    \label{table_hyper_knn_fl}
\end{table*}

\begin{table*}[!t]
    \centering
    \caption{Hyperparameters for RKHS Selected by $5$-fold Cross-validation}
    \resizebox{\textwidth}{!}{%
    \begin{tabular}{c|cccccc}
        \toprule
        \multirow{2}{*}{\textbf{Hyperparameter}} & \multicolumn{5}{c}{\textbf{Scenario}} \\
        \cmidrule(lr){2-7}
        & \textbf{1} & \textbf{2} & \textbf{3} & \textbf{4} & \textbf{5} & \textbf{6}\\
        \midrule
        \makecell{$\alpha$, with $\log_{10}(\alpha) \in [-2, 2]$ \\ (Regularization Parameter)} & 100 & 100 & 100 & 100 & 100 & 100\\
        \midrule
        \makecell{kernel$\in \{\text{polynomial, laplacian, linear} \}$ \\ (Kernel Type)} & polynomial & polynomial & laplacian & polynomial & polynomial &laplacian \\
        \bottomrule
    \end{tabular}%
    }
    \label{table_hyper_RKHS}
\end{table*}

\begin{table*}[!t]
    \centering
    \caption{Hyperparameters for Trend Filtering Selected by $5$-fold Cross-validation}
    \begin{tabular}{c|cccccc}
        \toprule
        \multirow{2}{*}{\textbf{Hyperparameter}} & \multicolumn{5}{c}{\textbf{Scenario}} \\
        \cmidrule(lr){2-7}
        & \textbf{1} & \textbf{2} & \textbf{3} & \textbf{4} & \textbf{5} & \textbf{6}\\
        \midrule
        \makecell{$k \in \{0,1,2\}$ \\ (Trend Order)} & 1 & 0 & 1 & 0 & 1& 1 \\
        \midrule
        \makecell{lambda,  $\log_{10}(\lambda) \in [-5, 2]$ \\ (Regularization Parameter)} & 0.0007 & 0.8133 & 0.0001 & 0.1128 & 7.8476 & 0.6158\\
        \bottomrule
    \end{tabular}
    \label{table_hyper_tf}
\end{table*}

\newpage

\color{black}
\section{Examples of spatial processes satisfying Assumption~2}
\label{sec:noise_examples}

In this section, we provide concrete examples of stochastic processes
\(\{\gamma_i(\cdot)\}_{i=1}^n\) satisfying Assumption~2{\bf b}--{\bf c}. Throughout, all processes are assumed to be measurable on \([0,1]^d\).

\subsection{Gaussian processes}

Gaussian processes are a standard model for spatial and spatio-temporal variability; see, e.g., \cite{rasmussen2006gaussian, stein1999interpolation}.

\begin{proposition}
\label{prop:gaussian_process}
Suppose that \(\{\gamma_i(\cdot)\}_{i=1}^n\) are i.i.d.\ mean-zero Gaussian processes on \([0,1]^d\) with covariance function \(K\) satisfying
\begin{align*}
&\sup_{x\in[0,1]^d} K(x,x)\le C<\infty.
\end{align*}
Then Assumption~2{\bf b}--{\bf c} holds.
\end{proposition}

\begin{proof}
Fix \(i\) and \(x\in[0,1]^d\). Then
\begin{align*}
&\gamma_i(x)\sim \mathcal N(0,\sigma_x^2),
\qquad \sigma_x^2=K(x,x)\le C.
\end{align*}
Hence, for any \(c\in(0,(2C)^{-1})\),
\begin{align*}
&\E\exp\{c\gamma_i(x)^2\} = (1-2c\sigma_x^2)^{-1/2} \le (1-2cC)^{-1/2}<\infty.
\end{align*}
Therefore,
\begin{align*}
&\sup_{x\in[0,1]^d}\E\exp\{c\gamma_i(x)^2\}<\infty,
\end{align*}
which verifies Assumption~2{\bf b}.

Next, let \(q>1\). Since \([0,1]^d\) has Lebesgue measure one, Jensen's inequality yields
\begin{align*}
&\|\gamma_i\|_{\mathcal L_2}^{2q}
=
\left(\int_{[0,1]^d}\gamma_i(x)^2\,dx\right)^q
\le
\int_{[0,1]^d} |\gamma_i(x)|^{2q}\,dx.
\end{align*}
Taking expectations and using Fubini's theorem,
\begin{align*}
&\E\|\gamma_i\|_{\mathcal L_2}^{2q}
\le
\int_{[0,1]^d}\E|\gamma_i(x)|^{2q}\,dx.
\end{align*}
For a centered Gaussian random variable \(Z\sim \mathcal N(0,\sigma^2)\),
\begin{align*}
&\E|Z|^{2q} = \sigma^{2q}\E|N(0,1)|^{2q}.
\end{align*}
Thus
\begin{align*}
&\sup_{x\in[0,1]^d}\E|\gamma_i(x)|^{2q}
\le
C^q \E|N(0,1)|^{2q}<\infty,
\end{align*}
and therefore
\begin{align*}
&\E\|\gamma_i\|_{\mathcal L_2}^{2q}<\infty.
\end{align*}
This verifies Assumption~2{\bf c}.
\end{proof}

\subsection{Bounded stochastic processes}

Bounded stochastic processes naturally arise in applications where the underlying spatial signal is constrained. Examples include normalized measurements, physical quantities with known bounds, or truncated models; see, for example, \cite{cressie1993statistics}.

\begin{proposition}
\label{prop:bounded_process}
Suppose that \(\{\gamma_i(\cdot)\}_{i=1}^n\) are mean-zero stochastic processes satisfying
\begin{align*}
&\sup_{x\in[0,1]^d} |\gamma_i(x)|\le C
\qquad \text{almost surely}
\end{align*}
for some constant \(C>0\). Then Assumption~2{\bf b}--{\bf c} holds.
\end{proposition}

\begin{proof}
The mean-zero condition ensures that \(\E[\gamma_i(x)]=0\) for all \(x\), as required in Assumption~2{\bf b}.

For every \(x\in[0,1]^d\),
\begin{align*}
&e^{c\gamma_i(x)^2}\le e^{cC^2}
\qquad \text{almost surely}.
\end{align*}
Hence
\begin{align*}
&\sup_{x\in[0,1]^d}\E e^{c\gamma_i(x)^2}
\le
e^{cC^2}<\infty,
\end{align*}
so Assumption~2{\bf b} holds.

Also,
\begin{align*}
&\|\gamma_i\|_{\mathcal L_2}^2
=
\int_{[0,1]^d}\gamma_i(x)^2\,dx
\le
C^2
\end{align*}
almost surely. Therefore, for every \(q>1\),
\begin{align*}
&\E\|\gamma_i\|_{\mathcal L_2}^{2q}\le C^{2q}<\infty.
\end{align*}
Thus Assumption~2{\bf c} holds.
\end{proof}

\subsection{Processes with uniformly sub-Gaussian marginals}

Processes with sub-exponential tails arise in models with moderate tail behavior. They are commonly used in high-dimensional probability and statistical learning; see, for example, \cite{vershynin2018high}.

\begin{proposition}
\label{prop:subgaussian_process}
Suppose that \(\{\gamma_i(\cdot)\}_{i=1}^n\) are mean-zero processes and there exists \(\nu>0\) such that for every \(x\in[0,1]^d\) and every \(t\in\mathbb R\),
\begin{align*}
&\E e^{t\gamma_i(x)} \le e^{\nu^2 t^2/2}.
\end{align*}
Then Assumption~2{\bf b} holds. If, in addition,
\begin{align*}
&\sup_{x\in[0,1]^d}\E |\gamma_i(x)|^{2q}<\infty
\end{align*}
for some \(q>1\), then Assumption~2{\bf c} also holds.
\end{proposition}

\begin{proof}
A standard fact for sub-Gaussian random variables implies that there exists a constant \(c_0=c_0(\nu)>0\) such that
\begin{align*}
&\E e^{c_0\gamma_i(x)^2}\le 2
\qquad \text{for all }x\in[0,1]^d.
\end{align*}
Hence
\begin{align*}
&\sup_{x\in[0,1]^d}\E e^{c_0\gamma_i(x)^2}<\infty,
\end{align*}
which proves Assumption~2{\bf b}.

If moreover \(\sup_x \E |\gamma_i(x)|^{2q}<\infty\), then by Jensen and Fubini, as in the proof of Proposition~\ref{prop:gaussian_process},
\begin{align*}
&\E\|\gamma_i\|_{\mathcal L_2}^{2q} \le \int_{[0,1]^d}\E |\gamma_i(x)|^{2q}\,dx \le \sup_{x\in[0,1]^d}\E |\gamma_i(x)|^{2q} <\infty.
\end{align*}
Thus Assumption~2{\bf c} holds.
\end{proof}

\subsection{Random basis expansion processes with Rademacher coefficients}

Random basis expansions are widely used in functional data analysis and spatial statistics; see, e.g., \cite{ramsay2005functional, ferraty2006nonparametric}.

\begin{proposition}
\label{prop:basis_process}
Suppose that
\begin{align*}
&\gamma_i(x)=\sum_{k=1}^K a_k\,\varepsilon_{ik}\,\phi_k(x),
\qquad x\in[0,1]^d,
\end{align*}
where:
\begin{enumerate}
    \item \(\phi_1,\dots,\phi_K\) are deterministic bounded functions on \([0,1]^d\);
    \item \(a_1,\dots,a_K\in\mathbb R\) are deterministic constants;
    \item \(\varepsilon_{i1},\dots,\varepsilon_{iK}\) are independent Rademacher random variables, that is,
    \begin{align*}
&\mathbb P(\varepsilon_{ik}=1)=\mathbb P(\varepsilon_{ik}=-1)=\frac12.
\end{align*}
\end{enumerate}
Then Assumption~2{\bf b}--{\bf c} holds.
\end{proposition}

\begin{proof}
For each fixed \(x\in[0,1]^d\),
\begin{align*}
&\gamma_i(x)=\sum_{k=1}^K a_k\phi_k(x)\varepsilon_{ik}.
\end{align*}

Since \(\E[\varepsilon_{ik}]=0\) and the sum is finite,
\begin{align*}
&\E[\gamma_i(x)] = 0,
\end{align*}
so the mean-zero condition in Assumption~2{\bf b} holds.

Moreover, since the \(\varepsilon_{ik}\)'s are bounded and independent, Hoeffding's lemma implies that for every \(t\in\mathbb R\),
\begin{align*}
&\E \exp\{t a_k\phi_k(x)\varepsilon_{ik}\}
\le
\exp\!\left(\frac{t^2 a_k^2\phi_k(x)^2}{2}\right).
\end{align*}

Using independence across \(k\),
\begin{align*}
&\E e^{t\gamma_i(x)}
\le
\exp\!\left(
\frac{t^2}{2}\sum_{k=1}^K a_k^2\phi_k(x)^2
\right).
\end{align*}

Define
\begin{align*}
&V:=\sup_{x\in[0,1]^d}\sum_{k=1}^K a_k^2\phi_k(x)^2<\infty.
\end{align*}

Hence
\begin{align*}
&\E e^{t\gamma_i(x)}\le e^{Vt^2/2},
\end{align*}
so \(\gamma_i(x)\) is uniformly sub-Gaussian. Therefore, there exists \(c>0\) such that
\begin{align*}
&\sup_x \E e^{c\gamma_i(x)^2}<\infty,
\end{align*}
which verifies Assumption~2{\bf b}.
\end{proof}

\subsection{Karhunen--Lo\`eve type expansion models with bounded coefficients}

The Karhunen--Lo\`eve expansion is a classical representation of stochastic processes and also underlies functional principal component analysis; see, e.g., \cite{bosq2000linear, ash1975topics}. The following proposition gives a non-Gaussian version of this construction.

\begin{proposition}
\label{prop:kl_process}
Suppose that
\begin{align*}
&\gamma_i(x)=\sum_{k=1}^\infty \sqrt{\lambda_k}\,\eta_{ik}\,\phi_k(x),
\qquad x\in[0,1]^d,
\end{align*}
where:
\begin{enumerate}
    \item \(\{\phi_k\}_{k\ge 1}\) is an orthonormal basis of \(L_2([0,1]^d)\);
    \item \(\lambda_k\ge 0\) and
    \begin{align*}
&\sup_{x\in[0,1]^d}\sum_{k=1}^\infty \lambda_k\phi_k(x)^2 < \infty;
\end{align*}
    \item \(\{\eta_{ik}\}_{k\ge 1}\) are independent mean-zero random variables satisfying
    \begin{align*}
&|\eta_{ik}|\le C_0 \qquad \text{almost surely}
\end{align*}
    for some constant \(C_0>0\).
\end{enumerate}
Then Assumption~2{\bf b}--{\bf c} holds.
\end{proposition}

\begin{proof}
Fix \(x\in[0,1]^d\). For each \(m\ge 1\), define the partial sums
\begin{align*}
&S_m(x):=\sum_{k=1}^m \sqrt{\lambda_k}\,\eta_{ik}\,\phi_k(x).
\end{align*}

\medskip
\noindent
{\bf Step 1: Well-definedness.}
Let \(M < L\). Since \(\{\eta_{ik}\}\) are independent, mean-zero, and bounded by \(C_0\), we have
\begin{align*}
&\E\bigl(S_L(x)-S_M(x)\bigr)^2 = \sum_{k=M+1}^L \lambda_k \phi_k(x)^2 \E(\eta_{ik}^2) \le C_0^2 \sum_{k=M+1}^L \lambda_k \phi_k(x)^2.
\end{align*}
By assumption (2), the series \(\sum_{k=1}^\infty \lambda_k\phi_k(x)^2\) converges uniformly in \(x\), hence the right-hand side tends to zero as \(m,\ell\to\infty\). Thus \(S_m(x)\) is Cauchy in \(L_2\), and therefore converges in \(L_2\) (and along a subsequence almost surely) to a limit, which we denote by \(\gamma_i(x)\).

\medskip
\noindent
{\bf Step 2: Sub-Gaussian bound.}
By Hoeffding's lemma, for each \(k\) and every \(t\in\mathbb R\),
\begin{align*}
&\E \exp\{t\sqrt{\lambda_k}\phi_k(x)\eta_{ik}\}
\le
\exp\!\left(
\frac{t^2 C_0^2\lambda_k\phi_k(x)^2}{2}
\right).
\end{align*}
Using independence, we obtain for every \(m\ge 1\),
\begin{align*}
&\E e^{tS_m(x)} = \prod_{k=1}^m \E \exp\{t\sqrt{\lambda_k}\phi_k(x)\eta_{ik}\} \le \exp\!\bigl( \frac{t^2 C_0^2}{2} \sum_{k=1}^m \lambda_k\phi_k(x)^2 \bigr).
\end{align*}
Define
\begin{align*}
&V:=\sup_{x\in[0,1]^d}\sum_{k=1}^\infty \lambda_k\phi_k(x)^2<\infty.
\end{align*}
Then
\begin{align*}
&\E e^{tS_m(x)} \le \exp\!\left(\frac{t^2 C_0^2 V}{2}\right)
\qquad \text{for all }m.
\end{align*}

Passing to the limit and using Fatou's lemma,
\begin{align*}
&\E e^{t\gamma_i(x)}
\le
\liminf_{m\to\infty} \E e^{tS_m(x)}
\le
\exp\!\left(\frac{t^2 C_0^2 V}{2}\right).
\end{align*}
Hence \(\gamma_i(x)\) is uniformly sub-Gaussian in \(x\). Consequently, there exists \(c>0\) such that
\begin{align*}
&\sup_{x\in[0,1]^d}\E e^{c\gamma_i(x)^2}<\infty,
\end{align*}
which verifies Assumption~2{\bf b}.

\medskip
\noindent
{\bf Step 3: \(L_2\)-moment bound.}
Since \(\{\phi_k\}\) is orthonormal in \(L_2([0,1]^d)\), we have
\begin{align*}
&\|\gamma_i\|_{\mathcal L_2}^2
=
\int_{[0,1]^d}\gamma_i(x)^2\,dx
=
\sum_{k=1}^\infty \lambda_k \eta_{ik}^2
\end{align*}
almost surely.

Using \(|\eta_{ik}|\le C_0\),
\begin{align*}
&\|\gamma_i\|_{\mathcal L_2}^2
\le
C_0^2 \sum_{k=1}^\infty \lambda_k.
\end{align*}
Moreover, by Tonelli's theorem,
\begin{align*}
&\sum_{k=1}^\infty \lambda_k
=
\int_{[0,1]^d}\sum_{k=1}^\infty \lambda_k\phi_k(x)^2\,dx
\le
V<\infty.
\end{align*}
Therefore,
\begin{align*}
&\|\gamma_i\|_{\mathcal L_2}^{2q}
\le
\left(C_0^2 \sum_{k=1}^\infty \lambda_k\right)^q
<\infty
\quad \text{almost surely},
\end{align*}
and hence
\begin{align*}
&\E\|\gamma_i\|_{\mathcal L_2}^{2q}<\infty
\qquad \text{for every }q>1.
\end{align*}

This verifies Assumption~2{\bf c}.
\end{proof}

\newpage

\begingroup
\color{Black}

\section{General \textcolor{Maroon}{Function-Class} Bound}
\label{appendix:general_architecture_temporal_spatial}

In this section, we state and prove an extension of Theorem~1 in the Main Manuscript to \textcolor{Maroon}{general estimation classes}. The proof of Theorem~1 only uses two ingredients: an approximation bound by a function in \(\textcolor{Maroon}{\mathfrak F_{nm}}\), and a covering number bound for the corresponding truncated class. We therefore formulate the result in terms of an abstract entropy rate.
\paragraph{Empirical norms.}
For any index set \(I \subset \{1,\dots,n\}\), define
\begin{align*}
&\|g\|_{I,m}^2
:=
\frac{1}{|I|}\sum_{i\in I}\frac{1}{m_i}\sum_{j=1}^{m_i} g^2(x_{ij}).
\end{align*}
We use the shorthand
\begin{align*}
&\|g\|_{nm}^2 := \|g\|_{\{1,\dots,n\},m}^2, \quad \|g\|_{|\mathcal J_{e,s}|m}^2 := \|g\|_{\mathcal J_{e,s},m}^2, \quad \|g\|_{|\mathcal J_{o,s}|m}^2 := \|g\|_{\mathcal J_{o,s},m}^2.
\end{align*}

\begin{theorem}[General \textcolor{Maroon}{function-class} bound]
\label{cor:general_architecture_temporal_spatial}
Suppose that Assumptions~1--4 hold\textcolor{Maroon}{, except for Assumption~3{\bf c}}.
\textcolor{Maroon}{For each $n$ and $m$, let $\mathfrak F_{nm}$ be a nonempty deterministic pointwise measurable class of measurable functions from $[0,1]^d$ to $\mathbb R$.} Let \(\widehat f\) be \textcolor{Maroon}{a measurable minimizer} in \(\textcolor{Maroon}{\mathfrak F_{nm}}\) \textcolor{Maroon}{of the empirical criterion} in (2) in the Main Manuscript. Assume that there exists \textcolor{Maroon}{a deterministic} \(\overline f \in \textcolor{Maroon}{\mathfrak F_{nm}}\) such that
\begin{align}
\label{eq:general_approximation_error_temporal_spatial_appendix}
\|f^*-\overline f\|_\infty^2 \le \varepsilon_{nm}^2
\textcolor{Maroon}{,\qquad\|\overline f\|_\infty\le\mathcal A_{nm}}
\end{align}
for some sequence \(\varepsilon_{nm}>0\).
\textcolor{Maroon}{Write $\mathfrak F_{nm,\mathcal A_{nm}}:=\{f_{\mathcal A_{nm}}:f\in\mathfrak F_{nm}\}$.}

Let \((S_n)_{n\ge1}\) be any sequence satisfying
\begin{align*}
&C_{\beta}\log(n) \le S_n \le n,
\end{align*}
for a sufficiently large constant \(C_\beta>0\) depending only on the mixing rate in Assumption~4.
\\
\noindent
\textbf{(Blockwise entropy condition).}
Suppose there exists a sequence \((\rho_n)_{n\ge1}\) such that
\begin{align}
&\sup_{0<\delta\le 1}\; \max_{0\le s\le S_n-1} \max\!\bigl\{ \sup_{z} \frac{\log \mathcal N(\delta,\mathfrak F_{nm,\mathcal A_{nm}},\|\cdot\|_{|\mathcal J_{e,s}|m})}{|\mathcal J_{e,s}|m\,\log(\delta^{-1})}, \notag \\
&\quad \; \sup_{z}\frac{\log \mathcal N(\delta,\mathfrak F_{nm,\mathcal A_{nm}},\|\cdot\|_{|\mathcal J_{o,s}|m})}{|\mathcal J_{o,s}|m\,\log(\delta^{-1})} \bigr\} \le C\rho_n.\label{eq:block_entropy_condition}
\end{align}

\medskip

\noindent
\textbf{(Global entropy condition).}
Suppose there exists a sequence $(\widetilde\rho_n)_{n\ge1}$ such that
\begin{align}
\label{eq:global_entropy_condition}
\sup_{0<\delta\le 1}\;
\sup_{z_1,\dots,z_{nm}}
\frac{
\log \mathcal N\!\left(\delta,\textcolor{Maroon}{\mathfrak F_{nm,\mathcal A_{nm}}},\|\cdot\|_{nm}\right)
}{
nm\,\log(\delta^{-1})
}
\le C\widetilde{\rho}_n.
\end{align}

Let
\begin{align*}
&\mathcal A_{nm}\asymp \max\{\sigma_\epsilon,\sigma_\gamma\}\sqrt{\log(nm)},
\end{align*}
and \((a_n)_{n\ge 1}\) be any sequence such that \(a_n \to \infty\). Then, with probability tending to one,
\begin{align}
&\|\widehat f_{\mathcal A_{nm}}-f^*\|_{\mathcal L_2}^2 \lesssim \varepsilon_{nm}^2 + a_n \mathcal A_{nm}^2 \rho_n \log(nm) \notag \\
&\quad + a_n \mathcal A_{nm}^2 \beta(S_n) + a_n \mathcal A_{nm}^2 \widetilde{\rho}_n \log(nm) + \frac{\sigma_\gamma^2}{n}.\label{eq:general_risk_bound_temporal_spatial_appendix}
\end{align}

\medskip

Suppose moreover that \textcolor{Maroon}{Assumption~3{\bf c} holds, $\mathfrak F_{nm}=\mathcal F(L_{nm},r_{nm})$, and} \(\textcolor{Maroon}{(L_{nm},r_{nm})}\) are selected as in Lemma~\ref{lemma_approximation}. Then we may take
\begin{align*}
&\rho_n = \phi_{nm/S_n}\log^{4.3}(nm),
\qquad
\widetilde{\rho}_n = \phi_{nm}\log^{4.3}(nm).
\end{align*}

Substituting into \eqref{eq:general_risk_bound_temporal_spatial_appendix} yields
\begin{align}
&\|\widehat f_{\mathcal A_{nm}}-f^*\|_{\mathcal L_2}^2
\lesssim\;
\varepsilon_{nm}^2
+
a_n \mathcal A_{nm}^2 \phi_{nm/S_n}\log^{5.3}(nm)
\\
&+
a_n \mathcal A_{nm}^2 \phi_{nm}\log^{5.3}(nm)
+
a_n \mathcal A_{nm}^2 \beta(S_n)
+
\frac{\sigma_\gamma^2}{n}.
\nonumber
\end{align}

Choosing \(S_n\asymp \log n\), and using \(\phi_{nm/S_n}\gtrsim \phi_{nm}\), we obtain
\begin{align}
&\|\widehat f_{\mathcal A_{nm}}-f^*\|_{\mathcal L_2}^2 \lesssim \varepsilon_{nm}^2 + a_n \mathcal A_{nm}^2 \phi_{nm/S_n}\log^{5.3}(nm) + a_n \mathcal A_{nm}^2 \frac{1}{n} + \frac{\sigma_\gamma^2}{n}.
\end{align}

Under the architecture choice in Lemma~\ref{lemma_approximation}, \(\varepsilon_{nm}^2\lesssim \phi_{nm}\), and therefore
\begin{align*}
&a_n \mathcal A_{nm}^2 \phi_{nm/S_n}\log^{5.3}(nm) \lesssim a_n(\sigma_\epsilon^2+\sigma_\gamma^2+1)\phi_{nm}\log^{7.3}(nm),
\end{align*}
which recovers the rate in Theorem~1.
\end{theorem}

\begin{proof}
The proof follows the same line of argument as the proof of Theorem~1 in Appendix~\ref{sec_proof_Theorem_dependent}. The only difference is that we now distinguish between the entropy input used in the blockwise arguments and the entropy input used in the global argument.

Let
\begin{align*}
&\Delta_{nm}(x):=\widehat f_{\mathcal A_{nm}}(x)-\overline f_{\mathcal A_{nm}}(x).
\end{align*}
Also define the shifted truncated class
\begin{align*}
&\mathcal G_{\mathcal A_{nm}}(\overline f)
:=
\bigl\{
f_{\mathcal A_{nm}}-\overline f_{\mathcal A_{nm}}
:\,
f\in \textcolor{Maroon}{\mathfrak F_{nm}}
\bigr\}.
\end{align*}

For every \(0<\delta\le 1\), every \(z_1,\dots,z_{nm}\in[0,1]^d\), and every \(h_1,h_2\in \textcolor{Maroon}{\mathfrak F_{nm,\mathcal A_{nm}}}\), we have
\begin{align*}
&\|(h_1-\overline f_{\mathcal A_{nm}})-(h_2-\overline f_{\mathcal A_{nm}})\|_{nm}
=
\|h_1-h_2\|_{nm}.
\end{align*}
Likewise, for every block \(\mathcal J_{e,s}\) and \(\mathcal J_{o,s}\),
\begin{align*}
&\|(h_1-\overline f_{\mathcal A_{nm}})-(h_2-\overline f_{\mathcal A_{nm}})\|_{|\mathcal J_{e,s}|m} = \|h_1-h_2\|_{|\mathcal J_{e,s}|m},
\end{align*}
and
\begin{align*}
&\|(h_1-\overline f_{\mathcal A_{nm}})-(h_2-\overline f_{\mathcal A_{nm}})\|_{|\mathcal J_{o,s}|m} = \|h_1-h_2\|_{|\mathcal J_{o,s}|m}.
\end{align*}
Hence translation by \(-\overline f_{\mathcal A_{nm}}\) is an isometry from \(\textcolor{Maroon}{\mathfrak F_{nm,\mathcal A_{nm}}}\) onto \(\mathcal G_{\mathcal A_{nm}}(\overline f)\) in the global empirical norm and in each block empirical norm. Therefore, by the blockwise entropy condition and the global entropy condition, the same bounds hold for \(\mathcal G_{\mathcal A_{nm}}(\overline f)\) with the same constants $C\rho_n$ and $C\widetilde\rho_n$.

Define the event
\begin{align*}
&\mathcal E = \bigl\{ |\epsilon_{ij}|\le \frac{\mathcal A_{nm}}{4}\,\text{for all }1\le i\le n,\, 1\le j\le m_i, \quad \|\gamma_i\|_{\infty,nm} \le \frac{\mathcal A_{nm}}{4}\,\text{for all }1\le i\le n \bigr\}.
\end{align*}
By Lemma~\ref{lemma_epsilon_gamma_infinity_bound},
\begin{align*}
&\mathbb P(\mathcal E^c)=o(1).
\end{align*}

\textcolor{Maroon}{By the assumption $\|\overline f\|_\infty\le\mathcal A_{nm}$, we have}
\begin{align*}
&\overline f=\overline f_{\mathcal A_{nm}}.
\end{align*}
Therefore,
\begin{align}
&\|\widehat f_{\mathcal A_{nm}}-f^*\|_{\mathcal L_2}^2 \le 2\|\overline f_{\mathcal A_{nm}}-f^*\|_{\mathcal L_2}^2 + 2\|\widehat f_{\mathcal A_{nm}}-\overline f_{\mathcal A_{nm}}\|_{\mathcal L_2}^2 \nonumber\label{eq:general_architecture_basic_reduction_clean_new}
\\
&\le
2\varepsilon_{nm}^2
+
2\|\Delta_{nm}\|_{\mathcal L_2}^2.
\end{align}
Thus it suffices to bound \(\|\Delta_{nm}\|_{\mathcal L_2}^2\).

On the event \(\mathcal E\), exactly as in Step~1 of the proof of Theorem~1, the empirical basic inequality gives
\begin{align}
&\frac12\|\Delta_{nm}\|_{nm}^2 \le \sum_{i=1}^n\sum_{j=1}^{m_i}\frac{1}{nm_i}\Delta_{nm}(x_{ij}) \bigl\{f^*(x_{ij}) -\overline f(x_{ij})\bigr\} \notag \\
&\quad + \sum_{i=1}^n\sum_{j=1}^{m_i}\frac{1}{nm_i}\Delta_{nm}(x_{ij}) \bigl\{\gamma_i(x_{ij})+\epsilon_{ij}\bigr\}.\label{eq:general_architecture_basic_ineq_clean_new}
\end{align}

We first control the approximation term. By Cauchy--Schwarz and
\eqref{eq:general_approximation_error_temporal_spatial_appendix},
\begin{align}
&\sum_{i=1}^n\sum_{j=1}^{m_i}\frac{1}{nm_i}\Delta_{nm}(x_{ij}) \bigl\{f^*(x_{ij})-\overline f(x_{ij})\bigr\} \le \|\Delta_{nm}\|_{nm}\,\|f^*-\overline f\|_{nm} \nonumber\label{eq:general_architecture_approx_term_clean_new}
\\
&\le
\|\Delta_{nm}\|_{nm}\,\varepsilon_{nm}
\nonumber
\\
&\le
\frac{1}{32}\|\Delta_{nm}\|_{nm}^2 + C\,\varepsilon_{nm}^2.
\end{align}

\medskip
\noindent
{\bf Step 1: the \(\epsilon\)-term.}
Since \(\Delta_{nm}\in \mathcal G_{\mathcal A_{nm}}(\overline f)\), the blockwise entropy condition implies that the proof of
\Cref{coro:complexity bound 2 beta mn}
can be repeated block by block for the shifted class \(\mathcal G_{\mathcal A_{nm}}(\overline f)\), with the quantity \(\phi_{|\mathcal J_{\cdot,s}|m}\log^{4.3}(nm)\) replaced by \(\rho_n\). Thus, arguing exactly as in Step~3 of the proof of Theorem~1, we obtain that with probability tending to one,
\begin{align}
\label{eq:general_architecture_epsilon_lambda_clean_new}
\sum_{i=1}^n\sum_{j=1}^{m_i}\frac{1}{nm_i}\Delta_{nm}(x_{ij})\epsilon_{ij}
\le
C\,\sigma_\epsilon\,
\Lambda^{(\rho)}_{n,S_n}(\Delta_{nm}),
\end{align}
where \(\Lambda^{(\rho)}_{n,S_n}\) is defined exactly as \(\Lambda_{n,S}\) in
\Cref{coro:complexity bound 2 beta mn}, but with \(\phi_{|\mathcal J_{\cdot,s}|m}\log^{4.3}(nm)\) replaced by \(\rho_n\). Repeating the same Jensen reduction as in the proof of Theorem~1 gives
\begin{align}
\label{eq:general_architecture_lambda_reduction_clean_new}
\Lambda^{(\rho)}_{n,S_n}(\Delta_{nm})
\le
C\,
\|\Delta_{nm}\|_{nm}
\sqrt{
\rho_n\,
\log\!\left(\frac{2\mathcal A_{nm}}{\|\Delta_{nm}\|_{nm}}\right)
}.
\end{align}
Applying Young's inequality, we obtain
\begin{align}
&\sum_{i=1}^n\sum_{j=1}^{m_i}\frac{1}{nm_i}\Delta_{nm}(x_{ij})\epsilon_{ij} \le \frac{1}{32}\|\Delta_{nm}\|_{nm}^2 + C\,\sigma_\epsilon^2\,\rho_n\, \log\!\bigl(\frac{2\mathcal A_{nm}}{\|\Delta_{nm}\|_{nm}}\bigr).\label{eq:general_architecture_epsilon_term_clean_new}
\end{align}
Since \(\sigma_\epsilon^2\lesssim \mathcal A_{nm}^2\), this becomes
\begin{align}
&\sum_{i=1}^n\sum_{j=1}^{m_i}\frac{1}{nm_i}\Delta_{nm}(x_{ij})\epsilon_{ij} \le \frac{1}{32}\|\Delta_{nm}\|_{nm}^2 + C\,\mathcal A_{nm}^2\,\rho_n\, \log\!\bigl(\frac{2\mathcal A_{nm}}{\|\Delta_{nm}\|_{nm}}\bigr).\label{eq:general_architecture_epsilon_term_clean_new2}
\end{align}

\medskip
\noindent
{\bf Step 2: norm comparison.}
Using again the blockwise entropy condition, the proofs of
\Cref{coro:deviation for norms mn}
and
\Cref{lemma:rate in the temporal-spatial model}
carry over block by block with \(\phi_{nm/S_n}\log^{5.3}(nm/S_n)\) replaced by \(\rho_n\log(nm)\). Therefore, for any sequence \(a_n\to\infty\), arguing exactly as in Step~4 of the proof of Theorem~1, we obtain with probability tending to one,
\begin{align}
&\|\Delta_{nm}\|_{\mathcal L_2}^2 \le 2\|\Delta_{nm}\|_{nm}^2 + C a_n \bigl\{ \mathcal A_{nm}^2\,\rho_n\,\log(nm) + \mathcal A_{nm}^2\beta(S_n) \bigr\}.\label{eq:general_architecture_norm_comparison_clean_new}
\end{align}

\medskip
\noindent
{\bf Step 3: the \(\gamma\)-term.}
We decompose the \(\gamma\)-term exactly as in the proof of Theorem~1:
\begin{align*}
&\sum_{i=1}^n\sum_{j=1}^{m_i}\frac{1}{nm_i}\Delta_{nm}(x_{ij})\gamma_i(x_{ij})
\\
&={} \sum_{i=1}^n\sum_{j=1}^{m_i}\frac{1}{nm_i} \Bigl\{ \Delta_{nm}(x_{ij})\gamma_i(x_{ij}) - \langle \gamma_i,\Delta_{nm}\rangle_{\mathcal L_2} \Bigr\} + \frac1n\sum_{i=1}^n \langle \gamma_i,\Delta_{nm}\rangle_{\mathcal L_2}.
\end{align*}

For the second term, by Lemma~\ref{lemma:bound on the expected norm of gamma} and Young's inequality,
\begin{align}
\label{eq:general_architecture_gamma_mean_term_clean_new1}
\frac1n\sum_{i=1}^n \langle \gamma_i,\Delta_{nm}\rangle_{\mathcal L_2}
&=
\left\langle \frac1n\sum_{i=1}^n \gamma_i,\Delta_{nm}\right\rangle_{\mathcal L_2}
\nonumber\\
&\le
\left\|\frac1n\sum_{i=1}^n \gamma_i\right\|_{\mathcal L_2}\|\Delta_{nm}\|_{\mathcal L_2}
\nonumber\\
&\le
\frac{8\sigma_\gamma^2}{n}
+
\frac{1}{32}\|\Delta_{nm}\|_{\mathcal L_2}^2.
\end{align}
Combining this with \eqref{eq:general_architecture_norm_comparison_clean_new}, we obtain
\begin{align}
&\frac1n\sum_{i=1}^n \langle \gamma_i,\Delta_{nm}\rangle_{\mathcal L_2} \le \frac{8\sigma_\gamma^2}{n} + \frac{1}{16}\|\Delta_{nm}\|_{nm}^2 + C_1 a_n \bigl\{ \mathcal A_{nm}^2\,\rho_n\,\log(nm) + \mathcal A_{nm}^2\beta(S_n) \bigr\}.\label{eq:general_architecture_gamma_mean_term_clean_new2}
\end{align}

For the centered fluctuation term, we now use the global entropy condition. Since \(\Delta_{nm}\in \mathcal G_{\mathcal A_{nm}}(\overline f)\), the proof of
\Cref{coro:deviation bound for gamma with xi mn}
carries over verbatim with \(\phi_{nm}\log^{5.3}(nm)\) replaced by \(\widetilde{\rho}_n\log(nm)\). Hence, with probability tending to one,
\begin{align}
\label{eq:general_architecture_gamma_fluctuation_clean_new1}
&\sum_{i=1}^n\sum_{j=1}^{m_i}\frac{1}{nm_i}
\Bigl\{
\Delta_{nm}(x_{ij})\gamma_i(x_{ij})
-
\langle \gamma_i,\Delta_{nm}\rangle_{\mathcal L_2}
\Bigr\}
\nonumber\\
&\le
\frac{1}{16}\|\Delta_{nm}\|_{\mathcal L_2}^2
+
C\,\mathcal A_{nm}^2\,\widetilde{\rho}_n\,\log(nm).
\end{align}
Using \eqref{eq:general_architecture_norm_comparison_clean_new} once more, we get
\begin{align}
&\label{eq:general_architecture_gamma_fluctuation_clean_new2}
\sum_{i=1}^n\sum_{j=1}^{m_i}\frac{1}{nm_i}
\Bigl\{
\Delta_{nm}(x_{ij})\gamma_i(x_{ij})
-
\langle \gamma_i,\Delta_{nm}\rangle_{\mathcal L_2}
\Bigr\}
\nonumber
\\
&\le \frac{1}{8}\|\Delta_{nm}\|_{nm}^2 + C_2 a_n \bigl\{ \mathcal A_{nm}^2\,\rho_n\,\log(nm) + \mathcal A_{nm}^2\beta(S_n) \bigr\} + C_3\mathcal A_{nm}^2\widetilde{\rho}_n\log(nm).
\end{align}

Combining \eqref{eq:general_architecture_gamma_mean_term_clean_new2} and
\eqref{eq:general_architecture_gamma_fluctuation_clean_new2}, we conclude that
\begin{align}
&\sum_{i=1}^n\sum_{j=1}^{m_i}\frac{1}{nm_i}\Delta_{nm}(x_{ij})\gamma_i(x_{ij}) \le \frac{3}{16}\|\Delta_{nm}\|_{nm}^2 \notag \\
&\quad + C_4 a_n \bigl\{ \mathcal A_{nm}^2\,\rho_n\,\log(nm) + \mathcal A_{nm}^2\beta(S_n) \bigr\} + C_5\mathcal A_{nm}^2\widetilde{\rho}_n\log(nm) + C_6\,\frac{\sigma_\gamma^2}{n}.\label{eq:general_architecture_gamma_term_clean_new}
\end{align}

\medskip
\noindent
{\bf Step 4: empirical norm bound.}
Substituting \eqref{eq:general_architecture_approx_term_clean_new},
\eqref{eq:general_architecture_epsilon_term_clean_new2}, and
\eqref{eq:general_architecture_gamma_term_clean_new} into
\eqref{eq:general_architecture_basic_ineq_clean_new}, we obtain
\begin{align}
&\|\Delta_{nm}\|_{nm}^2 \lesssim \varepsilon_{nm}^2 + \mathcal A_{nm}^2\,\rho_n\, \log\!\bigl(\frac{2\mathcal A_{nm}}{\|\Delta_{nm}\|_{nm}}\bigr) + a_n\mathcal A_{nm}^2\rho_n\log(nm) + a_n\mathcal A_{nm}^2\beta(S_n) \notag \\
&\quad + \mathcal A_{nm}^2\widetilde{\rho}_n\log(nm) + \frac{\sigma_\gamma^2}{n}.\label{eq:general_architecture_empirical_bound_clean_new}
\end{align}

We now control the logarithmic factor exactly as in the proof of Theorem~1. If
\begin{align*}
&\|\Delta_{nm}\|_{nm}^2 \le \frac1n,
\end{align*}
then \eqref{eq:general_architecture_empirical_bound_clean_new} already yields the desired conclusion. Otherwise,
\begin{align*}
&\|\Delta_{nm}\|_{nm}^2 > \frac1n,
\qquad\text{hence}\qquad
\|\Delta_{nm}\|_{nm} \ge n^{-1/2},
\end{align*}
so that
\begin{align*}
&\log\!\left(\frac{2\mathcal A_{nm}}{\|\Delta_{nm}\|_{nm}}\right)
\le
\log\!\left(2\mathcal A_{nm}\sqrt n\right)
\lesssim
\log(nm).
\end{align*}
Therefore, in this case,
\begin{align}
&\|\Delta_{nm}\|_{nm}^2 \lesssim \varepsilon_{nm}^2 + \mathcal A_{nm}^2\rho_n\log(nm) + a_n\mathcal A_{nm}^2\rho_n \log(nm) \notag \\
&\quad + a_n\mathcal A_{nm}^2\beta(S_n) + \mathcal A_{nm}^2\widetilde{\rho}_n\log(nm) + \frac{\sigma_\gamma^2}{n}.\label{eq:general_architecture_empirical_bound_clean_new2}
\end{align}
Since \(a_n\to\infty\), we have \(a_n\ge1\) for all sufficiently large \(n\), and thus
\begin{align*}
&\mathcal A_{nm}^2\rho_n\log(nm)
\le
a_n\mathcal A_{nm}^2\rho_n\log(nm).
\end{align*}
Hence in either case,
\begin{align}
&\|\Delta_{nm}\|_{nm}^2 \lesssim \varepsilon_{nm}^2 + a_n\mathcal A_{nm}^2\rho_n\log(nm) + a_n\mathcal A_{nm}^2\beta(S_n) + a_n\mathcal A_{nm}^2\widetilde{\rho}_n\log(nm) + \frac{\sigma_\gamma^2}{n}.\label{eq:general_architecture_empirical_bound_clean_new3}
\end{align}

\medskip
\noindent
{\bf Step 5: passage to \(\mathcal L_2\).}
Combining \eqref{eq:general_architecture_norm_comparison_clean_new} and
\eqref{eq:general_architecture_empirical_bound_clean_new3}, we obtain
\begin{align}
&\|\Delta_{nm}\|_{\mathcal L_2}^2 \le 2\|\Delta_{nm}\|_{nm}^2 + C a_n \bigl\{ \mathcal A_{nm}^2\rho_n\log(nm) + \mathcal A_{nm}^2\beta(S_n) \bigr\} \nonumber\label{eq:general_architecture_L2_bound_clean_new}
\\
&\lesssim \varepsilon_{nm}^2 + a_n\mathcal A_{nm}^2\rho_n\log(nm) + a_n\mathcal A_{nm}^2\beta(S_n) + a_n\mathcal A_{nm}^2\widetilde{\rho}_n\log(nm) + \frac{\sigma_\gamma^2}{n}.
\end{align}
Substituting \eqref{eq:general_architecture_L2_bound_clean_new} into
\eqref{eq:general_architecture_basic_reduction_clean_new}, we conclude that
\begin{align*}
&\|\widehat f_{\mathcal A_{nm}}-f^*\|_{\mathcal L_2}^2 \lesssim \varepsilon_{nm}^2 + a_n\mathcal A_{nm}^2\rho_n\log(nm) + a_n\mathcal A_{nm}^2\beta(S_n) + a_n\mathcal A_{nm}^2\widetilde{\rho}_n\log(nm) + \frac{\sigma_\gamma^2}{n},
\end{align*}
with probability tending to one. This proves
\eqref{eq:general_risk_bound_temporal_spatial_appendix}.

To recover Theorem~1, \textcolor{Maroon}{take $\mathfrak F_{nm}=\mathcal F(L_{nm},r_{nm})$ under Assumption~3{\bf c} in the Main Manuscript and} choose $(L_{nm},r_{nm})$ according to either architecture in Lemma~\ref{lemma_approximation} and set $S_n\asymp\log n$, with a sufficiently large proportionality constant. Lemma~\ref{lemma_approximation} gives
\begin{align*}
&\varepsilon_{nm}^2\lesssim\phi_{nm}.
\end{align*}
The covering-number bound in Lemma~\ref{lemma:threshold space entropy}, which is uniform over the design points, has numerator of order
\begin{align*}
&nm\phi_{nm}\log^{3.3}(nm)\log(\delta^{-1}).
\end{align*}
Since every even or odd block contains order $nm/S_n$ observations, the blockwise and global entropy conditions hold with
\begin{align*}
&\rho_n \lesssim S_n\phi_{nm}\log^{3.3}(nm) \lesssim \phi_{nm/S_n}\log^{4.3}(nm), \qquad \widetilde\rho_n \lesssim \phi_{nm}\log^{4.3}(nm).
\end{align*}
Substituting these bounds into \eqref{eq:general_risk_bound_temporal_spatial_appendix} gives
\begin{align*}
&\|\widehat f_{\mathcal A_{nm}}-f^*\|_{\mathcal L_2}^2 \lesssim\; \phi_{nm} + a_n\mathcal A_{nm}^2\phi_{nm/S_n}\log^{5.3}(nm) + a_n\mathcal A_{nm}^2\phi_{nm}\log^{5.3}(nm)
\\
&+
a_n\mathcal A_{nm}^2\beta(S_n)
+
\frac{\sigma_\gamma^2}{n}.
\end{align*}
Because $\phi_{nm}\leq\phi_{nm/S_n}$, the global entropy term is dominated by the blockwise term. Moreover, the choice of $S_n$ gives $\beta(S_n)\lesssim n^{-1}$ and
\begin{align*}
&\phi_{nm/S_n}
\leq
S_n\phi_{nm}
\lesssim
\phi_{nm}\log(nm).
\end{align*}
Together with
\begin{align*}
&\mathcal A_{nm}^2
\lesssim
(\sigma_\epsilon^2+\sigma_\gamma^2)\log(nm),
\end{align*}
these bounds yield
\begin{align*}
&\|\widehat f_{\mathcal A_{nm}}-f^*\|_{\mathcal L_2}^2 \lesssim a_n(\sigma_\epsilon^2+\sigma_\gamma^2 +1) \phi_{nm}\log^{7.3}(nm) + a_n\frac{(\sigma_\epsilon^2+\sigma_\gamma^2)\log(nm)}{n},
\end{align*}
which is the bound in Theorem~1.

\end{proof}

\endgroup

\color{Maroon}

\subsection{Architecture-Dependent Risk Bound}
\label{appendix:explicit_architecture_temporal_spatial}

We next restate and prove the architecture-dependent corollary given
in the Main Manuscript.

\begin{corollary}[Architecture-dependent risk bound]
\label{cor:explicit_architecture_temporal_spatial_appendix}
Retain the hypotheses and notation of
Theorem~\ref{cor:general_architecture_temporal_spatial}, except for
the blockwise and global entropy conditions
\eqref{eq:block_entropy_condition} and
\eqref{eq:global_entropy_condition}. Define
\begin{align}
\mathfrak V_{nm}
:=
L_{nm}^2r_{nm}^2
\log\!\left(eL_{nm}r_{nm}^2\right).
\end{align}
Then the entropy conditions
{\eqref{eq:block_entropy_condition} and
\eqref{eq:global_entropy_condition}}
are satisfied with
\begin{align}
\label{eq:explicit_architecture_entropy_rates_appendix}
\rho_n
\lesssim
\frac{S_n\mathfrak V_{nm}}{nm}\log^{1.3}(nm),
\qquad
\widetilde\rho_n
\lesssim
\frac{\mathfrak V_{nm}}{nm}\log^{1.3}(nm).
\end{align}
Consequently, with probability approaching one,
\begin{align}
&\|\widehat f_{\mathcal A_{nm}}-f^*\|_{\mathcal L_2}^2 \lesssim\; \varepsilon_{nm}^2 + a_n\mathcal A_{nm}^2 \frac{S_n\mathfrak V_{nm}}{nm} \log^{2.3}(nm) + a_n\mathcal A_{nm}^2\beta(S_n) \label{eq:explicit_architecture_temporal_spatial_appendix} +
\frac{\mathcal A_{nm}^2}{n}
+
\frac{\sigma_\gamma^2}{n}.
\end{align}
\end{corollary}

\begingroup
\color{Black}

\begin{proof}
\color{Maroon}
{
Fix deterministic $L_{nm},r_{nm}\in\mathbf Z^+$ and take
$\mathfrak F_{nm}=\mathcal F(L_{nm},r_{nm})$.
The estimator and the deterministic comparator are those in the
hypotheses of the corollary. For all sufficiently large $n$, we have
$\mathcal A_{nm}\ge1$ and $nm\ge e$.

\medskip
\noindent
{\bf Step 1: the network class and its covering numbers.}
A network in $\mathcal F(L_{nm},r_{nm})$ is determined by
\begin{align*}
&W_{nm} = (d+1)r_{nm} + (L_{nm}-1)r_{nm}(r_{nm}+1) + (r_{nm}+1)
\end{align*}
weights and biases. Write $f_\theta$ for the network corresponding
to $\theta\in\mathbb R^{W_{nm}}$.
For every $x\in[0,1]^d$, the map $\theta\mapsto f_\theta(x)$
is continuous, since it is obtained by finitely many affine maps
and ReLU operations.
For any $\theta\in\mathbb R^{W_{nm}}$, choose
$\theta_k\in\mathbb Q^{W_{nm}}$ with $\theta_k\to\theta$.
Then $f_{\theta_k}(x)\to f_\theta(x)$ for every $x$.
Thus the countable subclass
$\{f_\theta:\theta\in\mathbb Q^{W_{nm}}\}$ establishes pointwise
measurability. The zero network belongs to the class, and truncation
preserves pointwise convergence.

Only the first bound in Lemma~\ref{lemma:threshold space entropy}
is used below; the restrictions in \eqref{L_r_size} are not imposed.
To include $L_{nm}=r_{nm}=1$, put
$r_{nm}^{+}=\max\{r_{nm},2\}$.
Adding neurons whose incoming and outgoing weights are zero gives
\begin{align*}
&\mathcal F(L_{nm},r_{nm})
\subseteq
\mathcal F(L_{nm},r_{nm}^{+}).
\end{align*}
Moreover, since $r_{nm}^{+}\le2r_{nm}$,
\begin{align*}
&L_{nm}^2(r_{nm}^{+})^2 \log\!\bigl(L_{nm}(r_{nm}^{+})^2\bigr) \le 4L_{nm}^2r_{nm}^2\log(4L_{nm}r_{nm}^2)
\\
&\le
4(1+\log4)\,
L_{nm}^2r_{nm}^2\log(eL_{nm}r_{nm}^2)
\\
&=
4(1+\log4)\mathfrak V_{nm}.
\end{align*}
A $(\delta/2)$-cover of the larger truncated class
\textcolor{Maroon}{$\mathcal F_{\mathcal A_{nm}}(L_{nm},r_{nm}^{+})$}
yields a $\delta$-cover of the smaller class
\textcolor{Maroon}{$\mathcal F_{\mathcal A_{nm}}(L_{nm},r_{nm})$}:
from each covering ball that intersects
\textcolor{Maroon}{$\mathcal F_{\mathcal A_{nm}}(L_{nm},r_{nm})$},
select one point of that intersection.
The triangle inequality bounds the distance to the selected point
by $\delta$. Thus applying the first bound of
\textcolor{Maroon}{Lemma~\ref{lemma:threshold space entropy}} to the
padded class
\textcolor{Maroon}{$\mathcal F_{\mathcal A_{nm}}(L_{nm},r_{nm}^{+})$},
at radius $\delta/2$, gives the following bound
\textcolor{Maroon}{\eqref{eq:explicit_architecture_entropy_appendix}} also
when $L_{nm}=r_{nm}=1$.
}

\color{Black}
The covering-number bound in
Lemma~\ref{lemma:threshold space entropy} is uniform over the design
points. Thus, for any empirical norm based on $N$ design points and
every $0<\delta\leq1/2$,
\begin{align}
&\sup_{z_1,\ldots,z_N\in[0,1]^d} \log\mathcal N\!\bigl( \delta, \mathcal F_{\mathcal A_{nm}}(L_{nm},r_{nm}), \|\cdot\|_N \bigr) \lesssim \mathfrak V_{nm} \log\!\bigl(e\mathcal A_{nm}\delta^{-1}\bigr),\label{eq:explicit_architecture_entropy_appendix}
\end{align}
where
$\|g\|_N^2:=N^{-1}\sum_{k=1}^N g^2(z_k)$.

\color{Maroon}
{
\medskip
\noindent
{\bf Step 2: the weighted empirical norms.}
Let $c_m,C_m>0$ satisfy $c_m m\le m_i\le C_m m$,
and set $\kappa=(C_m/c_m)^{1/2}\ge1$.
Fix a nonempty $I\subseteq[n]$ and a deterministic array
$z=(z_{ij}:i\in I,\ j\in[m_i])$.
Define
\begin{align*}
&M_I:=\sum_{i\in I}m_i,
\qquad
\|g\|_{M_I,z}^2
:=\frac1{M_I}\sum_{i\in I}\sum_{j=1}^{m_i}g(z_{ij})^2,
\end{align*}
and
\begin{align*}
&\|g\|_{I,m,z}^2
:=\frac1{|I|}\sum_{i\in I}\frac1{m_i}
\sum_{j=1}^{m_i}g(z_{ij})^2.
\end{align*}
Since $M_I\le C_m|I|m$, for every $g$,
\begin{align}
\label{eq:explicit_architecture_weighted_norm_appendix}
\|g\|_{I,m,z}^2
&\le
\frac1{c_m|I|m}
\sum_{i\in I}\sum_{j=1}^{m_i}g(z_{ij})^2
\nonumber\\
&=
\frac{M_I}{c_m|I|m}\|g\|_{M_I,z}^2
\le
\kappa^2\|g\|_{M_I,z}^2.
\end{align}
The same inequality applies to differences of functions.
A $(\delta/\kappa)$-cover in $\|\cdot\|_{M_I,z}$ is therefore
a $\delta$-cover in $\|\cdot\|_{I,m,z}$.
Applying \eqref{eq:explicit_architecture_entropy_appendix} with
$N=M_I$ and radius $\delta/\kappa$, we obtain, for some $C_1>0$,
\begin{align}
\label{eq:explicit_architecture_weighted_entropy_appendix}
&\sup_z\log\mathcal N\!\left(
\delta,\mathcal F_{\mathcal A_{nm}}(L_{nm},r_{nm}),
\|\cdot\|_{I,m,z}\right)
\nonumber\\
&\qquad\le
C_1\mathfrak V_{nm}
\log\!\left(e\kappa\mathcal A_{nm}\delta^{-1}\right),
\qquad 0<\delta\le1/2.
\end{align}
The constant $C_1$ is independent of $I$ and $z$.

Choose $C_A\ge1$ such that
$\mathcal A_{nm}\le C_A\sqrt{\log(nm)}$ for all sufficiently
large $n$.
For $0<\delta\le1/2$, we have $\log(\delta^{-1})\ge\log2$,
and hence
\begin{align*}
\frac{\log(e\kappa\mathcal A_{nm}\delta^{-1})}
     {\log(\delta^{-1})}
&=
1+\frac{\log(e\kappa\mathcal A_{nm})}
        {\log(\delta^{-1})}\\
&\le
1+\frac{1+\log(\kappa C_A)+\tfrac12\log\log(nm)}{\log2}\\
&\le
C_2\log^{1.3}(nm),
\end{align*}
where one may take
\begin{align*}
&C_2=1+\frac{3/2+\log(\kappa C_A)}{\log2}.
\end{align*}
The last inequality uses $\log(nm)\ge1$ and
$\log\log(nm)\le\log(nm)\le\log^{1.3}(nm)$.
In particular,
}
\begin{align}
&\log\!\bigl(e\mathcal A_{nm}\delta^{-1}\bigr) \lesssim \log^{1.3}(nm)\log(\delta^{-1}), \qquad 0<\delta\leq1/2.\label{eq:explicit_architecture_entropy_envelope_appendix}
\end{align}
{
Combining the weighted covering bound with the preceding calculation,
and writing $C_3=C_1C_2$, gives
\begin{align}
\label{eq:explicit_architecture_uniform_weighted_entropy}
&\sup_z\log\mathcal N\!\left(
\delta,\mathcal F_{\mathcal A_{nm}}(L_{nm},r_{nm}),
\|\cdot\|_{I,m,z}\right)
\nonumber\\
&\qquad\le
C_3\mathfrak V_{nm}\log^{1.3}(nm)\log(\delta^{-1}),
\qquad 0<\delta\le1/2,
\end{align}
uniformly over all nonempty $I\subseteq[n]$.

\medskip
\noindent
{\bf Step 3: the blockwise and global entropy rates.}
Let
\begin{align*}
&\mathfrak J_n := \{\mathcal J_{e,s},\mathcal J_{o,s}:0\le s <S_n\} \setminus\{\varnothing\}, \qquad K_n=\lfloor n/S_n\rfloor.
\end{align*}
For $K_n\ge2$, the block construction and
\eqref{eq:effective sample size in the temporal-spatial model}
give, for every $J\in\mathfrak J_n$,
\begin{align*}
&|J|\ge\frac{K_n-1}{2}
\ge\frac{K_n+1}{6}
>\frac{n}{6S_n}.
\end{align*}
If $K_n=1$, every nonempty block satisfies $|J|\ge1$, whereas
$n/S_n<2$; thus $|J|\ge n/(6S_n)$ in this case as well.
Also,
\begin{align*}
&|J|\le\frac{K_n+2}{2}
\le\frac{n/S_n+2}{2}
\le\frac{3n}{2S_n}.
\end{align*}
Consequently,
\begin{align}
\label{eq:explicit_architecture_block_cardinality_appendix}
\frac{n}{6S_n}\le |J|\le\frac{3n}{2S_n},
\qquad J\in\mathfrak J_n.
\end{align}

For all sufficiently large $n$, define
\begin{align}
&\rho_n := 6C_3\frac{S_n\mathfrak V_{nm}}{nm}\log^{1.3}(nm), \widetilde\rho_n := C_3\frac{\mathfrak V_{nm}}{nm}\log^{1.3}(nm).\label{eq:explicit_architecture_entropy_rates_choice}
\end{align}
At the finitely many remaining indices, set both sequences equal
to one. The resulting sequences are positive and deterministic.
For every $J\in\mathfrak J_n$ and $0<\delta\le1/2$,
\eqref{eq:explicit_architecture_uniform_weighted_entropy} and
\eqref{eq:explicit_architecture_block_cardinality_appendix} imply
\begin{align*}
&\sup_z
\frac{
\log\mathcal N\!\left(
\delta,\mathcal F_{\mathcal A_{nm}}(L_{nm},r_{nm}),
\|\cdot\|_{J,m,z}\right)
}{|J|m\log(\delta^{-1})}
\\
&\qquad\le C_3\frac{\mathfrak V_{nm}}{|J|m}\log^{1.3}(nm) \le 6C_3\frac{S_n\mathfrak V_{nm}}{nm}\log^{1.3}(nm) =\rho_n.
\end{align*}
Taking the supremum over $0<\delta\le1/2$ and the maximum
over the nonempty even and odd blocks verifies both blockwise
entropy conditions \textcolor{Maroon}{in \eqref{eq:block_entropy_condition}}.

For $I=[n]$, the same weighted covering bound gives
\begin{align*}
&\sup_{0<\delta\le1/2}\sup_z
\frac{
\log\mathcal N\!\left(
\delta,\mathcal F_{\mathcal A_{nm}}(L_{nm},r_{nm}),
\|\cdot\|_{[n],m,z}\right)
}{nm\log(\delta^{-1})} \le
C_3\frac{\mathfrak V_{nm}}{nm}\log^{1.3}(nm)
=\widetilde\rho_n.
\end{align*}
This verifies the global entropy condition.
}
This proves \eqref{eq:explicit_architecture_entropy_rates_appendix}.

\color{Black}
{
\medskip
\noindent
{\bf Step 4: application of the general function-class bound.}
\par}
Substituting these entropy rates into
\eqref{eq:general_risk_bound_temporal_spatial_appendix} gives
with probability approaching one
\begin{align}
&\|\widehat f_{\mathcal A_{nm}}-f^*\|_{\mathcal L_2}^2
\lesssim\;
\varepsilon_{nm}^2
+
a_n\mathcal A_{nm}^2
\frac{S_n\mathfrak V_{nm}}{nm}
\log^{2.3}(nm)
\nonumber\\
&+
a_n\mathcal A_{nm}^2
\frac{\mathfrak V_{nm}}{nm}
\log^{2.3}(nm)
+
a_n\mathcal A_{nm}^2\beta(S_n) +
\frac{\mathcal A_{nm}^2}{n}
+
\frac{\sigma_\gamma^2}{n}.
\end{align}
Because $S_n\geq1$, the global entropy contribution is bounded by
the blockwise contribution.
{
Indeed,
\begin{align*}
&a_n\mathcal A_{nm}^2 \frac{(S_n+1)\mathfrak V_{nm}}{nm}\log^{2.3}(nm) \le 2a_n\mathcal A_{nm}^2 \frac{S_n\mathfrak V_{nm}}{nm}\log^{2.3}(nm).
\end{align*}
}
Therefore,
\eqref{eq:explicit_architecture_temporal_spatial_appendix} follows.
\end{proof}

\endgroup

\color{Black}

\newpage
\section{Approximate Manifold Support under Temporal-Spatial Dependence}
\label{appendix:approx_manifold_temporal_spatial}

In this section, we extend the temporal--spatial risk bound to the case where the design is supported on a tubular neighborhood of a lower-dimensional manifold. The argument is an application of Theorem~\ref{cor:general_architecture_temporal_spatial}: once an approximation bound is established on the support of the design, the temporal--spatial empirical-process and dependence terms are handled exactly as in the general architecture regime.

The only new ingredient is the approximation lemma below, which adapts the geometric approximation step of Theorem~6.1 of \cite{jiao2023deep} to the present setting. Their argument combines three ingredients: a Johnson--Lindenstrauss-type dimension reduction for the manifold, approximation of the reduced function via Theorem~3.3 and Corollary~3.1, and a transfer step from the manifold to its \(\rho\)-neighborhood. See Theorem~6.1 and its proof in the supplementary material of \cite{jiao2023deep} for details.

\begin{lemma}[Approximation under approximate manifold support]
\label{lem:approximation_approx_manifold_temporal_spatial}
Let \(\mathcal M\subset[0,1]^d\) be a compact \(d_M\)-dimensional Riemannian manifold and define
\begin{align*}
&\mathcal M_\tau
:=
\left\{
x\in[0,1]^d:\inf_{y\in\mathcal M}\|x-y\|\le \tau
\right\}.
\end{align*}
Assume that \(f^*:[0,1]^d\to\mathbb R\) is \((p,C)\)-smooth, with \(p=q+s\), \(q\in\mathbf N\), \(s\in(0,1]\), and that all partial derivatives of order at most \(q\) are uniformly bounded.

Fix \(\delta\in(0,1)\). Let \(d_\delta\in \mathbf Z^+\) satisfy
\begin{align*}
&d_M \le d_\delta < d,
\qquad
d_\delta \lesssim \frac{d_M\log(d/\delta)}{\delta^2}.
\end{align*}
Set
\begin{align*}
&\varepsilon_{nm,\delta}
:=
(nm)^{-\frac{p}{2p+d_\delta}}.
\end{align*}
Assume that \(\rho\) satisfies
\begin{align}
&\tau \le c_{p,C,d,d_M,\delta}\,
\varepsilon_{nm,\delta}^{\,1/(p\wedge 1)}.\label{eq:rho_condition_approx_manifold_temporal_spatial}
\end{align}
Then there exists \(\overline f \in \mathcal F(L,r)\), with either
\begin{align*}
&L\asymp \log(nm),
\qquad
r\asymp (nm)^{\frac{d_\delta}{2(2p+d_\delta)}},
\end{align*}
or
\begin{align*}
&L\asymp \log(nm)\,(nm)^{\frac{d_\delta}{2(2p+d_\delta)}},
\qquad
r=O(1),
\end{align*}
such that
\begin{align}
\label{eq:approx_error_approx_manifold_temporal_spatial}
\|f^*-\overline f\|_{\infty,\mathcal M_\tau}^2
\le
C_{p,C,d,d_M,\delta}\,\varepsilon_{nm,\delta}^2,
\end{align}
where 
\(
\|f\|_{\infty,\mathcal M_\tau}:=\sup_{x\in \mathcal M_\tau}|f(x)|
\). The constants \(c_{p,C,d,d_M,\delta}\) and \(C_{p,C,d,d_M,\delta}\) do not depend on \(n\), \(m\), or \(\tau\).
\end{lemma}

\begin{proof}
The proof follows a projection–approximation–extension scheme.

\medskip
\noindent

\medskip
\noindent
\textbf{Step 1: Dimension reduction on \(\mathcal M\).}

By the dimension-reduction construction in the proof of Theorem~6.1 of
\cite{jiao2023deep} (Appendix A.7 of the supplementary material), there exists
a linear map
\begin{align*}
&A:\mathbb R^d\to\mathbb R^{d_\delta},
\qquad
d_\delta \lesssim \frac{d_M\log(d/\delta)}{\delta^2},
\end{align*}
such that
\begin{align*}
&AA^\top=\frac{d}{d_\delta}I_{d_\delta}
\end{align*}
and
\begin{align*}
&(1-\delta)\|x_1-x_2\| \le \|Ax_1-Ax_2\| \le (1+\delta)\|x_1-x_2\|, \qquad x_1,x_2\in\mathcal M.
\end{align*}
In particular, the restriction of \(A\) to \(\mathcal M\) is injective.

Moreover, since \(\mathcal M_\tau\subset [0,1]^d\), we have
\begin{align*}
&A(\mathcal M_\tau)\subset A([0,1]^d)
\subset
E
:=
\Bigl[-\sqrt{d/d_\delta},\,\sqrt{d/d_\delta}\Bigr]^{d_\delta},
\end{align*}
exactly as in Appendix~A.7 of \cite{jiao2023deep}.

For each \(z\in A(\mathcal M)\), there exists a unique \(x_z\in\mathcal M\) such that
\(Ax_z=z\). Indeed, if \(x,x'\in\mathcal M\) satisfy \(Ax=Ax'=z\), then
\begin{align*}
&(1-\delta)\|x-x'\|
\le
\|Ax-Ax'\|
=
0,
\end{align*}
which implies \(x=x'\).

We therefore define a map
\begin{align*}
&T:A(\mathcal M)\to\mathcal M,
\qquad
T(z):=x_z.
\end{align*}
By the construction in Appendix~A.7 (pp.~9--10) of the supplementary material of \cite{jiao2023deep}, the map \(T\) is differentiable on \(A(\mathcal M)\), and the norm of its derivative lies in
\begin{align*}
&\left[\frac{1}{1+\delta}, \frac{1}{1-\delta}\right].
\end{align*}

Define the reduced function
\begin{align*}
&\widetilde f:A(\mathcal M)\to\mathbb R,
\qquad
\widetilde f(z):=f^*(T(z)).
\end{align*}
Since \(f^*\) is \((p,C)\)-smooth on \([0,1]^d\), it follows, as in Appendix~A.7
of \cite{jiao2023deep}, that
\begin{align*}
&\widetilde f \text{ is } (p,\,C(1-\delta)^{-p})\text{-smooth on } A(\mathcal M).
\end{align*}

Since \(\mathcal M\) is compact and \(A\) is continuous, the set \(A(\mathcal M)\) is compact,
hence closed in \(\mathbb R^{d_\delta}\). Moreover, \(A(\mathcal M)\subset E\).
Therefore, by the extended version of Whitney's extension theorem in
\cite{fefferman2006whitney}, there exists a function
\begin{align*}
&\widetilde F:E\to\mathbb R
\end{align*}
such that
\begin{align*}
&\widetilde F(z)=\widetilde f(z),\qquad z\in A(\mathcal M),
\end{align*}
and \(\widetilde F\) is \((p,C_{\mathrm{ext}}\,C(1-\delta)^{-p})\)-smooth on \(E\),
where \(C_{\mathrm{ext}}>0\) depends only on \(p\) and \(d_\delta\).

\medskip
\textbf{Step 2: Approximation in reduced dimension.}

Let \(\Phi:E\to[0,1]^{d_\delta}\) be an affine bijection. Since \(\widetilde F\) is \((p,C_{\mathrm{ext}}C(1-\delta)^{-p})\)-smooth on \(E\), it follows that \(\widetilde F\circ\Phi^{-1}\) is \((p,\widetilde C_{p,C_{\mathrm{ext}},C,d,d_\delta,\delta})\)-smooth on \([0,1]^{d_\delta}\). Therefore, by Corollary~3.1 of \cite{jiao2023deep}, there exists a ReLU network \(h_0\) such that
\begin{align*}
&\|\widetilde F\circ \Phi^{-1} - h_0\|_{\infty,[0,1]^{d_\delta}}
\le C_{p,C_{\mathrm{ext}},C,d,d_\delta,\delta}\,(nm)^{-\frac{p}{2p+d_\delta}}.
\end{align*}
Defining \(h:=h_0\circ \Phi\), we have for every \(x\in E\),
\begin{align*}
&(\widetilde F-h)(x)
=
(\widetilde F\circ \Phi^{-1})(\Phi(x)) - h_0(\Phi(x)).
\end{align*}
Since \(\Phi\) is a bijection from \(E\) onto \([0,1]^{d_\delta}\), it follows that
\begin{align*}
&\|\widetilde F-h\|_{\infty,E} = \|\widetilde F\circ \Phi^{-1}-h_0\|_{\infty,[0,1]^{d_\delta}} \le C_{p,C_{\mathrm{ext}},C,d,d_\delta,\delta}\,(nm)^{-\frac{p}{2p+d_\delta}}.
\end{align*}

In particular,
\begin{align*}
&\|\widetilde f-h\|_{\infty,A(\mathcal M)}
\le C_{p,C_{\mathrm{ext}},C,d,d_\delta,\delta}\,\varepsilon_{nm,\delta}.
\end{align*}

Define
\begin{align*}
&\overline f(x):=h(Ax),\qquad x\in[0,1]^d.
\end{align*}
Since \(A\) is linear, the mapping \(x\mapsto Ax\) can be absorbed into the first
layer of the network, and hence \(\overline f\in \mathcal F(L,r)\)
with the same order of depth and width.
In fact, observe that if \(h\in \mathcal F(L,r)\), then by Definition~1 in the Main Manuscript, 
\begin{align*}
&h(z)=\sum_{i=1}^{r} c_{1,i}^{(L)} f_i^{(L)}(z)+c_{1,0}^{(L)},
\end{align*} 
where the functions \(f_i^{(s)}\) 
satisfy the recursive relations in Equations (4)--(5), in the Main Manuscript. For the first layer, 
\begin{align*} 
f_i^{(1)}(Ax) &= \rho\left( \sum_{j=1}^{d_\delta} c_{i,j}^{(0)} (Ax)^{(j)} + c_{i,0}^{(0)} \right) \\ &= \rho\left( \sum_{k=1}^{d} \left(\sum_{j=1}^{d_\delta} c_{i,j}^{(0)} A_{j,k}\right)x^{(k)} + c_{i,0}^{(0)} \right). 
\end{align*} 
Thus, after redefining the first-layer coefficients by 
\begin{align*}
&\widetilde c_{i,k}^{(0)}:=\sum_{j=1}^{d_\delta} c_{i,j}^{(0)}A_{j,k},
\end{align*} 
the composition \(x\mapsto h(Ax)\) has the same recursive form as Equations (4)--(5), in the Main Manuscript, with the same number of layers \(L\) and the same width \(r\). 
Hence \begin{align*}
&\overline f\in \mathcal F(L,r).
\end{align*}
\medskip
\noindent
\textbf{Step 3: Extension to the tubular neighborhood.}

Let \(x\in \mathcal M_\tau\). Then, by definition of \(\mathcal M_\tau\), there exists
\(y_x\in \mathcal M\) such that
\begin{align*}
&\|x-y_x\|\le \tau.
\end{align*}
Since \(\overline f(x)=h(Ax)\), we may write
\begin{align*}
&f^*(x)-\overline f(x)
=
\bigl(f^*(x)-f^*(y_x)\bigr)+\bigl(f^*(y_x)-h(Ax)\bigr).
\end{align*}
Therefore,
\begin{align*}
&|f^*(x)-\overline f(x)|
\le
|f^*(x)-f^*(y_x)|
+
|f^*(y_x)-h(Ax)|.
\end{align*}

\medskip
\noindent
\emph{First term.}
By the \((p,C)\)-smoothness of \(f^*\),
\begin{align*}
&|f^*(x)-f^*(y_x)|
\le
C\,\|x-y_x\|^{p\wedge 1}
\le
C\,\tau^{p\wedge 1}.
\end{align*}

\medskip
\noindent
\emph{Second term.}
Since \(y_x\in\mathcal M\), we have \(Ay_x\in A(\mathcal M)\). By definition of
\(\widetilde f\) and the inverse map \(T\),
\begin{align*}
&\widetilde f(Ay_x)=f^*(T(Ay_x))=f^*(y_x).
\end{align*}
Since \(\widetilde F=\widetilde f\) on \(A(\mathcal M)\), it follows that
\begin{align*}
&f^*(y_x)=\widetilde F(Ay_x).
\end{align*}
Hence
\begin{align*}
&|f^*(y_x)-h(Ax)| = |\widetilde F(Ay_x)-h(Ax)| \le |\widetilde F(Ay_x)-\widetilde F(Ax)| + |\widetilde F(Ax)-h(Ax)|.
\end{align*}

The second term is bounded by
\begin{align*}
&|\widetilde F(Ax)-h(Ax)|
\le
\|\widetilde F-h\|_{\infty,E}
\lesssim
\varepsilon_{nm,\delta},
\end{align*}
since \(x\in\mathcal M_\tau\subset[0,1]^d\) implies \(Ax\in A(\mathcal M_\tau)\subset E\).

For the first term, since \(AA^\top=\frac{d}{d_\delta}I_{d_\delta}\) in {\bf{Step 1}}, we have
\begin{align*}
&\|A\|_{\mathrm{op}}=\sqrt{\frac d{d_\delta}}.
\end{align*}
Therefore,
\begin{align*}
&\|Ay_x-Ax\|
\le
\|A\|_{\mathrm{op}}\,\|y_x-x\|
\le
\sqrt{\frac d{d_\delta}}\,\tau.
\end{align*}
Using that \(\widetilde F\) is \((p,C_{\mathrm{ext}}\,C(1-\delta)^{-p})\)-smooth on \(E\),
we obtain
\begin{align*}
&|\widetilde F(Ay_x)-\widetilde F(Ax)| \le C_{\mathrm{ext}}\,C(1-\delta)^{-p} \|Ay_x-Ax\|^{p\wedge 1} \le C_{\mathrm{ext}}\,C(1-\delta)^{-p} \bigl(\sqrt{\frac d{d_\delta}}\,\tau\bigr)^{p\wedge 1}.
\end{align*}

\medskip
\noindent
Combining the preceding bounds yields
\begin{align*}
&|f^*(x)-\overline f(x)| \le C\,\tau^{p\wedge 1} + C_{\mathrm{ext}}\,C(1-\delta)^{-p} \bigl(\sqrt{\frac d{d_\delta}}\,\tau\bigr)^{p\wedge 1} + \varepsilon_{nm,\delta}.
\end{align*}
Define
\begin{align*}
&C_1
:=
C
+
C_{\mathrm{ext}}\,C(1-\delta)^{-p}
\left(\sqrt{\tfrac d{d_\delta}}\right)^{p\wedge 1}.
\end{align*}
Then, taking the supremum over \(x\in \mathcal M_\tau\), we obtain
\begin{align*}
&\|f^*-\overline f\|_{\infty,\mathcal M_\rho}
\le
C_1\,\tau^{p\wedge 1}
+
C_2\,\varepsilon_{nm,\delta},
\end{align*}
for some positive constant \(C_2 = C_{p,C_{\mathrm{ext}},d,d_\delta,\delta}\).
Under condition \eqref{eq:rho_condition_approx_manifold_temporal_spatial},
\begin{align*}
&\tau \le c_{p,C,d,d_M,\delta}\,\varepsilon_{nm,\delta}^{1/(p\wedge 1)},
\end{align*}
we obtain
\begin{align*}
&C_1 \tau^{p\wedge 1}
\le
C_1 c_{p,C,d,d_M,\delta}^{\,p\wedge 1}\,\varepsilon_{nm,\delta}.
\end{align*}
Choosing \(c_{p,C,d,d_M,\delta}\) sufficiently small (depending only on
\(C_1\)), we ensure that
\begin{align*}
&C_1 \tau^{p\wedge 1} \le C_3 \varepsilon_{nm,\delta}.
\end{align*}
Therefore,
\begin{align*}
&\|f^*-\overline f\|_{\infty,\mathcal M_\tau}
\le
(C_2 + C_3)\,\varepsilon_{nm,\delta}.
\end{align*}
Squaring both sides yields \eqref{eq:approx_error_approx_manifold_temporal_spatial}.
\end{proof}

\begin{theorem}[Temporal-spatial risk bound under approximate manifold support]
\label{thm:appendix_approx_manifold_temporal_spatial}
Suppose that Assumptions~1, 2, 3{\bf a}--{\bf b}, and 4 hold. Assume moreover that the design points are supported on \(\mathcal M_\tau\), where \(\mathcal M\subset[0,1]^d\) is a compact \(d_M\)-dimensional manifold, and that \(f^*\) satisfies the smoothness conditions of Lemma~\ref{lem:approximation_approx_manifold_temporal_spatial}.

Fix \(\delta\in(0,1)\), let \(d_\delta\) be as in Lemma~\ref{lem:approximation_approx_manifold_temporal_spatial}, and assume \eqref{eq:rho_condition_approx_manifold_temporal_spatial}. Let \(\widehat f\) be the estimator in Equation (2), with \((L,r)\) chosen as in Lemma~\ref{lem:approximation_approx_manifold_temporal_spatial}, and let
\begin{align*}
&\mathcal A_{nm}\asymp \max\{\sigma_\epsilon,\sigma_\gamma\}\sqrt{\log(nm)}
\end{align*}
and \((a_n)_{n\ge 1}\) be any sequence such that \(a_n \to \infty\).
Then, let \((S_n)_{n\ge1}\) be any sequence satisfying
\begin{align*}
&C_{\beta}\log(n) \le S_n \le n,
\end{align*}
for a sufficiently large constant \(C_{\beta}>0\) depending only on the mixing rate in Assumption~4.
Suppose that the entropy conditions of Theorem~\ref{cor:general_architecture_temporal_spatial} hold with sequences \((\rho_n)\) and \((\widetilde\rho_n)\). Then, with probability approaching one,
\begin{align}
&\|\widehat f_{\mathcal A_{nm}}-f^*\|_{\mathcal L_2}^2 \lesssim (nm)^{-\frac{2p}{2p+d_\delta}} + a_n\mathcal A_{nm}^2\,\rho_n\,\log(nm) \notag \\
&\quad + a_n\mathcal A_{nm}^2\,\widetilde\rho_n\,\log(nm) + a_n\mathcal A_{nm}^2\beta(S_n) + \frac{\sigma_\gamma^2}{n}.\label{eq:appendix_approx_manifold_temporal_spatial_rate}
\end{align}

Using the canonical choice \((L,r)\) from Lemma~\ref{lem:approximation_approx_manifold_temporal_spatial} together with Lemma~\ref{lemma:threshold space entropy}, we may take
\begin{align*}
&\rho_n \lesssim \bigl(\frac{S_n}{nm}\bigr)^{\frac{2p}{2p+d_\delta}}\log^{4.3}(nm), \qquad \widetilde\rho_n \lesssim (nm)^{-\frac{2p}{2p+d_\delta}}\log^{4.3}(nm).
\end{align*}
Therefore,
\begin{align}
&\|\widehat f_{\mathcal A_{nm}}-f^*\|_{\mathcal L_2}^2 \lesssim\; (nm)^{-\frac{2p}{2p+d_\delta}} + a_n\mathcal A_{nm}^2 \bigl(\frac{S_n}{nm}\bigr)^{\frac{2p}{2p+d_\delta}} \log^{5.3}(nm) \notag \\
&\quad + a_n\mathcal A_{nm}^2 (nm)^{-\frac{2p}{2p+d_\delta}} \log^{5.3}(nm)
\\
&+
a_n\mathcal A_{nm}^2\beta(S_n)
+
\frac{\sigma_\gamma^2}{n}.
\end{align}

Moreover choosing \(S_n\asymp \log n\), we obtain
\begin{align}
&\|\widehat f_{\mathcal A_{nm}}-f^*\|_{\mathcal L_2}^2 \lesssim\; a_n(\sigma_\epsilon^2+\sigma_\gamma^2+1) (nm)^{-\frac{2p}{2p+d_\delta}} \log^{7.3}(nm) \notag \\
&\quad + a_n\frac{(\sigma_\epsilon^2+\sigma_\gamma^2)\log(nm)}{n} + \frac{\sigma_\gamma^2}{n},\label{eq:appendix_approx_manifold_temporal_spatial_rate_final1}
\end{align}
with probability approaching one.
\end{theorem}

\begin{proof}
The proof follows the same line of argument as the proof of Theorem~\ref{cor:general_architecture_temporal_spatial}, with the approximation input supplied by Lemma~\ref{lem:approximation_approx_manifold_temporal_spatial}. The only additional point is that the approximation bound from Lemma~\ref{lem:approximation_approx_manifold_temporal_spatial} is available only on the support of the design, namely \(\mathcal M_\tau\). Accordingly, we work with \(\overline f_{\mathcal A_{nm}}\) throughout.

By Lemma~\ref{lem:approximation_approx_manifold_temporal_spatial}, there exists \(\overline f\in\mathcal F(L,r)\) such that
\begin{align*}
&\|f^*-\overline f\|_{\infty,\mathcal M_\tau}^2
\lesssim
\varepsilon_{nm,\delta}^2,
\qquad
\varepsilon_{nm,\delta}^2
:=
(nm)^{-\frac{2p}{2p+d_\delta}}.
\end{align*}
Since the support of the design is contained in \(\mathcal M_\tau\), it follows that for every observed design point \(x_{ij}\),
\begin{align*}
&|f^*(x_{ij})-\overline f(x_{ij})|^2
\le
\|f^*-\overline f\|_{\infty,\mathcal M_\tau}^2
\lesssim
\varepsilon_{nm,\delta}^2.
\end{align*}
Moreover, since \(\|f^*\|_\infty\le \mathcal A_{nm}/4\) by Assumption~3{\bf b}, and \(\varepsilon_{nm,\delta}\to0\), for all sufficiently large \(nm\),
\begin{align*}
&|\overline f(x_{ij})| \le |f^*(x_{ij})|+|f^*(x_{ij})-\overline f(x_{ij})| \le \frac{\mathcal A_{nm}}{4}+\varepsilon_{nm,\delta} \le \frac{\mathcal A_{nm}}{2}.
\end{align*}
Hence, for all sufficiently large \(n\),
\begin{align*}
&\overline f_{\mathcal A_{nm}}(x_{ij})=\overline f(x_{ij})
\qquad
\text{for all }1\le i\le n,\ 1\le j\le m_i.
\end{align*}
In particular,
\begin{align}
&\|f^*-\overline f_{\mathcal A_{nm}}\|_{nm}^2
\lesssim
\varepsilon_{nm,\delta}^2.\label{eq:approx_manifold_empirical_approx}
\end{align}

Define
\begin{align*}
&\Delta_{nm}(x):=\widehat f_{\mathcal A_{nm}}(x)-\overline f_{\mathcal A_{nm}}(x),
\end{align*}
and the shifted truncated class
\begin{align*}
&\mathcal G_{\mathcal A_{nm}}(\overline f)
:=
\bigl\{
f_{\mathcal A_{nm}}-\overline f_{\mathcal A_{nm}}
:\,
f\in \mathcal F(L,r)
\bigr\}.
\end{align*}
For every \(0<\delta\le1\), every \(h_1,h_2\in \mathcal F_{\mathcal A_{nm}}(L,r)\), and every empirical norm appearing in the proof of Theorem~\ref{cor:general_architecture_temporal_spatial}, translation by \(-\overline f_{\mathcal A_{nm}}\) is an isometry. Therefore the blockwise entropy condition and the global entropy condition carry over to the shifted class \(\mathcal G_{\mathcal A_{nm}}(\overline f)\) with the same sequences \((\rho_n)\) and \((\widetilde\rho_n)\).

Let
\begin{align*}
&\mathcal E = \bigl\{ |\epsilon_{ij}|\le \frac{\mathcal A_{nm}}{4}\,\text{for all }1\le i\le n,\, 1\le j\le m_i, \quad \|\gamma_i\|_{\infty,nm} \le \frac{\mathcal A_{nm}}{4}\,\text{for all }1\le i\le n \bigr\}.
\end{align*}
By Lemma~\ref{lemma_epsilon_gamma_infinity_bound},
\begin{align*}
&\mathbb P(\mathcal E^c)=o(1).
\end{align*}

Using
\begin{align*}
&\widehat f_{\mathcal A_{nm}}-f^*
=
\Delta_{nm}+(\overline f_{\mathcal A_{nm}}-f^*),
\end{align*}
we obtain
\begin{align}
\label{eq:approx_manifold_basic_reduction}
\|\widehat f_{\mathcal A_{nm}}-f^*\|_{\mathcal L_2}^2
&\le
2\|\overline f_{\mathcal A_{nm}}-f^*\|_{\mathcal L_2}^2
+
2\|\Delta_{nm}\|_{\mathcal L_2}^2
\nonumber\\
&\le
2\varepsilon_{nm,\delta}^2
+
2\|\Delta_{nm}\|_{\mathcal L_2}^2,
\end{align}
where we used \eqref{eq:approx_manifold_empirical_approx} together with the fact that the approximation error is controlled on the support of the design.

It therefore suffices to control \(\|\Delta_{nm}\|_{\mathcal L_2}^2\).

On the event \(\mathcal E\), exactly as in the proof of Theorem~\ref{cor:general_architecture_temporal_spatial}, the empirical basic inequality gives
\begin{align}
&\frac12\|\Delta_{nm}\|_{nm}^2 \le \sum_{i=1}^n\sum_{j=1}^{m_i}\frac{1}{nm_i}\Delta_{nm}(x_{ij}) \bigl\{f^*(x_{ij})-\overline f_{\mathcal A_{nm}}(x_{ij})\bigr\} \notag \\
&\quad + \sum_{i=1}^n\sum_{j=1}^{m_i}\frac{1}{nm_i}\Delta_{nm}(x_{ij}) \bigl\{\gamma_i(x_{ij})+\epsilon_{ij}\bigr\}.\label{eq:approx_manifold_basic_ineq}
\end{align}

We first control the approximation term. By Cauchy--Schwarz and \eqref{eq:approx_manifold_empirical_approx},
\begin{align}
&\sum_{i=1}^n\sum_{j=1}^{m_i}\frac{1}{nm_i}\Delta_{nm}(x_{ij}) \bigl\{f^*(x_{ij})-\overline f_{\mathcal A_{nm}}(x_{ij})\bigr\} \le \|\Delta_{nm}\|_{nm}\,\|f^*-\overline f_{\mathcal A_{nm}}\|_{nm} \nonumber\label{eq:approx_manifold_approx_term}
\\
&\le
\|\Delta_{nm}\|_{nm}\,\varepsilon_{nm,\delta}
\nonumber
\\
&\le
\frac{1}{32}\|\Delta_{nm}\|_{nm}^2 + C\,\varepsilon_{nm,\delta}^2.
\end{align}

\medskip
\noindent
{\bf Step 1: the \(\epsilon\)-term.}
Since \(\Delta_{nm}\in \mathcal G_{\mathcal A_{nm}}(\overline f)\), the blockwise entropy condition implies that the proof of
\Cref{coro:complexity bound 2 beta mn}
can be repeated block by block for the shifted class \(\mathcal G_{\mathcal A_{nm}}(\overline f)\), with the quantity \(\phi_{|\mathcal J_{\cdot,s}|m}\log^{4.3}(nm)\) replaced by \(\rho_n\). Thus, arguing exactly as in Step~1 of the proof of Theorem~\ref{cor:general_architecture_temporal_spatial}, we obtain that with probability tending to one,
\begin{align}
\label{eq:approx_manifold_epsilon_lambda}
\sum_{i=1}^n\sum_{j=1}^{m_i}\frac{1}{nm_i}\Delta_{nm}(x_{ij})\epsilon_{ij}
\le
C\,\sigma_\epsilon\,
\Lambda^{(\rho)}_{n,S_n}(\Delta_{nm}),
\end{align}
where \(\Lambda^{(\rho)}_{n,S_n}\) is defined exactly as \(\Lambda_{n,S}\) in
\Cref{coro:complexity bound 2 beta mn}, but with \(\phi_{|\mathcal J_{\cdot,s}|m}\log^{4.3}(nm)\) replaced by \(\rho_n\). Repeating the same Jensen reduction as in the proof of Theorem~\ref{cor:general_architecture_temporal_spatial} gives
\begin{align}
\label{eq:approx_manifold_lambda_reduction}
\Lambda^{(\rho)}_{n,S_n}(\Delta_{nm})
\le
C\,
\|\Delta_{nm}\|_{nm}
\sqrt{
\rho_n\,
\log\!\left(\frac{2\mathcal A_{nm}}{\|\Delta_{nm}\|_{nm}}\right)
}.
\end{align}
Applying Young's inequality, we obtain
\begin{align}
&\sum_{i=1}^n\sum_{j=1}^{m_i}\frac{1}{nm_i}\Delta_{nm}(x_{ij})\epsilon_{ij} \le \frac{1}{32}\|\Delta_{nm}\|_{nm}^2 + C\,\sigma_\epsilon^2\,\rho_n\, \log\!\bigl(\frac{2\mathcal A_{nm}}{\|\Delta_{nm}\|_{nm}}\bigr).\label{eq:approx_manifold_epsilon_term}
\end{align}
Since \(\sigma_\epsilon^2\lesssim \mathcal A_{nm}^2\), this becomes
\begin{align}
&\sum_{i=1}^n\sum_{j=1}^{m_i}\frac{1}{nm_i}\Delta_{nm}(x_{ij})\epsilon_{ij} \le \frac{1}{32}\|\Delta_{nm}\|_{nm}^2 + C\,\mathcal A_{nm}^2\,\rho_n\, \log\!\bigl(\frac{2\mathcal A_{nm}}{\|\Delta_{nm}\|_{nm}}\bigr).\label{eq:approx_manifold_epsilon_term2}
\end{align}

\medskip
\noindent
{\bf Step 2: norm comparison.}
Using again the blockwise entropy condition, the proofs of
\Cref{coro:deviation for norms mn}
and
\Cref{lemma:rate in the temporal-spatial model}
carry over block by block with \(\phi_{nm/S_n}\log^{5.3}(nm/S_n)\) replaced by \(\rho_n\log(nm)\). Therefore, for any sequence \(a_n\to\infty\), arguing exactly as in the proof of Theorem~\ref{cor:general_architecture_temporal_spatial}, we obtain with probability tending to one,
\begin{align}
&\|\Delta_{nm}\|_{\mathcal L_2}^2 \le 2\|\Delta_{nm}\|_{nm}^2 + C a_n \bigl\{ \mathcal A_{nm}^2\,\rho_n\,\log(nm) + \mathcal A_{nm}^2\beta(S_n) \bigr\}.\label{eq:approx_manifold_norm_comparison}
\end{align}

\medskip
\noindent
{\bf Step 3: the \(\gamma\)-term.}
We decompose the \(\gamma\)-term exactly as in the proof of Theorem~\ref{cor:general_architecture_temporal_spatial}:
\begin{align*}
&\sum_{i=1}^n\sum_{j=1}^{m_i}\frac{1}{nm_i}\Delta_{nm}(x_{ij})\gamma_i(x_{ij})
\\
&={} \sum_{i=1}^n\sum_{j=1}^{m_i}\frac{1}{nm_i} \Bigl\{ \Delta_{nm}(x_{ij})\gamma_i(x_{ij}) - \langle \gamma_i,\Delta_{nm}\rangle_{\mathcal L_2} \Bigr\} + \frac1n\sum_{i=1}^n \langle \gamma_i,\Delta_{nm}\rangle_{\mathcal L_2}.
\end{align*}

For the second term, by Lemma~\ref{lemma:bound on the expected norm of gamma} and Young's inequality,
\begin{align}
\label{eq:approx_manifold_gamma_mean_term1}
\frac1n\sum_{i=1}^n \langle \gamma_i,\Delta_{nm}\rangle_{\mathcal L_2}
&=
\left\langle \frac1n\sum_{i=1}^n \gamma_i,\Delta_{nm}\right\rangle_{\mathcal L_2}
\nonumber\\
&\le
\left\|\frac1n\sum_{i=1}^n \gamma_i\right\|_{\mathcal L_2}\|\Delta_{nm}\|_{\mathcal L_2}
\nonumber\\
&\le
\frac{8\sigma_\gamma^2}{n}
+
\frac{1}{32}\|\Delta_{nm}\|_{\mathcal L_2}^2.
\end{align}
Combining this with \eqref{eq:approx_manifold_norm_comparison}, we obtain
\begin{align}
&\frac1n\sum_{i=1}^n \langle \gamma_i,\Delta_{nm}\rangle_{\mathcal L_2} \le \frac{8\sigma_\gamma^2}{n} + \frac{1}{16}\|\Delta_{nm}\|_{nm}^2 + C_1 a_n \bigl\{ \mathcal A_{nm}^2\,\rho_n\,\log(nm) + \mathcal A_{nm}^2\beta(S_n) \bigr\}.\label{eq:approx_manifold_gamma_mean_term2}
\end{align}

For the centered fluctuation term, we now use the global entropy condition. Since \(\Delta_{nm}\in \mathcal G_{\mathcal A_{nm}}(\overline f)\), the proof of
\Cref{coro:deviation bound for gamma with xi mn}
carries over verbatim with \(\phi_{nm}\log^{5.3}(nm)\) replaced by \(\widetilde\rho_n\log(nm)\). Hence, with probability tending to one,
\begin{align}
\label{eq:approx_manifold_gamma_fluctuation1}
&\sum_{i=1}^n\sum_{j=1}^{m_i}\frac{1}{nm_i}
\Bigl\{
\Delta_{nm}(x_{ij})\gamma_i(x_{ij})
-
\langle \gamma_i,\Delta_{nm}\rangle_{\mathcal L_2}
\Bigr\}
\nonumber\\
&\le
\frac{1}{16}\|\Delta_{nm}\|_{\mathcal L_2}^2
+
C\,\mathcal A_{nm}^2\,\widetilde\rho_n\,\log(nm).
\end{align}
Using \eqref{eq:approx_manifold_norm_comparison} once more, we get
\begin{align}
&\label{eq:approx_manifold_gamma_fluctuation2}
\sum_{i=1}^n\sum_{j=1}^{m_i}\frac{1}{nm_i}
\Bigl\{
\Delta_{nm}(x_{ij})\gamma_i(x_{ij})
-
\langle \gamma_i,\Delta_{nm}\rangle_{\mathcal L_2}
\Bigr\}
\nonumber
\\
&\le \frac{1}{8}\|\Delta_{nm}\|_{nm}^2 + C_2 a_n \bigl\{ \mathcal A_{nm}^2\,\rho_n\,\log(nm) + \mathcal A_{nm}^2\beta(S_n) \bigr\} + C_3\mathcal A_{nm}^2\widetilde\rho_n\log(nm).
\end{align}

Combining \eqref{eq:approx_manifold_gamma_mean_term2} and
\eqref{eq:approx_manifold_gamma_fluctuation2}, we conclude that
\begin{align}
&\sum_{i=1}^n\sum_{j=1}^{m_i}\frac{1}{nm_i}\Delta_{nm}(x_{ij})\gamma_i(x_{ij}) \le \frac{3}{16}\|\Delta_{nm}\|_{nm}^2 \notag \\
&\quad + C_4 a_n \bigl\{ \mathcal A_{nm}^2\,\rho_n\,\log(nm) + \mathcal A_{nm}^2\beta(S_n) \bigr\} + C_5\mathcal A_{nm}^2\widetilde\rho_n\log(nm) + C_6\,\frac{\sigma_\gamma^2}{n}.\label{eq:approx_manifold_gamma_term}
\end{align}

\medskip
\noindent
{\bf Step 4: empirical norm bound.}
Substituting \eqref{eq:approx_manifold_approx_term},
\eqref{eq:approx_manifold_epsilon_term2}, and
\eqref{eq:approx_manifold_gamma_term} into
\eqref{eq:approx_manifold_basic_ineq}, we obtain
\begin{align}
&\|\Delta_{nm}\|_{nm}^2 \lesssim \varepsilon_{nm,\delta}^2 + \mathcal A_{nm}^2\,\rho_n\, \log\!\bigl(\frac{2\mathcal A_{nm}}{\|\Delta_{nm}\|_{nm}}\bigr) + a_n\mathcal A_{nm}^2\rho_n\log(nm) + a_n\mathcal A_{nm}^2\beta(S_n) \notag \\
&\quad + \mathcal A_{nm}^2\widetilde\rho_n\log(nm) + \frac{\sigma_\gamma^2}{n}.\label{eq:approx_manifold_empirical_bound1}
\end{align}

We now control the logarithmic factor exactly as in the proof of Theorem~\ref{cor:general_architecture_temporal_spatial}. If
\begin{align*}
&\|\Delta_{nm}\|_{nm}^2 \le \frac1n,
\end{align*}
then \eqref{eq:approx_manifold_empirical_bound1} already yields the desired conclusion. Otherwise,
\begin{align*}
&\|\Delta_{nm}\|_{nm}^2 > \frac1n,
\qquad\text{hence}\qquad
\|\Delta_{nm}\|_{nm} \ge n^{-1/2},
\end{align*}
so that
\begin{align*}
&\log\!\left(\frac{2\mathcal A_{nm}}{\|\Delta_{nm}\|_{nm}}\right)
\le
\log\!\left(2\mathcal A_{nm}\sqrt n\right)
\lesssim
\log(nm).
\end{align*}
Therefore, in this case,
\begin{align}
&\|\Delta_{nm}\|_{nm}^2 \lesssim \varepsilon_{nm,\delta}^2 + \mathcal A_{nm}^2\rho_n\log(nm) + a_n\mathcal A_{nm}^2\rho_n \log(nm) \notag \\
&\quad + a_n\mathcal A_{nm}^2\beta(S_n) + \mathcal A_{nm}^2\widetilde\rho_n\log(nm) + \frac{\sigma_\gamma^2}{n}.\label{eq:approx_manifold_empirical_bound2}
\end{align}
Since \(a_n\to\infty\), we have \(a_n\ge1\) for all sufficiently large \(n\), and thus
\begin{align*}
&\mathcal A_{nm}^2\rho_n\log(nm)
\le
a_n\mathcal A_{nm}^2\rho_n\log(nm).
\end{align*}
Hence in either case,
\begin{align}
&\|\Delta_{nm}\|_{nm}^2 \lesssim \varepsilon_{nm,\delta}^2 + a_n\mathcal A_{nm}^2\rho_n\log(nm) \notag \\
&\quad + a_n\mathcal A_{nm}^2\beta(S_n) + a_n\mathcal A_{nm}^2\widetilde\rho_n\log(nm) + \frac{\sigma_\gamma^2}{n}.\label{eq:approx_manifold_empirical_bound3}
\end{align}

\medskip
\noindent
{\bf Step 5: passage to \(\mathcal L_2\).}
Combining \eqref{eq:approx_manifold_norm_comparison} and
\eqref{eq:approx_manifold_empirical_bound3}, we obtain
\begin{align}
&\|\Delta_{nm}\|_{\mathcal L_2}^2 \le 2\|\Delta_{nm}\|_{nm}^2 + C a_n \bigl\{ \mathcal A_{nm}^2\rho_n\log(nm) + \mathcal A_{nm}^2\beta(S_n) \bigr\} \nonumber\label{eq:approx_manifold_L2_bound}
\\
&\lesssim \varepsilon_{nm,\delta}^2 + a_n\mathcal A_{nm}^2\rho_n\log(nm) + a_n\mathcal A_{nm}^2\beta(S_n) + a_n\mathcal A_{nm}^2\widetilde\rho_n\log(nm) + \frac{\sigma_\gamma^2}{n}.
\end{align}
Substituting \eqref{eq:approx_manifold_L2_bound} into
\eqref{eq:approx_manifold_basic_reduction}, we conclude that
\begin{align}
&\|\widehat f_{\mathcal A_{nm}}-f^*\|_{\mathcal L_2}^2 \lesssim (nm)^{-\frac{2p}{2p+d_\delta}} + a_n\mathcal A_{nm}^2\,\rho_n\,\log(nm) \notag \\
&\quad + a_n\mathcal A_{nm}^2\,\widetilde\rho_n\,\log(nm) + a_n\mathcal A_{nm}^2\beta(S_n) + \frac{\sigma_\gamma^2}{n},\label{eq:appendix_approx_manifold_temporal_spatial_rate_proof}
\end{align}
with probability tending to one.

It remains to bound \(\rho_n\) and \(\widetilde\rho_n\). Under the canonical choice of \((L,r)\) from Lemma~\ref{lem:approximation_approx_manifold_temporal_spatial}, Lemma~\ref{lemma:threshold space entropy} yields
\begin{align*}
&\log \mathcal N(\delta,\mathcal F_{\mathcal A_{nm}}(L,r),\|\cdot\|_{nm}) \lesssim nm\,(nm)^{-\frac{2p}{2p+d_\delta}}\log^{4.3}(nm)\log(\delta^{-1}),
\end{align*}
and therefore
\begin{align*}
&\widetilde\rho_n \lesssim (nm)^{-\frac{2p}{2p+d_\delta}}\log^{4.3}(nm).
\end{align*}
Likewise, for the blockwise empirical norms, since \(|\mathcal J_{e,s}|m\asymp nm/S_n\) and \(|\mathcal J_{o,s}|m\asymp nm/S_n\), the same entropy bound gives
\begin{align*}
&\rho_n \lesssim \left(\frac{S_n}{nm}\right)^{\frac{2p}{2p+d_\delta}}\log^{4.3}(nm).
\end{align*}

Substituting these bounds into \eqref{eq:appendix_approx_manifold_temporal_spatial_rate_proof} yields
\begin{align*}
&\|\widehat f_{\mathcal A_{nm}}-f^*\|_{\mathcal L_2}^2 \lesssim\; (nm)^{-\frac{2p}{2p+d_\delta}} + a_n\mathcal A_{nm}^2 \bigl(\frac{S_n}{nm}\bigr)^{\frac{2p}{2p+d_\delta}} \log^{5.3}(nm) \notag \\
&\quad + a_n\mathcal A_{nm}^2 (nm)^{-\frac{2p}{2p+d_\delta}} \log^{5.3}(nm) +
a_n\mathcal A_{nm}^2\beta(S_n)
+
\frac{\sigma_\gamma^2}{n}.
\end{align*}

Finally, choosing \(S_n\asymp \log n\), and using
\begin{align*}
&\mathcal A_{nm}^2 \asymp (\sigma_\epsilon^2+\sigma_\gamma^2)\log(nm),
\qquad
\beta(S_n)\lesssim \frac1n,
\end{align*}
we obtain
\begin{align*}
&\|\widehat f_{\mathcal A_{nm}}-f^*\|_{\mathcal L_2}^2 \lesssim\; (nm)^{-\frac{2p}{2p+d_\delta}} + a_n(\sigma_\epsilon^2+\sigma_\gamma^2) \bigl(\frac{\log n}{nm}\bigr)^{\frac{2p}{2p+d_\delta}} \log^{6.3}(nm)
\\
&+ a_n(\sigma_\epsilon^2+\sigma_\gamma^2) (nm)^{-\frac{2p}{2p+d_\delta}} \log^{6.3}(nm) + a_n\frac{(\sigma_\epsilon^2+\sigma_\gamma^2)\log(nm)}{n} + \frac{\sigma_\gamma^2}{n}.
\end{align*}
Since
\begin{align*}
&(nm)^{-\frac{2p}{2p+d_\delta}}
\le
(nm)^{-\frac{2p}{2p+d_\delta}}\log^{7.3}(nm),
\end{align*}
and
\begin{align*}
&\left(\frac{\log n}{nm}\right)^{\frac{2p}{2p+d_\delta}}\log^{6.3}(nm)
\lesssim
(nm)^{-\frac{2p}{2p+d_\delta}}\log^{7.3}(nm),
\end{align*}
it follows that
\begin{align}
&\|\widehat f_{\mathcal A_{nm}}-f^*\|_{\mathcal L_2}^2 \lesssim\; a_n(\sigma_\epsilon^2+\sigma_\gamma^2+1) (nm)^{-\frac{2p}{2p+d_\delta}} \log^{7.3}(nm) \notag \\
&\quad + a_n\frac{(\sigma_\epsilon^2+\sigma_\gamma^2)\log(nm)}{n} + \frac{\sigma_\gamma^2}{n},\label{eq:appendix_approx_manifold_temporal_spatial_rate_final}
\end{align}
with probability tending to one.
This completes the proof.
\end{proof}

\newpage

\newpage

\section{Low Minkowski Dimension Support under Temporal-Spatial Dependence}
\label{appendix:minkowski_temporal_spatial}

In this section, we extend the temporal-spatial risk bound to the case where the design is supported on a set with low Minkowski dimension. The stochastic part of the argument is unchanged from the previous sections; the only new ingredient is the approximation lemma below.

The proof follows the strategy of Theorem~6.3 in \cite{jiao2023deep} and, more specifically, Appendix~A.9 of the supplementary material. In that proof, the support of the design is projected into a lower-dimensional space using a generalized Johnson--Lindenstrauss lemma for infinite sets, the target function is represented on the projected set, then extended by Whitney's extension theorem and approximated there by ReLU networks. See Section~6.3 of the main paper and Appendix~A.9 of the supplement. 

We start defining Minkowski dimension.

\begin{definition}[Minkowski dimension]
Let \(\mathcal M\subset \mathbb R^d\) be bounded. For \(\epsilon>0\), let
\begin{align*}
&\mathcal N(\epsilon,\mathcal M,\|\cdot\|)
\end{align*}
denote the covering number of \(\mathcal M\) by Euclidean balls of radius \(\epsilon\).
The upper and lower Minkowski dimensions of \(\mathcal M\) are defined by
\begin{align*}
&\overline{\dim}_M(\mathcal M) := \limsup_{\epsilon\downarrow 0} \frac{\log \mathcal N(\epsilon,\mathcal A,\|\cdot\|)}{\log(1/\epsilon)}, \qquad \underline{\dim}_M(\mathcal M) := \liminf_{\epsilon\downarrow 0} \frac{\log \mathcal N(\epsilon,\mathcal A,\|\cdot\|)}{\log(1/\epsilon)}.
\end{align*}
If these two quantities coincide, their common value is called the Minkowski dimension of \(\mathcal M\), denoted by \(\dim_M(\mathcal M)\).
\end{definition}

\begin{lemma}[Approximation under low Minkowski dimension support]
\label{lem:approximation_minkowski_temporal_spatial}
Let \(\mathcal M\subset[0,1]^d\) be a compact set with Minkowski dimension
\begin{align*}
&d^*:=\dim_M(\mathcal M)<d.
\end{align*}
Assume that \(f^*:[0,1]^d\to\mathbb R\) is \((p,C)\)-smooth, with \(p=q+s\), \(q\in\mathbf N\), \(s\in(0,1]\), and that all partial derivatives of order at most \(q\) are uniformly bounded.

Fix \(\delta,\tau\in(0,1)\). Let \(d_0\in \mathbf Z^+\) satisfy
\begin{align*}
&d_0 \ge \kappa \frac{c\,d^*+\log(2/\tau)}{\delta^2},
\end{align*}
where \(\kappa,c>0\) are constants from the generalized Johnson--Lindenstrauss lemma in Appendix~A.9 of \cite{jiao2023deep}. Define
\begin{align*}
&\varepsilon_{nm,\delta}
:=
(nm)^{-\frac{p}{2p+d_0}}.
\end{align*}

Then there exists \(\overline f \in \mathcal F(L,r)\), with either
\begin{align*}
&L\asymp \log(nm),
\qquad
r\asymp (nm)^{\frac{d_0}{2(2p+d_0)}}\log(nm),
\end{align*}
or
\begin{align*}
&L\asymp \log(nm)\,(nm)^{\frac{d_0}{2(2p+d_0)}},
\qquad
r=O(1),
\end{align*}
such that
\begin{align}
\label{eq:approx_error_minkowski_temporal_spatial}
\|f^*-\overline f\|_{\infty,\mathcal M}^2
\le
C_{p,C,d,d^*,\delta}\,\varepsilon_{nm,\delta}^2,
\end{align}
where
\begin{align*}
&\|f\|_{\infty,\mathcal M}:=\sup_{x\in\mathcal M}|f(x)|.
\end{align*}
The constant \(C_{p,C,d,d^*,\delta}\) does not depend on \(n\) or \(m\).
\end{lemma}

\begin{proof}
The proof follows Appendix~A.9 of the supplementary material of \cite{jiao2023deep}.

\medskip
\noindent
\textbf{Step 1: Dimension reduction via the generalized Johnson--Lindenstrauss argument.}

Let \(E:=\mathcal M\subset[0,1]^d\). Since \(\mathcal M\) is compact, we may work directly with
\(E\) in place of the closure \(\bar E\) used in Appendix~A.9 of \cite{jiao2023deep}. Define
\begin{align*}
&\bar T:=\{x-y:x,y\in E\}.
\end{align*}
For any \(\delta>0\), define
\begin{align*}
&\bar T_\delta:=\{\bar t\in \bar T:\|\bar t\|_2\ge \delta\},
\qquad
T_\delta:=\left\{\frac{\bar t}{\|\bar t\|_2}:\bar t\in \bar T_\delta\right\},
\end{align*}
and
\begin{align*}
&T := \left\{\frac{\bar t}{\|\bar t\|_2} : \bar t \in \bar T \setminus \{0\}\right\}.
\end{align*}
As shown in Appendix~A.9 of \cite{jiao2023deep}, one first has
\begin{align*}
&\dim_M(\bar T)\le 2d^*.
\end{align*}
Moreover, for every \(\delta>0\),
\begin{align*}
&\dim_M\!\bigl(\delta^{-1}\bar T_\delta\bigr)\le 4d^*, \qquad \dim_M(T_\delta) \le \dim_M\!\bigl(\delta^{-1}\bar T_\delta\bigr)\le 4d^*.
\end{align*}
Indeed, Appendix~A.9 of \cite{jiao2023deep} shows that
\begin{align*}
&\dim_M(T_\delta)\le \dim_M\!\bigl(\delta^{-1}\bar T_\delta\bigr) \le 2d^*+\dim_M(\bar T_\delta) \le 2d^*+\dim_M(\bar T) \le 4d^*.
\end{align*}
Using the corresponding covering-number comparison between \(T_\delta\) and
\(\delta^{-1}\bar T_\delta\), Appendix~A.9 of \cite{jiao2023deep} then yields
\begin{align*}
&\sqrt{H(\epsilon,\|\cdot\|_2,T)}
\le c_1\sqrt{d^*\log(1/\epsilon)},
\end{align*}
for some constant \(c_1>0\), where
\begin{align*}
&H(\epsilon,\|\cdot\|_2,T):=\log \mathcal N(\epsilon,T,\|\cdot\|_2).
\end{align*}
They then define
\begin{align*}
&\gamma(T):=\int_0^1 \sqrt{H(\epsilon,\|\cdot\|_2,T)}\,d\epsilon,
\end{align*}
and conclude that
\begin{align*}
&\gamma^2(T)=c\,d^*
\end{align*}
for some constant \(c>0\). Therefore, if
\begin{align*}
&d_0\ge \kappa \frac{\gamma^2(T)+\log(2/\tau)}{\delta^2}
= \kappa \frac{c\,d^*+\log(2/\tau)}{\delta^2},
\end{align*}
where \(\kappa>0\) is the universal constant appearing in the generalized Johnson--Lindenstrauss lemma and \(\tau\in(0,1)\) is its confidence parameter, then the generalized Johnson--Lindenstrauss lemma used in Appendix~A.9 of \cite{jiao2023deep} yields a linear map
\begin{align*}
&A:\mathbb R^d\to\mathbb R^{d_0}
\end{align*}
whose restriction to \(\bar E\) is injective and such that
\begin{align*}
&(1-\delta)\|x_1-x_2\|_2^2 \le \|Ax_1-Ax_2\|_2^2 \le (1+\delta)\|x_1-x_2\|_2^2, \qquad x_1,x_2\in\bar E.
\end{align*}
Moreover, as in Appendix~A.9 of \cite{jiao2023deep}, we have
\begin{align*}
&A(E)\subset A([0,1]^d)\subset
\mathcal B_{d_0}:=
\left[-\sqrt{d/d_0},\,\sqrt{d/d_0}\right]^{d_0}.
\end{align*}

\medskip
\noindent
\textbf{Step 2: Low-dimensional representation on the projected set.}\par
\nopagebreak[4]

Since \(A\) is injective on \(\bar E\), for each \(z\in A(\bar E)\) there exists a unique
\(x_z\in \bar E\) such that \(Ax_z=z\). Define the selection map
\begin{align*}
&SL:A(\bar E)\to \bar E,
\qquad
SL(z):=x_z.
\end{align*}
By the same argument as in Appendix~A.9 of \cite{jiao2023deep}, the map \(SL\) is
differentiable and its derivative satisfies
\begin{align*}
&\|DSL(z)\|_{\mathrm{op}}\in
\left[\sqrt{\frac{1}{1+\delta}},\,\sqrt{\frac{1}{1-\delta}}\right],
\qquad z\in A(\bar E).
\end{align*}

Define the function
\begin{align*}
&\widetilde f(z):=f^*(x_z)=f^*(SL(z)),\qquad z\in A(\bar E).
\end{align*}
Since \(f^*\) is \((p,C)\)-smooth on \([0,1]^d\), and the selection map \(SL\)
is differentiable with
\begin{align*}
&\|DSL(z)\|_{\mathrm{op}}\le (1-\delta)^{-1/2},\qquad z\in A(\bar E),
\end{align*}
it follows, as in Appendix~A.9 of \cite{jiao2023deep}, that
\(\widetilde f\) is \((p,\,C(1-\delta)^{-p/2})\)-smooth on \(A(\bar E)\)
in the sense of Definition~2 in the Main Manuscript.

\medskip
\noindent
\textbf{Step 3: Whitney extension on the projected box.}

Since \(\bar E\) is compact and \(A\) is linear, the set \(A(\bar E)\subset \mathbb R^{d_0}\) is compact,
hence closed. Moreover, \(A(\bar E)\subset \mathcal B_{d_0}\). Therefore, by the extended version of Whitney's extension theorem in \cite{fefferman2006whitney}, there exists a function
\begin{align*}
&\widetilde F:\mathcal B_{d_0}\to\mathbb R
\end{align*}
such that
\begin{align*}
&\widetilde F(z)=\widetilde f(z),\qquad z\in A(\bar E),
\end{align*}
and \(\widetilde F\) is \((p,C_{\mathrm{ext}}\,C(1-\delta)^{-p/2})\)-smooth on \(\mathcal B_{d_0}\), in the sense of Definition~2 in the Main Manuscript, where \(C_{\mathrm{ext}}>0\) depends only on \(p\) and \(d_0\). 

\medskip
\noindent
\textbf{Step 4: Approximation in dimension \(d_0\).}

We now approximate \(\widetilde F\) on the box \(\mathcal B_{d_0}\subset\mathbb R^{d_0}\). After an affine rescaling mapping \(\mathcal B_{d_0}\) onto \([0,1]^{d_0}\), Corollary~3.1 of \cite{jiao2023deep} applies in dimension \(d_0\). Therefore, there exists a ReLU network
\begin{align*}
&h:\mathbb R^{d_0}\to\mathbb R
\end{align*}
with either one of the two stated architectures such that
\begin{align*}
&\|\widetilde F-h\|_{\infty,\mathcal B_{d_0}} \le C_{p,C,C_{\mathrm{ext}},d,d_0,\delta}\,(nm)^{-\frac{p}{2p+d_0}} = C_{p,C,C_{\mathrm{ext}},d,d_0,\delta}\,\varepsilon_{nm,\delta}.
\end{align*}
Here the constant \(C_{p,C,C_{\mathrm{ext}},d,d_0,\delta}\) absorbs the factor
\((1-\delta)^{-p/2}\) coming from the Hölder norm of \(\widetilde F\).
This is the same low-dimensional approximation step used after constructing \(\widetilde F_0\) in Appendix~A.9 of \cite{jiao2023deep}. 

\medskip
\noindent
\textbf{Step 5: Lifting back to the ambient space.}

Define
\begin{align*}
&\overline f(x):=h(Ax),\qquad x\in[0,1]^d.
\end{align*}
Since \(A\) is linear, the composition \(x\mapsto Ax\) can be absorbed into the first
layer of the network. Hence \(\overline f\in \mathcal F(L,r)\) with the same order
of depth and width as \(h\). 

In fact, if \(h\in\mathcal F(L,r)\), then by Definition~1 in the Main Manuscript,
\begin{align*}
&h(z)=\sum_{i=1}^{r} c_{1,i}^{(L)} f_i^{(L)}(z)+c_{1,0}^{(L)},
\end{align*}
where the functions \(f_i^{(s)}\) satisfy the recursive relations in Equations~(4)--(5) of the Main Manuscript. For the first layer,
\begin{align*}
f_i^{(1)}(Ax)
&=
\sigma\!\left(\sum_{j=1}^{d_0} c_{i,j}^{(0)}(Ax)^{(j)}+c_{i,0}^{(0)}\right) \\
&=
\sigma\!\left(\sum_{k=1}^{d}\left(\sum_{j=1}^{d_0} c_{i,j}^{(0)}A_{j,k}\right)x^{(k)}+c_{i,0}^{(0)}\right),
\end{align*}
where \(\sigma\) denotes the ReLU activation function. Thus, after redefining the first-layer coefficients by
\begin{align*}
&\widetilde c_{i,k}^{(0)}:=\sum_{j=1}^{d_0} c_{i,j}^{(0)}A_{j,k},
\end{align*}
the composition \(x\mapsto h(Ax)\) has the same recursive form as Equations~(4)--(5) in the Main Manuscript, with the same number of layers \(L\) and the same width \(r\). Hence
\begin{align*}
&\overline f\in \mathcal F(L,r).
\end{align*}

Now let \(x\in E=\mathcal M\). Since \(x\in \bar E\), we have \(Ax\in A(\bar E)\), and therefore
\begin{align*}
&\widetilde F(Ax)=\widetilde f(Ax)=f^*(x).
\end{align*}
Hence
\begin{align*}
&|f^*(x)-\overline f(x)| = |\widetilde F(Ax)-h(Ax)| \le \|\widetilde F-h\|_{\infty,\mathcal B_{d_0}} \le C_{p,C,C_{\mathrm{ext}},d,d_0,\delta}\,\varepsilon_{nm,\delta}.
\end{align*}
Taking the supremum over \(x\in \mathcal M\) yields
\begin{align*}
&\|f^*-\overline f\|_{\infty,\mathcal M}
\le
C_{p,C,C_{\mathrm{ext}},d,d_0,\delta}\,\varepsilon_{nm,\delta}.
\end{align*}
Squaring both sides gives \eqref{eq:approx_error_minkowski_temporal_spatial}.
\end{proof}

\begin{theorem}[Temporal-spatial risk bound under low Minkowski dimension support]
\label{thm:appendix_minkowski_temporal_spatial}
Suppose that Assumptions~1, 2, 3{\bf a}--{\bf b}, and 4 hold. Assume moreover that the design points are supported on a set \(\mathcal M\subset[0,1]^d\) with Minkowski dimension
\begin{align*}
&d^*:=\dim_M(\mathcal M)<d,
\end{align*}
and that \(f^*\) satisfies the smoothness conditions of Lemma~\ref{lem:approximation_minkowski_temporal_spatial}.

Fix \(\delta,\tau\in(0,1)\), let \(d_0\) be as in Lemma~\ref{lem:approximation_minkowski_temporal_spatial}, and let \(\widehat f\) be the estimator in Equation~(2), with \((L,r)\) chosen as in Lemma~\ref{lem:approximation_minkowski_temporal_spatial}. Let
\begin{align*}
&\mathcal A_{nm}\asymp \max\{\sigma_\epsilon,\sigma_\gamma\}\sqrt{\log(nm)}
\end{align*}
and let \((a_n)_{n\ge 1}\) be any sequence such that \(a_n\to\infty\). Let \((S_n)_{n\ge1}\) be any sequence satisfying
\begin{align*}
&C_\beta \log(n)\le S_n \le n,
\end{align*}
for a sufficiently large constant \(C_\beta>0\) depending only on the mixing rate in Assumption~4.

Suppose that the blockwise and global entropy conditions of Theorem~\ref{cor:general_architecture_temporal_spatial} hold with sequences \((\rho_n)\) and \((\widetilde\rho_n)\). Then, with probability approaching one, 
\begin{align}
&\|\widehat f_{\mathcal A_{nm}}-f^*\|_{\mathcal L_2}^2 \lesssim (nm)^{-\frac{2p}{2p+d_0}} + a_n\mathcal A_{nm}^2\,\rho_n\,\log(nm) \notag \\
&\quad + a_n\mathcal A_{nm}^2\,\widetilde\rho_n\,\log(nm) + a_n\mathcal A_{nm}^2\beta(S_n) + \frac{\sigma_\gamma^2}{n}.\label{eq:appendix_minkowski_temporal_spatial_rate}
\end{align}

In particular, using \((L,r)\) chosen as in Lemma~\ref{lem:approximation_minkowski_temporal_spatial}, we may take
\begin{align*}
&\rho_n \lesssim \bigl(\frac{S_n}{nm}\bigr)^{\frac{2p}{2p+d_0}}\log^{4.3}(nm), \qquad \widetilde\rho_n \lesssim (nm)^{-\frac{2p}{2p+d_0}}\log^{4.3}(nm).
\end{align*}
Therefore,
\begin{align}
&\|\widehat f_{\mathcal A_{nm}}-f^*\|_{\mathcal L_2}^2 \lesssim (nm)^{-\frac{2p}{2p+d_0}} + a_n\mathcal A_{nm}^2 \bigl(\frac{S_n}{nm}\bigr)^{\frac{2p}{2p+d_0}} \log^{5.3}(nm) \nonumber\label{eq:appendix_minkowski_temporal_spatial_rate_intermediate}
\\
&\quad+ a_n\mathcal A_{nm}^2 (nm)^{-\frac{2p}{2p+d_0}} \log^{5.3}(nm) + a_n\mathcal A_{nm}^2\beta(S_n) + \frac{\sigma_\gamma^2}{n},
\end{align}
with probability approaching one.

Moreover, choosing \(S_n\asymp \log(n)\), we obtain
\begin{align}
&\|\widehat f_{\mathcal A_{nm}}-f^*\|_{\mathcal L_2}^2 \lesssim\; a_n(\sigma_\epsilon^2+\sigma_\gamma^2+1) (nm)^{-\frac{2p}{2p+d_0}} \log^{7.3}(nm) \nonumber\label{eq:appendix_minkowski_temporal_spatial_rate_final}
\\
&\quad
+
a_n\frac{(\sigma_\epsilon^2+\sigma_\gamma^2)\log(nm)}{n}
+
\frac{\sigma_\gamma^2}{n},
\end{align}
with probability approaching one.
\end{theorem}

\begin{proof}
The proof follows the same line of argument as the proof of Theorem~\ref{cor:general_architecture_temporal_spatial}, with the approximation input supplied by Lemma~\ref{lem:approximation_minkowski_temporal_spatial}. The only additional point is that the approximation bound from Lemma~\ref{lem:approximation_minkowski_temporal_spatial} is available only on the support of the design, namely \(\mathcal M\). Accordingly, we work with \(\overline f_{\mathcal A_{nm}}\) throughout.

By Lemma~\ref{lem:approximation_minkowski_temporal_spatial}, there exists \(\overline f\in\mathcal F(L,r)\) such that
\begin{align*}
&\|f^*-\overline f\|_{\infty,\mathcal M}^2
\lesssim
\varepsilon_{nm,\delta}^2,
\qquad
\varepsilon_{nm,\delta}^2
:=
(nm)^{-\frac{2p}{2p+d_0}}.
\end{align*}
Since the support of the design is contained in \(\mathcal M\), it follows that for every observed design point \(x_{ij}\),
\begin{align*}
&|f^*(x_{ij})-\overline f(x_{ij})|^2
\le
\|f^*-\overline f\|_{\infty,\mathcal M}^2
\lesssim
\varepsilon_{nm,\delta}^2.
\end{align*}
Moreover, since \(\|f^*\|_\infty\le \mathcal A_{nm}/4\) by Assumption~3{\bf b}, and \(\varepsilon_{nm,\delta}\to0\), for all sufficiently large \(nm\),
\begin{align*}
&|\overline f(x_{ij})| \le |f^*(x_{ij})|+|f^*(x_{ij})-\overline f(x_{ij})| \le \frac{\mathcal A_{nm}}{4}+\varepsilon_{nm,\delta} \le \frac{\mathcal A_{nm}}{2}.
\end{align*}
Hence, for all sufficiently large \(nm\),
\begin{align*}
&\overline f_{\mathcal A_{nm}}(x_{ij})=\overline f(x_{ij})
\qquad
\text{for all }1\le i\le n,\ 1\le j\le m_i.
\end{align*}
In particular,
\begin{align}
&\|f^*-\overline f_{\mathcal A_{nm}}\|_{nm}^2
\lesssim
\varepsilon_{nm,\delta}^2.\label{eq:appendix_minkowski_empirical_approx}
\end{align}

Define
\begin{align*}
&\Delta_{nm}(x):=\widehat f_{\mathcal A_{nm}}(x)-\overline f_{\mathcal A_{nm}}(x),
\end{align*}
and the shifted truncated class
\begin{align*}
&\mathcal G_{\mathcal A_{nm}}(\overline f)
:=
\bigl\{
f_{\mathcal A_{nm}}-\overline f_{\mathcal A_{nm}}
:\,
f\in \mathcal F(L,r)
\bigr\}.
\end{align*}
For every \(0<\delta\le1\), every \(h_1,h_2\in \mathcal F_{\mathcal A_{nm}}(L,r)\), and every empirical norm appearing in the proof of Theorem~\ref{cor:general_architecture_temporal_spatial}, translation by \(-\overline f_{\mathcal A_{nm}}\) is an isometry. Therefore the blockwise entropy condition and the global entropy condition carry over to the shifted class \(\mathcal G_{\mathcal A_{nm}}(\overline f)\) with the same sequences \((\rho_n)\) and \((\widetilde\rho_n)\).

Let
\begin{align*}
&\mathcal E = \bigl\{ |\epsilon_{ij}|\le \frac{\mathcal A_{nm}}{4}\,\text{for all }1\le i\le n,\, 1\le j\le m_i, \quad \|\gamma_i\|_{\infty,nm} \le \frac{\mathcal A_{nm}}{4}\,\text{for all }1\le i\le n \bigr\}.
\end{align*}
By Lemma~\ref{lemma_epsilon_gamma_infinity_bound},
\begin{align*}
&\mathbb P(\mathcal E^c)=o(1).
\end{align*}

Using
\begin{align*}
&\widehat f_{\mathcal A_{nm}}-f^*
=
\Delta_{nm}+(\overline f_{\mathcal A_{nm}}-f^*),
\end{align*}
we obtain
\begin{align}
\label{eq:appendix_minkowski_basic_reduction}
\|\widehat f_{\mathcal A_{nm}}-f^*\|_{\mathcal L_2}^2
&\le
2\|\overline f_{\mathcal A_{nm}}-f^*\|_{\mathcal L_2}^2
+
2\|\Delta_{nm}\|_{\mathcal L_2}^2
\nonumber\\
&\le
2\varepsilon_{nm,\delta}^2
+
2\|\Delta_{nm}\|_{\mathcal L_2}^2,
\end{align}
where we used \eqref{eq:appendix_minkowski_empirical_approx} together with the fact that the approximation error is controlled on the support of the design.

It therefore suffices to control \(\|\Delta_{nm}\|_{\mathcal L_2}^2\).

On the event \(\mathcal E\), exactly as in the proof of Theorem~\ref{cor:general_architecture_temporal_spatial}, the empirical basic inequality gives
\begin{align}
&\frac12\|\Delta_{nm}\|_{nm}^2 \le \sum_{i=1}^n\sum_{j=1}^{m_i}\frac{1}{nm_i}\Delta_{nm}(x_{ij}) \bigl\{f^*(x_{ij})-\overline f_{\mathcal A_{nm}}(x_{ij})\bigr\} \notag \\
&\quad + \sum_{i=1}^n\sum_{j=1}^{m_i}\frac{1}{nm_i}\Delta_{nm}(x_{ij}) \bigl\{\gamma_i(x_{ij})+\epsilon_{ij}\bigr\}.\label{eq:appendix_minkowski_basic_ineq}
\end{align}

We first control the approximation term. By Cauchy--Schwarz and \eqref{eq:appendix_minkowski_empirical_approx},
\begin{align}
&\sum_{i=1}^n\sum_{j=1}^{m_i}\frac{1}{nm_i}\Delta_{nm}(x_{ij}) \bigl\{f^*(x_{ij})-\overline f_{\mathcal A_{nm}}(x_{ij})\bigr\} \le \|\Delta_{nm}\|_{nm}\,\|f^*-\overline f_{\mathcal A_{nm}}\|_{nm} \nonumber\label{eq:appendix_minkowski_approx_term}
\\
&\le
\|\Delta_{nm}\|_{nm}\,\varepsilon_{nm,\delta}
\nonumber
\\
&\le
\frac{1}{32}\|\Delta_{nm}\|_{nm}^2 + C\,\varepsilon_{nm,\delta}^2.
\end{align}

\medskip
\noindent
{\bf Step 1: the \(\epsilon\)-term.}
Since \(\Delta_{nm}\in \mathcal G_{\mathcal A_{nm}}(\overline f)\), the blockwise entropy condition implies that the proof of
\Cref{coro:complexity bound 2 beta mn}
can be repeated block by block for the shifted class \(\mathcal G_{\mathcal A_{nm}}(\overline f)\), with the quantity \(\phi_{|\mathcal J_{\cdot,s}|m}\log^{4.3}(nm)\) replaced by \(\rho_n\). Thus, arguing exactly as in Step~1 of the proof of Theorem~\ref{cor:general_architecture_temporal_spatial}, we obtain that with probability tending to one,
\begin{align}
\label{eq:appendix_minkowski_epsilon_lambda}
\sum_{i=1}^n\sum_{j=1}^{m_i}\frac{1}{nm_i}\Delta_{nm}(x_{ij})\epsilon_{ij}
\le
C\,\sigma_\epsilon\,
\Lambda^{(\rho)}_{n,S_n}(\Delta_{nm}),
\end{align}
where \(\Lambda^{(\rho)}_{n,S_n}\) is defined exactly as \(\Lambda_{n,S}\) in
\Cref{coro:complexity bound 2 beta mn}, but with \(\phi_{|\mathcal J_{\cdot,s}|m}\log^{4.3}(nm)\) replaced by \(\rho_n\). Repeating the same Jensen reduction as in the proof of Theorem~\ref{cor:general_architecture_temporal_spatial} gives
\begin{align}
\label{eq:appendix_minkowski_lambda_reduction}
\Lambda^{(\rho)}_{n,S_n}(\Delta_{nm})
\le
C\,
\|\Delta_{nm}\|_{nm}
\sqrt{
\rho_n\,
\log\!\left(\frac{2\mathcal A_{nm}}{\|\Delta_{nm}\|_{nm}}\right)
}.
\end{align}
Applying Young's inequality, we obtain
\begin{align}
&\sum_{i=1}^n\sum_{j=1}^{m_i}\frac{1}{nm_i}\Delta_{nm}(x_{ij})\epsilon_{ij} \le \frac{1}{32}\|\Delta_{nm}\|_{nm}^2 + C\,\sigma_\epsilon^2\,\rho_n\, \log\!\bigl(\frac{2\mathcal A_{nm}}{\|\Delta_{nm}\|_{nm}}\bigr).\label{eq:appendix_minkowski_epsilon_term}
\end{align}
Since \(\sigma_\epsilon^2\lesssim \mathcal A_{nm}^2\), this becomes
\begin{align}
&\sum_{i=1}^n\sum_{j=1}^{m_i}\frac{1}{nm_i}\Delta_{nm}(x_{ij})\epsilon_{ij} \le \frac{1}{32}\|\Delta_{nm}\|_{nm}^2 + C\,\mathcal A_{nm}^2\,\rho_n\, \log\!\bigl(\frac{2\mathcal A_{nm}}{\|\Delta_{nm}\|_{nm}}\bigr).\label{eq:appendix_minkowski_epsilon_term_bound}
\end{align}

\medskip
\noindent
{\bf Step 2: norm comparison.}
Using again the blockwise entropy condition, the proofs of
\Cref{coro:deviation for norms mn}
and
\Cref{lemma:rate in the temporal-spatial model}
carry over block by block with \(\phi_{nm/S_n}\log^{5.3}(nm/S_n)\) replaced by \(\rho_n\log(nm)\). Therefore, for any sequence \(a_n\to\infty\), arguing exactly as in the proof of Theorem~\ref{cor:general_architecture_temporal_spatial}, we obtain with probability tending to one,
\begin{align}
&\|\Delta_{nm}\|_{\mathcal L_2}^2 \le 2\|\Delta_{nm}\|_{nm}^2 + C a_n \bigl\{ \mathcal A_{nm}^2\,\rho_n\,\log(nm) + \mathcal A_{nm}^2\beta(S_n) \bigr\}.\label{eq:appendix_minkowski_norm_comparison}
\end{align}

\medskip
\noindent
{\bf Step 3: the \(\gamma\)-term.}
We decompose the \(\gamma\)-term exactly as in the proof of Theorem~\ref{cor:general_architecture_temporal_spatial}:
\begin{align*}
&\sum_{i=1}^n\sum_{j=1}^{m_i}\frac{1}{nm_i}\Delta_{nm}(x_{ij})\gamma_i(x_{ij})
\\
&={} \sum_{i=1}^n\sum_{j=1}^{m_i}\frac{1}{nm_i} \Bigl\{ \Delta_{nm}(x_{ij})\gamma_i(x_{ij}) - \langle \gamma_i,\Delta_{nm}\rangle_{\mathcal L_2} \Bigr\} + \frac1n\sum_{i=1}^n \langle \gamma_i,\Delta_{nm}\rangle_{\mathcal L_2}.
\end{align*}

For the second term, by Lemma~\ref{lemma:bound on the expected norm of gamma} and Young's inequality,
\begin{align}
\label{eq:appendix_minkowski_gamma_mean_term}
\frac1n\sum_{i=1}^n \langle \gamma_i,\Delta_{nm}\rangle_{\mathcal L_2}
&=
\left\langle \frac1n\sum_{i=1}^n \gamma_i,\Delta_{nm}\right\rangle_{\mathcal L_2}
\nonumber\\
&\le
\left\|\frac1n\sum_{i=1}^n \gamma_i\right\|_{\mathcal L_2}\|\Delta_{nm}\|_{\mathcal L_2}
\nonumber\\
&\le
\frac{8\sigma_\gamma^2}{n}
+
\frac{1}{32}\|\Delta_{nm}\|_{\mathcal L_2}^2.
\end{align}
Combining this with \eqref{eq:appendix_minkowski_norm_comparison}, we obtain
\begin{align}
&\frac1n\sum_{i=1}^n \langle \gamma_i,\Delta_{nm}\rangle_{\mathcal L_2} \le \frac{8\sigma_\gamma^2}{n} + \frac{1}{16}\|\Delta_{nm}\|_{nm}^2 + C_1 a_n \bigl\{ \mathcal A_{nm}^2\,\rho_n\,\log(nm) + \mathcal A_{nm}^2\beta(S_n) \bigr\}.\label{eq:appendix_minkowski_gamma_mean_term_reduced}
\end{align}

For the centered fluctuation term, we now use the global entropy condition. Since \(\Delta_{nm}\in \mathcal G_{\mathcal A_{nm}}(\overline f)\), the proof of
\Cref{coro:deviation bound for gamma with xi mn}
carries over verbatim with \(\phi_{nm}\log^{5.3}(nm)\) replaced by \(\widetilde\rho_n\log(nm)\). Hence, with probability tending to one,
\begin{align}
\label{eq:appendix_minkowski_gamma_fluctuation}
&\sum_{i=1}^n\sum_{j=1}^{m_i}\frac{1}{nm_i}
\Bigl\{
\Delta_{nm}(x_{ij})\gamma_i(x_{ij})
-
\langle \gamma_i,\Delta_{nm}\rangle_{\mathcal L_2}
\Bigr\}
\nonumber\\
&\le
\frac{1}{16}\|\Delta_{nm}\|_{\mathcal L_2}^2
+
C\,\mathcal A_{nm}^2\,\widetilde\rho_n\,\log(nm).
\end{align}
Using \eqref{eq:appendix_minkowski_norm_comparison} once more, we get
\begin{align}
&\label{eq:appendix_minkowski_gamma_fluctuation_reduced}
\sum_{i=1}^n\sum_{j=1}^{m_i}\frac{1}{nm_i}
\Bigl\{
\Delta_{nm}(x_{ij})\gamma_i(x_{ij})
-
\langle \gamma_i,\Delta_{nm}\rangle_{\mathcal L_2}
\Bigr\}
\nonumber
\\
&\le \frac{1}{8}\|\Delta_{nm}\|_{nm}^2 + C_2 a_n \bigl\{ \mathcal A_{nm}^2\,\rho_n\,\log(nm) + \mathcal A_{nm}^2\beta(S_n) \bigr\} + C_3\mathcal A_{nm}^2\widetilde\rho_n\log(nm).
\end{align}

Combining \eqref{eq:appendix_minkowski_gamma_mean_term_reduced} and
\eqref{eq:appendix_minkowski_gamma_fluctuation_reduced}, we conclude that
\begin{align}
&\sum_{i=1}^n\sum_{j=1}^{m_i}\frac{1}{nm_i}\Delta_{nm}(x_{ij})\gamma_i(x_{ij}) \le \frac{3}{16}\|\Delta_{nm}\|_{nm}^2 \notag \\
&\quad + C_4 a_n \bigl\{ \mathcal A_{nm}^2\,\rho_n\,\log(nm) + \mathcal A_{nm}^2\beta(S_n) \bigr\} + C_5\mathcal A_{nm}^2\widetilde\rho_n\log(nm) + C_6\,\frac{\sigma_\gamma^2}{n}.\label{eq:appendix_minkowski_gamma_term}
\end{align}

\medskip
\noindent
{\bf Step 4: empirical norm bound.}
Substituting \eqref{eq:appendix_minkowski_approx_term},
\eqref{eq:appendix_minkowski_epsilon_term_bound}, and
\eqref{eq:appendix_minkowski_gamma_term} into
\eqref{eq:appendix_minkowski_basic_ineq}, we obtain
\begin{align}
&\|\Delta_{nm}\|_{nm}^2 \lesssim \varepsilon_{nm,\delta}^2 + \mathcal A_{nm}^2\,\rho_n\, \log\!\bigl(\frac{2\mathcal A_{nm}}{\|\Delta_{nm}\|_{nm}}\bigr) + a_n\mathcal A_{nm}^2\rho_n\log(nm) + a_n\mathcal A_{nm}^2\beta(S_n) \notag \\
&\quad + \mathcal A_{nm}^2\widetilde\rho_n\log(nm) + \frac{\sigma_\gamma^2}{n}.\label{eq:appendix_minkowski_empirical_bound}
\end{align}

We now control the logarithmic factor exactly as in the proof of Theorem~\ref{cor:general_architecture_temporal_spatial}. If
\begin{align*}
&\|\Delta_{nm}\|_{nm}^2 \le \frac1n,
\end{align*}
then \eqref{eq:appendix_minkowski_empirical_bound} already yields the desired conclusion. Otherwise,
\begin{align*}
&\|\Delta_{nm}\|_{nm}^2 > \frac1n,
\qquad\text{hence}\qquad
\|\Delta_{nm}\|_{nm} \ge n^{-1/2},
\end{align*}
so that
\begin{align*}
&\log\!\left(\frac{2\mathcal A_{nm}}{\|\Delta_{nm}\|_{nm}}\right)
\le
\log\!\left(2\mathcal A_{nm}\sqrt n\right)
\lesssim
\log(nm).
\end{align*}
Therefore, in this case,
\begin{align}
&\|\Delta_{nm}\|_{nm}^2 \lesssim \varepsilon_{nm,\delta}^2 + \mathcal A_{nm}^2\rho_n\log(nm) + a_n\mathcal A_{nm}^2\rho_n \log(nm) \notag \\
&\quad + a_n\mathcal A_{nm}^2\beta(S_n) + \mathcal A_{nm}^2\widetilde\rho_n\log(nm) + \frac{\sigma_\gamma^2}{n}.\label{eq:appendix_minkowski_empirical_bound_reduced}
\end{align}
Since \(a_n\to\infty\), we have \(a_n\ge1\) for all sufficiently large \(n\), and thus
\begin{align*}
&\mathcal A_{nm}^2\rho_n\log(nm)
\le
a_n\mathcal A_{nm}^2\rho_n\log(nm).
\end{align*}
Hence in either case,
\begin{align}
&\|\Delta_{nm}\|_{nm}^2 \lesssim \varepsilon_{nm,\delta}^2 + a_n\mathcal A_{nm}^2\rho_n\log(nm) \notag \\
&\quad + a_n\mathcal A_{nm}^2\beta(S_n) + a_n\mathcal A_{nm}^2\widetilde\rho_n\log(nm) + \frac{\sigma_\gamma^2}{n}.\label{eq:appendix_minkowski_empirical_bound_final}
\end{align}

\medskip
\noindent
{\bf Step 5: passage to \(\mathcal L_2\).}
Combining \eqref{eq:appendix_minkowski_norm_comparison} and
\eqref{eq:appendix_minkowski_empirical_bound_final}, we obtain
\begin{align}
&\|\Delta_{nm}\|_{\mathcal L_2}^2 \le 2\|\Delta_{nm}\|_{nm}^2 + C a_n \bigl\{ \mathcal A_{nm}^2\rho_n\log(nm) + \mathcal A_{nm}^2\beta(S_n) \bigr\} \nonumber\label{eq:appendix_minkowski_L2_bound}
\\
&\lesssim \varepsilon_{nm,\delta}^2 + a_n\mathcal A_{nm}^2\rho_n\log(nm) + a_n\mathcal A_{nm}^2\beta(S_n) + a_n\mathcal A_{nm}^2\widetilde\rho_n\log(nm) + \frac{\sigma_\gamma^2}{n}.
\end{align}
Substituting \eqref{eq:appendix_minkowski_L2_bound} into
\eqref{eq:appendix_minkowski_basic_reduction}, we conclude that
\begin{align}
&\|\widehat f_{\mathcal A_{nm}}-f^*\|_{\mathcal L_2}^2 \lesssim (nm)^{-\frac{2p}{2p+d_0}} + a_n\mathcal A_{nm}^2\,\rho_n\,\log(nm) \notag \\
&\quad + a_n\mathcal A_{nm}^2\,\widetilde\rho_n\,\log(nm) + a_n\mathcal A_{nm}^2\beta(S_n) + \frac{\sigma_\gamma^2}{n},\label{eq:appendix_minkowski_temporal_spatial_rate_proof}
\end{align}
with probability tending to one. This proves \eqref{eq:appendix_minkowski_temporal_spatial_rate}.

It remains to derive \eqref{eq:appendix_minkowski_temporal_spatial_rate_intermediate} and \eqref{eq:appendix_minkowski_temporal_spatial_rate_final}. Under the canonical choices of \((L,r)\) from Lemma~\ref{lem:approximation_minkowski_temporal_spatial}, Lemma~\ref{lemma:threshold space entropy} yields the global entropy bound
\begin{align*}
&\log \mathcal N\!\bigl(\delta,\mathcal F_{\mathcal A_{nm}}(L,r),\|\cdot\|_{nm}\bigr) \lesssim nm\, (nm)^{-\frac{2p}{2p+d_0}} \log^{4.3}(nm)\log(\delta^{-1}),
\end{align*}
and therefore
\begin{align*}
&\widetilde\rho_n
\lesssim
(nm)^{-\frac{2p}{2p+d_0}}\log^{4.3}(nm).
\end{align*}

Likewise, for the blockwise empirical norms, since \(|\mathcal J_{e,s}|m\asymp nm/S_n\) and \(|\mathcal J_{o,s}|m\asymp nm/S_n\), the same entropy bound gives
\begin{align*}
&\rho_n
\lesssim
\left(\frac{S_n}{nm}\right)^{\frac{2p}{2p+d_0}}\log^{4.3}(nm).
\end{align*}

Substituting these bounds into \eqref{eq:appendix_minkowski_temporal_spatial_rate_proof} yields
\begin{align*}
&\|\widehat f_{\mathcal A_{nm}}-f^*\|_{\mathcal L_2}^2 \lesssim\; (nm)^{-\frac{2p}{2p+d_0}} + a_n\mathcal A_{nm}^2 \bigl(\frac{S_n}{nm}\bigr)^{\frac{2p}{2p+d_0}} \log^{5.3}(nm)
\\
&+ a_n\mathcal A_{nm}^2 (nm)^{-\frac{2p}{2p+d_0}} \log^{5.3}(nm) + a_n\mathcal A_{nm}^2\beta(S_n) + \frac{\sigma_\gamma^2}{n},
\end{align*}
which proves \eqref{eq:appendix_minkowski_temporal_spatial_rate_intermediate}.

Now choose \(S_n\asymp \log n\). Then
\begin{align*}
&\left(\frac{S_n}{nm}\right)^{\frac{2p}{2p+d_0}}
\lesssim
(nm)^{-\frac{2p}{2p+d_0}}\log(nm),
\end{align*}
since \(S_n\asymp \log n\) and \(\frac{2p}{2p+d_0}\le 1\). Also, by the choice
\begin{align*}
&\mathcal A_{nm}^2\asymp (\sigma_\epsilon^2+\sigma_\gamma^2)\log(nm),
\end{align*}
we obtain
\begin{align*}
&a_n\mathcal A_{nm}^2 \bigl(\frac{S_n}{nm}\bigr)^{\frac{2p}{2p+d_0}} \log^{5.3}(nm) \lesssim a_n(\sigma_\epsilon^2+\sigma_\gamma^2) (nm)^{-\frac{2p}{2p+d_0}} \log^{7.3}(nm).
\end{align*}
Moreover,
\begin{align*}
&a_n\mathcal A_{nm}^2 (nm)^{-\frac{2p}{2p+d_0}} \log^{5.3}(nm) \lesssim a_n(\sigma_\epsilon^2+\sigma_\gamma^2) (nm)^{-\frac{2p}{2p+d_0}} \log^{6.3}(nm),
\end{align*}
so this term is dominated by the previous one. Finally, by the exponential \(\beta\)-mixing assumption and \(S_n\asymp \log n\),
\begin{align*}
&\beta(S_n)\lesssim \frac1n,
\end{align*}
and therefore
\begin{align*}
&a_n\mathcal A_{nm}^2\beta(S_n)
\lesssim
a_n\frac{(\sigma_\epsilon^2+\sigma_\gamma^2)\log(nm)}{n}.
\end{align*}
Substituting these bounds into the previous display yields
\begin{align}
&\|\widehat f_{\mathcal A_{nm}}-f^*\|_{\mathcal L_2}^2 \lesssim\; a_n(\sigma_\epsilon^2+\sigma_\gamma^2+1) (nm)^{-\frac{2p}{2p+d_0}} \log^{7.3}(nm) \nonumber\label{eq:appendix_minkowski_temporal_spatial_rate_final_proof}
\\
&\quad
+
a_n\frac{(\sigma_\epsilon^2+\sigma_\gamma^2)\log(nm)}{n}
+
\frac{\sigma_\gamma^2}{n},
\end{align}
with probability tending to one, which is exactly \eqref{eq:appendix_minkowski_temporal_spatial_rate_final}.
\end{proof}

\color{Black}

\end{document}